\documentclass{article}

\PassOptionsToPackage{numbers,sort&compress}{natbib}
\usepackage[final,main]{neurips_2026}

\usepackage[utf8]{inputenc}
\usepackage[T1]{fontenc}

\usepackage{url}
\usepackage{microtype}
\usepackage{graphicx}
\usepackage{booktabs}
\usepackage{array}
\usepackage[table]{xcolor}
\usepackage{wrapfig}
\usepackage{tabularx}
\usepackage{threeparttable}
\usepackage[export]{adjustbox}
\usepackage{comment}
\usepackage{subcaption}
\usepackage{capt-of}
\usepackage{changepage}

\makeatletter
\newwrite\apx@out

\newcommand{\apxtocline}[4]{%
  \par\noindent
  \ifnum#1=2\hspace*{1.5em}\fi
  \hyperlink{#4}{#2}\leaders\hbox{.\kern.5pt}\hfill #3\par
}

\newcommand{\printappendixtoc}{%
  \begingroup
  \section*{Appendix Contents}%
  \parindent=0pt
  \InputIfFileExists{\jobname.apx}{}{%
    \textit{(Appendix contents will appear after the second compile.)}%
  }%
  \endgroup
}

\newcommand{\startappendixtoc}{%
  \immediate\openout\apx@out=\jobname.apx
  \let\apx@orig@section\section
  \let\apx@orig@subsection\subsection

  \def\section{\@ifstar{\apx@orig@section*}{\apx@section}}%
  \def\apx@section{\@dblarg\apx@section@i}%
  \def\apx@section@i[##1]##2{%
    \apx@orig@section[##1]{##2}%
    \protected@write\apx@out{}{%
      \string\apxtocline{1}{##2}{\thepage}{\@currentHref}%
    }%
  }%

  \def\subsection{\@ifstar{\apx@orig@subsection*}{\apx@subsection}}%
  \def\apx@subsection{\@dblarg\apx@subsection@i}%
  \def\apx@subsection@i[##1]##2{%
    \apx@orig@subsection[##1]{##2}%
    \protected@write\apx@out{}{%
      \string\apxtocline{2}{##2}{\thepage}{\@currentHref}%
    }%
  }%
}

\newcommand{\stopappendixtoc}{%
  \immediate\closeout\apx@out
  \let\section\apx@orig@section
  \let\subsection\apx@orig@subsection
}
\makeatother

\usepackage{amsmath}
\usepackage{amssymb}
\usepackage{mathtools}
\usepackage{amsthm}
\usepackage{enumitem}
\usepackage{siunitx}
\usepackage{multirow}
\usepackage{wrapfig}
\usepackage{xspace}

\newcommand{\dm}{\textsc{Distance Marching}\xspace}

\usepackage{hyperref}
\usepackage[nameinlink,capitalize]{cleveref}

\makeatletter
\AtBeginDocument{%
  \def\ltx@label#1{\cref@label{#1}}%
  \def\label@in@display@noarg#1{\cref@old@label@in@display{#1}}%
}
\makeatother

\crefname{equation}{Eq.}{Eqs.}
\Crefname{equation}{Eq.}{Eqs.}

\crefname{figure}{Fig.}{Figs.}
\Crefname{figure}{Fig.}{Figs.}

\crefname{section}{Sec.}{Secs.}
\Crefname{section}{Sec.}{Secs.}

\crefname{table}{Tab.}{Tabs.}
\Crefname{table}{Tab.}{Tabs.}

\theoremstyle{plain}
\newtheorem{theorem}{Theorem}[section]

\newtheorem{lemma}[theorem]{Lemma}
\newtheorem{corollary}[theorem]{Corollary}
\theoremstyle{definition}
\newtheorem{definition}[theorem]{Definition}

\crefname{definition}{def.}{defs.}
\Crefname{definition}{Def.}{Defs.}
\newtheorem{assumption}[theorem]{Assumption}
\theoremstyle{remark}
\newtheorem{remark}[theorem]{Remark}

\usepackage[textsize=tiny]{todonotes}

\definecolor{DarkGreen}{RGB}{0,100,0}

\newif\ifshowcomments
\showcommentsfalse %

\title{Distance Marching for Generative Modeling}

\author{%
  Zimo Wang, Ishit Mehta, Haolin Lu, Xunpeng Huang \\
  \textbf{Chung-En Sun, Ge Yan, Tsui-Wei Weng, Tzu-Mao Li} \\
  UC San Diego
}

\begin{document}

\vspace{-2em}

\maketitle

\vspace{-2em}

\begin{abstract}
Modern generative models based on diffusion or flow matching transport noise toward the data manifold, without representing the data manifold explicitly.
We propose \dm that explores an alternative design by explicitly modeling a distance field of the data manifold.
A distance field offers a geometric view of generation as a descent process toward the data manifold using only the current position.
The negative gradient of the distance field provides the direction, and the distance value can optionally determine the step size.
However, sparse training data in high-dimensional space make it difficult to identify the closest points on the data manifold needed to train the distance field.
To address this challenge, we design new loss functions that emphasize denoising targets from closer data points.
Our analysis shows that, without time input, emphasizing closer targets reduces denoising ambiguity and the bias toward the data mean.
This distance-field perspective naturally motivates two principled deterministic samplers, showing fast convergence in inference.
As a result, our model is the only one without time input that remains competitive with time-conditioned models across datasets and architectures.
Moreover, our distance prediction is also helpful for early stopping during sampling and for out-of-distribution detection.
We hope distance-field modeling can open a promising direction for exploring generative modeling from a geometric perspective.

\end{abstract}

\section{Introduction}
\label{sec:intro}

\begin{wrapfigure}[19]{r}{0.5\linewidth}
    \centering
    \vspace{-1.5em}
    \includegraphics[width=\linewidth]{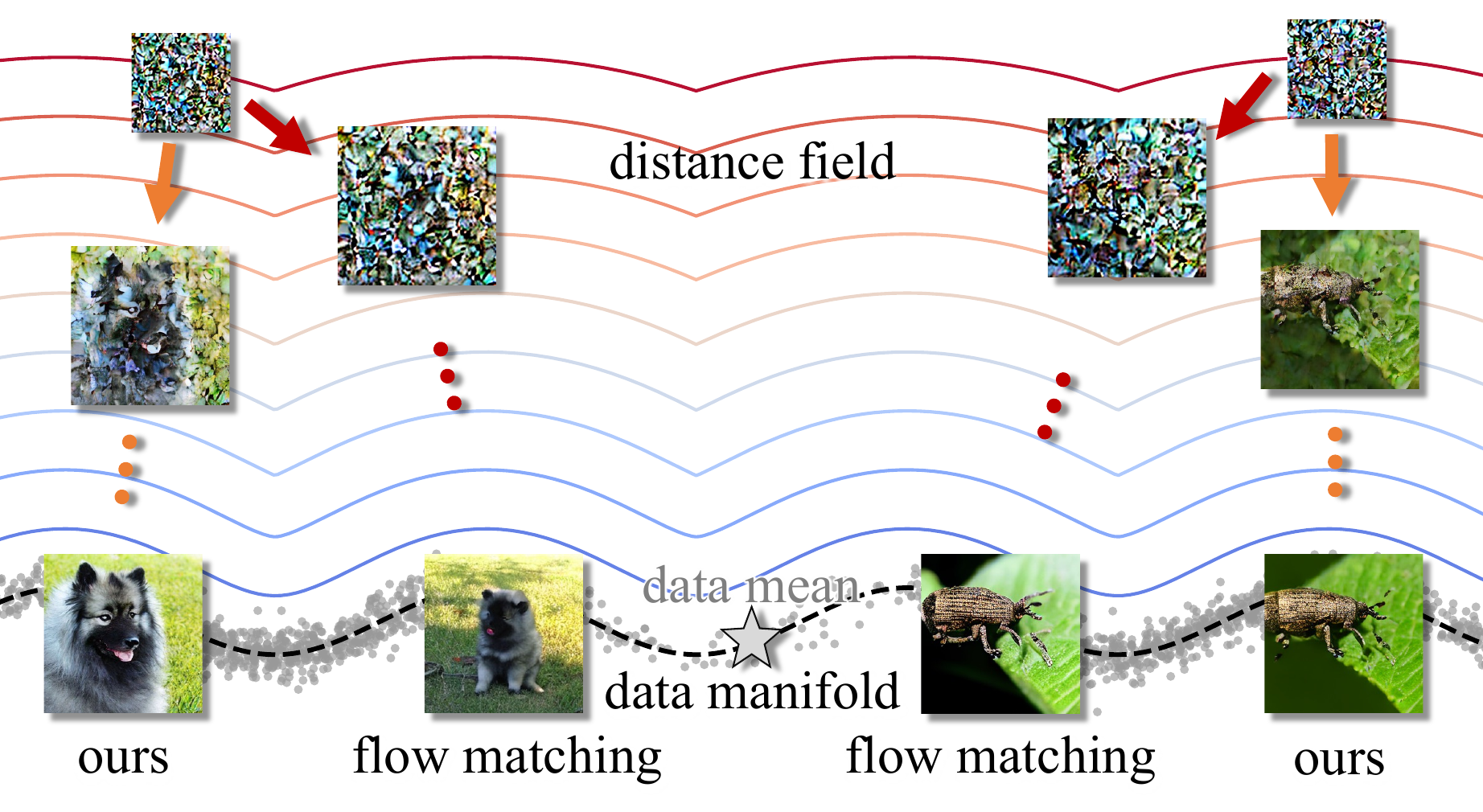}
    \vspace{-1.5em}
    \caption{
    We propose \dm to generate images without time input.
    Inspired by distance-field modeling, our method learns a denoising direction that focuses on closer targets. In contrast, the flow matching loss favors a direction biased toward the data mean in the early stages of generation.
    We compare outputs after 20\% of the sampling steps to show that our sampling trajectories wander less.
    See~\Cref{sec:analysis} for a detailed analysis.}
    \label{fig:teaser}
    \vspace{-1em}
\end{wrapfigure}

Modern generative models, typically based on diffusion or flow matching~\cite{song2020score,lipman2022flow}, learn denoising vector fields that transport noise toward the data manifold.
These methods typically do not represent the data manifold explicitly.
Recent theory suggests that learning the geometry of the data manifold can enable generative models to produce realistic and novel samples~\cite{li2025scores,shen2026manifold}.
Motivated by these findings, we explore a generative model that explicitly learns a \emph{distance field} of the data manifold.
The field records the distance from each point in the ambient space to the manifold.

Our method is based on the observation that moving along the negative gradient of the distance field brings samples closer to the data manifold using only the current position.
This observation is consistent with recent theoretical work~\cite{li2025scores,shen2026manifold} on geometric generative modeling.
Furthermore, the distance itself tells us how far we should go in the negative gradient direction, which induces a closest-point projection onto the data manifold.
However, designing a loss function to learn a distance field for high-dimensional data is challenging.
Sparse training samples in high-dimensional space do not provide reliable supervision for learning the closest-point projection onto the data manifold.
To address this challenge, we introduce loss functions that emphasize denoising targets from closer data points.
We also propose two principled samplers, gradient descent and sphere tracing, that guide noise samples toward the data manifold faster.

Empirically, the learned field provides meaningful distance estimates~(\Cref{fig:dist_error}) and captures local geometry consistent with the denoising trajectories~(\Cref{fig:distance_head_main}).
The learned distance field further improves robustness under train--test mismatch (\Cref{fig:combined_baseline_range}) and enables downstream tasks (see~\Cref{subsection:applications_distance}).

Our distance-field formulation naturally supports generation without time input, connecting our method to recent time-unconditional generative models~\cite{sun2025noise,balcerak2025energy,wang2025equilibrium}.
Simply removing time input from standard diffusion~\cite{ho2020denoising} or flow matching objectives~\cite{lipman2022flow}, however, leaves the denoising targets ambiguous~\cite{li2025improved,bertrand2025closed,karras2022elucidating,gagneux2025generation}.
Averaging these targets biases the learned direction toward the data mean.
The same emphasis on closer targets, motivated by distance-field learning, also mitigates this ambiguity.
Our theoretical analysis shows that this emphasis reduces the bias toward the data mean, and yields more effective denoising directions in practice~(\Cref{fig:fid-compare}).

Our experiments show strong performance across datasets, architectures, and model sizes. On CIFAR-10, \dm ties the best time-unconditional result~(\Cref{tab:cifar-10}). On ImageNet, it reduces Fr\'echet Inception Distance (FID) by 9.5\%--24.7\% relative to the time-conditioned flow matching backbone~(\Cref{fig:different_sizes}). Under classifier-free guidance, it improves FID by 4.9\% and Inception Score (IS) by 10.4\% over the strongest compared baseline~(\Cref{tab:compcfg}). Using the learned distance estimates, our sphere tracing sampler further improves generation quality over gradient descent~(\Cref{tab:cifar-10,tab:compcfg,tab:sampler_ablation}) and reaches flow matching's final FID using fewer than 60\% of its sampling steps~(\Cref{fig:faster}).

In summary, our main contributions are:
\begin{itemize}[leftmargin=1.2em]
  \item We propose \dm, a generative model that explicitly learns a distance field of the data manifold without time input. From this formulation, we design two principled samplers, gradient descent and sphere tracing, that guide noise toward the manifold.

  \item We design loss functions that emphasize closer denoising targets, analyzing how they reduce target ambiguity and bias toward the data mean. We further establish conditions for distance-field approximation and convergence toward the manifold.

  \item We validate the accuracy and geometric consistency of the learned distance field. The distance estimates improve generation quality through sphere tracing, enhance the robustness under train--test mismatch, and enable adaptive stopping and out-of-distribution detection.

  \item We demonstrate superior generation quality over time-conditioned flow matching on both CIFAR-10 and ImageNet, achieving consistent gains across architectures without time input.
\end{itemize}

\section{Related Work}
\label{sec:related}

\noindent\textbf{Time-conditioned diffusion and flow models.}
Diffusion~\cite{sohl2015deep, ho2020denoising, song2020denoising, song2020score} and flow-based models~\cite{rezende2015variational, chen2018neural, lipman2022flow} typically rely on explicit time or noise conditioning and carefully designed parameterizations.
Prior work improves sample quality or reduces sampling steps by refining schedules, coefficients, and conditioning mechanisms~\cite{nichol2021improved, karras2022elucidating, ho2022classifier, lu2022dpm, peebles2023scalable}.
These methods are effective but suffer from exposure bias: during inference, the model is queried at discrete time steps on states it generated itself, which no longer match the true forward-noise distribution, causing error accumulation~\cite{ning2023elucidating, li2023alleviating}.
A related train--test mismatch appears when the inference initialization differs from the training noise, which we study empirically in~\Cref{fig:combined_baseline_range}.

\noindent\textbf{Time-unconditional generative modeling.}
Removing explicit time conditioning has a long history in score matching and energy-based modeling~\cite{hyvarinen2005estimation,carreira2005contrastive,lecun2006tutorial,du2019implicit}.
Recent approaches have shown that time-unconditional image generation can be competitive with carefully designed objectives, delicate coefficient constraints, and sampling schedules~\cite{sun2025noise,balcerak2025energy,wang2025equilibrium}.
At the same time, removing time or noise inputs makes the denoising target less explicit, leaving the model to average signals from targets at different distances.
We study this ambiguity through a distance-field lens and design losses that explicitly emphasize closer targets.

\noindent\textbf{Geometric Learning for Generative Models.}
Distance-based modeling has shown promise for generation~\citep{permenter2023interpreting,wan2024elucidating,bar2025geometric}. 
This line of work leaves open how distance-field principles can directly shape training and inference for competitive image generation.
Prior work has observed a deep connection between the manifold hypothesis and generative modeling.
In particular, some work~\cite{li2025scores,shen2026manifold} suggested that recovering the \emph{support} of the data manifold is easier and more useful than recovering the density, in contrast to the optimal transport perspective~\cite{lipman2024flow,tong2023improving, pooladian2023multisample}. 
Similar to previous time-unconditional models, our model prioritizes support recovery over density. 
Empirically, these methods, including ours, still perform competitively in image generation tasks.

\noindent\textbf{Distance-field modeling and sphere tracing.}
Our method draws inspiration from neural distance fields in computer graphics and vision~\cite{hoppe1992surface, park2019deepsdf, wang2025hotspot}.
Distance fields guide points toward the zero level set using a direction and magnitude that depend only on spatial coordinates, and sphere tracing~\cite{hart1996sphere} exploits this property to take adaptive, well-calibrated steps.
Prior work~\cite{gomes2003vector, atzmon2020sald, ma2020neural} has introduced representations and constraints to better regularize the learned distance field.

\vspace{-0.5em}
\section{Method}
\subsection{Motivation}

\newif\iffigtwonewstyle
\figtwonewstyletrue

\begin{figure*}[t]
\centering
\setlength{\tabcolsep}{1pt}
\iffigtwonewstyle
  \def\figtwoA{figures/fig2_dm_with_gt_geometry.pdf}
  \def\figtwoB{figures/fig2_dm_geometry.pdf}
  \def\figtwoC{figures/fig2_em_geometry.pdf}
  \def\figtwoD{figures/fig2_eqm_geometry.pdf}
\else
  \def\figtwoA{figures/model_linear_scalar_with_gt.pdf}
  \def\figtwoB{figures/model_linear_scalar.pdf}
  \def\figtwoC{figures/2d_appendix/levelset_em.png}
  \def\figtwoD{figures/eqm_levelset.png}
\fi
\begin{tabular}{@{}c c c c c@{}}
\hypertarget{fig:2d_motivation:a}{\includegraphics[width=0.195\textwidth]{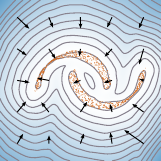}} &
\hypertarget{fig:2d_motivation:b}{\includegraphics[width=0.195\textwidth]{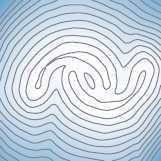}} &
\hypertarget{fig:2d_motivation:c}{\includegraphics[width=0.195\textwidth]{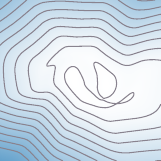}} &
\hypertarget{fig:2d_motivation:d}{\includegraphics[width=0.195\textwidth]{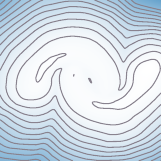}} &
\hypertarget{fig:2d_motivation:e}{\includegraphics[width=0.195\textwidth]{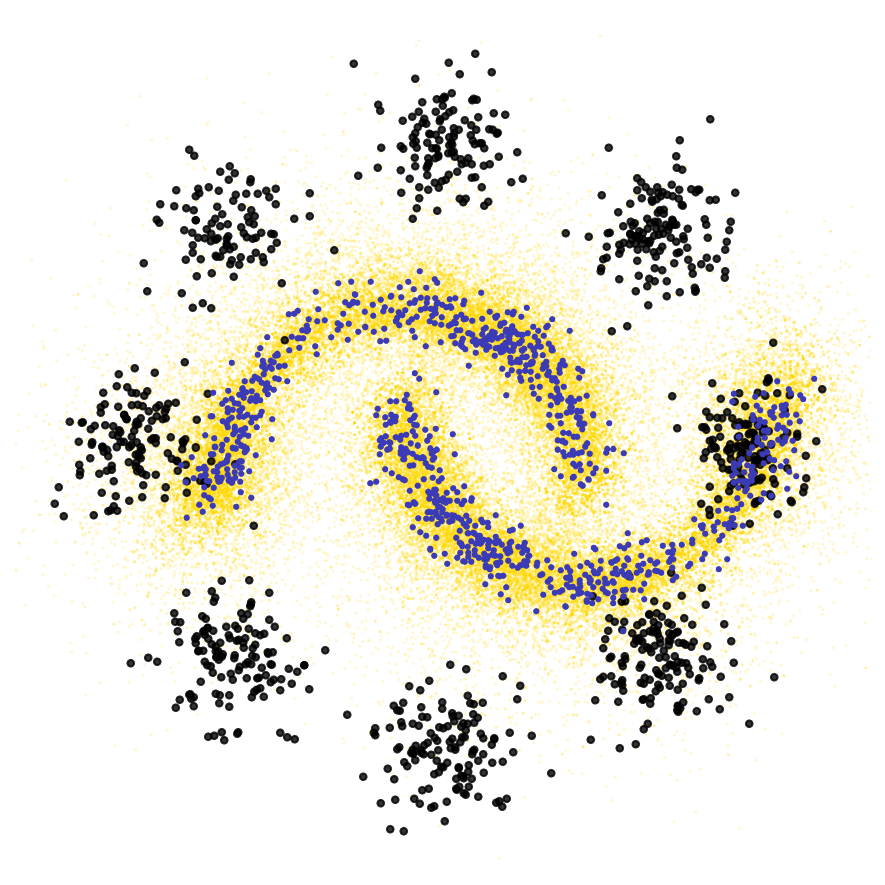}} \\
\parbox[t]{0.195\textwidth}{\centering\footnotesize (a) Target two-moons (\textcolor[HTML]{FA7F57}{orange}) and learned $u_\theta$ with $-\nabla_{\mathbf{x}}u_\theta$} &
\parbox[t]{0.195\textwidth}{\centering\footnotesize (b) Our learned distance field level sets $u_\theta$} &
\parbox[t]{0.195\textwidth}{\centering\footnotesize (c) EM~\cite{balcerak2025energy} learns a rough energy field.} &
\parbox[t]{0.195\textwidth}{\centering\footnotesize (d) EqM~\cite{wang2025equilibrium} offers rough energy nearby the data manifold.} &
\parbox[t]{0.195\textwidth}{\centering\footnotesize (e) Our samples: init (black), mid (\textcolor[HTML]{F2D024}{yellow}), and final (\textcolor[HTML]{008FFF}{blue})} \\
\end{tabular}
\caption{
In 2D, we learn a distance-like scalar field $u_\theta(\mathbf{x})$ (\protect\hyperlink{fig:2d_motivation:a}{a} and \protect\hyperlink{fig:2d_motivation:b}{b}) that transports samples onto the support of the data distribution (\protect\hyperlink{fig:2d_motivation:a}{a}).
\iffigtwonewstyle In (a--d), darker shading indicates larger scalar values, with each field normalized separately. \fi
In contrast, prior energy-based methods (\protect\hyperlink{fig:2d_motivation:c}{c}, \protect\hyperlink{fig:2d_motivation:d}{d}), even when using minibatch nearest-neighbor during training to disambiguate targets, do not capture energy landscape details.
Because they inherit the flow-matching loss, when rematching is impractical due to high-dimensional hubness, they become unreliable for determining denoising directions, as analyzed in~\Cref{sec:analysis}.
See~\Cref{tab:2d_metrics,app:2d_exp} for quantitative sampling comparisons.}
\label{fig:2d_motivation}
\vspace{-1em}
\end{figure*}

\dm applies the geometric intuition of distance fields to generative modeling by learning distance-like fields toward the data manifold with losses that naturally emphasize closer targets.
\Cref{fig:2d_motivation} shows a working 2D example: using distance information, noise samples are pulled towards the data support and converge on the target manifold.

We use the definitional property of a distance function that it induces a closest-point projection. Let $\mathcal{S}$ be the target data manifold and $d(\mathbf{x})$ be the distance from point $\mathbf{x}$ to the data manifold $\mathcal{S}$. The distance field provides two pieces of information: the \emph{magnitude} $d(\mathbf{x})$ indicates how far $\mathbf{x}$ is from $\mathcal{S}$, and the \emph{direction} $-\nabla_{\mathbf{x}} d(\mathbf{x})$ indicates where to move to approach the support of the data manifold. Hence, the closest point on the data manifold $\pi_\mathcal{S}(\mathbf{x}) \in \mathcal{S}$ is recovered by:
\begin{equation}
    \pi_\mathcal{S}(\mathbf{x}) = \mathbf{x} - d(\mathbf{x}) \nabla d(\mathbf{x}).
    \label{eq:one_step}
\end{equation}
This suggests that we can learn a scalar field $u_{\theta}: \mathbb{R}^D \rightarrow \mathbb{R}$ to approximate the behavior of the distance function $d$, where $D$ is the spatial dimension.
We model the scalar field using a neural network, where $\theta$ are the network parameters.

Directly obtaining $\pi_\mathcal{S}$ as supervision from the training data can be challenging, since a finite training set is sparse and suffers from \emph{hubness}. 
In low-dimensional space, we treat the problem similarly to neural surface reconstruction and directly use the nearest neighbor within a minibatch like~\cite{atzmon2020sald}.
This strategy, unfortunately, does not scale to high-dimensional space, as discussed in \Cref{sec:analysis}.

Prior methods~\cite{balcerak2025energy, wang2025equilibrium} on time-unconditional and energy-based generative modeling also aim to learn a landscape (\cref{fig:2d_motivation}(c) and (d)). However, they cannot learn landscape details well and do not have the properties of the distance field. Below, we introduce our losses for learning the distance field.

\subsection{Distance Field Modeling}
\label{subsec:dfmodeling}

We now define our losses on a training pair $(\mathbf{x},\mathbf{s})$, where $\mathbf{x}$ is a noisy point in the ambient space and $\mathbf{s}\in\mathcal{S}$ is the denoising target, as specified later in~\Cref{subsec:dm}.
To learn the projection mapping as in \cref{eq:one_step}, we propose a \emph{One-Step Loss} function:
\begin{equation}
    L_{\mathrm{OS}}
    =
    \frac{1}{n}
    \sum_{i=1}^{n}
    w_i
        \big\|
        \mathbf{x}_{i}
        - u_{\theta}(\mathbf{x}_{i})\nabla u_{\theta}(\mathbf{x}_{i})
        - \mathbf{s}_{i}
        \big\|_2^{2},
    \label{eq:osl}
\end{equation}
where we take a weighted average over the dataset with weights $w_i$.

When the projection $\pi_\mathcal{S}$ is accessible, we set $\mathbf{s}=\pi_\mathcal{S}(\mathbf{x})$.
When $\pi_\mathcal{S}(\mathbf{x})$ is not accessible and $\mathbf{s}$ is randomly sampled on the manifold, a fixed $\mathbf{x}$ may be paired with different denoising targets $\mathbf{s}_i$. Then the weight should emphasize data points $\mathbf{s}_{i}$ that are closer to the query point $\mathbf{x}$ to make $u_\theta$ behave like a distance field.
At the same time, we want different query points $\mathbf{x}_{i}$, whether far from or close to the manifold, to receive balanced supervision.
We therefore choose:
\begin{equation}
    w_i = \frac{1}{\big\|\mathbf{x}_{i} - \mathbf{s}_{i}\big\|_2^2 + \epsilon},
\end{equation}
where $\epsilon > 0$ is a small constant for numerical stability. As we will show in~\Cref{sec:analysis} and~\Cref{app:ablation}, the choice of this weight is critical to our work and enables us to mitigate bias towards the data mean during inference and lead to faster inference and higher quality samples.

\begin{wrapfigure}[8]{r}{0.2\linewidth}
\vspace{-2.0em}
  \centering
  \includegraphics[width=\linewidth]{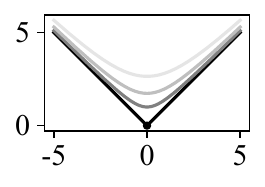}
  \vspace{-2.0em}
  \caption{1D $\hat{d}(\mathbf{x})$ to the origin; lighter line, larger $c_0$.}
  \label{fig:dfamily}
\end{wrapfigure}
The one-step loss alone is not sufficient.
Although~\Cref{eq:one_step} characterizes a one-step mapping to the surface, it does not uniquely determine the distance field. In particular, for any constant $C\ge 0$, the family
\begin{equation}
    \hat{d}(\mathbf{x}) = \pm \sqrt{\|\mathbf{x} - \pi_\mathcal{S}(\mathbf{x})\|_{2}^{2} + C}
    \label{eq:analytical}
\end{equation}
induces the \emph{same} one-step projection in~\Cref{eq:one_step} wherever the gradient exists, as shown in~\Cref{fig:dfamily} for different values of $C$. We prove it in~\Cref{app:projection_invariance}.

To remove this non-uniqueness of the offset $C$ and sign, we introduce our second constraint, the \textit{Directional Eikonal Loss}, which fixes the offset by selecting the positive sign and $C=c_0$. Specifically, we match the gradient from \Cref{eq:analytical} as our target with $C=c_0$:
\begin{equation}
    L_{\mathrm{DE}}
    =
    \frac{1}{n}
    \sum_{i=1}^{n}
    \bigg\|
        \nabla u_{\theta}(\mathbf{x}_i) -
        \frac{
            \mathbf{x}_i - \mathbf{s}_i
        }{
            \sqrt{
                \big\|\mathbf{x}_i - \mathbf{s}_i\big\|_2^2 + c_0
            }
        }
    \bigg\|_2^{2}.
    \label{eq:del}
\end{equation}
Intuitively, one-step loss enforces that $u_\theta$ approximates the desired one-step mapping, while directional eikonal loss disambiguates the offset $C$ as $c_0$ of the learned field by constraining the gradient.
In practice, we find that the exact value of $c_0$ does not matter much as shown in~\Cref{app:ablation}.

Our one-step loss~\cref{eq:osl} differs from the classical flow matching loss~\cite{lipman2022flow}, both in its motivation and in its form.
We abandon the full denoising vector from $\mathbf{X}_0$ to $\mathbf{X}_1$, and explicitly choose a weight that emphasizes closer targets, while maintaining balanced supervision across inputs $\mathbf{X}$ with different noise levels.
Next, given our loss, we detail how we train our model and use it for inference.

\subsection{Implementation of Distance Marching}
\label{subsec:dm}

In practice, we find that training the same network for both the magnitude $d$ and the direction $-\nabla d$ can be inefficient due to both the expense of computing the gradient, and the numerical instability of training with higher-order derivatives~\cite{czarnecki2017sobolev,li2024neural,chetan2025accurate}.
Therefore, for our higher-dimensional image generation tasks, we replace the gradient $\nabla u_\theta(\mathbf{x})$ with another vector network $\mathbf{v}_\theta(\mathbf{x})$, and simply replace $\nabla u_\theta$ in our losses with the vector field $\mathbf{v}_\theta$.
Previous energy-based methods have made similar observations and design choices~\cite{wang2025equilibrium}.
From our theory in~\Cref{app:properties} and~\Cref{app:convergence_approx}, the separation of $u_\theta$ and $\mathbf{v}_\theta$ also doesn't hurt the displacement after inverse-squared-distance reweighting.
Experiments confirm convergence toward the manifold~(\Cref{app:convergence}). The distance head predicts meaningful distances for downstream tasks~(\Cref{fig:dist_error,subsection:applications_distance}), and its gradient aligns with the vector head~(\Cref{fig:distance_head_main}).

\noindent\textbf{Inference.}
Given the scalar field $u_\theta$ and the vector field $\mathbf{v}_\theta$, we can then start from a Gaussian $\mathbf{x}^{(0)}$, and move closer to the data manifold using either of the following two iterations with step size $\eta$:
\par\noindent
\begingroup
\makebox[\linewidth][c]{%
  \refstepcounter{equation}\label{eq:image_st}%
  \makebox[0.49\linewidth][c]{%
    $\displaystyle
    \mathbf{x}^{(i+1)} = \mathbf{x}^{(i)} - \eta \,u_\theta(\mathbf{x}^{(i)} )\,\mathbf{v}_\theta(\mathbf{x}^{(i)} ).
    $\hspace{0.5em}\textup{(\theequation)}%
  }%
  \hfill
  \refstepcounter{equation}\label{eq:gd}%
  \makebox[0.49\linewidth][c]{%
    $\displaystyle
    \mathbf{x}^{(i+1)} = \mathbf{x}^{(i)} - \eta\,\,\mathbf{v}_\theta(\mathbf{x}^{(i)}).
    $\hspace{0.5em}\textup{(\theequation)}%
  }%
}
\endgroup

The first update rule (\Cref{eq:image_st}), which we name Sphere Tracing~\cite{hart1996sphere}, uses both the distance and directional information.
The second update rule uses only the directional information, which we name Gradient Descent. Their motivations are direct from a distance-field-modeling perspective. To better explore the manifold, Markov chain Monte Carlo (MCMC) is also optional with an extra random shift term, as proved in~\Cref{app:allocation} and compared in~\Cref{subsec:2dmechanism} and~\Cref{app:ula}.

For class-conditional generation, we independently sample the class label $y \in \{1,\dots,K\}$ as an additional model input and fix it during the iterations. This can be viewed as selecting the corresponding class-conditional manifold.

\noindent\textbf{Training.}
For training, we need to generate the pair $(\mathbf{x}_{i}, \mathbf{s}_{i})$.
To obtain noisy points that cover the ambient space, while also sampling regions near the manifold, we use the standard linear interpolation in~\Cref{def:genproc} to define denoising targets.
We use $L=\lambda_{\mathrm{OS}}L_\text{OS}+\lambda_{\mathrm{DE}}L_\text{DE},\,\lambda_{\mathrm{OS}},\lambda_{\mathrm{DE}}>0$ as our final loss.

\newcommand{\dd}{\,\mathrm{d}}
\newcommand{\Unif}{\mathrm{Unif}}
\newcommand{\Normal}{\mathcal{N}}
\begin{definition}
\label[definition]{def:genproc}
Let the dataset be $\mathcal{D}=\{\mathbf{s}_{1},\dots,\mathbf{s}_{N}\}\subset\mathbb{R}^D$.
Sample an index $I\sim\Unif(\{1,\dots,N\})$ and set $\mathbf{X}_1:=\mathbf{s}_{I}$.
Independently sample $\mathbf{X}_0\sim\Normal(\mathbf{0},\mathbf{I}_D)$ and an interpolation coefficient $T\sim p_T$ with density $p_T(t)$ on $(0,1)$.
Define the noised samples as the only input $\mathbf{X}:=(1-T)\mathbf{X}_0+T\mathbf{X}_1$.
\end{definition}

Here, $p_T(t)$ is a user-chosen time-sampling distribution~\cite{karras2022elucidating,lee2024improving,lipman2024flow}. Our training pair realization $(\mathbf{x}_{i}, \mathbf{s}_{i})$ is drawn directly from the distribution $(\mathbf{X}, \mathbf{X}_1)$ without reassignment.

\vspace{-0.5em}
\section{Analysis and Discussions}
\label{sec:analysis}

\noindent\textbf{Failure of nearest-neighbor selection in high dimensions.}
When the closest-projection $\pi_\mathcal{S}(\mathbf{x})$ is accessible in 2D, our method directly sets training pairs $(\mathbf{x}, \mathbf{s})$ as $(\mathbf{x}, \pi_\mathcal{S}(\mathbf{x}))$. 
To verify our loss design, we show in 2D that our losses with a known closest-point projection from minibatch nearest neighbor learn the distance well, while other time-unconditional methods always fail to capture landscape details (\cref{fig:2d_motivation}).
We report results in~\Cref{tab:2d_metrics} using this variant and MCMC in 2D, supporting the effectiveness of distance-field-based generation.

However, for higher-dimensional tasks such as image generation, using the nearest neighbor as a proxy for the closest projection easily leads to mode collapse.
As shown in \cref{fig:mode_collapse_examples}, nearest-neighbor training favors low-contrast images in the finite dataset.
To quantify this bias, we measure the coverage rate on CIFAR-10 with 50k training images, i.e., the fraction of real images selected as the nearest neighbor of query points.

\begin{wrapfigure}[5]{r}{0.50\textwidth}
  \centering
  \vspace{-2em}
  \setlength{\tabcolsep}{0pt}

  \begin{tabular}{@{}cccccccc@{}}
    \includegraphics[width=0.125\linewidth]{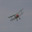} &
    \includegraphics[width=0.125\linewidth]{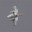} &
    \includegraphics[width=0.125\linewidth]{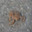} &
    \includegraphics[width=0.125\linewidth]{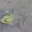} &
    \includegraphics[width=0.125\linewidth]{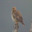} &
    \includegraphics[width=0.125\linewidth]{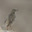} &
    \includegraphics[width=0.125\linewidth]{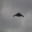} &
    \includegraphics[width=0.125\linewidth]{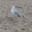}
  \end{tabular}
\vspace{-1em}
  \caption{The leftmost image alone accounts for $62.2\%$ of nearest-neighbor selections due to its grayness and low contrast.}
  \label{fig:mode_collapse_examples}
\end{wrapfigure}
Near the pure noise, the coverage rate drops to $0.17\%$, and the top $8$ images account for $92.8\%$ of nearest-neighbor selections as shown in~\Cref{fig:mode_collapse_examples}.
The full curve is shown in~\Cref{fig:mode_collapse_curve}.
\citet{radovanovic2010hubs} theoretically characterize this as high-dimensional \emph{hubness}, and \citet{frosio2018statistical} find it from an image-denoising perspective, indicating a severe risk of mode collapse under nearest-neighbor matching within a finite dataset.

One possible mitigation is minibatch re-pairing without replacement, as done in optimal transport re-matching~\cite{tong2023improving, pooladian2023multisample}. 
The re-pairing explicitly selects closer targets while spreading assignments across targets.
However, once additional conditions (e.g., class labels) are introduced, re-matching alters the posterior at a fixed noise location, creating a train--inference mismatch that hurts performance~\cite{cheng2025curse}.

Our ablation confirms these two weaknesses, showing clear performance degradation in~\Cref{tab:ot_failure} and~\Cref{fig:ablation_coupling}.
Our method, by contrast, is built on distance modeling and support recovery through a new loss design and samplers, without using such re-matching and thus avoiding these issues.

\begin{figure}[t]
    \vspace{-1em}
    \centering

    \begin{minipage}[t]{0.49\linewidth}
        \centering
        \includegraphics[width=0.56\linewidth]{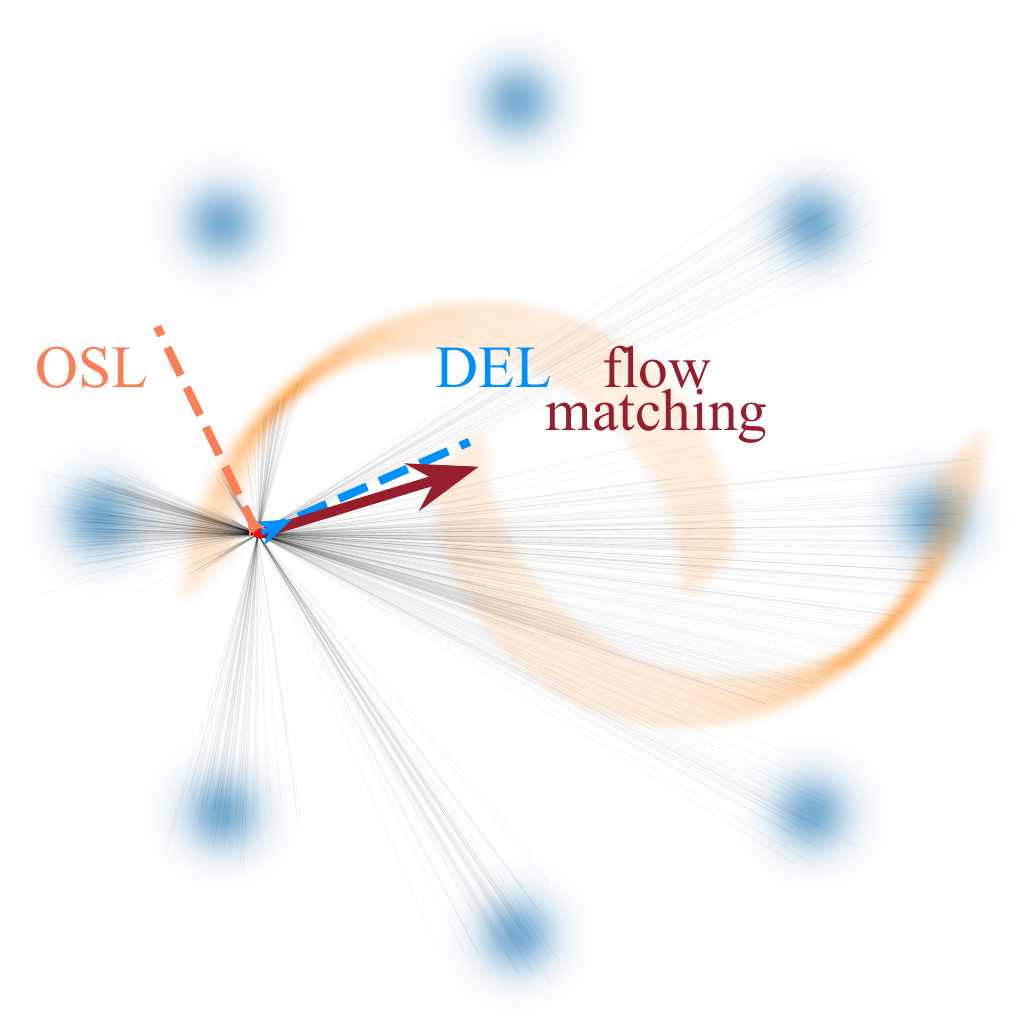}%
        \raisebox{0.3\height}{%
          \includegraphics[width=0.43\linewidth]{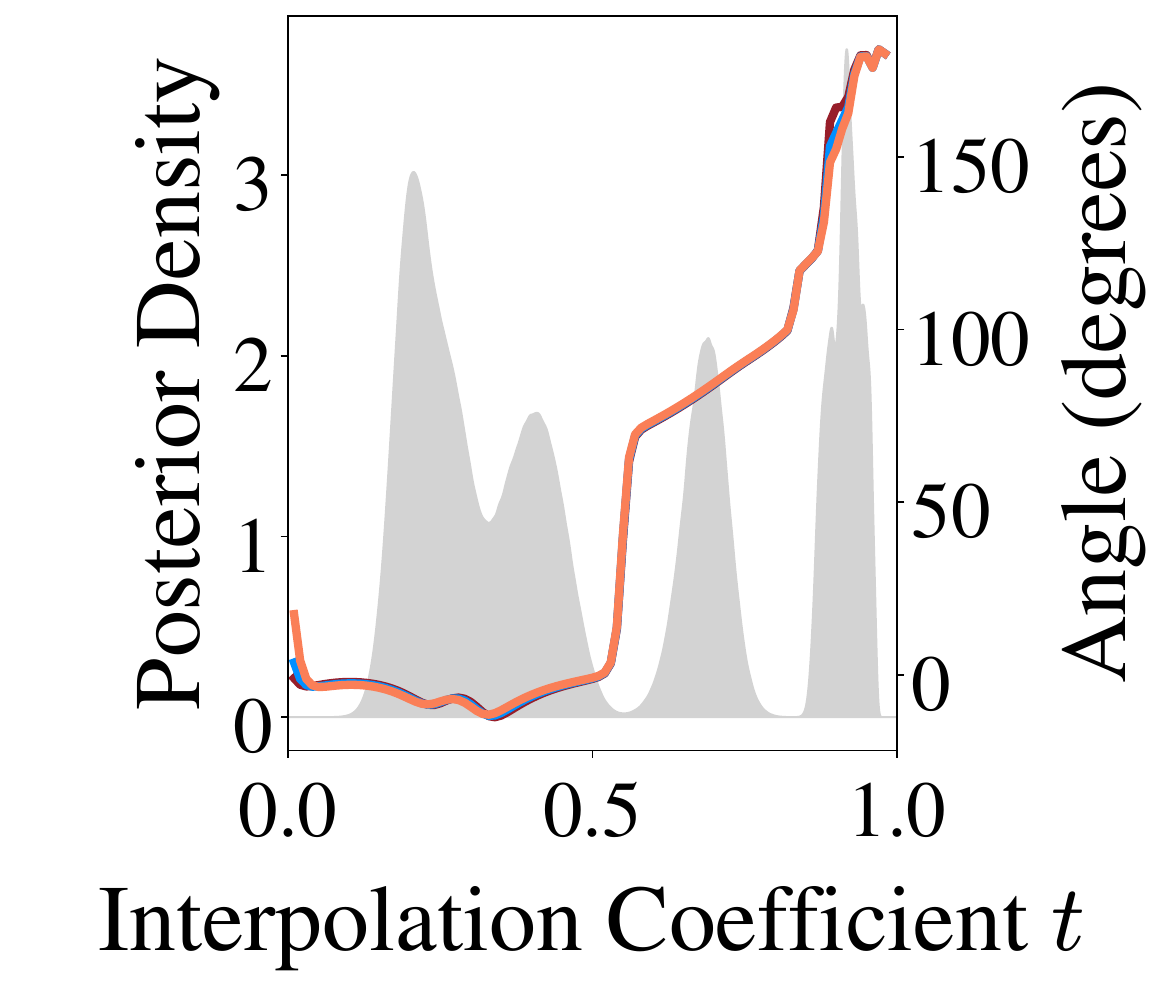}%
        }
        \captionof{figure}{\textbf{Flow matching does not always help denoising.} We compare the minimizer vectors induced by flow matching (\textcolor[HTML]{971E2A}{red}), directional eikonal loss (\textcolor[HTML]{008FFF}{blue}), and one-step loss (\textcolor[HTML]{FA7F57}{orange}) at the same location and show the posterior over random matches from 8-Gaussians to two-moons conditioned on this point. The right panel shows the posterior of the interpolation coefficient $t$ (gray) and azimuth angles of each minimizer at that position conditioned on different values of $t$.}
        \label{fig:failure_fm}
    \end{minipage}\hfill
    \begin{minipage}[t]{0.49\linewidth}
        \centering
        \includegraphics[width=0.56\linewidth]{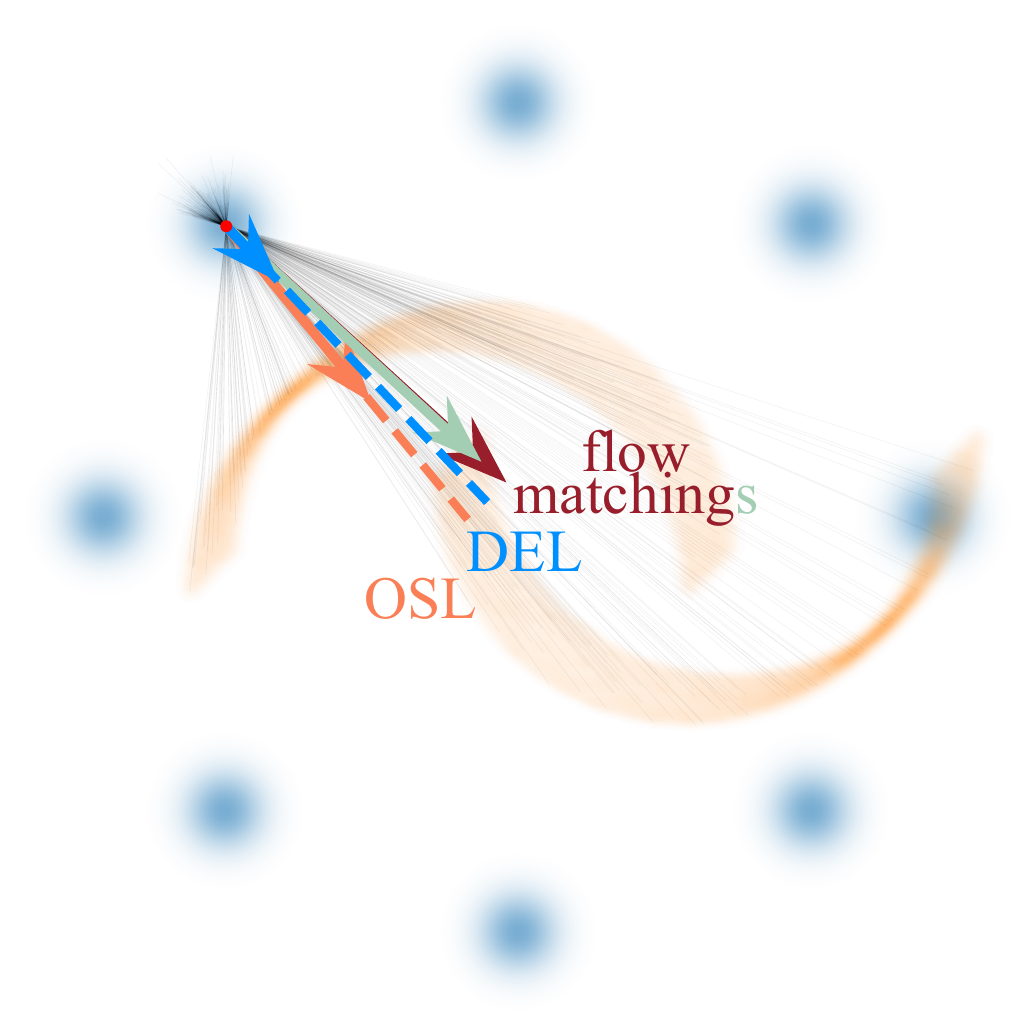}%
        \raisebox{0.3\height}{%
          \includegraphics[width=0.43\linewidth]{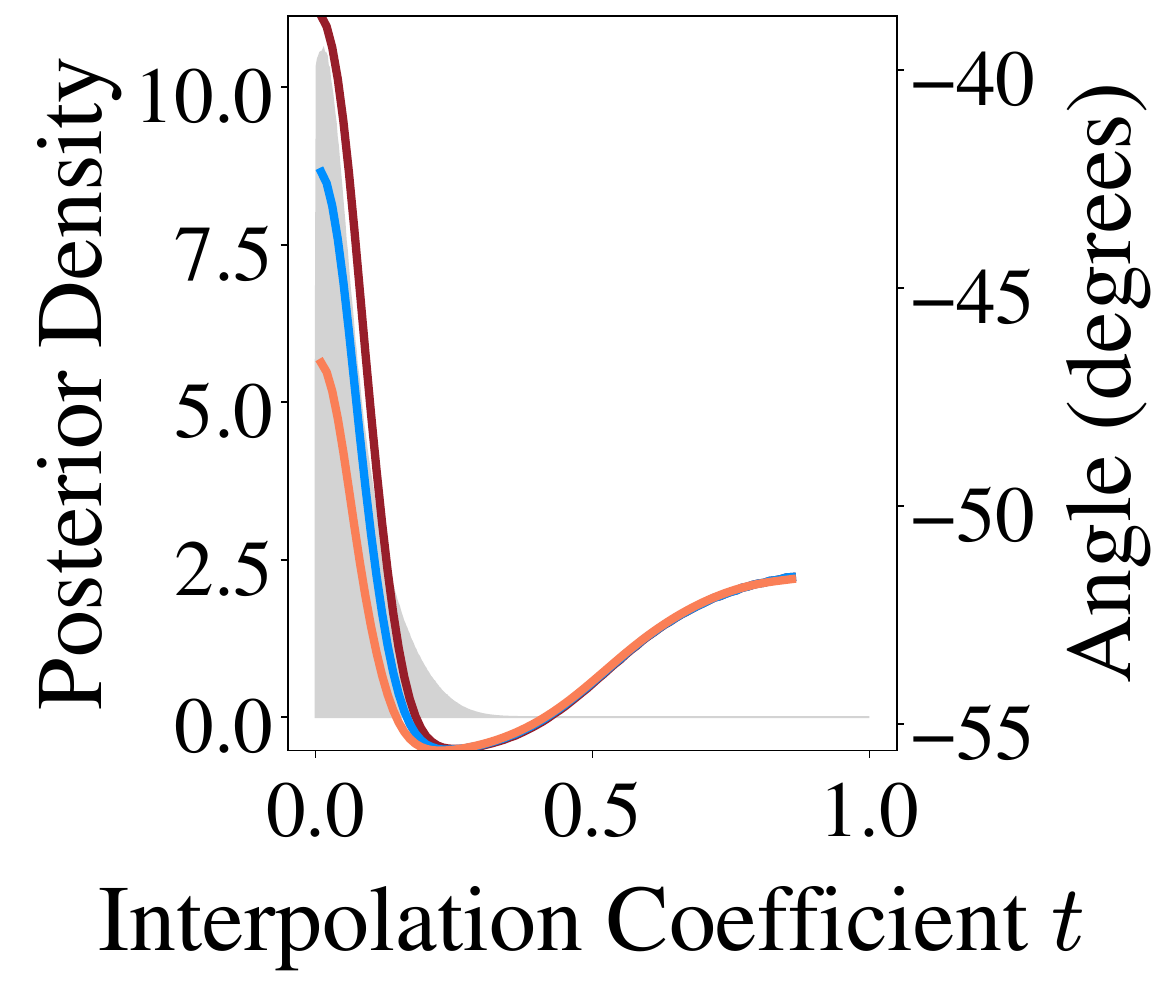}%
        }
        \captionof{figure}{\textbf{Flow matching minimizers can be biased toward the data mean, whereas our losses intentionally emphasize closer targets.} Even with the $(1-t)^{-2}$ reweighting, the reweighted minimizer (\textcolor[HTML]{A3CEB3}{green}) remains very close to standard flow matching (\textcolor[HTML]{971E2A}{red}). In contrast, one-step loss (\textcolor[HTML]{FA7F57}{orange}) and directional eikonal loss (\textcolor[HTML]{008FFF}{blue}) yield minimizers that are less affected by the data mean, leading to more effective early-stage denoising (see~\Cref{fig:fid-compare}). Same right panel in~\Cref{fig:failure_fm}.}
        \label{fig:failure_reweighted}
    \end{minipage}
    \vspace{-1.5em}
\end{figure}

\noindent\textbf{Disambiguation via inverse-squared-distance reweighting.}
In the time-unconditional setting, removing time input marginalizes the targets and can be misleading without concentration on closer neighbors: at the same location $\mathbf{x}$, the learned direction may point away from high-density regions of the target distribution, whereas our one-step loss and directional eikonal loss yield more stable directions by concentrating on closer targets~(\Cref{fig:failure_fm}). 

This failure mode is rooted in posterior ambiguity. Conditioning on $\mathbf{X}=\mathbf{x}$ (\Cref{def:genproc}) while marginalizing the interpolation coefficient $T$ mixes targets from different time steps, so the posterior $(\mathbf{X}_0,\mathbf{X}_1)\mid \mathbf{X}=\mathbf{x}$ becomes multi-modal over heterogeneous directions. Consequently, the squared-loss minimizer
$\mathbf{v}^{*}(\mathbf{x})=\mathbb{E}\!\left[\mathbf{X}_1-\mathbf{X}_0\mid \mathbf{X}=\mathbf{x}\right]$
averages incompatible denoising signals, which can produce a direction that does not correspond to any meaningful denoising trajectory. Conditioning on $t$ collapses the posterior, and the minimizer becomes well-defined and interpretable~(\Cref{fig:time-conditioning}).

\begin{figure}[!htbp]
  \centering
  \begin{minipage}[t]{0.40\linewidth}
    \vspace{0pt}
    \centering
    \begin{minipage}[t][1.45in][t]{\linewidth}
    \begin{subfigure}[t]{0.48\linewidth}
      \centering
      \includegraphics[width=\linewidth]{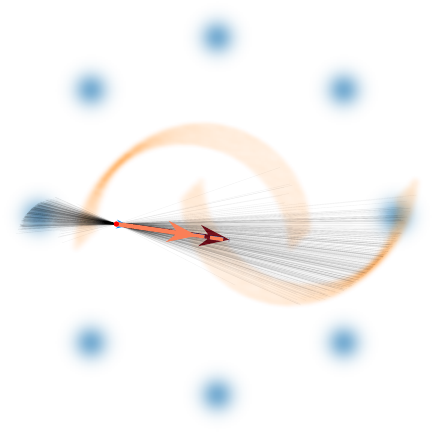}
      \caption{$t=0.25$}
      \label{fig:fixedt1}
    \end{subfigure}\hfill
    \begin{subfigure}[t]{0.48\linewidth}
      \centering
      \includegraphics[width=\linewidth]{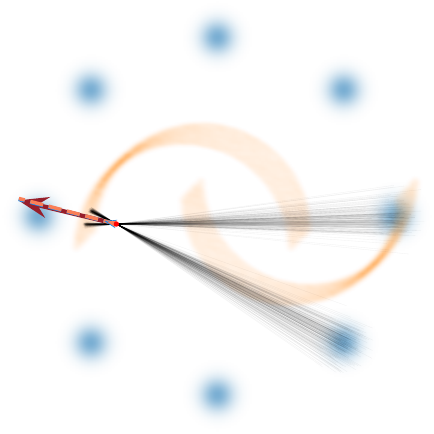}
      \caption{$t=0.90$}
      \label{fig:fixedt2}
    \end{subfigure}
    \end{minipage}
    \caption{With an extra condition on the interpolation coefficient, the directions become meaningful and consistent. Both panels use the same coordinate scale and a common arrow scale.}
    \label{fig:time-conditioning}
  \end{minipage}\hfill
  \begin{minipage}[t]{0.25\linewidth}
    \vspace{0pt}
    \centering
    \begin{minipage}[t][1.45in][t]{\linewidth}
    \includegraphics[width=\linewidth]{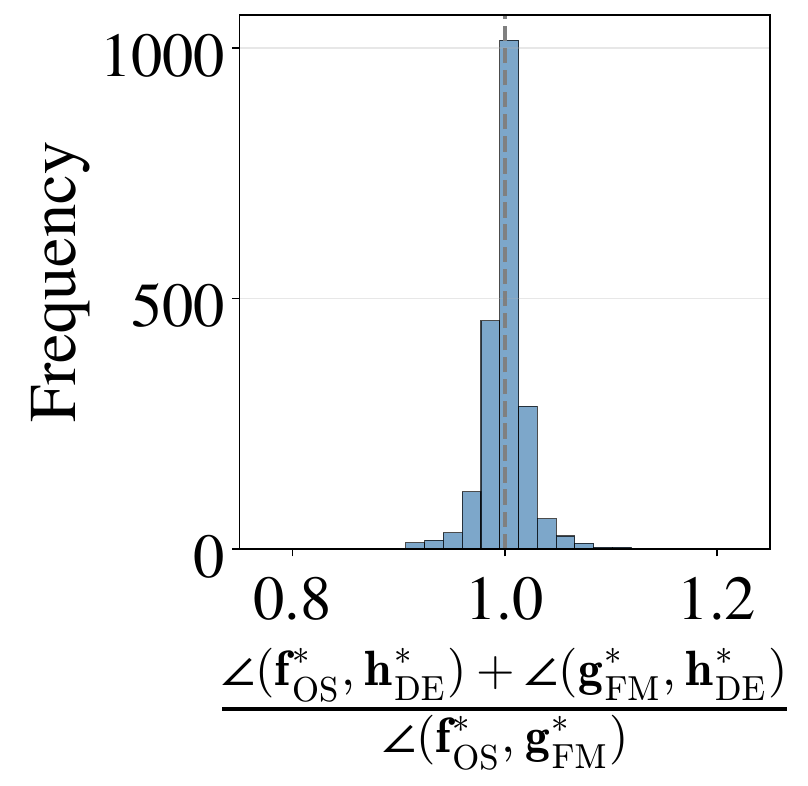}
    \end{minipage}
    \caption{An additive relation like the 2D angles is observed with a mean of 1.0003 in image space.}
    \label{fig:verify-equa}
  \end{minipage}\hfill
  \begin{minipage}[t]{0.31\linewidth}
    \vspace{0pt}
    \centering
    \begin{minipage}[t][1.45in][t]{\linewidth}
    \includegraphics[width=\linewidth]{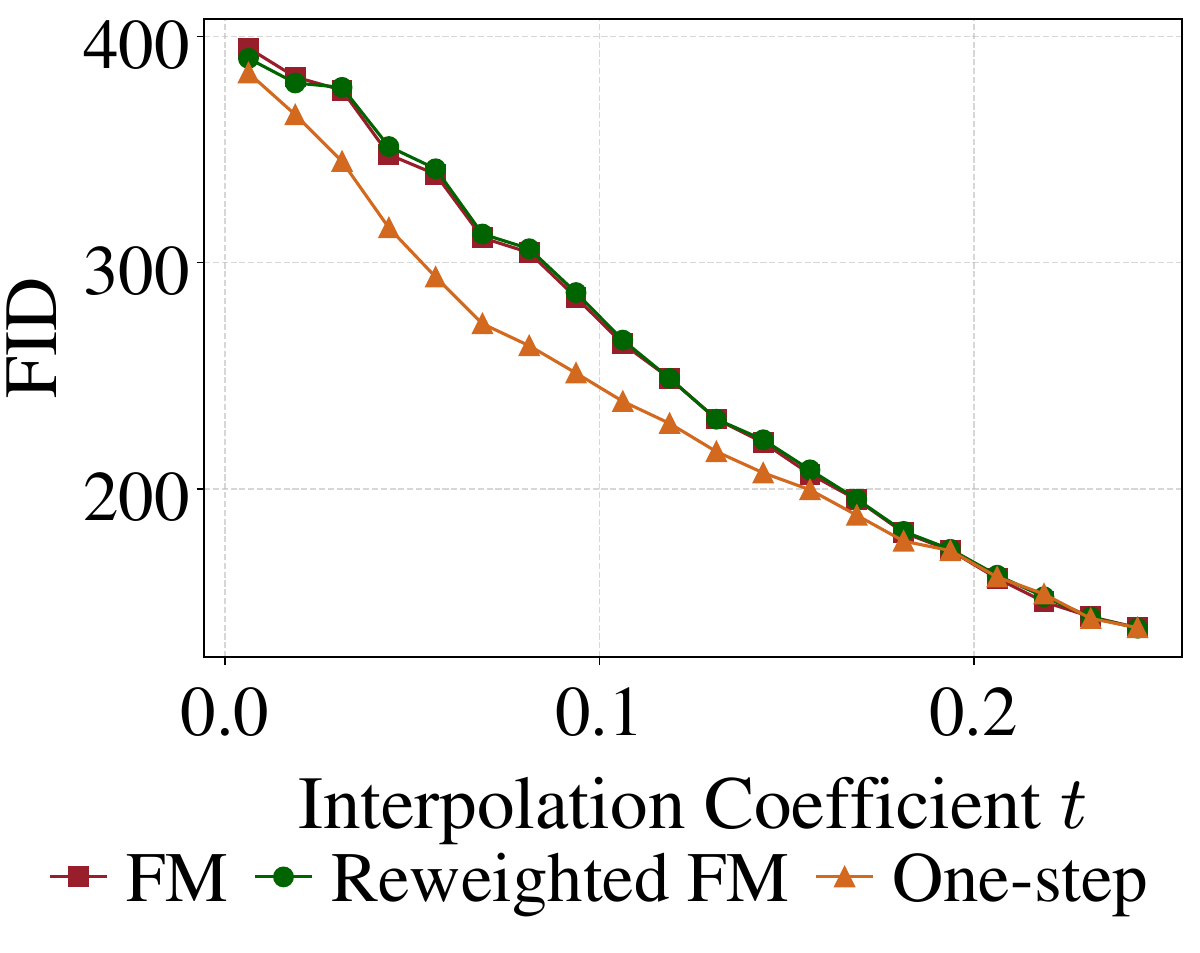}
    \end{minipage}
    \caption{With all other variables controlled, our \textcolor[HTML]{FA7F57}{one-step loss} offers a denoising direction that achieves lower FID, corroborating~\Cref{fig:failure_reweighted}.}
    \label{fig:fid-compare}
  \end{minipage}
\end{figure}

\begin{wrapfigure}[9]{r}{0.18\textwidth}
  \vspace{-2em}
  \centering
  \includegraphics[width=\linewidth]{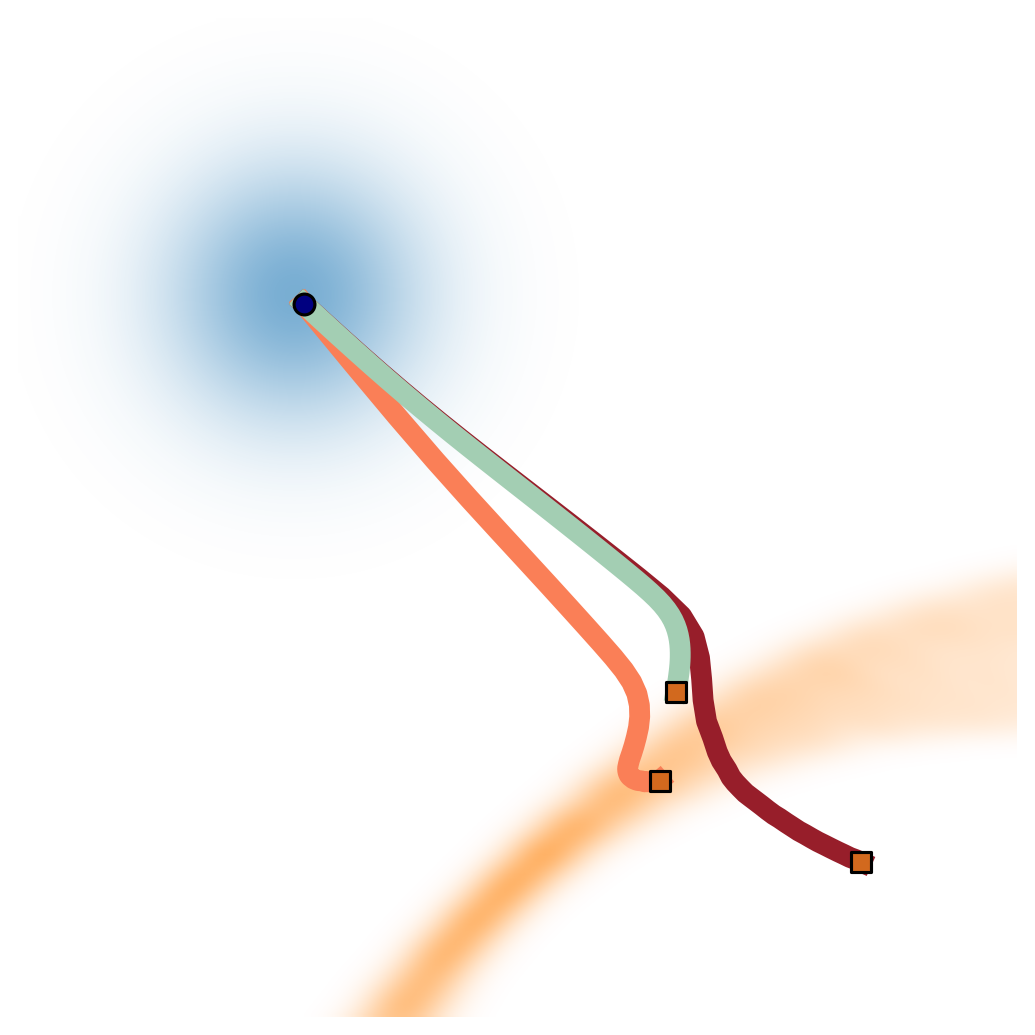}
  \captionsetup{justification=centering,labelsep=newline,skip=2pt}
  \caption{Denoising paths.}
  \label{fig:paths}
\end{wrapfigure}
A natural modification is to reweight the flow-matching loss by $(1-t)^{-2}$ to emphasize larger $t$.
However, this reweighted loss still tends to yield directions that point toward the data mean, as shown in~\Cref{fig:failure_reweighted}.
Because its minimizer initially biases toward the geometric center of the target distribution, it does not induce locality (i.e., focus on closer targets) in the early stage, even in the absence of multi-modality.
\Cref{fig:paths} shows different denoising paths starting from the same point in~\Cref{fig:failure_reweighted} with minimizers of our one-step loss (\textcolor[HTML]{FA7F57}{orange}), flow matching loss (\textcolor[HTML]{971E2A}{red}), or reweighted flow matching loss (\textcolor[HTML]{A3CEB3}{green}).
\Cref{subsec:8d} shows further qualitative experiments in 8D.

Our one-step loss explicitly prioritizes closer targets, while directional eikonal loss offers an adjustable intermediate concentration between one-step loss and standard flow matching. 
We formally study the properties of the losses in \Cref{app:properties} and summarize the results below:

\begin{adjustwidth}{1em}{0pt}
\textbf{(i) Flow matching loss vs.\ directional eikonal loss.}
Time-unconditional flow matching loss makes longer pairs dominate the direction, whereas directional eikonal loss is length-neutral.

\textbf{(ii) One-step loss.}
One-step loss explicitly concentrates supervision on nearby targets as a feature via inverse-distance reweighting, thereby reducing the time-marginalized target ambiguity.
\end{adjustwidth}

Our inverse-distance reweighting makes the learned ``distance'' explicit: under one-step loss, the model learns a denoising displacement from a noisy input toward the data manifold, with supervision biased toward closer targets via inverse-squared-distance reweighting.

Finally, we confirm our finding by an exploratory analysis on the CIFAR-10 image dataset.

At the fixed condition $t=0$, we sample multiple initial points 
$\mathbf{X}=\mathbf{X}_0 \sim \Normal(\mathbf{0},\mathbf{I}_D)$.
By enumerating all targets $\mathbf{X}_1 \in \mathcal{D}$, we numerically compute the optimal denoising directions induced by the directional eikonal loss, the one-step loss, and flow matching, denoted by $\mathbf{h}_{\mathrm{DE}}^{*}$, $\mathbf{f}_{\mathrm{OS}}^{*}$, and $\mathbf{g}_{\mathrm{FM}}^{*}$, respectively.
As in the 2D case, $\mathbf{h}_{\mathrm{DE}}^{*}$ lies between $\mathbf{f}_{\mathrm{OS}}^{*}$ and $\mathbf{g}_{\mathrm{FM}}^{*}$.
We further empirically verify the additive angle relation
$\angle(\mathbf{f}_{\mathrm{OS}}^{*},\mathbf{g}_{\mathrm{FM}}^{*})
=\angle(\mathbf{f}_{\mathrm{OS}}^{*},\mathbf{h}_{\mathrm{DE}}^{*})
+\angle(\mathbf{g}_{\mathrm{FM}}^{*},\mathbf{h}_{\mathrm{DE}}^{*})$
in~\Cref{fig:verify-equa} using $2048$ samples of $\mathbf{x}\sim\Normal(\mathbf{0},\mathbf{I}_D)$.

\par
\noindent\begin{minipage}[t]{0.48\linewidth}
\noindent\setlength{\parskip}{0.4em}%
\Cref{fig:unbalanced} further shows that the flow matching loss induces highly imbalanced sample weights on CIFAR-10 because of the variation in the pairwise length factor
$\ell:=\|\mathbf{X}_1-\mathbf{X}_0\|_2$ in~\Cref{thm:dir_fm_vs_de}.
Since flow matching learns the full denoising vector, pairs with larger $\ell$ contribute larger weights and can disproportionately affect the learned denoising direction. On CIFAR-10, we verify that $\ell$ fluctuates substantially, ranging from $50$ to $80$ in~\Cref{fig:unbalanced}.

With all other settings fixed, one-step loss achieves substantially lower FID than flow matching in~\Cref{fig:fid-compare}, consistent with the failure case in~\Cref{fig:failure_reweighted} (see~\Cref{subsec:ablation}).
\Cref{fig:full_fid_comparison} provides full curves.
\end{minipage}\hfill
\begin{minipage}[t]{0.49\linewidth}
  \vspace{-\ht\strutbox}
  \captionsetup{type=figure}
  \centering
\begin{subfigure}[t]{0.48\linewidth}
    \vspace{0pt}
    \centering
    \includegraphics[width=\linewidth]{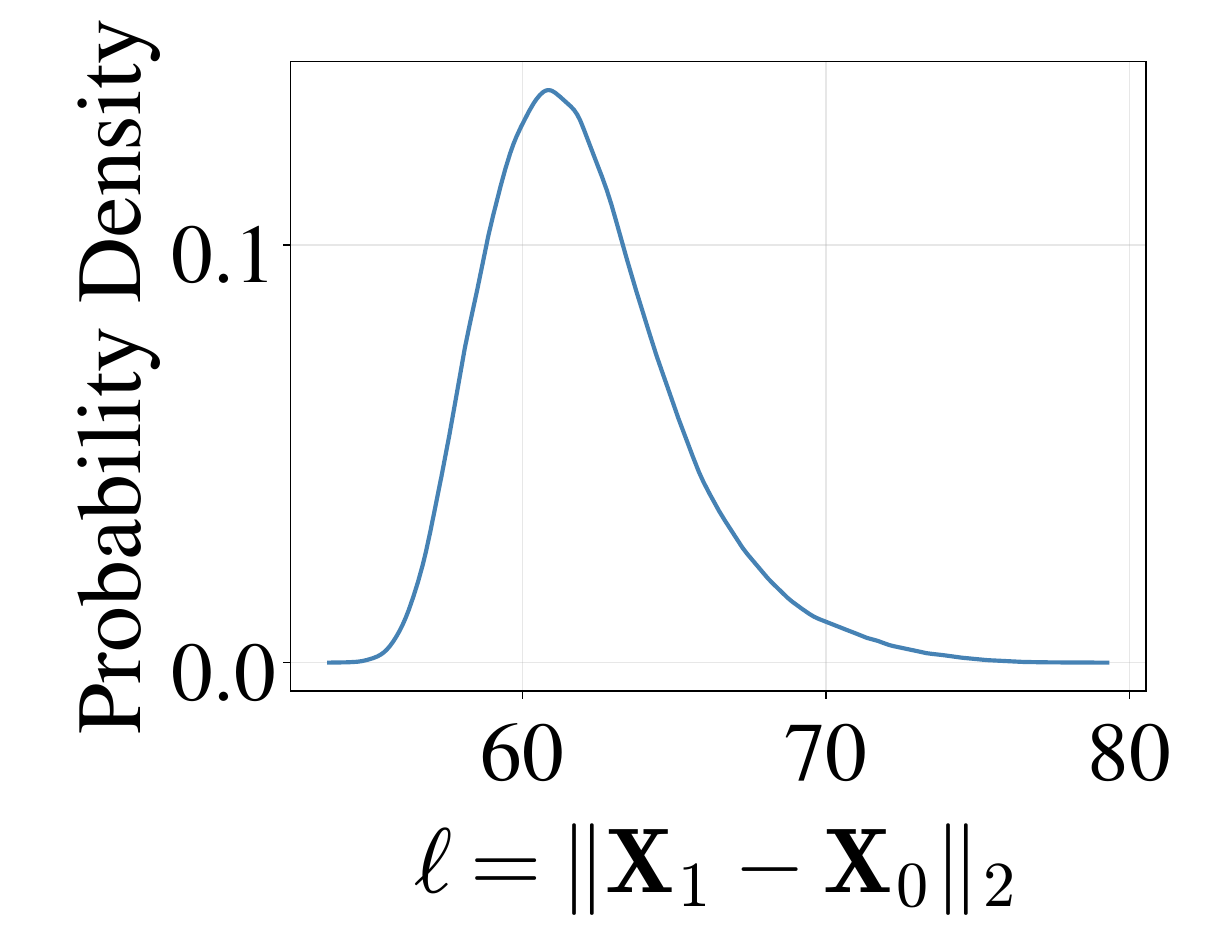}
    \caption{The pairwise length factor $\ell$ varies widely, acting as a random reweighting in denoising.}
  \end{subfigure}\hfill
  \begin{subfigure}[t]{0.48\linewidth}
    \vspace{0pt}
    \centering
    \includegraphics[width=\linewidth]{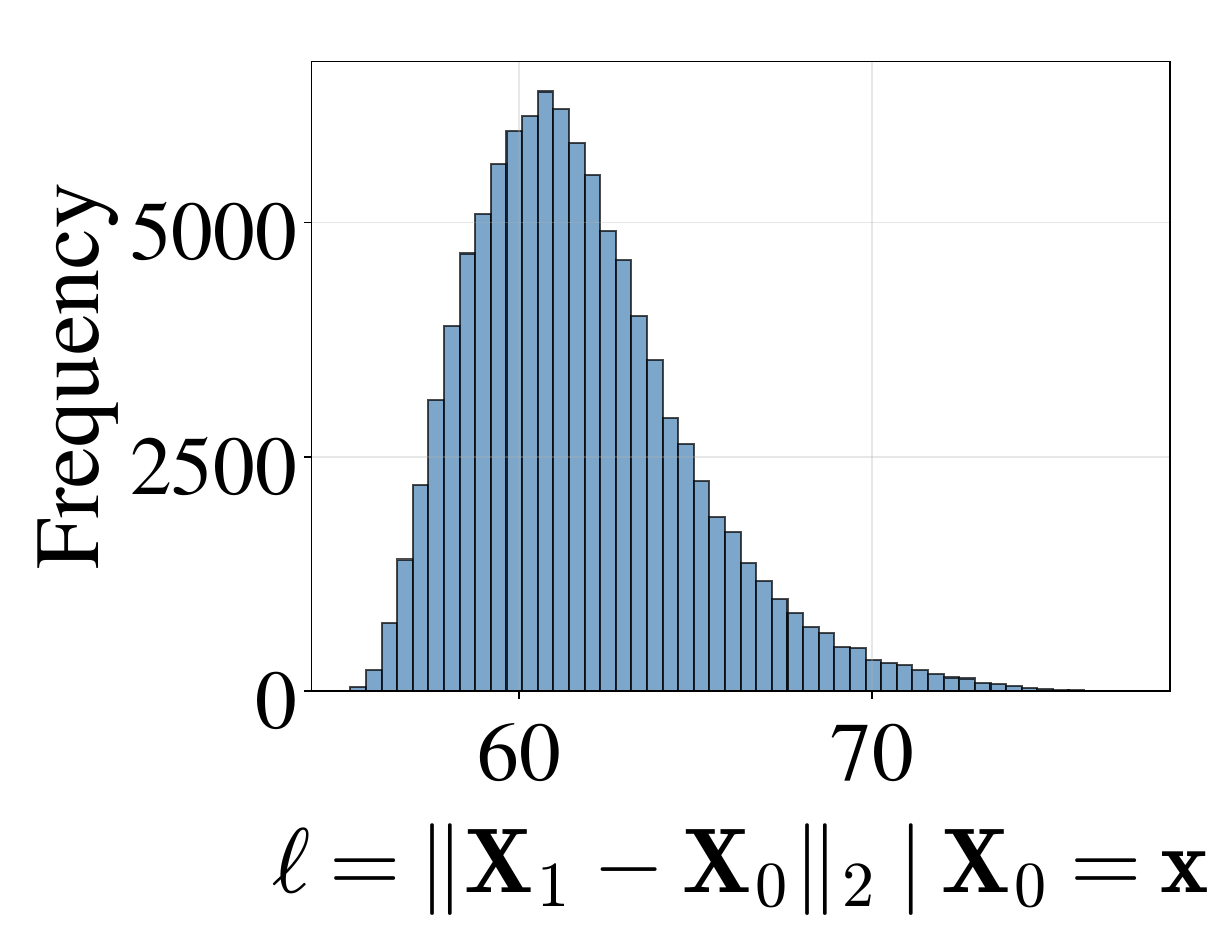}
    \caption{For a fixed noise sample $\mathbf{x}$, distances from real images to $\mathbf{x}$ also exhibit high variance.}
  \end{subfigure}

  \caption{Flow matching loss implicitly upweights signals from longer pairs, while our losses reduce this bias (see~\Cref{thm:dir_fm_vs_de}).}
  \label{fig:unbalanced}
  \end{minipage}
\par

\noindent\textbf{Support recovery versus density recovery.}
Our problem formulation naturally focuses on pulling samples to the data manifold support, unlike classical flow matching, which also reconstructs the data density.
This is a feature of our method. The usefulness of the feature depends on the application~\cite{li2025scores}.
In \Cref{app:allocation}, we show that the distance field can be used to naturally induce a normalized Hausdorff measure on the data manifold using a Gibbs transformation. This allows the distance field to take MCMC as an optional exploration step, which we use in~\Cref{tab:2d_metrics} and when the initialization is poor (\Cref{app:ula}). Other methods that explore the data manifold can also be combined~\cite{de2025provable}.

In practice, we avoid costly MCMC and adopt deterministic processes. 
That is, we start from an initial distribution and deterministically follow our samplers (\cref{eq:image_st,eq:gd}). 
In \Cref{app:pushforward}, we use a pushforward analysis to show that a distance field linearly projects distributions from high-dimensional space to the manifold. 
This perspective also explains the practical robustness of our method: recovering a particular non-uniform density on $\mathcal{S}$ requires a corresponding input distribution, but reaching the data manifold does not require this exact density match. 
Rather, a wide range of initial noise in the ambient space can still be guided to the data manifold by the learned field. 

In essence, flow matching and diffusion models, from a density-recovery perspective, also need a specific input distribution to recover a specific output distribution.
For example, when starting from a uniform initialization at inference time that differs from the Gaussian initialization used during training, flow matching significantly fails to produce realistic images.
In contrast, our method robustly generates images without MCMC and achieves a 63.3\% lower FID, as shown in~\Cref{fig:combined_baseline_range}.
Thus, the initial distribution becomes a practical degree of freedom for shaping the final manifold density, rather than a strict prior that must be matched exactly for generation to work.

\noindent\textbf{Convergence to the true distance field.}
Our inverse-squared-distance reweighting addresses a key tradeoff: nearest-neighbor supervision can collapse modes due to hubness, while uniformly selecting target samples biases denoising directions toward the data mean.

In \Cref{app:convergence_approx}, we show that under some condition on the flatness of the manifold, our inverse-squared-distance weight and our losses, when fully optimized, lead to a field that approximates the actual distance field and converges to the actual distance as the distance goes to zero.
We further show that the repeated update using our Sphere Tracing under inverse-squared-distance reweighting still constantly pulls samples toward the target manifold.

\section{Experiments}
\label{sec:exp}

\begin{figure}[t]
  \centering

  \begin{minipage}[t]{0.5\linewidth}
    \vspace{0pt}
    \centering
    \captionsetup{type=table,skip=6pt}
    \caption{2D point-cloud distance metrics between $10$k generated samples and $10$k target two-moons samples for the 8-Gaussians $\rightarrow$ two-moons task (lower is better).}
    \label{tab:2d_metrics}

    \tiny
    \setlength{\tabcolsep}{3pt}
    \renewcommand{\arraystretch}{1.05}

    \resizebox{\linewidth}{!}{%
\begin{tabular}{@{}>{\raggedright\arraybackslash}p{0.46\linewidth}
                >{\centering\arraybackslash}p{0.12\linewidth}
                *{3}{>{\centering\arraybackslash}p{0.10\linewidth}}@{}}
\toprule
\textbf{Method} & \textbf{w/o $t$} & \textbf{W2 $\downarrow$} & \textbf{HD $\downarrow$} & \textbf{CD $\downarrow$} \\
\midrule
FM~\cite{lipman2022flow}               & \texttimes  & 0.577 & 1.337 & 0.026 \\
OT-CFM~\cite{tong2023improving}        & \texttimes  & \textbf{0.218} & 1.233 & 0.019 \\
AM~\cite{neklyudov2023action}          & \texttimes  & 0.581 & 2.154 & 0.140 \\
AM Swish (as above)                    & \texttimes  & 0.442 & 1.719 & 0.033 \\
EM~\cite{balcerak2025energy}           & \checkmark  & 0.523 & 1.279 & 0.031 \\
EqM~\cite{wang2025equilibrium}         & \checkmark  & 1.525 & 1.691 & 0.150 \\
Distance Marching (ours)               & \checkmark  & 0.440 & \textbf{1.069} & \textbf{0.014} \\
\bottomrule
\multicolumn{5}{@{}l}{\scriptsize W2: Wasserstein-2; HD: Hausdorff distance; CD: Chamfer distance.}
\end{tabular}
}
  \end{minipage}
  \hfill
  \begin{minipage}[t]{0.47\linewidth}
    \vspace{0pt}
    \centering
    \includegraphics[width=\linewidth]{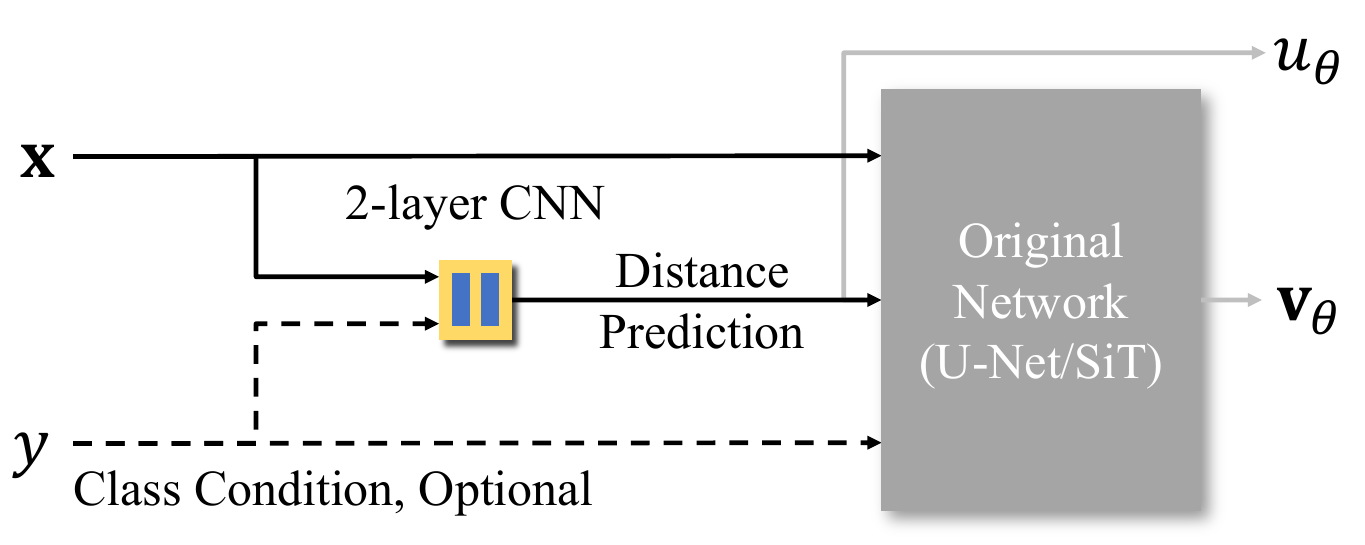}

    \captionsetup{type=figure,skip=6pt}
    \caption{Image generation pipeline. We predict $\mathbf{v}_\theta$ directly without enforcing $\mathbf{v}_\theta=\nabla u_\theta$, reducing computation cost and instability while preserving the key properties in~\Cref{sec:analysis}.}
    \label{fig:architecture}
  \end{minipage}
\vspace{-1em}
\end{figure}

\subsection{Mechanism Study in 2D}
\label{subsec:2dmechanism}

We test whether our losses recover a useful distance-like field and enable competitive generation in 2D.
Using the losses in~\Cref{subsec:dfmodeling} and minibatch nearest neighbors to approximate closest-point projections, we train a simple scalar network.
The learned field (\Cref{fig:2d_motivation}) achieves the lowest Hausdorff and Chamfer distances and the second-lowest Wasserstein distance (\Cref{tab:2d_metrics}).
These results suggest that a well-learned distance-like field can provide a reliable geometric signal for sampling.
Under the same experimental settings, other energy-based baselines learn less useful geometric landscapes and underperform in these metrics, with or without MCMC (\Cref{app:2d_exp}).

\subsection{Datasets and Implementation}
We evaluate image generation on CIFAR-10~\cite{krizhevsky2009learning} using the U-Net backbone of~\citet{lipman2024flow}, and on class-conditional ImageNet-256~\cite{deng2009imagenet} using the transformer backbone of~\citet{ma2024sit}, where each class label defines a class-specific target manifold, i.e., a different level set. 
We keep the original backbone implementations to avoid confounding architectural differences, as illustrated in~\Cref{fig:architecture}, and add only a lightweight two-layer CNN to predict $u_{\theta}(\mathbf{x})$, increasing parameters by 0.009\%--0.19\%. 
Further details are in~\Cref{sec:exp_set}. 
We also retrieve nearest neighbors from the training set to verify that the model does not simply copy; see~\Cref{app:memo}.

\subsection{Time- and Class-Unconditional CIFAR-10 Generation}
\label{subsec:cifar10}

\newcolumntype{L}{>{\raggedright\arraybackslash}X}
\newcommand{\darrow}{\ensuremath{\downarrow}}

\begin{figure}[t]
\centering

\begin{minipage}[t]{0.37\linewidth}
\vspace{0pt}

\captionof{table}{FID on unconditional CIFAR-10 generation. Bold and underline: best; bold: second-best.}
\label{tab:cifar-10}
\centering
\tiny
\setlength{\tabcolsep}{6pt}
\renewcommand{\arraystretch}{1.10}

\resizebox{\linewidth}{!}{%
\begin{tabular}{@{}l c r@{}}
\toprule
\textbf{Method} & \textbf{\mbox{w/o $t$}} & \textbf{FID}~\darrow \\
\midrule
NCSN++~\cite{song2020score}                                & \texttimes & 2.45 \\
DDPM++~\cite{kim2021soft}                                  & \texttimes & 3.45 \\
FM~\cite{lipman2022flow}                        & \texttimes & 6.35 \\
EDM~\cite{karras2022elucidating} & \texttimes & \underline{\textbf{1.99}} \\
1-RF FM~\cite{liu2022flow}  & \texttimes & 2.53 \\
Cooperative DRL-large~\cite{zhu2023learning}              & \texttimes & 3.68 \\
CLEL-large~\cite{lee2023guiding}                           & \checkmark & 8.61 \\
uEDM~\cite{sun2025noise}                                   & \checkmark & \textbf{2.23} \\
EM~\cite{balcerak2025energy}                  & \checkmark & 3.34 \\
EqM~\cite{wang2025equilibrium}            & \checkmark & 3.36 \\
Distance Marching (ours, gradient descent)                  & \checkmark & 2.54 \\
Distance Marching (ours, sphere tracing)                    & \checkmark & \textbf{2.23} \\
\bottomrule
\end{tabular}%
}

\medskip
\vspace{-1em}
\captionof{table}{Best results across different CFG scales on ImageNet generation.
}
\label{tab:compcfg}
\centering
\tiny
\setlength{\tabcolsep}{4pt}
\renewcommand{\arraystretch}{1.05}
\begin{tabular}{@{}cccccc@{}}
\toprule
\textbf{Metrics} & SiT & uEDM & EqM & Our ST & Our GD \\
\midrule
FID $\downarrow$ & 7.66 & 7.96 & 6.93 & \textbf{6.59} & 6.67 \\
IS $\uparrow$  & 254.77 & -- & 281.87 & 310.02 & \textbf{311.31} \\
\bottomrule
\end{tabular}

\end{minipage}%
\hfill
\begin{minipage}[t]{0.60\linewidth}
\vspace{0pt}
\centering

\includegraphics[width=0.99\linewidth]{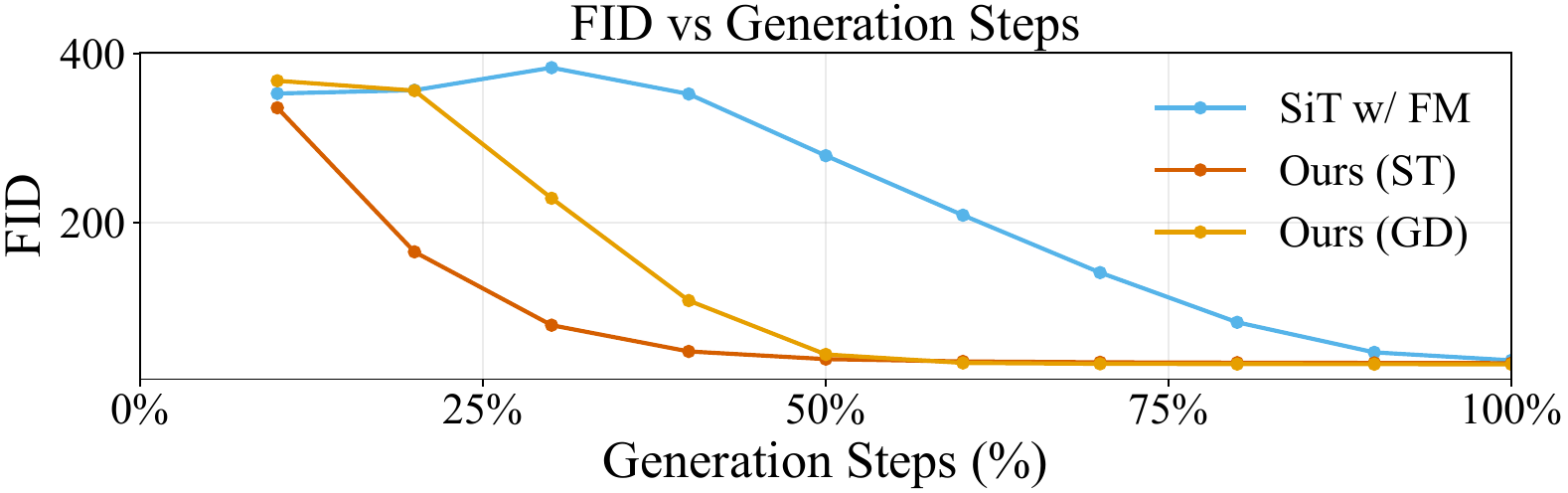}

\begingroup
\newcommand{\labw}{0.07\linewidth}
\newcommand{\thumbw}{\dimexpr(\linewidth-\labw)/11\relax}
\newcommand{\rowlabel}[1]{%
  \vbox to \thumbw{%
    \vss
    \hbox to \labw{\hss\footnotesize #1\hss}%
    \vss
  }%
}

\noindent\vbox{\offinterlineskip
  \hbox to \linewidth{%
    \rowlabel{ST:}%
    \includegraphics[width=\thumbw,height=\thumbw]{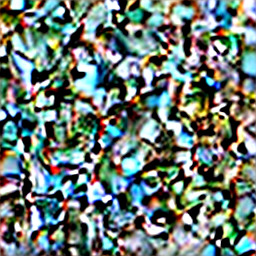}%
    \includegraphics[width=\thumbw,height=\thumbw]{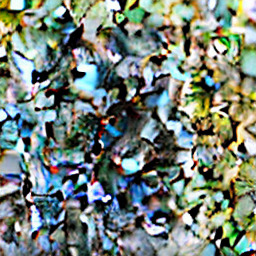}%
    \includegraphics[width=\thumbw,height=\thumbw]{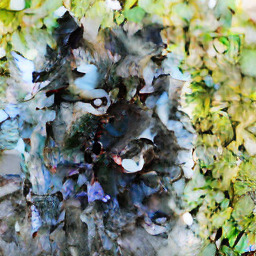}%
    \includegraphics[width=\thumbw,height=\thumbw]{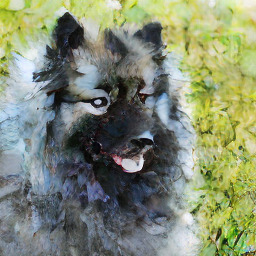}%
    \includegraphics[width=\thumbw,height=\thumbw]{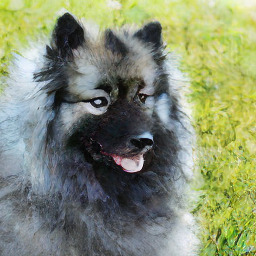}%
    \includegraphics[width=\thumbw,height=\thumbw]{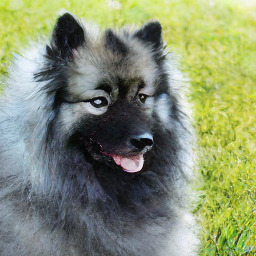}%
    \includegraphics[width=\thumbw,height=\thumbw]{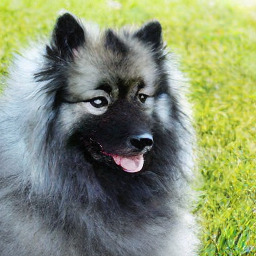}%
    \includegraphics[width=\thumbw,height=\thumbw]{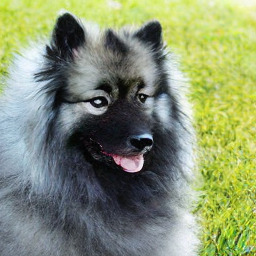}%
    \includegraphics[width=\thumbw,height=\thumbw]{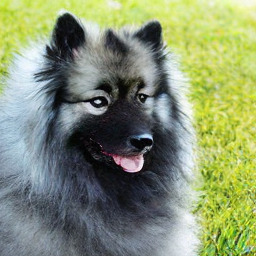}%
    \includegraphics[width=\thumbw,height=\thumbw]{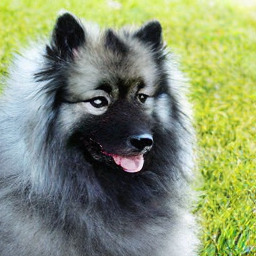}%
    \includegraphics[width=\thumbw,height=\thumbw]{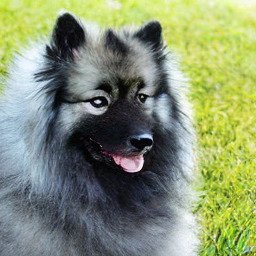}%
  }%
  \hbox to \linewidth{%
    \rowlabel{GD:}%
    \includegraphics[width=\thumbw,height=\thumbw]{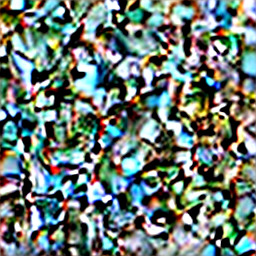}%
    \includegraphics[width=\thumbw,height=\thumbw]{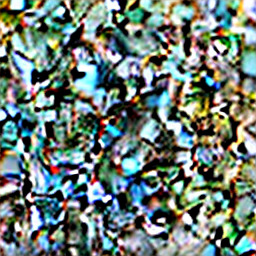}%
    \includegraphics[width=\thumbw,height=\thumbw]{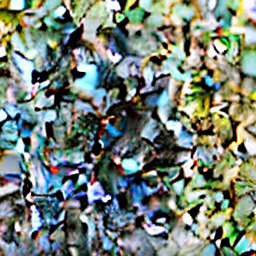}%
    \includegraphics[width=\thumbw,height=\thumbw]{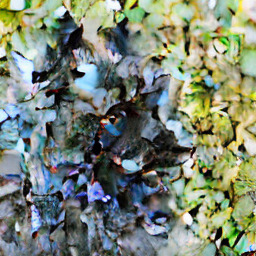}%
    \includegraphics[width=\thumbw,height=\thumbw]{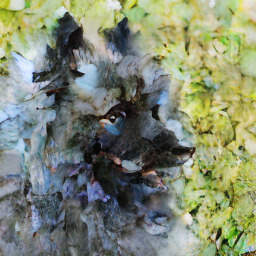}%
    \includegraphics[width=\thumbw,height=\thumbw]{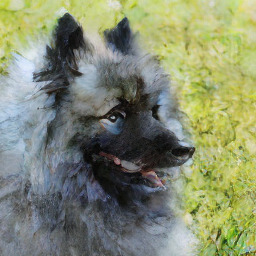}%
    \includegraphics[width=\thumbw,height=\thumbw]{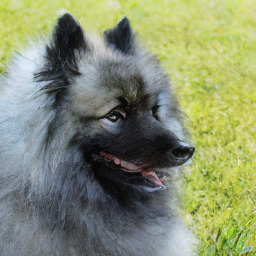}%
    \includegraphics[width=\thumbw,height=\thumbw]{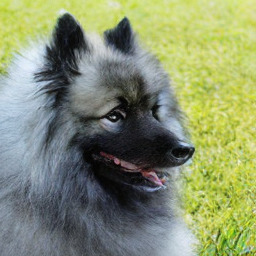}%
    \includegraphics[width=\thumbw,height=\thumbw]{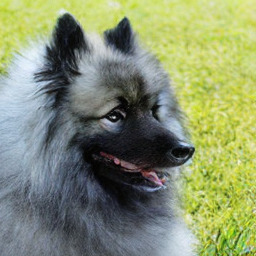}%
    \includegraphics[width=\thumbw,height=\thumbw]{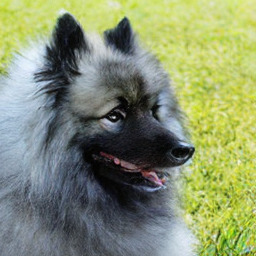}%
    \includegraphics[width=\thumbw,height=\thumbw]{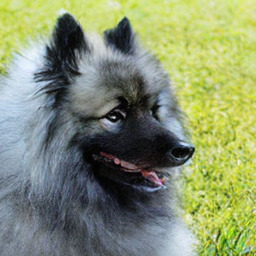}%
  }%
  \hbox to \linewidth{%
    \rowlabel{FM:}%
    \includegraphics[width=\thumbw,height=\thumbw]{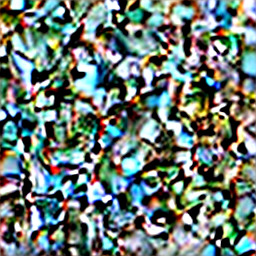}%
    \includegraphics[width=\thumbw,height=\thumbw]{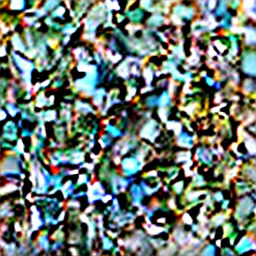}%
    \includegraphics[width=\thumbw,height=\thumbw]{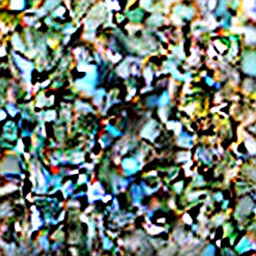}%
    \includegraphics[width=\thumbw,height=\thumbw]{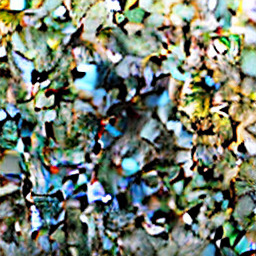}%
    \includegraphics[width=\thumbw,height=\thumbw]{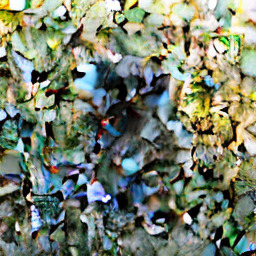}%
    \includegraphics[width=\thumbw,height=\thumbw]{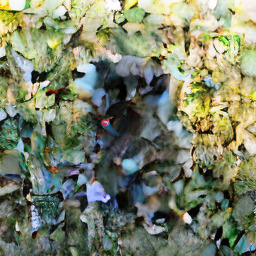}%
    \includegraphics[width=\thumbw,height=\thumbw]{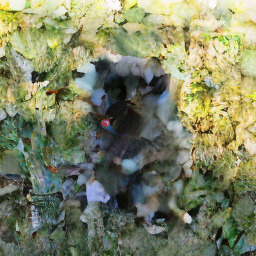}%
    \includegraphics[width=\thumbw,height=\thumbw]{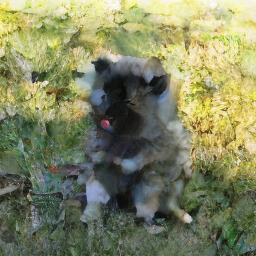}%
    \includegraphics[width=\thumbw,height=\thumbw]{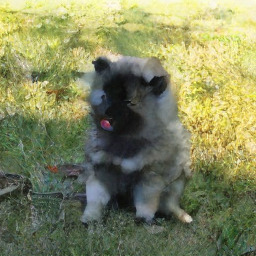}%
    \includegraphics[width=\thumbw,height=\thumbw]{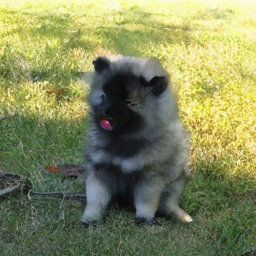}%
    \includegraphics[width=\thumbw,height=\thumbw]{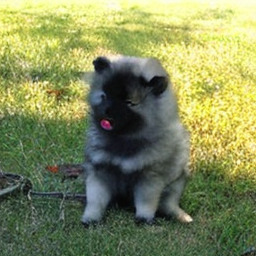}%
  }%
}
\endgroup

\caption{With step sizes tuned for best 250-step performance, gradient descent and sphere tracing produce perceptually better images earlier than flow matching (SiT+dopri5).
We reach flow matching's final FID within $<\!60\%$ of the steps.
FID uses no CFG (B/2 base); sequences use CFG$=4$ (XL/2 base, every 10\% steps).
See~\Cref{app:gpc} for more details.}
\label{fig:faster}

\end{minipage}
\vspace{-1em}
\end{figure}

\begin{figure}[!t]
\noindent\begin{minipage}[t]{0.40\linewidth}
\vspace{0pt}
\centering
\captionof{table}{ImageNet FID without CFG, with similar model sizes and the same NFE (250). With tuned CFG scales, sphere tracing achieves lower FID than gradient descent~(\Cref{tab:compcfg}).}
\label{tab:imagenet}
\vspace{-0.4em}

\tiny
\setlength{\tabcolsep}{3pt}
\renewcommand{\arraystretch}{1.08}

\begin{tabularx}{\linewidth}{@{}L c r@{}}
\toprule
\textbf{Method} & \textbf{w/o $t$} & \textbf{FID}~\darrow \\
\midrule
DiT-B/2~\cite{peebles2023scalable}                             & \texttimes & 43.47 \\
SiT-B/2 (standard FM)~\cite{ma2024sit}                          & \texttimes & 36.68 \\
SiT-B/2 (best; $w_t^{\mathrm{KL},\,\eta}\!$ SDE)~\cite{ma2024sit} & \texttimes & 33.02 \\
uEDM~\cite{sun2025noise}                                        & \checkmark & 40.80 \\
Equilibrium Matching-B/2~\cite{wang2025equilibrium}             & \checkmark & 32.85 \\
Distance Marching-B/2 (ours, sphere tracing)                    & \checkmark & 33.44 \\
Distance Marching-B/2 (ours, gradient descent)                  & \checkmark & \textbf{32.16} \\
\bottomrule
\end{tabularx}
\end{minipage}\hfill
\begin{minipage}[t]{0.58\linewidth}
\vspace{0pt}
\centering
\captionsetup{type=figure,skip=2pt}

\begin{subfigure}[t]{0.48\linewidth}
  \centering
  \includegraphics[width=\linewidth]{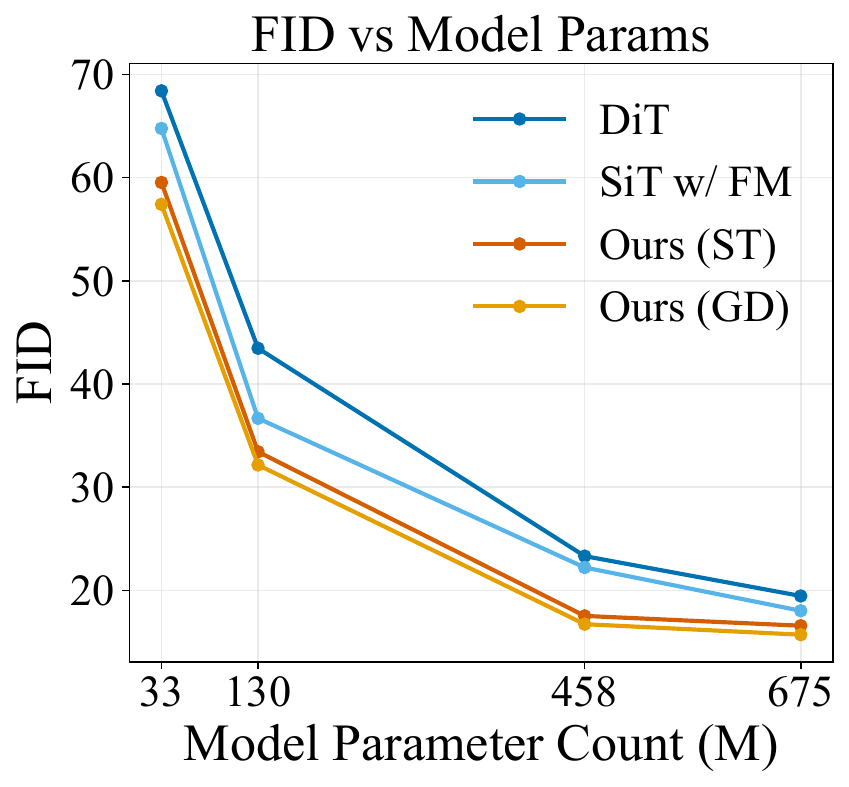}
\end{subfigure}
\hfill
\begin{subfigure}[t]{0.48\linewidth}
  \centering
  \includegraphics[width=\linewidth]{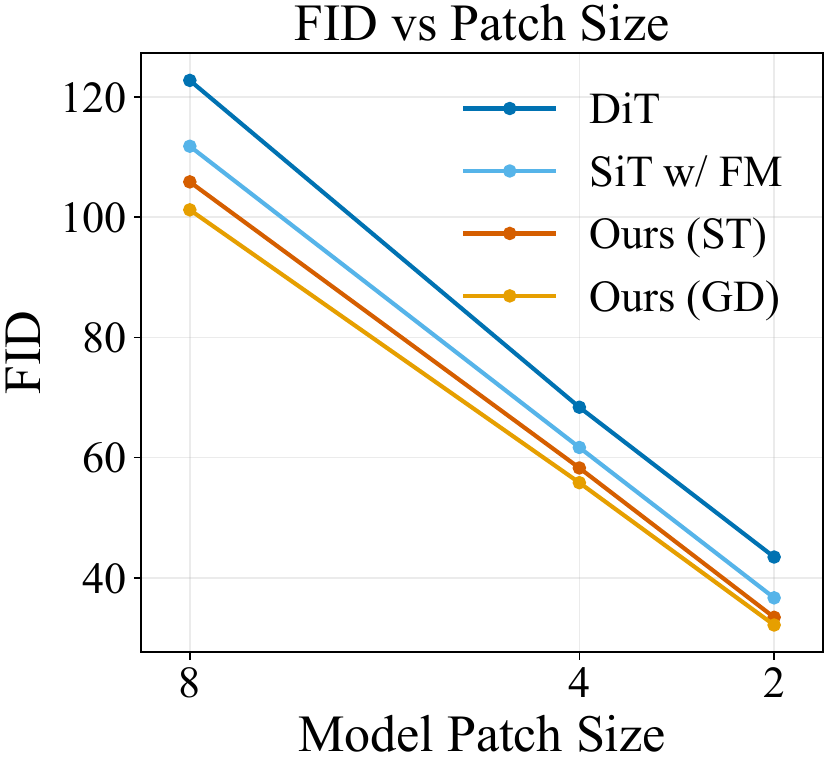}
\end{subfigure}

\caption{
Compared to standard FM SiT with time, our method reduces FID by 9.5\%--24.7\% across S/2, B/2, B/4, B/8, L/2, and XL/2 models, and 13.6\% on average.}
\label{fig:different_sizes}
\end{minipage}
\end{figure}

\dm offers a simple yet effective time-unconditional method for CIFAR-10 image generation, as shown in~\Cref{tab:cifar-10}.
\Cref{fig:uncurated_cifar10} in the appendix shows uncurated samples.

\begin{wrapfigure}{r}{0.51\textwidth}
    \centering
    \vspace{-0.3em}
    \includegraphics[width=\linewidth]{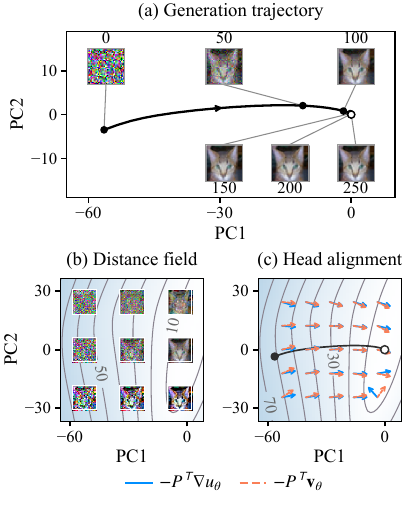}
    \vspace{-0.4em}
    \captionsetup{skip=0pt}
    \vspace{-0.4em}
    \caption{CIFAR-10 geometry in a PCA plane. Across 800 queries from 32 seeds, the scalar-field gradient and vector head are closely aligned in the PCA plane, with a mean projected angle discrepancy of only $8.90^\circ$.}
    \label{fig:distance_head_main}
\end{wrapfigure}

Our method ranks first (tied) among time-unconditional methods~\cite{lee2023guiding,sun2025noise, balcerak2025energy, wang2025equilibrium} and first among energy-based models~\cite{zhu2023learning,lee2023guiding,balcerak2025energy}.
uEDM's~\cite{sun2025noise} quality drops when scaling to class-conditional ImageNet, providing FID 26.9\% higher than ours without CFG, which we attribute to its delicate coefficient schedule.

\noindent\textbf{Distance estimate evaluation.}
We further evaluate the accuracy of distance estimation~(\Cref{fig:dist_error}).
Even with a tiny distance head $u_\theta$, the error remains below 5\% for most samples when using smoothed distance $\sqrt{\|\mathbf{x}-\mathbf{x}_1\|_2^2 + c_0}$ as the reference.
When $\|\mathbf{x}-\mathbf{x}_1\|_2$ is large, the distance outputs are slightly but systematically smaller than the reference values.
This is consistent with our earlier analysis that the learned output is dominated by shorter paths.

\noindent\textbf{Geometry of the learned distance field.}
\Cref{fig:distance_head_main} visualizes the learned field in the plane spanned by the first two principal components of a generation trajectory.
The trajectory shows progressive denoising, while the surrounding level sets relate image structure to predicted distance.
The scalar-field gradient and vector head also align within this plane, supporting the geometric interpretation of our sampler.
See~\Cref{app:distance_head} for experimental details.
\par\WFclear

\subsection{Time-unconditional ImageNet Generation}

\Cref{tab:imagenet} compares models of similar size at the same NFE (250) for class-conditional ImageNet generation without CFG. Our method achieves the lowest FID.
In comparisons without CFG across all backbone sizes (S/2, B/2, B/4, B/8, L/2, and XL/2), our method achieves 13.6\% lower FID on average than flow matching, which requires time input whereas ours does not~(\Cref{fig:different_sizes}).

\noindent\textbf{Compatibility with classifier-free guidance.} In~\Cref{tab:compcfg}, our methods show a clearer advantage over baselines.
Comparing with the runner-up, we improve more than $4.9\%$ on FID and $10.4\%$ on IS.
The full table is attached in~\Cref{tab:fid_cfg} and~\Cref{tab:is_cfg}.
Our samplers lead in coverage at eight of nine CFG scales, with higher density than SiT and EqM at all nine~(\Cref{tab:coverage_cfg,tab:density_cfg}).
We present samples in~\Cref{fig:curated_imagenet_grid}.

\noindent\textbf{Comparison between gradient descent and sphere tracing.}
Sphere tracing achieves lower FID than gradient descent with tuned CFG scales~(\Cref{tab:compcfg}) and converges faster~(\Cref{fig:faster}).
The visual sequences in~\Cref{fig:faster,fig:evolution} and the convergence experiments in~\Cref{app:convergence} also empirically support convergence without structurally constraining $\mathbf{v}_\theta$ to be the gradient of a scalar potential.

\vspace{-0.3em}
\subsection{Ablation Study}
\label{subsec:ablation}

We compare one-step denoising FID across objectives~(\Cref{fig:fid-compare}).
We fix the time-unconditional model architecture and training setup and change only the loss.

The denoising update depends on the objective: under one-step loss, the network predicts an additive correction that is directly added to the input, whereas for flow matching objectives, we rescale the output by $1-t$ before applying the update. We exclude directional eikonal loss because it does not define an explicit denoising destination.
After all models are trained for the same number of epochs to convergence, one-step loss yields lower FID, consistent with the benefit of learning from closer targets.
\Cref{fig:full_fid_comparison} provides full curves.

Additional ablations~(\Cref{app:ablation}) favor sphere tracing over gradient descent with and without directional eikonal loss~(\Cref{tab:sampler_ablation}).
We verify mode collapse and condition mismatch~(\Cref{fig:ablation_coupling}).
FID worsens with $c_0=\epsilon=0$, and denominator removal raises sphere-tracing FID by $74\%$~(\Cref{tab:loss_variation}).

\subsection{Applications of the Distance Estimation}
\label{subsection:applications_distance}

\begin{wraptable}{r}{0.45\textwidth}
\vspace{-1em}
\centering
\tiny
\setlength{\tabcolsep}{3pt}
\renewcommand{\arraystretch}{1.05}

\caption{Adaptive stopping for gradient descent sampling using $u_\theta$ on ImageNet generation.}
\label{tab:adaptive}
\begin{tabular}{@{}l c c c c c@{}}
\toprule
\textbf{Method} & $\eta$ & \textbf{Avg. Steps$\downarrow$} & \textbf{FID$\downarrow$} & \textbf{sFID$\downarrow$} & \textbf{IS$\uparrow$} \\
\midrule
original GD  & 0.64 & 250.0 & \textbf{32.16} & 7.52 & 45.16 \\
adaptive GD  & 0.64 & \textbf{175.8} & 32.73 & \textbf{6.70} & \textbf{45.86} \\
\bottomrule
\end{tabular}

\vspace{0.8em}

\caption{OOD detection AUROC$\uparrow$ on CIFAR-10 (in-distribution) with SVHN, Textures, and Constant as OOD datasets. Baseline results from~\citet{yoon2023energy} and~\citet{wang2025equilibrium}.}
\label{tab:ood}
\setlength{\tabcolsep}{2.5pt}
\renewcommand{\arraystretch}{1.10}
\begin{tabular}{@{}l c c c c c@{}}
\toprule
\textbf{Method} & \textbf{Params} & \textbf{SVHN} & \textbf{Textures} & \textbf{Constant} & \textbf{Avg.} \\
\midrule
PixelCNN++~\cite{salimans2017pixelcnn++} & 54M   & 0.32 & 0.33 & 0.71 & 0.45 \\
GLOW~\cite{kingma2018glow}       & 44M   & 0.24 & 0.27 & --   & 0.26 \\
IGEBM~\cite{du2019implicit}      & 19M   & 0.63 & 0.48 & 0.39 & 0.50 \\
EqM~\cite{wang2025equilibrium}   & 130M  & 0.55 & \textbf{0.49} & \textbf{1.00} & 0.68 \\
Ours                              & 0.010M & \textbf{0.94} & \textbf{0.49} & 0.95 & \textbf{0.79} \\
\bottomrule
\end{tabular}
\vspace{-1em}
\end{wraptable}
\textbf{Adaptive image generation.}
Our scalar predictor $u_\theta(\mathbf{x})$ not only suggests an appropriate denoising step length but also serves as a practical stopping signal.
When $u_\theta$ increases consistently over successive updates, further denoising is no longer making progress and the process should terminate.
Based on this observation, we propose an adaptive gradient descent sampler that stops once $u_\theta$ has increased for 10 consecutive steps~(\Cref{tab:adaptive}).
It reduces the average number of steps by $29.7\%$, improving sFID and IS over the original gradient descent sampler, at the cost of a small FID increase; the resulting FID remains better than all other baselines.

\noindent\textbf{Out-of-distribution (OOD) detection.}
Our distance predictor $u_\theta(\mathbf{x})$ achieves the best OOD detection performance in~\Cref{tab:ood} despite being three orders of magnitude smaller than prior energy-based models.

\vspace{-0.3em}
\section{Limitations and Future Work}
\label{sec:limitations_future_work}

It is important to understand how approximation errors in the learned fields affect the induced measure and pushforward dynamics, and how finite-sample training generalizes to the continuous data manifold assumed in our analysis.
More broadly, our method emphasizes geometric support recovery rather than density allocation, raising the question of which principle is more important.
Practically, better initialization strategies and more efficient procedures for exploring the induced manifold distribution may further improve generation.
See~\Cref{app:limitations_future_work} for an extended discussion.

\section{Conclusion}
\label{sec:conclusion}

In this paper, we introduced \dm (DM), a distance-field view of time-unconditional generation with two principled samplers.
Our model offers a new perspective explaining the benefit of distance-field modeling: closer-target supervision helps disambiguate denoising signals.
This observation leads to a practical step toward large-scale time-unconditional image generation.
To the best of our knowledge, it is the first method to develop distance-field-inspired losses that explicitly emphasize closer targets.
Together with deterministic distance-field samplers, the resulting framework achieves strong performance and empirically faster convergence in sample quality.
We hope this work promotes distance-field modeling as a principled lens for designing objectives and inference in high-dimensional generative modeling.

\renewcommand{\acksection}{\section*{Acknowledgments}}
\begin{ack}
We thank Zilu Li, Liwen Wu, Tongtong Liang, Yingyu Lin, Peiyuan Zhang, and Yiping Wang for helpful discussions, assistance, and suggestions.
Zimo Wang also thanks Professors Ravi Ramamoorthi, Alex Cloninger, and Nikolay Atanasov for their thoughtful feedback and encouragement during his Ph.D. candidacy examination.
\end{ack}

\bibliography{main}
\bibliographystyle{unsrtnat}
\newpage
\appendix

\printappendixtoc

\startappendixtoc
\section{Proofs and Derivations}
\subsection{Invariance of the One-Step Projection Under Radial Reparameterizations}
\label{app:projection_invariance}

This subsection proves the claim in \Cref{eq:analytical}: the one-step map
\(
\mathbf{x} \mapsto \mathbf{x} - d(\mathbf{x}) \nabla d(\mathbf{x})
\)
does not uniquely identify the underlying (unsigned) distance field.

\paragraph{Setup.}
Let $\mathcal{S}\subset \mathbb{R}^D$ be a nonempty closed set (e.g., a target manifold).
Define the (unsigned) distance-to-set function
\begin{equation}
d(\mathbf{x}) \coloneqq \min_{\mathbf{s}\in \mathcal{S}} \|\mathbf{x}-\mathbf{s}\|_2,
\qquad
g(\mathbf{x}) \coloneqq d(\mathbf{x})^2 = \min_{\mathbf{s}\in \mathcal{S}} \|\mathbf{x}-\mathbf{s}\|_2^2.
\label{eq:app_def_d_g}
\end{equation}
Let the (set-valued) closest-point projection be
\begin{equation}
\Pi_{\mathcal{S}}(\mathbf{x}) \coloneqq \arg\min_{\mathbf{s}\in \mathcal{S}} \|\mathbf{x}-\mathbf{s}\|_2.
\label{eq:app_projection_set}
\end{equation}
We say the closest-point projection is \emph{unique} at $\mathbf{x}$ if $\Pi_{\mathcal{S}}(\mathbf{x})=\{\pi_{\mathcal{S}}(\mathbf{x})\}$ for some single point $\pi_{\mathcal{S}}(\mathbf{x})\in \mathcal{S}$.
Uniqueness holds, for example, away from the medial axis of $\mathcal{S}$.
We also use this notation $\pi_{\mathcal{S}}(\mathbf{x})$ to represent the unique closest-point projection from where uniqueness holds.

\paragraph{A brief introduction to Danskin's theorem~\cite{bertsekas2016}.}
Consider $G(\mathbf{x})=\min_{\mathbf{s}\in \mathcal{S}} \phi(\mathbf{x},\mathbf{s})$ where $\mathcal{S}$ is compact and $\phi$ is continuous in $(\mathbf{x},\mathbf{s})$, differentiable in $\mathbf{x}$, and $\nabla_{\mathbf{x}}\phi$ is continuous.
If the minimizer is unique at $\mathbf{x}$, i.e., $\arg\min_{\mathbf{s}\in\mathcal{S}}\phi(\mathbf{x},\mathbf{s})=\{\mathbf{s}(\mathbf{x})\}$, then $G$ is differentiable at $\mathbf{x}$ and
\begin{equation}
\nabla G(\mathbf{x}) = \nabla_{\mathbf{x}} \phi(\mathbf{x}, \mathbf{s}(\mathbf{x})).
\label{eq:app_danskin}
\end{equation}
(When $\mathcal{S}$ is merely closed, the same conclusion holds locally under standard regularity assumptions ensuring existence/uniqueness of the minimizer and local Lipschitzness; we use the theorem only at points where the minimizer is unique and the gradient exists.)

\paragraph{Lemma 1 (Gradient of squared distance under unique projection).}
Assume $\Pi_{\mathcal{S}}(\mathbf{x})=\{\pi_{\mathcal{S}}(\mathbf{x})\}$. Then $g$ in \Cref{eq:app_def_d_g} is differentiable at $\mathbf{x}$ and
\begin{equation}
\nabla g(\mathbf{x}) = 2(\mathbf{x}-\pi_{\mathcal{S}}(\mathbf{x})).
\label{eq:app_grad_g}
\end{equation}

\begin{proof}
Let $\phi(\mathbf{x},\pi_{\mathcal{S}}(\mathbf{x}))=\|\mathbf{x}-\pi_{\mathcal{S}}(\mathbf{x})\|_2^2$. Then $\nabla_{\mathbf{x}}\phi(\mathbf{x},\pi_{\mathcal{S}}(\mathbf{x}))=2(\mathbf{x}-\pi_{\mathcal{S}}(\mathbf{x}))$.
By Danskin's theorem (\Cref{eq:app_danskin}) and uniqueness of the minimizer, we obtain
$\nabla g(\mathbf{x})=\nabla_{\mathbf{x}}\phi(\mathbf{x},\pi_{\mathcal{S}}(\mathbf{x}))=2(\mathbf{x}-\pi_{\mathcal{S}}(\mathbf{x}))$.
That is, the gradient does not depend on the derivative of the minimizer $\pi_{\mathcal{S}}(\mathbf{x})$ with respect to $\mathbf{x}$.
\end{proof}

\paragraph{Proposition 1 (Family inducing the same one-step projection).}
Fix any constant $C\ge 0$ and any sign $\sigma_\text{sign}\in\{+1,-1\}$, and define
\begin{equation}
\hat{d}(\mathbf{x}) \coloneqq \sigma_\text{sign}\sqrt{g(\mathbf{x})+C}
= \sigma_\text{sign}\sqrt{\|\mathbf{x}-\pi_{\mathcal{S}}(\mathbf{x})\|_2^2 + C}.
\label{eq:app_dhat}
\end{equation}
Assume $\Pi_{\mathcal{S}}(\mathbf{x})$ is unique and $\nabla g(\mathbf{x})$ exists.
At any point where $\hat{d}(\mathbf{x})\neq 0$ and $\nabla \hat{d}(\mathbf{x})$ exists, the one-step update recovers the same closest point:
\begin{equation}
\mathbf{x} - \hat{d}(\mathbf{x})\,\nabla \hat{d}(\mathbf{x}) = \pi_{\mathcal{S}}(\mathbf{x}).
\label{eq:app_same_projection}
\end{equation}
Consequently, the family in \Cref{eq:app_dhat} induces the same one-step projection behavior as in \Cref{eq:one_step} wherever the gradients exist.

\begin{proof}
By the chain rule applied to \Cref{eq:app_dhat}, for $\hat{d}(\mathbf{x})\neq 0$ we have
\begin{equation}
\nabla \hat{d}(\mathbf{x})
=
\sigma_\text{sign}\cdot \frac{1}{2\sqrt{g(\mathbf{x})+C}}\,\nabla g(\mathbf{x}).
\label{eq:app_grad_dhat_1}
\end{equation}
Using $\hat{d}(\mathbf{x})=\sigma_\text{sign}\sqrt{g(\mathbf{x})+C}$, we rewrite the prefactor as
\(
\sigma_\text{sign} / \sqrt{g(\mathbf{x})+C} = 1/\hat{d}(\mathbf{x})
\),
hence
\begin{equation}
\nabla \hat{d}(\mathbf{x})
=
\frac{1}{2\hat{d}(\mathbf{x})}\,\nabla g(\mathbf{x}).
\label{eq:app_grad_dhat_2}
\end{equation}
Applying \Cref{eq:app_grad_g} from Lemma~1 yields
\begin{equation}
\nabla \hat{d}(\mathbf{x})
=
\frac{1}{2\hat{d}(\mathbf{x})}\cdot 2(\mathbf{x}-\pi_{\mathcal{S}}(\mathbf{x}))
=
\frac{\mathbf{x}-\pi_{\mathcal{S}}(\mathbf{x})}{\hat{d}(\mathbf{x})}.
\label{eq:app_grad_dhat_3}
\end{equation}
Multiplying both sides by $\hat{d}(\mathbf{x})$ gives
\(
\hat{d}(\mathbf{x})\,\nabla \hat{d}(\mathbf{x}) = \mathbf{x}-\pi_{\mathcal{S}}(\mathbf{x})
\),
which rearranges to \Cref{eq:app_same_projection}.
\end{proof}

\paragraph{Remark.}
The proposition shows that constraining only the one-step map
\(
\mathbf{x}\mapsto \mathbf{x}-u(\mathbf{x})\nabla u(\mathbf{x})
\)
(e.g., via the One-Step Loss) cannot uniquely determine the scalar field itself:
many distinct scalar fields $\hat{d}$ of the form \Cref{eq:app_dhat} yield the same closest-point projection wherever gradients exist.
This motivates adding additional constraints (e.g., Eikonal-type regularization) if one wishes to identify an actual distance field.
By solving~\Cref{eq:app_same_projection}, we prove that our two losses uniquely determine a distance-like function~\Cref{eq:app_dhat}, given a unique closest-point projection $\pi_{\mathcal{S}}(\mathbf{x})$.

\subsection{Distance-induced Measures Converge to Normalized Hausdorff Measure}
\label{app:allocation}

Throughout this paper, we distinguish geometric support recovery from the density allocation and recovery.
We make explicit that our learned field can faithfully transport samples toward the data support, but the induced distribution on that support is a separate question.
For example, a normalized Hausdorff measure can be approached using MCMC with the learned distance, whereas the deterministic samplers induce a distribution shaped by the initialization (see~\Cref{app:pushforward}).

In this subsection, we show that the Gibbs measure induced by the distance field converges to the normalized Hausdorff measure on the manifold (or other compact subset) in the low-temperature limit $\sigma\to 0$.
We first show that, when the support $\mathcal{S}$ is modeled as a compact smooth manifold, the probability measure induced by the distance field converges to the normalized Hausdorff measure on $\mathcal{S}$.
We then provide a proof for the distance field of a compact full-dimensional subset.

Since the limiting measure here is uniform on the manifold, it is in line with the data distribution, when we assume data is obtained uniformly on the manifold.
Some recent papers~\cite{li2025scores,shen2026manifold} argue that this geometric target can already be sufficient for strong generation performance.
We also provide our analysis in~\Cref{app:allocation} and experiments in~\Cref{app:ula}.

\begin{theorem}[Distance-induced measure on a manifold]
\label{thm:distance_hausdorff}
Let $\mathcal{S} \subset \mathbb{R}^D$ be a nonempty compact $m$-dimensional $C^2$ embedded submanifold without boundary, and let
\begin{equation}
d(\mathbf{x},\mathcal{S}) := \inf_{\mathbf{s}\in\mathcal{S}} \|\mathbf{x}-\mathbf{s}\|_2
\end{equation}
be the Euclidean distance to $\mathcal{S}$. For each $\sigma>0$, define the probability measure
\begin{equation}
d\mu_\sigma(\mathbf{x})
:=
\frac{1}{Z_\sigma}
\exp\!\left(-\frac{d(\mathbf{x},\mathcal{S})}{\sigma}\right)\,d\mathbf{x},
\end{equation}
where
\begin{equation}
Z_\sigma
:=
\int_{\mathbb{R}^D}
\exp\!\left(-\frac{d(\mathbf{x},\mathcal{S})}{\sigma}\right)\,d\mathbf{x}.
\end{equation}
Then $Z_\sigma<\infty$ for every $\sigma>0$, and as $\sigma\to 0$,
\begin{equation}
\mu_\sigma \rightharpoonup \bar{\mu},
\end{equation}
where $\bar{\mu}$ is the normalized $m$-dimensional Hausdorff measure on $\mathcal{S}$, namely
\begin{equation}
\bar{\mu}(A)
=
\frac{\mathcal{H}^m(A\cap \mathcal{S})}{\mathcal{H}^m(\mathcal{S})}
\end{equation}
for every Borel set $A\subset \mathbb{R}^D$. Equivalently, for every bounded continuous test function $\varphi\in C_b(\mathbb{R}^D)$,
\begin{equation}
\lim_{\sigma\to 0}
\int_{\mathbb{R}^D}\varphi(\mathbf{x})\,d\mu_\sigma(\mathbf{x})
=
\frac{1}{\mathcal{H}^m(\mathcal{S})}
\int_{\mathcal{S}} \varphi(\mathbf{s})\,d\mathcal{H}^m(\mathbf{s}).
\end{equation}
\end{theorem}

\begin{proof}
Since $\mathcal{S}$ is compact, there exists $R>0$ such that $\mathcal{S}\subset B(0,R)$. Hence
\begin{equation}
d(\mathbf{x},\mathcal{S})\ge \|\mathbf{x}\|_2-R,
\end{equation}
so
\begin{equation}
\exp\!\left(-\frac{d(\mathbf{x},\mathcal{S})}{\sigma}\right)
\le
e^{R/\sigma}e^{-\|\mathbf{x}\|_2/\sigma}.
\end{equation}
Because $e^{-\|\mathbf{x}\|_2/\sigma}\in L^1(\mathbb{R}^D)$, we obtain $0<Z_\sigma<\infty$.

We now identify the low-temperature limit. Since $\mathcal{S}$ is a compact $C^2$ embedded submanifold, it has positive reach. Therefore there exists $\rho>0$ such that every point $\mathbf{x}$ in the tubular neighborhood
\begin{equation}
U_\rho:=\{\mathbf{x}\in\mathbb{R}^D: d(\mathbf{x},\mathcal{S})<\rho\}
\end{equation}
admits a unique nearest-point projection $\Pi(\mathbf{x})\in \mathcal{S}$, and each $\mathbf{x}\in U_\rho$ can be written uniquely as
\begin{equation}
\mathbf{x}=\mathbf{s}+\mathbf{n},
\qquad
\mathbf{s}\in \mathcal{S},\quad
\mathbf{n}\in N_{\mathbf{s}}\mathcal{S},\quad
\|\mathbf{n}\|_2<\rho,
\end{equation}
where $N_{\mathbf{s}}\mathcal{S}$ is the normal space of $\mathcal{S}$ at $\mathbf{s}$. In these tubular coordinates, $d(\mathbf{x},\mathcal{S})=\|\mathbf{n}\|_2$, and the ambient volume element satisfies
\begin{equation}
d\mathbf{x}
=
J(\mathbf{s},\mathbf{n})\,d\mathbf{n}\,d\mathcal{H}^m(\mathbf{s}),
\end{equation}
where $J(\mathbf{s},\mathbf{n})$ is continuous and
\begin{equation}
J(\mathbf{s},\mathbf{0})=1
\qquad
\text{for all } \mathbf{s}\in \mathcal{S}.
\end{equation}
Moreover, because the tubular map $\Phi(\mathbf{s},\mathbf{n})=\mathbf{s}+\mathbf{n}$ is affine in $\mathbf{n}$, the Jacobian $J(\mathbf{s},\mathbf{n})$ is a polynomial in $\mathbf{n}$ of degree at most $m$, with coefficients continuous in $\mathbf{s}$; this is the standard tube-formula structure for a $C^2$ submanifold. In particular, $\nabla_{\mathbf{n}}J$ is jointly continuous on $\mathcal{S}\times \overline{B_\rho^\perp}$ after possibly shrinking $\rho$. Combined with $J(\mathbf{s},\mathbf{0})=1$ and compactness of $\mathcal{S}$, this yields
\begin{equation}
|J(\mathbf{s},\mathbf{n})-1|\le C\|\mathbf{n}\|_2
\qquad
\text{for all } \mathbf{s}\in \mathcal{S},\ \|\mathbf{n}\|_2<\rho
\end{equation}
for some constant $C>0$.

Let $\varphi\in C_b(\mathbb{R}^D)$. Split
\begin{equation}
\int_{\mathbb{R}^D}\varphi(\mathbf{x})e^{-d(\mathbf{x},\mathcal{S})/\sigma}\,d\mathbf{x}
=
I_\sigma^{\mathrm{near}}+I_\sigma^{\mathrm{far}},
\end{equation}
where
\begin{equation}
I_\sigma^{\mathrm{near}}
:=
\int_{U_\rho}\varphi(\mathbf{x})e^{-d(\mathbf{x},\mathcal{S})/\sigma}\,d\mathbf{x},
\end{equation}
and
\begin{equation}
I_\sigma^{\mathrm{far}}
:=
\int_{\mathbb{R}^D\setminus U_\rho}\varphi(\mathbf{x})e^{-d(\mathbf{x},\mathcal{S})/\sigma}\,d\mathbf{x}.
\end{equation}

We first show that the far-field contribution is negligible on the leading $\sigma^q$ scale, where $q:=D-m$. On $\mathbb{R}^D\setminus U_\rho$, one has $d(\mathbf{x},\mathcal{S})\ge \rho$, hence
\begin{equation}
e^{-d(\mathbf{x},\mathcal{S})/\sigma}
=
e^{-d(\mathbf{x},\mathcal{S})/(2\sigma)}e^{-d(\mathbf{x},\mathcal{S})/(2\sigma)}
\le
e^{-\rho/(2\sigma)}e^{-d(\mathbf{x},\mathcal{S})/(2\sigma)}.
\end{equation}
Therefore
\begin{equation}
|I_\sigma^{\mathrm{far}}|
\le
\|\varphi\|_\infty e^{-\rho/(2\sigma)}
\int_{\mathbb{R}^D} e^{-d(\mathbf{x},\mathcal{S})/(2\sigma)}\,d\mathbf{x}.
\end{equation}
It remains to prove that
\begin{equation}
\int_{\mathbb{R}^D} e^{-d(\mathbf{x},\mathcal{S})/(2\sigma)}\,d\mathbf{x}
=
O(\sigma^q).
\end{equation}
We do this directly, without using tubular coordinates outside $U_\rho$.

First, inside $U_\rho$, tubular coordinates are valid, so
\begin{equation}
\int_{U_\rho} e^{-d(\mathbf{x},\mathcal{S})/(2\sigma)}\,d\mathbf{x}
=
\int_{\mathcal{S}}
\int_{\|\mathbf{n}\|_2<\rho}
e^{-\|\mathbf{n}\|_2/(2\sigma)}
J(\mathbf{s},\mathbf{n})
\,d\mathbf{n}\,d\mathcal{H}^m(\mathbf{s}).
\end{equation}
Using $|J(\mathbf{s},\mathbf{n})|\le 1+C\rho$ on $U_\rho$, we obtain
\begin{equation}
\int_{U_\rho} e^{-d(\mathbf{x},\mathcal{S})/(2\sigma)}\,d\mathbf{x}
\le
(1+C\rho)\,\mathcal{H}^m(\mathcal{S})
\int_{\mathbb{R}^q} e^{-\|\mathbf{n}\|_2/(2\sigma)}\,d\mathbf{n}
=
O(\sigma^q).
\end{equation}

Next, on $\mathbb{R}^D\setminus U_\rho$, split into the bounded and unbounded parts:
\begin{equation}
\mathbb{R}^D\setminus U_\rho
=
\bigl(B(0,2R)\setminus U_\rho\bigr)\cup \bigl(\mathbb{R}^D\setminus B(0,2R)\bigr).
\end{equation}
On $B(0,2R)\setminus U_\rho$, we still have $d(\mathbf{x},\mathcal{S})\ge \rho$, so
\begin{equation}
\int_{B(0,2R)\setminus U_\rho} e^{-d(\mathbf{x},\mathcal{S})/(2\sigma)}\,d\mathbf{x}
\le
|B(0,2R)|\,e^{-\rho/(2\sigma)}
=
O\!\left(e^{-\rho/(2\sigma)}\right).
\end{equation}
On $\mathbb{R}^D\setminus B(0,2R)$, since $\mathcal{S}\subset B(0,R)$, we have
\begin{equation}
d(\mathbf{x},\mathcal{S})\ge \|\mathbf{x}\|_2-R\ge \frac{\|\mathbf{x}\|_2}{2}.
\end{equation}
Hence
\begin{equation}
\int_{\mathbb{R}^D\setminus B(0,2R)} e^{-d(\mathbf{x},\mathcal{S})/(2\sigma)}\,d\mathbf{x}
\le
\int_{\mathbb{R}^D\setminus B(0,2R)} e^{-\|\mathbf{x}\|_2/(4\sigma)}\,d\mathbf{x}
=
O(\sigma^D)
=
O(\sigma^q).
\end{equation}
Combining the three pieces yields
\begin{equation}
\int_{\mathbb{R}^D} e^{-d(\mathbf{x},\mathcal{S})/(2\sigma)}\,d\mathbf{x}
=
O(\sigma^q).
\end{equation}
Substituting this into the previous bound gives
\begin{equation}
|I_\sigma^{\mathrm{far}}|
\le
\|\varphi\|_\infty e^{-\rho/(2\sigma)} O(\sigma^q)
=
O\!\left(e^{-\rho/(2\sigma)}\sigma^q\right)
=
o(\sigma^q).
\end{equation}
Thus the far-field contribution is negligible compared with the leading $\sigma^q$ scale.

We now evaluate the near-field contribution. Using tubular coordinates,
\begin{equation}
I_\sigma^{\mathrm{near}}
=
\int_{\mathcal{S}}
\int_{\|\mathbf{n}\|_2<\rho}
\varphi(\mathbf{s}+\mathbf{n})
e^{-\|\mathbf{n}\|_2/\sigma}
J(\mathbf{s},\mathbf{n})
\,d\mathbf{n}\,d\mathcal{H}^m(\mathbf{s}).
\end{equation}
Rescale the normal variable by $\mathbf{n}=\sigma \mathbf{u}$. Then
\begin{equation}
I_\sigma^{\mathrm{near}}
=
\sigma^q
\int_{\mathcal{S}}
\int_{\|\mathbf{u}\|_2<\rho/\sigma}
\varphi(\mathbf{s}+\sigma\mathbf{u})
e^{-\|\mathbf{u}\|_2}
J(\mathbf{s},\sigma\mathbf{u})
\,d\mathbf{u}\,d\mathcal{H}^m(\mathbf{s}).
\end{equation}
Since $\varphi$ is bounded and continuous, and $J(\mathbf{s},\sigma\mathbf{u})\to 1$, dominated convergence yields
\begin{equation}
I_\sigma^{\mathrm{near}}
=
\sigma^q
\left(
\int_{\mathbb{R}^q}e^{-\|\mathbf{u}\|_2}\,d\mathbf{u}
\right)
\left(
\int_{\mathcal{S}} \varphi(\mathbf{s})\,d\mathcal{H}^m(\mathbf{s})
\right)
+o(\sigma^q).
\end{equation}
Denoting
\begin{equation}
c_q:=\int_{\mathbb{R}^q}e^{-\|\mathbf{u}\|_2}\,d\mathbf{u},
\end{equation}
we therefore have
\begin{equation}
\int_{\mathbb{R}^D}\varphi(\mathbf{x})e^{-d(\mathbf{x},\mathcal{S})/\sigma}\,d\mathbf{x}
=
c_q\,\sigma^q
\int_{\mathcal{S}} \varphi(\mathbf{s})\,d\mathcal{H}^m(\mathbf{s})
+o(\sigma^q).
\end{equation}
Applying the same expansion with $\varphi\equiv 1$ gives
\begin{equation}
Z_\sigma
=
c_q\,\sigma^q\,\mathcal{H}^m(\mathcal{S})+o(\sigma^q).
\end{equation}
Hence
\begin{equation}
\int_{\mathbb{R}^D}\varphi(\mathbf{x})\,d\mu_\sigma(\mathbf{x})
=
\frac{
c_q\,\sigma^q\int_{\mathcal{S}} \varphi(\mathbf{s})\,d\mathcal{H}^m(\mathbf{s})
+o(\sigma^q)
}{
c_q\,\sigma^q\,\mathcal{H}^m(\mathcal{S})+o(\sigma^q)
}
\to
\frac{1}{\mathcal{H}^m(\mathcal{S})}
\int_{\mathcal{S}} \varphi(\mathbf{s})\,d\mathcal{H}^m(\mathbf{s}).
\end{equation}
This proves the desired weak convergence.
\end{proof}

Theorem~\ref{thm:distance_hausdorff} requires $m<D$, since a nonempty
compact $C^2$ embedded submanifold of $\mathbb{R}^D$ without boundary
cannot be full-dimensional. The complementary case $m=D$ is covered
by the following short theorem, whose proof reduces to a direct
application of dominated convergence.

\begin{theorem}[Distance-induced measure in the full-dimensional case]
\label{thm:full_dim}
Let $\mathcal{S}\subset\mathbb{R}^D$ be a compact set with
$|\mathcal{S}|>0$, where $|\cdot|$ denotes Lebesgue measure on
$\mathbb{R}^D$. Define $\mu_\sigma$ as in
Theorem~\ref{thm:distance_hausdorff}. Then $Z_\sigma<\infty$ for every
$\sigma>0$, and as $\sigma\to 0$,
\begin{equation}
\mu_\sigma \rightharpoonup \bar{\mu},
\qquad
\bar{\mu}(A)=\frac{|A\cap\mathcal{S}|}{|\mathcal{S}|}
=\frac{\mathcal{H}^D(A\cap\mathcal{S})}{\mathcal{H}^D(\mathcal{S})}
\end{equation}
for every Borel set $A\subset\mathbb{R}^D$.
\end{theorem}

\begin{proof}
Finiteness of $Z_\sigma$ follows exactly as in
Theorem~\ref{thm:distance_hausdorff} via the reverse triangle
inequality $d(\mathbf{x},\mathcal{S})\ge\|\mathbf{x}\|_2-R$, where $R$
is any radius with $\mathcal{S}\subset B(0,R)$.

Because $\mathcal{S}$ is closed, $d(\mathbf{x},\mathcal{S})=0$ if and
only if $\mathbf{x}\in\mathcal{S}$; in particular
$d(\mathbf{x},\mathcal{S})>0$ on $\mathbb{R}^D\setminus\mathcal{S}$.
Fix $\varphi\in C_b(\mathbb{R}^D)$ and split
\begin{equation}
\int_{\mathbb{R}^D}\varphi(\mathbf{x})
e^{-d(\mathbf{x},\mathcal{S})/\sigma}\,d\mathbf{x}
=
\int_{\mathcal{S}}\varphi(\mathbf{x})\,d\mathbf{x}
+
\int_{\mathbb{R}^D\setminus\mathcal{S}}
\varphi(\mathbf{x})e^{-d(\mathbf{x},\mathcal{S})/\sigma}\,d\mathbf{x},
\end{equation}
using $d=0$ on $\mathcal{S}$.

Fix any $\sigma_0>0$. For every $\sigma\in(0,\sigma_0]$ and every
$\mathbf{x}\in\mathbb{R}^D\setminus\mathcal{S}$,
\begin{equation}
|\varphi(\mathbf{x})|\,e^{-d(\mathbf{x},\mathcal{S})/\sigma}
\le
\|\varphi\|_\infty\,e^{-d(\mathbf{x},\mathcal{S})/\sigma_0},
\end{equation}
and the right-hand side is integrable on $\mathbb{R}^D$ since
$Z_{\sigma_0}<\infty$. Moreover, on
$\mathbb{R}^D\setminus\mathcal{S}$ we have
$d(\mathbf{x},\mathcal{S})>0$, so
$e^{-d(\mathbf{x},\mathcal{S})/\sigma}\to 0$ pointwise as
$\sigma\to 0$. Dominated convergence yields
\begin{equation}
\int_{\mathbb{R}^D\setminus\mathcal{S}}
\varphi(\mathbf{x})e^{-d(\mathbf{x},\mathcal{S})/\sigma}\,d\mathbf{x}
\longrightarrow 0.
\end{equation}
Therefore
\begin{equation}
\int_{\mathbb{R}^D}\varphi(\mathbf{x})
e^{-d(\mathbf{x},\mathcal{S})/\sigma}\,d\mathbf{x}
\longrightarrow
\int_{\mathcal{S}}\varphi(\mathbf{x})\,d\mathbf{x},
\qquad
Z_\sigma\longrightarrow |\mathcal{S}|.
\end{equation}

Since $|\mathcal{S}|>0$, taking the ratio gives
\begin{equation}
\int_{\mathbb{R}^D}\varphi(\mathbf{x})\,d\mu_\sigma(\mathbf{x})
\longrightarrow
\frac{1}{|\mathcal{S}|}
\int_{\mathcal{S}}\varphi(\mathbf{x})\,d\mathbf{x},
\end{equation}
which is exactly the weak convergence $\mu_\sigma\rightharpoonup
\bar{\mu}$ claimed. The final equality
$|\cdot|=\mathcal{H}^D(\cdot)$ on $\mathbb{R}^D$ is standard.
\end{proof}

The preceding theorems characterize the distance-induced target measure and give a normalized interpretation of the geometric target.
To fit this measure on the manifold, we adopt MCMC methods in our 2D experiments.
We also attach comparisons in~\Cref{app:ula} to discuss the function of MCMC for the image generation.
In practice, with a limited step budget, we observe that MCMC provides significant improvements when the noise initialization could only be transported to a small subset of the manifold.
Without MCMC, our sampling processes are better understood not as exact sampling from the empirical data density, but as moving toward the high-likelihood regions of the distance-induced measure.
Under the Gibbs transformation, these regions receive exponentially larger weight when they are closer to the manifold.
This in turn requires a well-spread initial distribution nearby the manifold.
Starting from standard Gaussian noise and even uniform noise, we observe that our two generation processes, sphere tracing and gradient descent, provide samples that cover the image manifold well.
We also provide further analysis in~\Cref{app:pushforward} when MCMC is not applied.

\subsection{Pushforward Analysis}
\label{app:pushforward}

In this subsection, starting from an initial Gaussian distribution, we characterize the measure of our samples on the data manifold after a deterministic projection from an optimal distance field.

We begin our analysis with assumptions required in the following sections.

\begin{assumption}[Geometry of the target image manifold $\mathcal{S}$]\label{ass:geom}
$\mathcal{S}\subseteq\mathbb{R}^D$ is a closed (as a subset of $\mathbb{R}^D$, no boundary), $C^2$, $m$-dimensional submanifold with $1\le m<D$.
\end{assumption}
With this assumption, we can further define the manifold without loss of generality: near each point $y_0\in\mathcal{S}$, $\mathcal{S}$ looks like the graph of a $C^2$ function: there exist orthonormal coordinates of $\mathbb{R}^D$ centered at $y_0$ in which
\begin{equation}
    \label{app_sec:manifold_def}
    \mathcal{S}\cap U \;=\; \bigl\{\,(u,\varphi(u)):u\in W\subseteq\mathbb{R}^m\,\bigr\}
\end{equation}
for some open $W$, $U$, and a $C^2$ map $\varphi:W\to\mathbb{R}^{D-m}$.
Under this condition, we can define the 
$m$-dimensional Hausdorff measure on $\mathcal{S}$, which allows us to specify a well-defined probability density on $\mathcal{S}$ after the pushforward.
Specifically, pick a local parametrization $\psi:W\subseteq\mathbb{R}^m\to\mathcal{S}$, $\psi(u)=(u,\varphi(u))$. Then for any Borel $B\subseteq\mathcal{S}\cap\psi(W)$ the area formula gives
\begin{equation}\label{eq:areaformula}
\mathcal{H}^m(B)
\;=\;
\int_{\psi^{-1}(B)}\sqrt{\det\!\bigl(J\psi(u)^\top J\psi(u)\bigr)}\;du,
\end{equation}
where
\begin{equation*}
    J \psi(u) = \left[
        \begin{matrix}
            \frac{\partial \psi_1}{\partial u_1} &  \ldots & \frac{\partial \psi_1}{\partial u_m}\\
            \ldots & \ldots & \ldots\\
            \frac{\partial \psi_{D}}{\partial u_1} &  \ldots & \frac{\partial \psi_{D}}{\partial u_m} 
        \end{matrix}
    \right]\in \mathbb{R}^{D\times m}.
\end{equation*}

\begin{assumption}[Positive reach]\label{ass:reach}
There exists $\tau>0$, called the \emph{reach} of $\mathcal{S}$, such that the open ``tube''
\[
U_\tau\;:=\;\{x\in\mathbb{R}^D:d(x,\mathcal{S})<\tau\}
\]
has the following property: for every $x\in U_\tau$, there is a \emph{unique} closest point in $\mathcal{S}$, denoted $\pi_\mathcal{S}(x)$. Equivalently, the map
\[
\Phi:\{(y,v):y\in\mathcal{S},\,v\in N_y\mathcal{S},\,|v|<\tau\}\;\longrightarrow\;U_\tau,\qquad\Phi(y,v)=y+v,
\]
is a $C^1$-diffeomorphism.
\end{assumption}

Positive reach is the curved analog of the statement that \(V^\perp\) exists and that orthogonal projection is well-defined. For a linear subspace, the reach satisfies \(\tau=\infty\). For a smooth compact manifold, \(\tau\) is finite but positive; it is bounded from below by quantities determined by the minimal radius of curvature and the minimal separation between non-adjacent parts of \(\mathcal S\). 

\newcommand{\Med}{\operatorname{Med}}
This assumption implies the unique-projection condition required for Eq.~\eqref{eq:app_grad_g} due to Federer’s result~\cite{federer1959curvature}, which allows us to define the measure zero medial axis, i.e.,
\[
\Med(\mathcal{S})\;:=\;\{x\in\mathbb{R}^D:x\text{ has more than one nearest point in }\mathcal{S}\}
\]

\begin{assumption}[Optimality]\label{ass:opt}
The learned map satisfies $x - u_\theta(x) \nabla u_{\theta}(x)=\pi_\mathcal{S}(x)$ a.e.\ on $\mathbb{R}^D$. As before, this gives
\[
s_\theta(x)\coloneqq x - u_\theta(x)\nabla u_\theta(x)\;=\;\pi_\mathcal{S}(x)\quad\text{for almost every }x\in\mathbb{R}^D,
\]
which matches the assumptions required in Section~\ref{app:projection_invariance}.
\end{assumption}

In the following, we highlight one key property of our algorithm: under a typical Gaussian-type initialization in $\mathbb{R}^D$ and an exact projection map, the generated samples induce a distribution on $\mathcal{S}$ with a density that is positive with respect to the Hausdorff measure on $\mathcal{S}$.

For the proof convenience, we introduce the definition of the fiber decomposition to match the parameterization $u_\theta(\cdot)$ and $\nabla u_\theta(\cdot)$. Specifically, for $y\in\mathcal{S}$, define
\[
F_y\;:=\;\{v\in N_y\mathcal{S}:\pi_\mathcal{S}(y+v)=y\}\;\subseteq\;N_y\mathcal{S}.
\]
By Assumption~\ref{ass:reach}, $F_y$ contains the open ball $\{v\in N_y\mathcal{S}:|v|<\tau\}$. Moreover the map
\[
\Phi:\bigl\{(y,v):y\in\mathcal{S},\,v\in F_y\bigr\}\;\longrightarrow\;\mathbb{R}^D\setminus\Med(\mathcal{S}),\qquad(y,v)\mapsto y+v,
\]
is a measurable bijection. (Surjectivity: every $x\notin\Med(\mathcal{S})$ has a unique closest point $y=\pi_\mathcal{S}(x)$, and then $v=x-y\in N_y\mathcal{S}$ with $\pi_\mathcal{S}(y+v)=y$, so $v\in F_y$. Injectivity: if $y+v=y'+v'$ with $v\in F_y$, $v'\in F_{y'}$, then $y'=\pi_\mathcal{S}(y+v)=y$ and hence $v'=v$.)
Besides, this definition is tightly connected with our previous projection definition, and we have
\begin{equation*}
    s_{\theta}(\Phi(y,v)) = \Phi(y,v) -  u_\theta(\Phi(y,v))\cdot \nabla u_\theta(\Phi(y,v)) = y, \quad \text{for}\quad y\in \mathcal{S} \quad \text{and} \quad v \in F_y.
\end{equation*}

\begin{lemma}[Jacobian of the Local Normal-Coordinate Map]
Let $\mathcal{S}\subset \mathbb{R}^D$ be a $C^2$ embedded $m$-dimensional manifold, and let $k=D-m$. For $y_0\in\mathcal{S}$, let $\psi\colon W\subseteq \mathbb{R}^m \rightarrow \mathcal{S}$ be a local parametrization with $\psi(0)=y_0$. Let $n_1(y),\dots,n_{D-m}(y)$ be a smooth orthonormal frame of $N_y\mathcal{S}$ for $y\in\psi(W)$. Define the local normal-coordinate map $\tilde{\Phi}\colon W\times\mathbb{R}^{D-m}\to\mathbb{R}^D$ by
\begin{equation*}
    \small
    \tilde{\Phi}(u, t)\coloneqq \Phi\bigl(\psi(u),\,t_1n_1(\psi(u))+\cdots+t_{D-m}n_{D-m}(\psi(u))\bigr),
\end{equation*}
where $u\in W\subseteq\mathbb{R}^m$ and $t\in\mathbb{R}^{D-m}$. Then the Jacobian determinant of $\tilde{\Phi}$ satisfies
\begin{equation}
    \bigl|\det J\widetilde{\Phi}(u,t)\bigr|
=
\sqrt{\det(J\psi(u)^\top J\psi(u))}
\cdot
\bigl|\det(I_m-A(u,t))\bigr|,
\end{equation}
where
\begin{equation*}
\
    A(u,t)\coloneqq - \begin{bmatrix}
I_m\\
J\varphi(u)
\end{bmatrix}
\left(
I_m+J\varphi(u)^TJ\varphi(u)
\right)^{-1}
\begin{bmatrix}
I_m & J\varphi(u)^T
\end{bmatrix} \cdot J_yv(y).
\end{equation*}

\end{lemma}

\begin{proof}
Due to the bijection property of the normal-coordinate map on the relevant domain, we compute the Jacobian of $\Phi(y,v)$ locally as in the usual change-of-variables formula.

By definition,
\begin{equation*}
    \begin{aligned}
        J \tilde{\Phi}(u,t) & = \left[
            \begin{matrix}
                \frac{\partial \tilde{\Phi}}{\partial u_1}(u,t) & \ldots & \frac{\partial \tilde{\Phi}}{\partial u_m}(u,t) & \frac{\partial \tilde{\Phi}}{\partial t_1}(u,t) & \ldots & \frac{\partial \tilde{\Phi}}{\partial t_{D-m}}(u,t)
            \end{matrix}
        \right].
    \end{aligned}
\end{equation*}
We first compute the columns corresponding to the $u$-variables. An explicit differentiation gives
\begin{equation}
    \label{eq:col_tildephi_grad}
    \begin{aligned}
        \frac{\partial\tilde{\Phi}}{\partial u_i}(u,t) & = \partial_{u_i}\psi(u)+\sum_j t_j \partial_{u_i}n_j(\psi(u))\\
        & = \partial_{u_i}\psi(u)+ \underbrace{\sum_{j} t_j\cdot J_y n_j(\psi(u))\cdot \partial_{u_i}\psi(u)}_{\text{Term\ 1}}.
    \end{aligned}
\end{equation}
where $y=\psi(u)$.
We now decompose Term 1 in Eq.~\eqref{eq:col_tildephi_grad}. Specifically, set
\begin{equation*}
    J_yv(y)= \sum_{j=1}^{D-m}t_jJ_yn_j(y)\quad \text{where}\quad v(y)=\sum_{j=1}^{D-m} t_j n_j(y),
\end{equation*}
so that $\text{Term 1}=J_yv(y)\cdot \partial_{u_i}\psi(u)$. Decomposing this vector into its tangent and normal components at $y$ gives
\[
J_yv(y)\cdot \partial_{u_i}\psi(u)
=
\underbrace{\left[
J_yv(y)\cdot \partial_{u_i}\psi(u)
\right]^\top}_{\in T_y\mathcal{S}}
+
\underbrace{\left[
J_yv(y)\cdot \partial_{u_i}\psi(u)
\right]^\perp}_{\in N_y\mathcal{S}}.
\]
Since $y=\psi(u)$ and, in local graph coordinates,
\begin{equation*}
    T_y\mathcal{S} = T_{\psi(u)}\mathcal{S} = T_{(u,\varphi(u))}\mathcal{S} = \left\{(h, J \varphi(u)\cdot h): h\in\mathbb{R}^{m}\right\},
\end{equation*}
we express the tangent component using the orthogonal projection onto $T_y\mathcal{S}$. This yields
\begin{equation*}
    \left[
J_yv(y)\cdot \partial_{u_i}\psi(u)
\right]^\top = \begin{bmatrix}
I_m\\
J\varphi(u)
\end{bmatrix}
\left(
I_m+J\varphi(u)^TJ\varphi(u)
\right)^{-1}
\begin{bmatrix}
I_m & J\varphi(u)^T
\end{bmatrix} \cdot J_yv(y)\cdot \partial_{u_i}\psi(u).
\end{equation*}
With the notation $A(u,t)$, the tangent component of Term 1 is $-A(u,t)\partial_{u_i}\psi(u)$, and hence
\[
\text{Term\ 1}
=
-
A(u,t)\partial_{u_i}\psi(u)
+
\left[
J_yv(y)\cdot \partial_{u_i}\psi(u)
\right]^\perp.
\]
Substituting this decomposition into Eq.~\eqref{eq:col_tildephi_grad}, we obtain
\begin{equation*}
    \partial_{u_i}\widetilde{\Phi}(u,t)
=
\left(I-A(u,t)\right)\partial_{u_i}\psi(u)
+
\left[
J_yv(y)\cdot \partial_{u_i}\psi(u)
\right]^\perp.
\end{equation*}

We now compute the determinant. Set
\[
e_i:=\partial_{u_i}\psi(u),
\qquad i=1,\dots,m\quad \text{and}\quad q_i:=
\left[
J_yv(y)e_i
\right]^\perp
\in N_y\mathcal S.
\]
Then
\[
\partial_{u_i}\widetilde{\Phi}(u,t)
=
\left(I_{T_y\mathcal S}-A(u,t)\right)e_i+q_i.
\]
On the other hand, differentiating with respect to the normal coordinates gives
\[
\partial_{t_\alpha}\widetilde{\Phi}(u,t)=n_\alpha(y),
\qquad
\alpha=1,\dots,k.
\]
Thus, the full Jacobian matrix has columns
\[
J\widetilde{\Phi}(u,t)
=
\begin{bmatrix}
\left(I_{T_y\mathcal S}-A(u,t)\right)e_1+q_1
&
\cdots
&
\left(I_{T_y\mathcal S}-A(u,t)\right)e_m+q_m
&
n_1
&
\cdots
&
n_k
\end{bmatrix}.
\]

To compute its determinant, choose an oriented orthonormal basis of
\(\mathbb R^D\) adapted to the splitting
\[
\mathbb R^D=T_y\mathcal S\oplus N_y\mathcal S.
\]
In such an adapted orthonormal basis, the vectors \(n_1,\dots,n_k\) form
the normal block.
Therefore, the matrix of \(J\widetilde{\Phi}(u,t)\) takes the block form
\[
\left[\begin{matrix}
\left(I_{T_y\mathcal S}-A(u,t)\right)\cdot \partial_u \psi(u) & 0\\
Q(u,t) & I_k
\end{matrix}\right],
\]
where $Q(u,t)$ collects the normal components $q_i$. Since this matrix is block lower triangular, its determinant is the product of the determinants of the diagonal blocks. Consequently,
\[
\bigl|\det J\widetilde{\Phi}(u,t)\bigr|
=
\sqrt{\det\left(J\psi(u)^\top J\psi(u)\right)}
\,
\left|\det\nolimits_{T_y\mathcal S}
\left(I_{T_y\mathcal S}-A(u,t)\right)\right|.
\]
After identifying \(T_y\mathcal S\) with \(\mathbb R^m\) through the local
tangent coordinates, we write this as
\begin{equation}
    \label{eq:change_of_var_manifold}
    \bigl|\det J\widetilde{\Phi}(u,t)\bigr|
=
\sqrt{\det(J\psi(u)^\top J\psi(u))}
\cdot
\bigl|\det(I_m-A(u,t))\bigr|.
\end{equation}
This completes the proof.
\end{proof}

\begin{lemma}[Change of Variables on the Normal Bundle]
\label{lem:change_of_var_nor_bud}
Let $\mathcal{S}\subset \mathbb{R}^D$ be an $m$-dimensional embedded manifold, and let $\Med(\mathcal{S})$ denote the set of points in $\mathbb{R}^D$ having more than one nearest projection onto $\mathcal{S}$. Suppose that $\Med(\mathcal{S})$ has Lebesgue measure zero. Then, for any Borel function $g:\mathbb{R}^D\to\mathbb{R}$, the following change-of-variables formula holds:
\begin{equation}
    \begin{aligned}
        \int_{\mathbb{R}^D} g(x) dx
        =
        \int_{\mathcal{S}}\int_{F_y} g(y+ v)\cdot \left|\mathrm{det}\left(I_m - \bar{A}(y,v)\right)\right| dv d\mathcal{H}^m(y).
    \end{aligned}
\end{equation}
Here $F_y\subset N_y\mathcal{S}$ denotes the set of normal vectors $v$ such that $y$ is the unique nearest point on $\mathcal{S}$ to $y+v$. Moreover, ``$dv$'' is the standard $(D-m)$-dimensional Lebesgue measure on the linear space $N_y\mathcal{S}$, defined using any orthonormal basis, and the integral is independent of this choice.
\end{lemma}

\begin{proof}
Since the set $\Med(\mathcal{S})$ of points having more than one projection onto $\mathcal{S}$ has Lebesgue measure zero, for any Borel function $g:\mathbb{R}^D\to\mathbb{R}$ we have
\begin{equation*}
    \begin{aligned}
        \int_{\mathbb{R}^D} g(x) dx = \int_{\mathbb{R}^D\setminus \Med(\mathcal{S})} g(x) dx. 
    \end{aligned}
\end{equation*}
On $\mathbb{R}^D\setminus \Med(\mathcal{S})$, almost every $x$ has a unique nearest point on $\mathcal{S}$, denoted by $y = \pi_\mathcal{S}(x)$.
Hence the normal-coordinate map is one-to-one on this domain. Therefore, it suffices to verify the change-of-variables formula locally on $\mathrm{Im}(\tilde{\Phi})$.

Let $\psi$ be a local parametrization of $\mathcal{S}$ and let
\[
    \tilde{\Phi}(u,t)
    =
    \psi(u)+\sum_{j=1}^{D-m} t_j n_j(\psi(u)),
\]
where $\{n_j(\psi(u))\}_{j=1}^{D-m}$ is an orthonormal frame of the normal space $N_{\psi(u)}\mathcal{S}$. By the Euclidean change-of-variables formula, together with Eq.~\eqref{eq:change_of_var_manifold}, we obtain
\begin{equation*}
    \begin{aligned}
        & \int_{\mathrm{Im}(\tilde{\Phi})} g(x) dx = \int g(\tilde{\Phi}(u,t))\left|\mathrm{det} J\tilde{\Phi}(u,t)\right| du dt\\
        & = \int g\left(\psi(u)+\sum_{j=1}^{D-m} t_j\cdot n_j(\psi(u))\right)\cdot \sqrt{\det(J\psi(u)^\top J\psi(u))} \cdot \bigl|\det(I_m-A(u,t))\bigr| du dt
    \end{aligned}
\end{equation*}
which follows from Eq.~\eqref{eq:change_of_var_manifold}. 

We now rewrite the variables in geometric form. Set
\begin{equation*}
    y = \psi(u)\quad \text{and}\quad v(y)=\sum_{j=1}^{D-m} t_j n_j(y).
\end{equation*}
For any fixed $y\in\mathcal{S}$ and $v\in F_y\subset N_y\mathcal{S}$, since $\{n_j(y)\}_{j=1}^{D-m}$ is an orthonormal basis of $N_y\mathcal{S}$, there exists a unique coefficient vector $(t_1,\ldots,t_{D-m})$ such that the second equality holds. Therefore, for fixed $y$, the coordinates $t$ are simply orthonormal coordinates for the normal vector $v$.

Next, due to the formulation of $A(u,t)$, it can also be expressed in terms of $y$ and $v$. Namely,
\begin{equation*}
    \begin{aligned}
        & A(u,t)=- \begin{bmatrix}
I_m\\
J\varphi(u)
\end{bmatrix}
\left(
I_m+J\varphi(u)^TJ\varphi(u)
\right)^{-1}
\begin{bmatrix}
I_m & J\varphi(u)^T
\end{bmatrix} \cdot J_yv(y)\\
 & = - \begin{bmatrix}
I_m\\
J\left(y_\text{proj}
\right)
\end{bmatrix}\left(I_m+J^\top\left(y_\text{proj}
\right)\cdot J\left(y_\text{proj}
\right)\right)^{-1}\cdot \left[
    \begin{matrix}
        I_m & J^\top\left(y_\text{proj}
\right)
    \end{matrix}
\right]\cdot J_y v(y) \coloneqq \bar{A}(y,v)
    \end{aligned}
\end{equation*}
where $y_\text{proj}$ denotes the projection of $y$ onto its first $m$ coordinates. Thus the determinant factor may be rewritten as
\[
    \bigl|\det(I_m-A(u,t))\bigr|
    =
    \left|\mathrm{det}\left(I_m - \bar{A}(y,v)\right)\right|.
\]

Finally, for each fixed $y$, the map
\[
    t\mapsto v=\sum_{j=1}^{D-m} t_j n_j(y)
\]
is an isometry from $\mathbb{R}^{D-m}$ onto $N_y\mathcal{S}$. Hence
\begin{equation*}
    d v = dt \quad \text{and}\quad \sqrt{\det(J\psi(u)^\top J\psi(u))} d u = d \mathcal{H}^m(y)
\end{equation*}
following from Eq.~\eqref{eq:areaformula}.
Therefore, the local integral becomes
\[
    \int_{\mathrm{Im}(\tilde{\Phi})} g(x) dx
    =
    \int_{\psi(U)}\int_{F_y}
    g(y+ v)\cdot \left|\mathrm{det}\left(I_m - \bar{A}(y,v)\right)\right|
    dv d\mathcal{H}^m(y).
\]
Patching such local coordinate charts over $\mathcal{S}$ and using the fact that $\Med(\mathcal{S})$ has Lebesgue measure zero yields
\begin{equation}
    \label{eq:cov}
    \begin{aligned}
        \int_{\mathcal{S}}\int_{F_y} g(y+ v)\cdot \left|\mathrm{det}\left(I_m - \bar{A}(y,v)\right)\right| dv d\mathcal{H}^m(y)
    \end{aligned}
\end{equation}
as the desired change-of-variables expression.

Here ``$dv$'' is the standard $(D-m)$-dimensional Lebesgue measure on the linear space $N_y\mathcal{S}$, using any orthonormal basis; the integral does not depend on the choice. Also, $d\mathcal{H}^m(y)$ denotes integration against the $m$-dimensional Hausdorff measure on $\mathcal{S}$.
\end{proof}

\begin{theorem}[Pushforward density on a manifold]
Under Assumptions~\ref{ass:geom}--\ref{ass:opt}, the pushforward $p^\prime=(s_{\theta})_\#\mathcal{N}(0,\sigma_p^2\cdot I_D)$ is concentrated on $\mathcal{S}$. It is absolutely continuous with respect to the $m$-dimensional Hausdorff measure $\mathcal{H}^m\!\restriction_\mathcal{S}$, with density
\[
\frac{dp'}{d\mathcal{H}^m\!\restriction_\mathcal{S}}(y)\;=\;e^{-w(y)},\qquad y\in\mathcal{S},
\]
where
\[
w(y)\;:=\;
\begin{cases}
\dfrac{|y|^2}{2\sigma_p^2}+\dfrac{D}{2}\log(2\pi\sigma_p^2)-\log Z(y), & y\in\mathcal{S},\\[6pt]
+\infty, & y\notin\mathcal{S},
\end{cases}
\]
and
\begin{equation*}
    Z(y):=\int_{F_y}e^{(-2\langle y,v\rangle-|v|^2)/(2\sigma_p^2)}K(y,v)\,dv, \quad\text{where}\quad K(y,v):=\bigl|\det(I_m-A_y(v))\bigr|.
\end{equation*}
\end{theorem}
\begin{proof}

We now compute $p'=(s_\theta)_\#\mu$ where $\mu=\mathcal{N}(0,\sigma_p^2\cdot I_D)$. Let $\rho(x)=(2\pi\sigma_p^2)^{-D/2}e^{-\|x\|^2/(2\sigma_p^2)}$ be the standard Gaussian density.
For any nonnegative Borel $\phi:\mathcal{S}\to\mathbb{R}$, we have
\begin{equation*}
    \begin{aligned}
        \int_\mathcal{S}\phi(y)\,p'(dy) = \int_{\mathbb{R}^D\setminus\Med(\mathcal{S})}\phi(\pi_\mathcal{S}(x))\,\rho(x)\,dx
    \end{aligned}
\end{equation*}
where this follows from the definition of the pushforward measure together with Assumption~\ref{ass:opt}.
Then, by introducing Lemma~\ref{lem:change_of_var_nor_bud}, we have
\begin{equation*}
    \begin{aligned}
        &\int_\mathcal{S}\phi(y)\,p'(dy) = \int_\mathcal{S}\!\int_{F_y}\phi(\pi_\mathcal{S}(y+v))\,\rho(y+v)\left|\mathrm{det}\left(I_m - \bar{A}(y,v)\right)\right|\,dv\,d\mathcal{H}^m(y) \\
        & =\int_\mathcal{S}\phi(y)\left[\int_{F_y}\rho(y+v)\left|\mathrm{det}\left(I_m - \bar{A}(y,v)\right)\right|\,dv\right]d\mathcal{H}^m(y)
    \end{aligned}
\end{equation*}
where in the last step we used $\pi_\mathcal{S}(y+v)=y$ for $v\in F_y$. Since $\phi$ is arbitrary, $p'$ is absolutely continuous with respect to $\mathcal{H}^m\!\restriction_\mathcal{S}$, with density
\begin{equation}\label{eq:density}
\boxed{
p'(y)\;:=\;\frac{dp'}{d\mathcal{H}^m\!\restriction_\mathcal{S}}(y)
\;=\;\int_{F_y}(2\pi\sigma_p^2)^{-D/2}\,e^{-|y+v|^2/(2\sigma_p^2)}\,\bigl|\det(I_m-A_y(v))\bigr|\,dv,
\qquad y\in\mathcal{S}.
}
\end{equation}
And $p'(\mathbb{R}^D\setminus\mathcal{S})=0$, since by Assumption~\ref{ass:opt}, $s_\theta(x)=\pi_\mathcal{S}(x) \in \mathcal{S}$ for $\mu$-a.e.\
$x\in\mathbb{R}^D\setminus\mathrm{Med}(\mathcal{S})$, and
$\mu(\mathrm{Med}(\mathcal{S}))=0$.
\end{proof}

\begin{remark}\label{lem:posfin}
For every $y\in\mathcal{S}$, $0<p'(y)<\infty$.
\end{remark}

\begin{proof}
\emph{Positivity.} On the ball $\{|v|<\tau\}\subseteq F_y$, the map $v\mapsto K(y,v)=|\det(I-A_y(v))|$ is continuous and equals $1$ at $v=0$, so $K(y,v)\ge 1/2$ on a small ball $|v|<\delta$. Then
\[
p'(y)\ge\int_{|v|<\delta}(2\pi\sigma_p^2)^{-D/2}e^{-|y+v|^2/(2\sigma_p^2)}\cdot\tfrac12\,dv>0.
\]
\emph{Finiteness.} $K(y,v)$ is a polynomial in $v$ of degree at most $m$, so it is dominated by $C(1+|v|)^m$ for some $C=C(y)$. Then
\[
p'(y)\le C(2\pi\sigma_p^2)^{-D/2}\int_{N_y\mathcal{S}}e^{-|y+v|^2/(2\sigma_p^2)}(1+|v|)^m\,dv<\infty,
\]
using a standard Gaussian integral on the linear space $N_y\mathcal{S}\cong\mathbb{R}^{D-m}$.
\end{proof}

\begin{remark}[Initial distribution shapes the pushforward density]
\label{rem:init_shapes_pushforward}
The Gaussian assumption in the theorem above is convenient but not essential.
If we replace $\mathcal{N}(0,\sigma_p^2 I_D)$ by an arbitrary initial density
$\rho:\mathbb{R}^D\to\mathbb{R}_{\ge 0}$, the same normal-bundle change of
variables (Lemma~\ref{lem:change_of_var_nor_bud}) gives
\begin{equation}\label{eq:general_pushforward}
    p'(y) \;=\; \int_{F_y} \rho(y+v)\,K(y,v)\,dv,
    \qquad y\in\mathcal{S}.
\end{equation}
This is a \emph{linear functional} in $\rho$: an inner product
$\langle \rho(y+\,\cdot\,),\,K(y,\,\cdot\,)\rangle_{F_y}$ between the initial
density restricted to the fiber and the geometric weight $K(y,v)$. In
particular, the choice of $\rho$ enters $p'$ pointwise on $\mathcal{S}$,
which yields the following qualitative consequences:
\begin{itemize}
    \item Uniform initialization on the tube $U_\tau$: $p'(y) \propto
    \int_{F_y} K(y,v)\,dv$, a purely geometric quantity.
    \item Initialization whose support misses the fiber $F_y$ entirely:
    $p'(y) = 0$.
\end{itemize}
The last item is the most instructive: \emph{whether $p'$ has full support on
$\mathcal{S}$, and how concentrated it is, is decided fiber-by-fiber by how
much mass $\rho$ places along each $F_y$}.

This is an important property indicating an ideal distance field transforms the original density in space $\mathbb{R}^D$ to the density upon the manifold linearly.
While during training, a distance field does not learn the density on the manifold, final density can still be customized by varying initial distributions following the distance field's guidance.
\end{remark}

Empirically, this characterization is consistent with our experiments in~\Cref{app:ula}, where inference can start from uniform distributions rather than the standard Gaussian initialization used in training, while still producing low FID and diverse images.
This behavior is difficult to explain from a density-allocation perspective alone, since the inference initialization no longer matches the usual training-time prior.

In contrast, the support-recovery view makes the choice of inference-time initialization less tightly tied to the training prior, because a good denoising field should always push noisy data toward the manifold.
As shown by our Remark~\ref{rem:init_shapes_pushforward}, the final distribution can be customized by the initial distribution in inference.
This is a feature of our method.
Consistently, when not starting with standard Gaussian noise but uniform noise, the baseline flow matching model cannot generate realistic images in our experiments (see~\Cref{app:ula}).

This feature also has practical value. A density-allocation view typically ties generation to a specific prior distribution, such as the Gaussian noise used during training; if the inference-time initialization deviates from this prior, generation can fail entirely. In contrast, the support-recovery view decouples image generation from exact density allocation on the manifold. Although after deterministic samplers, the final non-uniform density on $\mathcal{S}$ depends on the initialization, it allows broader initializations to be mapped toward the data manifold and produce realistic samples. Thus, the initial distribution may be a practical control knob for shaping the final manifold density, rather than a strict requirement for generation to work. This opens an interesting direction for future studies on initialization-dependent generation.

\citet{shen2026manifold} provide a complementary theorem toward a quantitative coverage definition and also attribute the generalization to geometric learning.
Another study~\cite{li2025scores} also theoretically suggests that geometric support recovery is easier than empirical density recovery.

The support recovery perspective also empirically matches previous time-unconditional models~\cite{sun2025noise,balcerak2025energy,wang2025equilibrium}.
By removing the explicit time variable $t$, these methods weaken the usual density-allocation interpretation of diffusion or flow matching settings, yet they still achieve competitive image-generation results.
However, they largely leave this phenomenon unexplained: they contain little discussion of density allocation or geometric support recovery, let alone an explicit objective for modeling geometric support recovery.

Taken together, these analyses~\cite{li2025scores, shen2026manifold, liang2026inductive} and empirical findings~\cite{sun2025noise, balcerak2025energy, wang2025equilibrium} suggest that geometric support recovery is a promising strategy providing competitive image-generation performance.

To improve geometric support recovery, we make two contributions. We offer a novel loss design for learning the distance field, and introduce two principled generation processes: sphere tracing and gradient descent.

We should also emphasize that how support recovery and density allocation jointly influence data quality and diversity remains open.
A stable and efficient high-dimensional MCMC scheme, or a better initial distribution near the manifold, may further improve performance within our method.
We leave this to future work.

\subsection{Properties of Our Loss Design}
\label{app:properties}

In~\Cref{app:projection_invariance}, we prove that our loss design helps the field converge to a distance function given a closest-point projection. In~\Cref{app:allocation}, we further prove the distance-induced measure is a uniform measure on the manifold, which can be achieved with a long MCMC scheme. In~\Cref{app:pushforward}, we also discuss how the ideal distance function guarantees meaningful coverage on the data manifold from the standard Gaussian distribution, even if no MCMC is applied.

However, our analysis in~\Cref{sec:analysis} shows that obtaining an accurate closest-point projection $\mathbf{s}(\mathbf{x})$ is problematic with finite samples in high-dimensional space, whether by traversing the dataset or by rematching within the minibatch.
To obtain an approximation of the closest-point projection, we next prove how our loss preferentially emphasizes closer supervision, sharpening the distribution of the denoising vector and improving the image quality.
In this section, we prove our design finally learns an inverse-distance-reweighted approximation as formulated in~\Cref{cor:os_discrete}.
\emph{This core property is intact when we model $u(\mathbf{x})\nabla u(\mathbf{x})$ with $u_\theta(\mathbf{x})\mathbf{v}_\theta(\mathbf{x})$.}

This inverse-squared-distance reweighting also explains why multi-step denoising remains necessary, even though our method substantially accelerates this process.

\paragraph{(i) Time-unconditional flow matching is directionally pulled by longer pairs, while directional eikonal loss is length-neutral.}

\theoremstyle{definition}
\begin{theorem}
\label{thm:dir_fm_vs_de}
Under \Cref{def:genproc}, fix a spatial location and work under this condition given
$\mathbf{X}=\mathbf{x}$.
We rewrite the conditional expectation as $\mathbb{E}_{\mathbf{x}}[\cdot]:=\mathbb{E}[\cdot\mid \mathbf{X}=\mathbf{x}]$.

\begin{enumerate}[label=(\roman*)]
\item{With the standard flow matching loss, the learned direction is length-weighted.}

Let $\mathbf{\Delta}:=\mathbf{X}_1-\mathbf{X}_0$, $\ell:=\|\mathbf{\Delta}\|_2$, and
\begin{equation}
\mathbf{u}
:=
\begin{cases}
\mathbf{\Delta}/\ell, & \ell>0,\\
\mathbf{0}, & \ell=0.
\end{cases}
\label{eq:unit_dir_fm}
\end{equation}
Consider the time-unconditional squared regression at $\mathbf{x}$:
\begin{equation}
\mathcal{J}_{\mathrm{FM}}(\mathbf{g})
:=
\mathbb{E}_{\mathbf{x}}\!\left[
\|\mathbf{g}-\mathbf{\Delta}\|_2^2
\right],
\ \mathbf{g}\in\mathbb{R}^D .
\label{eq:J_fm}
\end{equation}
Then $\mathcal{J}_{\mathrm{FM}}$ is strictly convex and has the unique minimizer
\begin{equation}
\mathbf{g}_{\mathrm{FM}}^{*}(\mathbf{x})
=
\mathbb{E}_{\mathbf{x}}\!\left[\mathbf{\Delta} \right]
=
\mathbb{E}_{\mathbf{x}}\!\left[\ell\,\mathbf{u} \right].
\label{eq:vstar_fm_dir}
\end{equation}
Therefore, the \emph{direction} of the fixed point $\mathbf{g}_{\mathrm{FM}}^{*}(\mathbf{x})$ is a
length-weighted average of denoising directions $\mathbf{u}$, so larger $\ell$ pairs exert proportionally larger influence for the final direction.

\item{With the directional eikonal loss, the learned direction uses a bounded distance weight.}
Let $\mathbf{X}_1$ denote the random target paired with $\mathbf{x}$ under our sampling, and define
\begin{equation}
r:=\|\mathbf{x}-\mathbf{X}_1\|_2,
\qquad
\mathbf{v}
:=
\begin{cases}
(\mathbf{x}-\mathbf{X}_1)/r, & r>0,\\
\mathbf{0}, & r=0.
\end{cases}
\label{eq:unit_dir_de}
\end{equation}
Here $\mathbf{v}$ has the same conditional distribution as $-\mathbf{u}$ given $\mathbf{X}=\mathbf{x}$.
Recall the directional eikonal target in \cref{eq:del} (with constant $c_0>0$):
\begin{equation}
\mathbf{v}'(\mathbf{x},\mathbf{X}_1)
: =
\frac{\mathbf{x}-\mathbf{X}_1}{\sqrt{r^2+c_0}}
=
\underbrace{\frac{r}{\sqrt{r^2+c_0}}}_{\eqqcolon\,\alpha(r)\in[0,1)}
\,\mathbf{v}.
\label{eq:dstar_dir_de}
\end{equation}
Consider the time-unconditional squared regression at $\mathbf{x}$:
\begin{equation}
\mathcal{J}_{\mathrm{DE}}(\mathbf{h})
:=
\mathbb{E}_{\mathbf{x}}\!\left[
\|\mathbf{h}-\mathbf{v}'\|_2^2
\right],
\  \mathbf{h}\in\mathbb{R}^D .
\label{eq:J_de}
\end{equation}
Then $\mathcal{J}_{\mathrm{DE}}$ is strictly convex and has the unique minimizer
\begin{equation}
\mathbf{h}_{\mathrm{DE}}^{*}(\mathbf{x})
=
\mathbb{E}_{\mathbf{x}}\!\left[\mathbf{v}' \right]
=
\mathbb{E}_{\mathbf{x}}\!\left[\alpha(r)\,\mathbf{v} \right].
\label{eq:gstar_de_dir}
\end{equation}
Since $\alpha(r)=r/\sqrt{r^2+c_0}$ saturates to $1$ as $r\to\infty$,
directional eikonal loss does \emph{not} indicate unbalanced weights for longer paths at $\mathbf{x}$.

\end{enumerate}

\end{theorem}

From this theorem, we observe that the standard flow matching loss is biased toward directions induced by longer matching pairs, which does not provide a consistent way to resolve the ambiguity in the posterior expectation.
From a denoising perspective, as we move closer to the target along the path, the denoising direction should become increasingly certain to a consistent direction toward that target, or the trajectory gets drifted away as shown in~\Cref{fig:paths}.

\textbf{(ii) One-step loss reduces target ambiguity by focusing on closer data points.}

The above shows directional eikonal loss avoids strong dominance from far/long pairs. We next show one-step loss goes further by explicitly reweighting supervision toward closer targets.
From the previous derivation, standard flow matching loss even attaches higher weights when the denoising vector $\mathbf{\Delta}$ is longer.
In the following paragraphs, we compare our loss with a weighted version $\mathbb{E}\!\left[
(1-T)^{-2}\,\|\mathbf{v}_\theta(\mathbf{x})-\mathbf{\Delta}\|_2^2
\,\middle|\,
\mathbf{X}=\mathbf{x}
\right]$ to show our superiority.

\begin{theorem}[Posterior on index]
\label{thm:posteriorindex}
Under \Cref{def:genproc}, for any $\mathbf{X}=\mathbf{x}\in\mathbb{R}^D$, we rewrite the conditional probability as $\mathbb{P}_{\mathbf{x}}[\cdot]:=\mathbb{P}[\cdot\mid \mathbf{X}=\mathbf{x}]$. The posterior distribution of $I$ is categorical with
\begin{equation}
\mathbb{P}_{\mathbf{x}}(I=i)
=
\frac{W_i(\mathbf{x})}{\sum_{j=1}^N W_j(\mathbf{x})}.
\label{eq:postI}
\end{equation}
\begin{equation}
W_i(\mathbf{x})
:=
\int_0^1
p_T(t)\,(1-t)^{-D}
\exp\!\left(
-\frac{\|\mathbf{x}-t\mathbf{s}_{i}\|_2^2}{2(1-t)^2}
\right)\dd t .
\label{eq:Wi}
\end{equation}
\end{theorem}

\begin{proof}
Fix $i\in\{1,\dots,N\}$ and condition on $(I=i,T=t)$. Then
\begin{equation}
\mathbf{X} = (1-t)\mathbf{X}_0 + t\mathbf{s}_{i},
\qquad \mathbf{X}\mid(I=i,T=t)\sim
\Normal\!\Big(t\mathbf{s}_{i},\, (1-t)^2\mathbf{I}_D\Big).
\label{eq:Ycond}
\end{equation}
Hence its density is
\begin{equation}
\begin{split}
p(\mathbf{x}\mid I=i,T=t)
&=(2\pi)^{-D/2}(1-t)^{-D}
\exp\!\left(
-\frac{\|\mathbf{x}-t\mathbf{s}_{i}\|_2^2}{2(1-t)^2}
\right).
\end{split}
\label{eq:lik}
\end{equation}
By marginalizing $T$,
\begin{equation}
p(\mathbf{x}\mid I=i)
=
\int_0^1 p_T(t)\,p(\mathbf{x}\mid I=i,T=t)\dd t.
\label{eq:margT}
\end{equation}
By Bayes' rule and the uniform prior $\mathbb{P}(I=i)=1/N$,
\begin{equation}
\mathbb{P}_{\mathbf{x}}(I=i)
=
\frac{p(\mathbf{x}\mid I=i)}{\sum_{j=1}^N p(\mathbf{x}\mid I=j)}.
\label{eq:bayes}
\end{equation}
Finally, the factor $(2\pi)^{-D/2}$ is common to all $i$ in \cref{eq:lik} and cancels after normalization, yielding \cref{eq:postI,eq:Wi}.
\end{proof}

\begin{theorem}[Bayes-optimal update under \cref{eq:osl}]
\label{thm:os_weighted_posterior_mean}
Let $\epsilon>0$. Consider the population version of \cref{eq:osl}, and fix $\mathbf{X}=\mathbf{x}\in\mathbb{R}^D$.
Let $\mathbf{X}_1$ denote the random target $\mathbf{s}_{\mathrm{data}}$ paired with $\mathbf{x}$, with conditional law $\mathbf{X}_1\mid\mathbf{X}=\mathbf{x}$.
Define
\[
\widehat{\mathbf{s}}_{\theta}(\mathbf{x})
:=
\mathbf{x}-u_{\theta}(\mathbf{x})\,\nabla u_{\theta}(\mathbf{x}).
\]
Assume the optimization reaches the global infimum of the population risk.
Then any globally optimal solution must satisfy, for (almost every) fixed $\mathbf{x}$,
\begin{align}
\widehat{\mathbf{s}}_{\theta}^{\,*}(\mathbf{x})
&=
\frac{
\mathbb{E}_{\mathbf{x}}\!\left[
\frac{\mathbf{X}_1}{\|\mathbf{x}-\mathbf{X}_1\|_2^2+\epsilon}
\right]
}{
\mathbb{E}_{\mathbf{x}}\!\left[
\frac{1}{\|\mathbf{x}-\mathbf{X}_1\|_2^2+\epsilon}
\right]
}.
\label{eq:os_bayes}
\end{align}
\end{theorem}

\begin{proof}
Consider the pointwise objective induced by \cref{eq:osl} at fixed $\mathbf{X}=\mathbf{x}$:
\[
\mathcal{J}_{\mathbf{x}}(\mathbf{y})
:=
\mathbb{E}_{\mathbf{x}}\!\left[
\frac{\|\mathbf{y}-\mathbf{X}_1\|_2^2}{\|\mathbf{x}-\mathbf{X}_1\|_2^2+\epsilon}
\right],
\qquad \mathbf{y}\in\mathbb{R}^D.
\]
Since $\epsilon>0$, the weight $w(\mathbf{X}_1) := (\|\mathbf{x}-\mathbf{X}_1\|_2^2+\epsilon)^{-1}$ is strictly positive and finite, and $\mathcal{J}_{\mathbf{x}}(\mathbf{y})$ is a convex quadratic function of $\mathbf{y}$.
Expanding and differentiating under the expectation gives
\[
\nabla_{\mathbf{y}}\mathcal{J}_{\mathbf{x}}(\mathbf{y})
=
2\,\mathbb{E}_{\mathbf{x}}\!\left[
\frac{\mathbf{y}-\mathbf{X}_1}{\|\mathbf{x}-\mathbf{X}_1\|_2^2+\epsilon}
\right].
\]
Setting the gradient to zero yields
\[
\mathbb{E}_{\mathbf{x}}\!\left[
\frac{\mathbf{y}}{\|\mathbf{x}-\mathbf{X}_1\|_2^2+\epsilon}
\right]
=
\mathbb{E}_{\mathbf{x}}\!\left[
\frac{\mathbf{X}_1}{\|\mathbf{x}-\mathbf{X}_1\|_2^2+\epsilon}
\right].
\]
Because $\mathbf{y}$ is deterministic given $\mathbf{x}$, it can be pulled out on the left and then dividing by the (strictly positive) scalar denominator yields \cref{eq:os_bayes}.
\end{proof}

\begin{corollary}[Discrete posterior over dataset indices]
\label{cor:os_discrete}
In the setting where $\mathbf{X}_1\mid\mathbf{X}=\mathbf{x}$ is supported on $\{\mathbf{s}_{1},\dots,\mathbf{s}_{N}\}$ with posterior masses
$\pi_i(\mathbf{x}) := \mathbb{P}_{\mathbf{x}}(I=i)$,
\begin{equation}
\label{eq:os_discrete}
\widehat{\mathbf{s}}_{\theta}^{\,*}(\mathbf{x})
=
\frac{
\sum_{i=1}^N \pi_i(\mathbf{x})\,\omega_i(\mathbf{x})\,\mathbf{s}_{i}
}{
\sum_{i=1}^N \pi_i(\mathbf{x})\,\omega_i(\mathbf{x})
},
\ \text{where }\,
\omega_i(\mathbf{x})
:=
\frac{1}{\|\mathbf{x}-\mathbf{s}_{i}\|_2^2+\epsilon}.
\end{equation}

\end{corollary}

\begin{remark}[The ideal case of $\epsilon= 0$]
\label{rem:osl_epsilon_zero}
The statement in \Cref{thm:os_weighted_posterior_mean} is given for $\epsilon>0$ not only for technical convenience, but also for modeling and optimization reasons.
When $\epsilon=0$, the weight $1/\|\mathbf{x}-\mathbf{X}_1\|_2^2$ becomes singular as soon as $\mathbf{X}_1=\mathbf{x}$, and can become arbitrarily large when $\mathbf{X}_1$ is very close to $\mathbf{x}$.
Such singular behavior makes the loss overly unstable numerically, and is also undesirable for neural network training because it may lead to unstable gradients and a harder optimization landscape.
Introducing $\epsilon>0$ removes this singularity and yields a bounded, smooth reweighting, which is more suitable for learning.

At the same time, the formula in \Cref{eq:os_bayes} still has a natural interpretation as $\epsilon\downarrow 0$.
Because the conditional law of $\mathbf{X}_1$ given $\mathbf{X}=\mathbf{x}$ places no mass at $\mathbf{x}$ and the corresponding inverse-square weight remains integrable, the Bayes-optimal update converges to the inverse-distance-squared reweighted average obtained by formally setting $\epsilon=0$.
In this sense, the $\epsilon=0$ expression describes the limiting geometric preference of \cref{eq:osl} and the ideal case, while $\epsilon>0$ is the regularized form that avoids singularities and is better behaved in practice.
\end{remark}

\begin{corollary}[Compact manifold-supported target distribution]
\label{cor:os_manifold}
Assume the setting of \Cref{thm:os_weighted_posterior_mean}, except that the target endpoint
$\mathbf{X}_1$ is drawn from a probability measure supported on a compact data manifold
$\mathcal{S}\subset\mathbb{R}^D$.
Let
$\nu_{\mathbf{x}}:=\mathcal{L}(\mathbf{X}_1\mid \mathbf{X}=\mathbf{x})$
denote the posterior distribution of the target endpoint.
Then, for any $\epsilon>0$, any globally optimal solution satisfies, for
$p_{\mathbf{X}}$-almost every $\mathbf{x}$,
\begin{equation}
\widehat{\mathbf{s}}_{\theta}^{\,*}(\mathbf{x})
=
\frac{
\int_{\mathcal{S}}
\frac{\mathbf{s}}{\|\mathbf{x}-\mathbf{s}\|_2^2+\epsilon}
\,\dd\nu_{\mathbf{x}}(\mathbf{s})
}{
\int_{\mathcal{S}}
\frac{1}{\|\mathbf{x}-\mathbf{s}\|_2^2+\epsilon}
\,\dd\nu_{\mathbf{x}}(\mathbf{s})
}.
\label{eq:os_manifold_eps}
\end{equation}
Moreover, if
\begin{equation}
\int_{\mathcal{S}}
\frac{1}{\|\mathbf{x}-\mathbf{s}\|_2^2}
\,\dd\nu_{\mathbf{x}}(\mathbf{s})
<\infty
\label{eq:os_manifold_integrability}
\end{equation}
for a given $\mathbf{x}$, then
\begin{equation}
\lim_{\epsilon\downarrow0}
\widehat{\mathbf{s}}_{\theta}^{\,*}(\mathbf{x})
=
\frac{
\int_{\mathcal{S}}
\frac{\mathbf{s}}{\|\mathbf{x}-\mathbf{s}\|_2^2}
\,\dd\nu_{\mathbf{x}}(\mathbf{s})
}{
\int_{\mathcal{S}}
\frac{1}{\|\mathbf{x}-\mathbf{s}\|_2^2}
\,\dd\nu_{\mathbf{x}}(\mathbf{s})
}.
\label{eq:os_manifold_zero}
\end{equation}
\end{corollary}

The $\epsilon>0$ formula \cref{eq:os_manifold_eps} is an immediate specialization of \Cref{thm:os_weighted_posterior_mean} with $\nu_{\mathbf{x}}$ on $\mathcal{S}$. For the limit \cref{eq:os_manifold_zero}, the denominator follows from monotone convergence since $1/(\|\mathbf{x}-\mathbf{s}\|_2^2+\epsilon)\uparrow 1/\|\mathbf{x}-\mathbf{s}\|_2^2$, whose integral is finite by \cref{eq:os_manifold_integrability}; the numerator follows from dominated convergence, with the $\nu_{\mathbf{x}}$-integrable envelope $\|\mathbf{s}\|_2/\|\mathbf{x}-\mathbf{s}\|_2^2\le M/\|\mathbf{x}-\mathbf{s}\|_2^2$ provided by the compactness of $\mathcal{S}$.
This follows directly from \Cref{thm:os_weighted_posterior_mean} by replacing the
finite posterior over dataset indices with the posterior measure
$\nu_{\mathbf{x}}$ on $\mathcal{S}$.

\begin{corollary}[Pointwise Bayes solution for time-unconditional weighted regression]
\label{lem:bayes_weighted_regression}
Let $(\mathbf{X},T,\mathbf{\Delta})$ be random variables with $X\in\mathbb{R}^D$, $T\in(0,1)$, and $\mathbf{\Delta}\in\mathbb{R}^D$.
Let $w:(0,1)\to(0,\infty)$ be measurable with
$\mathbb{E}[w(T)\mid \mathbf{X}=\mathbf{x}]\in(0,\infty)$.
For fixed $x\in\mathbb{R}^D$, consider
\[
\mathcal{J}_\mathbf{x}(\mathbf{y})
:=
\mathbb{E}_{\mathbf{x}}\!\left[
w(T)\,\|\mathbf{y}-\mathbf{\Delta}\|_2^2
\right],
\qquad \mathbf{y}\in\mathbb{R}^D.
\]
Then $\mathcal{J}_\mathbf{x}(\mathbf{y})$ is strictly convex in $\mathbf{y}$ and admits the unique minimizer
\begin{equation}
\mathbf{y}^{*}(\mathbf{x})
=
\frac{\mathbb{E}_{\mathbf{x}}[w(T)\,\mathbf{\Delta}]}{\mathbb{E}_{\mathbf{x}}[w(T)]} .
\label{eq:bayes_vstar}
\end{equation}
Equivalently,
\begin{equation}
\label{eq:q_def}
\mathbf{y}^{*}(\mathbf{x})
=
\int_0^1 q(t\mid \mathbf{x})\,
\mathbb{E}[\mathbf{\Delta}\mid \mathbf{X}=\mathbf{x},T=t]\;\dd t,
\qquad
q(t\mid \mathbf{x})\propto p(t\mid \mathbf{x})\,w(t).
\end{equation}
\end{corollary}

\begin{proof}
Fix $\mathbf{X}=\mathbf{x}$.
The convexity and the minimizer \cref{eq:bayes_vstar} follow from the previous proof. 
We only justify \cref{eq:q_def}.

By the tower property and measurability of $w(T)$ w.r.t.\ $T$,
\[
\mathbb{E}_{\mathbf{x}}[w(T)\mathbf{\Delta}]
=
\mathbb{E}_{\mathbf{x}}\!\big[w(T)\,\mathbb{E}[\mathbf{\Delta}\mid \mathbf{X}=\mathbf{x},T]\big].
\]
Let $p(t\mid \mathbf{x})$ be the conditional density of $T$ given $\mathbf{X}=\mathbf{x}$. Converting the outer expectation into an integral yields
\begin{equation}
\label{eq:Ex_alpha_ints}
\mathbb{E}_{\mathbf{x}}[w(T)\mathbf{\Delta}]
=
\int_0^1 w(t)\,p(t\mid \mathbf{x})\,
\mathbb{E}[\mathbf{\Delta}\mid \mathbf{X}=\mathbf{x},T=t]\;\dd t,
\qquad
\mathbb{E}_{\mathbf{x}}[w(T)]
=
\int_0^1 w(t)\,p(t\mid \mathbf{x})\;\dd t.
\end{equation}
Substituting these into \cref{eq:bayes_vstar} and defining
\[
q(t\mid \mathbf{x})
:=
\frac{w(t)\,p(t\mid \mathbf{x})}{\int_0^1 w(\tau)\,p(\tau\mid \mathbf{x})\,\dd \tau}
\ \ \Big(\propto\ p(t\mid \mathbf{x})\,w(t)\Big),
\]
we obtain
\[
\mathbf{y}^*(\mathbf{x})
=
\int_0^1 q(t\mid \mathbf{x})\;
\mathbb{E}[\mathbf{\Delta}\mid \mathbf{X}=\mathbf{x},T=t]\;\dd t,
\]
which proves \cref{eq:q_def}.
\end{proof}

\begin{remark}[Effective time-weight under flow matching and why it does not directly penalize $\|\mathbf{x}-\mathbf{s}_i\|_2$]
\label{rem:fm_time_weight_x0}
Consider the interpolation model
$\mathbf{X}=(1-T)\mathbf{X}_0+T\mathbf{X}_1$ with $\mathbf{X}_0\sim\Normal(\mathbf{0},\mathbf{I}_D)$, $T\sim p_T$ on $(0,1)$,
and $\mathbf{X}_1\in\{\mathbf{s}_{1},\dots,\mathbf{s}_{N}\}$ uniformly.
For the flow-matching-style target $\mathbf{\Delta}:=\mathbf{X}_1-\mathbf{X}_0$ and time weight $w(t)=(1-t)^{-2}$,
\cref{lem:bayes_weighted_regression} implies that, at fixed $\mathbf{X}=\mathbf{x}$,
different $t$ values contribute according to the effective density
\begin{equation}
q(t\mid \mathbf{x})\propto p(t\mid \mathbf{x})\,(1-t)^{-2}.
\label{eq:q_fm}
\end{equation}
Moreover, by Bayes' rule,
\begin{equation}
p(t\mid \mathbf{x})
\propto
p_T(t)\,\sum_{i=1}^N
(1-t)^{-D}\exp\!\left(
-\frac{\|\mathbf{x}-t\mathbf{s}_{i}\|_2^2}{2(1-t)^2}
\right),
\label{eq:pt_given_x_mix}
\end{equation}
hence
\begin{equation}
q(t\mid \mathbf{x})
\propto
p_T(t)\,\sum_{i=1}^N
(1-t)^{-(D+2)}\exp\!\left(
-\frac{\|\mathbf{x}-t\mathbf{s}_{i}\|_2^2}{2(1-t)^2}
\right).
\label{eq:qtx_explicit}
\end{equation}
Crucially, the exponential term depends on $\|\mathbf{x}-t\mathbf{s}_{i}\|_2$ (distance to the segment $\{t\mathbf{s}_{i}:t\in[0,1]\}$), rather than $\|\mathbf{x}-\mathbf{s}_{i}\|_2$.
In particular, when $\mathbf{x}$ is directly drawn from $\Normal(\mathbf{0},\mathbf{I}_D)$, the effective density $q(t \mid \mathbf{x})$ is typically biased toward smaller $t$. In this regime, $t\mathbf{s}_{i}\approx \mathbf{0}$ for all $i$, so the exponent in \cref{eq:qtx_explicit} becomes nearly $i$-independent: $\|\mathbf{x}-t\mathbf{s}_{i}\|_2\approx \|\mathbf{x}\|_2$.
Thus, far data points $\mathbf{s}_{i}$ can still explain the same contribution via small $t$, and the flow matching time-weight $(1-t)^{-2}$ reweights time rather than introducing an explicit additional penalty in $\|\mathbf{x}-\mathbf{s}_{i}\|_2$.
This contrasts with distance-reweighted objectives (e.g., \cref{eq:osl}), which directly downweight far $\mathbf{s}_{i}$ through factors such as $(\|\mathbf{x}-\mathbf{s}_{i}\|_2^2+\epsilon)^{-1}$.
\end{remark}

\subsection{Closest-Point Projection Approximation and Iterated Contraction with Inverse-Distance-Reweighting}
\label{app:convergence_approx}

We now further interpret the inverse-distance-reweighted solution in
Corollary~\ref{cor:os_manifold} from a geometric perspective.
The key point is that the learned one-step prediction is a barycenter under a
reweighted posterior.
For a fixed $\mathbf{x}$, define the inverse-distance-reweighted posterior
$q_{\mathbf{x}}^\epsilon$ on $\mathcal{S}$ by
\begin{equation}
\dd q_{\mathbf{x}}^\epsilon(\mathbf{s})
:=
\frac{
(\|\mathbf{x}-\mathbf{s}\|_2^2+\epsilon)^{-1}
\,\dd\nu_{\mathbf{x}}(\mathbf{s})
}{
\int_{\mathcal{S}}
(\|\mathbf{x}-\mathbf{s}'\|_2^2+\epsilon)^{-1}
\,\dd\nu_{\mathbf{x}}(\mathbf{s}')
}.
\label{eq:q_x_eps_def}
\end{equation}
Then \Cref{eq:os_manifold_eps} can be written as
\begin{equation}
\widehat{\mathbf{s}}_{\theta}^{\,*}(\mathbf{x})
=
F_\epsilon(\mathbf{x})
:=
\int_{\mathcal{S}}\mathbf{s}\,\dd q_{\mathbf{x}}^\epsilon(\mathbf{s}).
\label{eq:F_eps_barycenter}
\end{equation}
Thus, the Bayes-optimal one-step prediction is the barycenter of target points
under the inverse-distance-reweighted posterior.

Let $\pi_{\mathcal{S}}(\mathbf{x})$ denote the closest-point projection onto
$\mathcal{S}$, wherever the projection is unique, and write
\begin{equation}
d_{\mathcal{S}}(\mathbf{x})
:=
\operatorname{dist}(\mathbf{x},\mathcal{S})
=
\|\mathbf{x}-\pi_{\mathcal{S}}(\mathbf{x})\|_2 .
\label{eq:dist_to_S}
\end{equation}
In a tubular neighborhood of a smooth manifold, the distance function satisfies
\begin{equation}
\nabla d_{\mathcal{S}}(\mathbf{x})
=
\frac{\mathbf{x}-\pi_{\mathcal{S}}(\mathbf{x})}
{\|\mathbf{x}-\pi_{\mathcal{S}}(\mathbf{x})\|_2}.
\label{eq:true_dist_grad}
\end{equation}
Therefore, if $F_\epsilon(\mathbf{x})$ is close to
$\pi_{\mathcal{S}}(\mathbf{x})$, then the residual
$\mathbf{x}-F_\epsilon(\mathbf{x})$ approximates the true distance vector
$\mathbf{x}-\pi_{\mathcal{S}}(\mathbf{x})$.

\begin{assumption}[Barycentric first-moment localization]
\label{assump:first_moment_localization}
Assume $\mathbf{x}$ lies in a neighborhood of $\mathcal{S}$ where
$\pi_{\mathcal{S}}(\mathbf{x})$ is unique. Let
\begin{equation}
\mathbf{p}:=\pi_{\mathcal{S}}(\mathbf{x}),
\qquad
d:=d_{\mathcal{S}}(\mathbf{x}).
\end{equation}
There exists a radius $\eta(\mathbf{x})$ such that
\begin{equation}
\eta(\mathbf{x})=o(d)
\label{eq:local_radius_od}
\end{equation}
and the nonlocal weighted first moment satisfies
\begin{equation}
\left\|
\int_{\mathcal{S}\setminus B(\mathbf{p},\eta(\mathbf{x}))}
(\mathbf{s}-\mathbf{p})
\,\dd q_{\mathbf{x}}^\epsilon(\mathbf{s})
\right\|_2
=
o(d).
\label{eq:external_first_moment_od}
\end{equation}
\end{assumption}

Below we show a few examples that satisfy the assumption above (we call them ``positive cases'').

\begin{remark}[Positive case: locally flat symmetric support]
\label{rem:flat_symmetric_support}
A simple positive case occurs when the support is locally flat around
$\mathbf{p}=\pi_{\mathcal{S}}(\mathbf{x})$.
Suppose that, in a neighborhood of $\mathbf{p}$, points on $\mathcal{S}$ can be
written as
\begin{equation}
\mathbf{s}=\mathbf{p}+\mathbf{u},
\qquad
\mathbf{u}\in T_{\mathbf{p}}\mathcal{S},
\end{equation}
and
\begin{equation}
\mathbf{x}=\mathbf{p}+\mathbf{n},
\qquad
\mathbf{n}\perp T_{\mathbf{p}}\mathcal{S}.
\end{equation}
Then
\begin{equation}
\|\mathbf{x}-\mathbf{s}\|_2^2
=
\|\mathbf{n}\|_2^2+\|\mathbf{u}\|_2^2.
\end{equation}
Thus the inverse-distance weight is an even function of the tangent coordinate
$\mathbf{u}$. If the local posterior is symmetric in the sense that the
contributions of $\mathbf{u}$ and $-\mathbf{u}$ match, then the tangent first
moment cancels:
\begin{equation}
\int \mathbf{u}\,\dd q_{\mathbf{x}}^\epsilon(\mathbf{u})
=
\mathbf{0}.
\end{equation}
Consequently, the local barycenter equals $\mathbf{p}$, and the projection error
vanishes in the ideal locally flat and symmetric case.
\end{remark}

\begin{remark}[Positive case: piecewise-flat supports away from seams]
\label{rem:piecewise_flat_support}
The same reasoning applies to a piecewise-flat support away from seams and
corners. Suppose
\begin{equation}
\mathcal{S}=\bigcup_{\ell=1}^M \mathcal{S}_\ell,
\end{equation}
where each $\mathcal{S}_\ell$ is a flat patch. If
$\pi_{\mathcal{S}}(\mathbf{x})$ lies in the interior of one patch
$\mathcal{S}_\ell$, and the local neighborhood seen by
$q_{\mathbf{x}}^\epsilon$ does not intersect seams or other patches, then the
locally flat argument in Remark~\ref{rem:flat_symmetric_support} applies on
$\mathcal{S}_\ell$.
Near seams or corners, multiple flat patches may contribute simultaneously, and
their barycenter can lie between patches. In that case, additional dominance or
first-moment cancellation assumptions are needed.
\end{remark}

\begin{remark}[Positive case: low-curvature support with localized tangent spread]
\label{rem:low_curvature_support}
For a smooth curved support, the projection error is controlled by curvature
times the effective tangent variance. Around
$\mathbf{p}=\pi_{\mathcal{S}}(\mathbf{x})$, write the manifold locally as
\begin{equation}
\mathbf{s}(\mathbf{u})
=
\mathbf{p}+\mathbf{u}+\mathbf{h}(\mathbf{u}),
\qquad
\mathbf{u}\in T_{\mathbf{p}}\mathcal{S},
\end{equation}
where $\mathbf{h}(\mathbf{u})$ is the normal displacement. If the curvature is
bounded by $\kappa$, then locally
\begin{equation}
\|\mathbf{h}(\mathbf{u})\|_2
\lesssim
\kappa\|\mathbf{u}\|_2^2.
\end{equation}
If the tangent first moment is negligible and the effective tangent spread
satisfies
\begin{equation}
\mathbb{E}_{q_{\mathbf{x}}^\epsilon}
\left[\|\mathbf{u}\|_2^2\right]
\le
R_{\mathrm{eff}}^2,
\end{equation}
then
\begin{equation}
\|F_\epsilon(\mathbf{x})-\pi_{\mathcal{S}}(\mathbf{x})\|_2
\lesssim
\kappa R_{\mathrm{eff}}^2.
\end{equation}
Therefore, Assumption~\ref{assump:first_moment_localization} is satisfied whenever
\begin{equation}
\kappa R_{\mathrm{eff}}^2=o(d_{\mathcal{S}}(\mathbf{x})).
\end{equation}
In particular, if $R_{\mathrm{eff}}=O(d_{\mathcal{S}}(\mathbf{x}))$ and the
curvature is bounded, then the projection error is
$O(d_{\mathcal{S}}(\mathbf{x})^2)$, which is
$o(d_{\mathcal{S}}(\mathbf{x}))$.
\end{remark}

\begin{remark}[Positive case: finite or isolated target supports]
\label{rem:finite_isolated_support}
Another positive case is a finite target set, or more generally an isolated
target atom. Suppose the support is discrete and
$\mathbf{p}=\pi_{\mathcal{S}}(\mathbf{x})$ is an isolated nearest target with
positive posterior mass. If the other targets are separated from $\mathbf{p}$
by a positive margin and $\epsilon\downarrow0$, then the inverse-distance weight
of $\mathbf{p}$ dominates the weights of all other targets as
$\mathbf{x}\to\mathbf{p}$. Hence
\begin{equation}
q_{\mathbf{x}}^\epsilon(\{\mathbf{p}\})\to1,
\end{equation}
and
\begin{equation}
F_\epsilon(\mathbf{x})\to\mathbf{p}
=
\pi_{\mathcal{S}}(\mathbf{x}).
\end{equation}
In this case the nonlocal first moment is typically of order
$O(d_{\mathcal{S}}(\mathbf{x})^2)$ under the ideal inverse-distance limit, and
is therefore $o(d_{\mathcal{S}}(\mathbf{x}))$.
For countably infinite supports, discreteness alone is not sufficient; one also
needs isolation or sufficient local separation near the nearest target.
\end{remark}

\begin{theorem}[Distance-vector consistency from first-moment localization]
\label{thm:distance_vector_consistency}
Assume Assumption~\ref{assump:first_moment_localization}. Then
\begin{equation}
F_\epsilon(\mathbf{x})-\pi_{\mathcal{S}}(\mathbf{x})
=
o(d_{\mathcal{S}}(\mathbf{x})).
\label{eq:F_projection_od}
\end{equation}
Consequently,
\begin{equation}
\|\mathbf{x}-F_\epsilon(\mathbf{x})\|_2
=
d_{\mathcal{S}}(\mathbf{x})+o(d_{\mathcal{S}}(\mathbf{x})),
\label{eq:distance_magnitude_consistency}
\end{equation}
and
\begin{equation}
\frac{\mathbf{x}-F_\epsilon(\mathbf{x})}
{\|\mathbf{x}-F_\epsilon(\mathbf{x})\|_2}
\to
\nabla d_{\mathcal{S}}(\mathbf{x})
\label{eq:distance_direction_consistency}
\end{equation}
as $d_{\mathcal{S}}(\mathbf{x})\to0$.
\end{theorem}

\begin{proof}
Let
\begin{equation}
\mathbf{p}:=\pi_{\mathcal{S}}(\mathbf{x}),
\qquad
d:=d_{\mathcal{S}}(\mathbf{x}),
\qquad
A_{\mathbf{x}}
:=
\mathcal{S}\cap B(\mathbf{p},\eta(\mathbf{x})).
\end{equation}
Using \Cref{eq:F_eps_barycenter}, decompose
\begin{equation}
F_\epsilon(\mathbf{x})-\mathbf{p}
=
\int_{A_{\mathbf{x}}}
(\mathbf{s}-\mathbf{p})
\,\dd q_{\mathbf{x}}^\epsilon(\mathbf{s})
+
\int_{\mathcal{S}\setminus A_{\mathbf{x}}}
(\mathbf{s}-\mathbf{p})
\,\dd q_{\mathbf{x}}^\epsilon(\mathbf{s}).
\end{equation}
The local term is bounded by
\begin{equation}
\left\|
\int_{A_{\mathbf{x}}}
(\mathbf{s}-\mathbf{p})
\,\dd q_{\mathbf{x}}^\epsilon(\mathbf{s})
\right\|_2
\le
\eta(\mathbf{x})
=
o(d),
\end{equation}
and the nonlocal term is $o(d)$ by
\Cref{eq:external_first_moment_od}. Hence
\begin{equation}
F_\epsilon(\mathbf{x})-\mathbf{p}=o(d),
\end{equation}
which proves \Cref{eq:F_projection_od}.

Next,
\begin{equation}
\mathbf{x}-F_\epsilon(\mathbf{x})
=
\mathbf{x}-\pi_{\mathcal{S}}(\mathbf{x})
+
o(d_{\mathcal{S}}(\mathbf{x})).
\end{equation}
Taking norms gives \Cref{eq:distance_magnitude_consistency}. Normalizing both
sides and using \Cref{eq:true_dist_grad} gives
\Cref{eq:distance_direction_consistency}.
\end{proof}

\begin{theorem}[Convergence of the true distance under repeated reweighted updates]
\label{thm:true_distance_convergence}
Assume that the iterates are defined by
\begin{equation}
\mathbf{x}^{(k+1)}
=
F_{\epsilon_k}(\mathbf{x}^{(k)})
=
\int_{\mathcal{S}}\mathbf{s}\,\dd q_{\mathbf{x}^{(k)}}^{\epsilon_k}(\mathbf{s}).
\end{equation}
Suppose that the iterates remain in a neighborhood of $\mathcal{S}$ where the
closest-point projection $\pi_{\mathcal{S}}$ is unique, and that
Assumption~\ref{assump:first_moment_localization} holds along the iterates in the uniform
sense that
\begin{equation}
\frac{
\|F_{\epsilon_k}(\mathbf{x}^{(k)})-\pi_{\mathcal{S}}(\mathbf{x}^{(k)})\|_2
}{
d_{\mathcal{S}}(\mathbf{x}^{(k)})
}
\to 0
\qquad
\text{as } k\to\infty .
\label{eq:uniform_projection_error_along_iterates}
\end{equation}
Then
\begin{equation}
d_{\mathcal{S}}(\mathbf{x}^{(k)})\to 0 .
\end{equation}
More specifically, the convergence is eventually linearly contractive:
for any fixed $\rho\in(0,1)$, there exists an integer $K\ge 0$ such that
for all $k\ge K$,
\begin{equation}
d_{\mathcal{S}}(\mathbf{x}^{(k+1)})
\le
\rho\, d_{\mathcal{S}}(\mathbf{x}^{(k)}).
\label{eq:true_distance_contraction}
\end{equation}
\end{theorem}

\begin{proof}
Let
\begin{equation}
\mathbf{p}_k:=\pi_{\mathcal{S}}(\mathbf{x}^{(k)}),
\qquad
d_k:=d_{\mathcal{S}}(\mathbf{x}^{(k)}).
\end{equation}
Since $\mathbf{p}_k\in\mathcal{S}$, we have
\begin{equation}
d_{\mathcal{S}}(\mathbf{x}^{(k+1)})
=
\operatorname{dist}(F_{\epsilon_k}(\mathbf{x}^{(k)}),\mathcal{S})
\le
\|F_{\epsilon_k}(\mathbf{x}^{(k)})-\mathbf{p}_k\|_2 .
\end{equation}
By \Cref{eq:uniform_projection_error_along_iterates},
\begin{equation}
\|F_{\epsilon_k}(\mathbf{x}^{(k)})-\mathbf{p}_k\|_2
=
o(d_k).
\end{equation}
Therefore, for any fixed $\rho\in(0,1)$, there exists $K$ such that for all
$k\ge K$,
\begin{equation}
\|F_{\epsilon_k}(\mathbf{x}^{(k)})-\mathbf{p}_k\|_2
\le
\rho d_k.
\end{equation}
Hence, for all $k\ge K$,
\begin{equation}
d_{\mathcal{S}}(\mathbf{x}^{(k+1)})
\le
\rho d_{\mathcal{S}}(\mathbf{x}^{(k)}).
\end{equation}
Iterating this inequality yields
\begin{equation}
d_{\mathcal{S}}(\mathbf{x}^{(K+n)})
\le
\rho^n d_{\mathcal{S}}(\mathbf{x}^{(K)}),
\end{equation}
and therefore $d_{\mathcal{S}}(\mathbf{x}^{(k)})\to0$.
\end{proof}

\begin{theorem}[A smaller mean-radius contraction certificate]
\label{thm:mean_radius_contraction_certificate}
Let $\mathbf{x}$ lie in a neighborhood of $\mathcal{S}$ where the closest-point
projection is unique. Write
\begin{equation}
\mathbf{p}:=\pi_{\mathcal{S}}(\mathbf{x}),
\qquad
d:=d_{\mathcal{S}}(\mathbf{x}),
\end{equation}
and define the projection radius
\begin{equation}
R(\mathbf{s})
:=
\|\mathbf{s}-\mathbf{p}\|_2 .
\end{equation}
Let
\begin{equation}
F_{\nu}(\mathbf{x})
:=
\int_{\mathcal{S}}\mathbf{s}\,\dd\nu_{\mathbf{x}}(\mathbf{s})
\end{equation}
be the ordinary posterior barycenter, and recall that
\begin{equation}
F_{\epsilon}(\mathbf{x})
:=
\int_{\mathcal{S}}\mathbf{s}\,\dd q_{\mathbf{x}}^\epsilon(\mathbf{s})
\end{equation}
is the inverse-distance-reweighted posterior barycenter.

Define the first-moment, or mean-radius, contraction certificates
\begin{equation}
\rho_{\nu}^{(1)}(\mathbf{x})
:=
\frac{
\int_{\mathcal{S}} R(\mathbf{s})\,\dd\nu_{\mathbf{x}}(\mathbf{s})
}{
d_{\mathcal{S}}(\mathbf{x})
},
\end{equation}
and
\begin{equation}
\rho_{\epsilon}^{(1)}(\mathbf{x})
:=
\frac{
\int_{\mathcal{S}} R(\mathbf{s})\,\dd q_{\mathbf{x}}^\epsilon(\mathbf{s})
}{
d_{\mathcal{S}}(\mathbf{x})
}.
\end{equation}
The superscript $(1)$ indicates that the certificate is based on the first
moment of the projection radius.

Assume that there exists a strictly decreasing function
$g:[0,\infty)\to\mathbb{R}_{>0}$ such that
\begin{equation}
(\|\mathbf{x}-\mathbf{s}\|_2^2+\epsilon)^{-1}
=
g(R(\mathbf{s}))
\end{equation}
for $\nu_{\mathbf{x}}$-almost every $\mathbf{s}$, and that $R(\mathbf{s})$ is not
$\nu_{\mathbf{x}}$-almost surely constant. Then
\begin{equation}
\rho_{\epsilon}^{(1)}(\mathbf{x})
<
\rho_{\nu}^{(1)}(\mathbf{x}).
\label{eq:mean_radius_cert_strict}
\end{equation}
Moreover, the reweighted update satisfies
\begin{equation}
d_{\mathcal{S}}(F_{\epsilon}(\mathbf{x}))
\le
\rho_{\epsilon}^{(1)}(\mathbf{x})d_{\mathcal{S}}(\mathbf{x}),
\label{eq:mean_radius_cert_distance_bound}
\end{equation}
while the ordinary posterior barycenter satisfies
\begin{equation}
d_{\mathcal{S}}(F_{\nu}(\mathbf{x}))
\le
\rho_{\nu}^{(1)}(\mathbf{x})d_{\mathcal{S}}(\mathbf{x}).
\label{eq:mean_radius_cert_distance_bound_nu}
\end{equation}
Thus inverse-distance reweighting gives a strictly smaller certified one-step
distance bound in the mean-radius sense, which guarantees faster movement toward the manifold if $\rho_{\nu}^{(1)} < 1$.
\end{theorem}

\begin{proof}
Since $\mathbf{p}\in\mathcal{S}$, we have
\begin{equation}
d_{\mathcal{S}}(F_{\epsilon}(\mathbf{x}))
\le
\|F_{\epsilon}(\mathbf{x})-\mathbf{p}\|_2 .
\end{equation}
Using the definition of $F_{\epsilon}$ and the triangle inequality,
\begin{equation}
\begin{split}
\|F_{\epsilon}(\mathbf{x})-\mathbf{p}\|_2
&=
\left\|
\int_{\mathcal{S}}(\mathbf{s}-\mathbf{p})
\,\dd q_{\mathbf{x}}^\epsilon(\mathbf{s})
\right\|_2  \\
&\le
\int_{\mathcal{S}}
\|\mathbf{s}-\mathbf{p}\|_2
\,\dd q_{\mathbf{x}}^\epsilon(\mathbf{s}) \\
&=
\int_{\mathcal{S}}
R(\mathbf{s})
\,\dd q_{\mathbf{x}}^\epsilon(\mathbf{s}).
\end{split}
\end{equation}
This proves \Cref{eq:mean_radius_cert_distance_bound}. The same argument with
$\nu_{\mathbf{x}}$ in place of $q_{\mathbf{x}}^\epsilon$ proves
\Cref{eq:mean_radius_cert_distance_bound_nu}.

It remains to compare the two certificates. We use the notation
\begin{equation}
\operatorname{Cov}_{\nu_{\mathbf{x}}}(f,h)
:=
\int_{\mathcal{S}} f(\mathbf{s})h(\mathbf{s})\,\dd\nu_{\mathbf{x}}(\mathbf{s})
-
\left(
\int_{\mathcal{S}} f(\mathbf{s})\,\dd\nu_{\mathbf{x}}(\mathbf{s})
\right)
\left(
\int_{\mathcal{S}} h(\mathbf{s})\,\dd\nu_{\mathbf{x}}(\mathbf{s})
\right).
\end{equation}
By definition of $q_{\mathbf{x}}^\epsilon$,
\begin{equation}
\int_{\mathcal{S}} R(\mathbf{s})\,\dd q_{\mathbf{x}}^\epsilon(\mathbf{s})
=
\frac{
\int_{\mathcal{S}} R(\mathbf{s})g(R(\mathbf{s}))
\,\dd\nu_{\mathbf{x}}(\mathbf{s})
}{
\int_{\mathcal{S}} g(R(\mathbf{s}))
\,\dd\nu_{\mathbf{x}}(\mathbf{s})
}.
\label{eq:mean_radius_reweighted_expectation}
\end{equation}
Equivalently,
\begin{equation}
\int_{\mathcal{S}} R(\mathbf{s})\,\dd q_{\mathbf{x}}^\epsilon(\mathbf{s})
=
\int_{\mathcal{S}} R(\mathbf{s})\,\dd\nu_{\mathbf{x}}(\mathbf{s})
+
\frac{
\operatorname{Cov}_{\nu_{\mathbf{x}}}
\left(
R(\mathbf{s}),g(R(\mathbf{s}))
\right)
}{
\int_{\mathcal{S}} g(R(\mathbf{s}))
\,\dd\nu_{\mathbf{x}}(\mathbf{s})
}.
\label{eq:mean_radius_cov_identity}
\end{equation}
We now show that the covariance term is strictly negative. Let
$\mathbf{S}$ and $\mathbf{S}'$ be two independent samples from
$\nu_{\mathbf{x}}$. Then
\begin{equation}
\begin{split}
2\operatorname{Cov}_{\nu_{\mathbf{x}}}
\left(
R(\mathbf{s}),g(R(\mathbf{s}))
\right)
=
\mathbb{E}
\left[
\left(R(\mathbf{S})-R(\mathbf{S}')\right)
\left(g(R(\mathbf{S}))-g(R(\mathbf{S}'))\right)
\right].
\end{split}
\end{equation}
Since $g$ is strictly decreasing, the integrand is always nonpositive and is
strictly negative whenever $R(\mathbf{S})\neq R(\mathbf{S}')$. Because
$R(\mathbf{s})$ is not $\nu_{\mathbf{x}}$-almost surely constant, the latter
event has positive probability. Hence
\begin{equation}
\operatorname{Cov}_{\nu_{\mathbf{x}}}
\left(
R(\mathbf{s}),g(R(\mathbf{s}))
\right)
<0.
\end{equation}
Substituting this into \Cref{eq:mean_radius_cov_identity} gives
\begin{equation}
\int_{\mathcal{S}} R(\mathbf{s})\,\dd q_{\mathbf{x}}^\epsilon(\mathbf{s})
<
\int_{\mathcal{S}} R(\mathbf{s})\,\dd\nu_{\mathbf{x}}(\mathbf{s}).
\end{equation}
Dividing both sides by $d_{\mathcal{S}}(\mathbf{x})>0$ proves
\Cref{eq:mean_radius_cert_strict}.
\end{proof}

\begin{remark}[Geometric meaning of the monotonicity condition]
\label{rem:mean_radius_monotonicity_geometry}
The monotonicity assumption holds exactly in the locally flat projection
geometry. Indeed, if locally
\begin{equation}
\mathbf{x}=\mathbf{p}+\mathbf{n},
\qquad
\mathbf{s}=\mathbf{p}+\mathbf{u},
\qquad
\mathbf{n}\perp T_{\mathbf{p}}\mathcal{S},
\qquad
\mathbf{u}\in T_{\mathbf{p}}\mathcal{S},
\end{equation}
then
\begin{equation}
\|\mathbf{x}-\mathbf{s}\|_2^2
=
\|\mathbf{n}\|_2^2+\|\mathbf{u}\|_2^2
=
d_{\mathcal{S}}(\mathbf{x})^2+R(\mathbf{s})^2.
\end{equation}
Hence
\begin{equation}
(\|\mathbf{x}-\mathbf{s}\|_2^2+\epsilon)^{-1}
=
g(R(\mathbf{s})),
\qquad
g(r):=
\frac{1}{
d_{\mathcal{S}}(\mathbf{x})^2+r^2+\epsilon
},
\end{equation}
where $g$ is strictly decreasing in $r$. Thus inverse-distance reweighting
shifts posterior mass toward targets closer to the current projection point,
reducing the mean-radius contraction certificate. This certificate is an upper
bound rather than the exact barycentric projection error, but it directly
explains why repeated inverse-distance-reweighted updates can require fewer
denoising steps.
\end{remark}

It should also be noted that this inverse-squared-distance reweighting doesn't appear only when we model the denoising displacement as $u_\theta(\mathbf{x}) \nabla u_\theta(\mathbf{x})$. Instead, Corollary~\ref{cor:os_manifold} still holds when we use $u_\theta(\mathbf{x}) \mathbf{v}_\theta(\mathbf{x})$ to learn the displacement.

\subsection{Extended Limitations and Future Work}
\label{app:limitations_future_work}

Our method provides a practical way for introducing distance-field modeling into image generation. While the results are encouraging, several theoretical and practical questions remain beyond the scope of this paper.

From a theoretical perspective, better understanding how approximations influence the associated measure and pushforward dynamics remains a valuable direction for future work.
How to choose initializations that faithfully recover desired manifold distributions under the projection in Remark~\ref{rem:init_shapes_pushforward} is also intriguing.
Another open question is how neural-network inductive biases generalize from finite training samples to the continuous data manifold assumed in our analysis.
More broadly, the target measure suggested by our distance-field perspective, namely the uniform measure on the data manifold, is consistent with the traditional probabilistic view only when images are assumed to be uniformly sampled from that manifold.
When geometric coverage and empirical density allocation are not aligned, determining which principle is more important for generating diverse and realistic images remains an open question.

From a practical perspective, several directions may further improve the proposed framework.
Building on our experiments in~\Cref{app:ula}, better noise initialization could facilitate the pushforward process, and more efficient MCMC procedures could further explore the induced uniform distribution.
Our experiments are conducted across UNet (CIFAR-10) and DiT (ImageNet-256), in line with recent time-unconditional baselines~\cite{sun2025noise,wang2025equilibrium}. 
We also evaluate multiple model scales and patch sizes, including S/2, B/2, L/2, XL/2, B/4, and B/8.
However, extending the study to additional datasets, more training epochs, and broader architectures remains an interesting direction for future work.
We have shown that our implementation (separately learning $u(\mathbf{x})$ and $\mathbf{v}(\mathbf{x})$) preserves properties in~\Cref{app:properties} that focuses on closer targets, and still enjoys convergence empirically in~\Cref{app:convergence}.
Future work could explore learning a single large scalar field $u(\mathbf{x})$ using better architectures that enable efficient and stable second-order backpropagation.
In addition, since our method already substantially accelerates sampling compared with standard flow matching, it would be interesting to investigate how much further distillation can help few-step generation.

We view these limitations as promising directions opened by the distance-field perspective. As discussed in~\Cref{sec:related}, distance fields have been widely studied in graphics-oriented applications, and several time-unconditional generative models have achieved competitive results despite relaxing exact density allocation.
Recent theoretical works also suggest that distance-field structures may facilitate generation, and that geometric support recovery can play a central role in image-generation performance beyond empirical density allocation.

Building on these insights, our model offers a new perspective explaining the benefit of distance-field modeling: closer-target supervision helps disambiguate denoising signals.
This observation leads to a practical step toward large-scale time-unconditional image generation.
To the best of our knowledge, it is the first method to develop distance-field-inspired losses that explicitly emphasize closer targets.
Together with deterministic distance-field samplers, the resulting framework achieves strong performance and empirically faster convergence in sample quality.

\subsection{Broader Impact} 
\label{subsec:broader_impact}

This work addresses the mathematical problem of synthesizing samples similar to input training data. A current popular application of such generative models involves the generation of photorealistic or artistic images, which we used as a proxy benchmark in the evaluation. While generation of images can be helpful for prototyping in artistic creation, they can also be used negatively for disinformation. Detecting or watermarking synthesized imagery is an important future research direction. 
Beyond image generation, our method and its extensions can also be potentially used in scientific research such as protein folding~\cite{wu2024protein}, and other applications such as robotics~\cite{wolf2025diffusion}.

\clearpage

\section{Experiment Settings}
\label{sec:exp_set}

In our experiments, we report FID computed using $50$k generated samples unless otherwise specified.

\paragraph{Third-party assets.}
We use CIFAR-10~\cite{krizhevsky2009learning} and ImageNet~\cite{deng2009imagenet} under their distributors' terms, a frozen latent VAE from~\citet{rombach2022high}, and public code stacks aligned with flow matching~\cite{lipman2024flow}, SiT~\cite{ma2024sit}, and Equilibrium Matching (EqM)~\cite{wang2025equilibrium}, in each case under the corresponding licenses and usage conditions.

To output $u_{\theta}(\mathbf{x})$, we introduce a lightweight distance prediction network implemented as a two-layer CNN with $10{,}177$ parameters for CIFAR-10 ($+0.018\%$ over the original model) and $62{,}801$ parameters for ImageNet ($+0.0093\% \sim 0.19\%$ over different backbone sizes) as shown in~\Cref{fig:distpredictor}.

For ImageNet generation, we follow SiT~\cite{ma2024sit} and other latent diffusion frameworks: we train and sample in the latent space of a pretrained VAE~\cite{rombach2022high}, and keep the VAE encoder/decoder frozen throughout all experiments.

\begin{figure}[!htbp]
    \centering
    \includegraphics[width=0.85\linewidth]{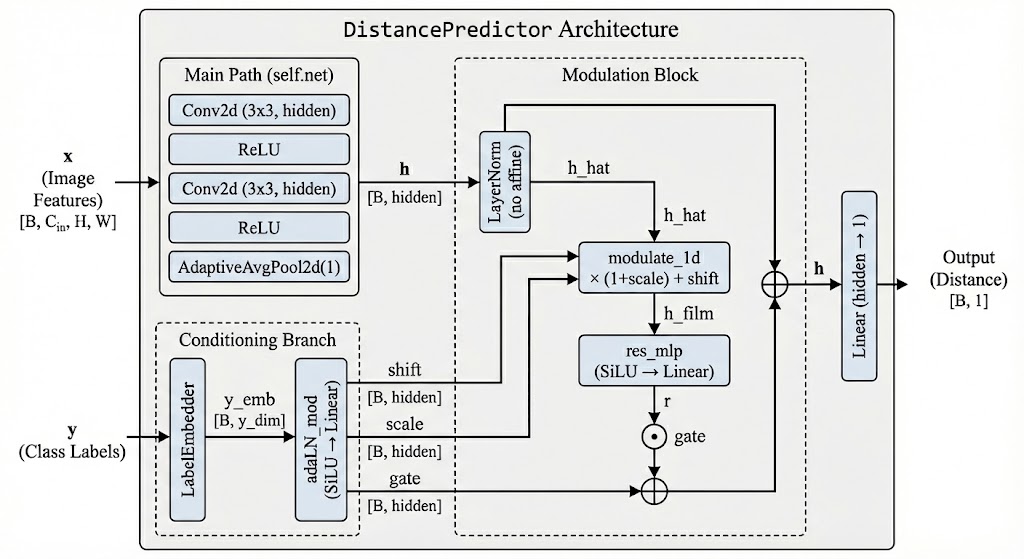}
    \caption{Distance predictor architecture.}
    \label{fig:distpredictor}
\end{figure}

We also report the speed comparison in~\Cref{tab:speed_comparison}.
Since our method only adds a tiny CNN with 0.009\%--0.19\% more parameters, its training and inference speeds are nearly identical to the baselines.

\begin{table}[!htbp]
\centering
\small
\caption{Training and CFG inference speed comparison on ImageNet.
For training, we use 4 Nvidia A100 (80G) GPUs. For inference, we use one Nvidia A10 to generate 64 images in one batch for 250 steps.}
\label{tab:speed_comparison}
\begin{tabular}{lccc}
\toprule
Metric & SiT & EqM & Ours \\
\midrule
Training throughput (steps/sec) & 2.19 & 2.18 & 2.18 \\
Inference time w/ CFG (sec) & 83.36 & 83.30 & 84.32 \\
\bottomrule
\end{tabular}
\end{table}

\begin{table}[t]
\caption{Configuration for the CIFAR-10 \dm run.
U-Net architecture entries follow the CIFAR-10 preset used in the codebase~\cite{lipman2024flow}.
Our losses in~\Cref{eq:osl} and~\Cref{eq:del} supervise a quasi-normalized denoising direction for $\mathbf{v}_\theta$, which induces a different target scale than flow matching.
We therefore rescale the network output $\mathbf{v}_\theta$ to align with the loss-defined magnitude.}
\centering
\small
\setlength{\tabcolsep}{6pt}
\renewcommand{\arraystretch}{1.10}
\begin{tabular}{l c}
\toprule
\textbf{Setting} & \textbf{Value} \\

\midrule
\multicolumn{2}{l}{\textbf{U-Net Architecture}} \\
input size & $3 \times 32 \times 32$ \\
base channels & 128 \\
channel multipliers & [2, 2, 2] \\
ResBlocks per level & 4 \\
attention resolutions & [2] \\
attention heads & 1 \\
head channels & 256 \\
dropout & 0.3 \\
time embedding max period & 1,000,000 \\
$\mathbf{v}_\theta(\cdot)$ output rescaling & 1/120.0 \\
$\mathbf{v}_\theta(\cdot)$ distance input rescaling & 1.0 \\

\midrule
\multicolumn{2}{l}{\textbf{Distance Predictor}} \\
input channels & 3 \\
convolution layer 1 & Conv2d(3$\to$32, $3\times 3$, padding=1) $\to$ ReLU \\
convolution layer 2 & Conv2d(32$\to$32, $3\times 3$, padding=1) $\to$ ReLU \\
final layer & AdaptiveAvgPool2d(1) $\to$ Linear(32$\to$1) \\

\midrule
\multicolumn{2}{l}{\textbf{Training}} \\
GPUs & 4 Nvidia A100 (80G)\\
batch size per GPU & 128 \\
optimizer         & AdamW \\
optimizer $\beta_1$ & 0.9 \\
optimizer $\beta_2$ & 0.95 \\
weight decay      & 0.01 \\
learning rate & $1 \times 10^{-4}$ \\
lr schedule       & constant \\
lr warmup         & none \\
interpolation coefficient distribution & default skewed  distribution~\cite{lipman2024flow} \\
training epochs & 3000 \\
test epoch & 1810 \\
$\lambda_{\mathrm{OS}}$ & 0.1 \\
$\lambda_{\mathrm{DE}}$ & 1.0 \\
$\epsilon$ in OSL & 4.0 \\
$c_0$ in DEL & 4.0 \\

\midrule
\multicolumn{2}{l}{\textbf{Inference}} \\
sphere tracing step size & 0.031 \\
gradient descent step size & 0.46 \\
steps & 250 \\

\bottomrule
\end{tabular}

\label{tab:dm_cifar10_config}
\end{table}

\begin{table}[t]
\caption{Configuration for the ImageNet $256 \times 256$ \dm run. Transformer architecture entries follow the SiT preset used in the codebase~\cite{ma2024sit}.}
\centering
\small
\begin{tabular}{l *{4}{>{\centering\arraybackslash}m{0.1\linewidth}}}
\toprule
model & S/2 & B/2 & L/2 & XL/2 \\
\midrule
\multicolumn{5}{l}{\textbf{SiT Architecture}} \\
latent space size & \multicolumn{4}{c}{$4 \times 32 \times 32$} \\
params (M)  & 33  & 130 & 458  & 675 \\
depth       & 12  & 12  & 24   & 28  \\
hidden dim  & 384 & 768 & 1024 & 1152 \\
patch size  & 2   & 2   & 2    & 2   \\
heads       & 6   & 12  & 16   & 16  \\
$\mathbf{v}_\theta(\cdot)$ output rescaling & \multicolumn{4}{c}{1/120.0} \\
$\mathbf{v}_\theta(\cdot)$ distance input rescaling  & \multicolumn{4}{c}{1/100.0} \\
\midrule
\multicolumn{5}{l}{\textbf{Distance Predictor}} \\
input channels & \multicolumn{4}{c}{4}\\
convolution layer 1 & \multicolumn{4}{c}{Conv2d(4$\to$64, $3\times 3$, padding=1) $\to$ ReLU}\\
convolution layer 2 & \multicolumn{4}{c}{Conv2d(64$\to$64, $3\times 3$, padding=1) $\to$ ReLU }\\
pooling & \multicolumn{4}{c}{AdaptiveAvgPool2d(1) }\\
normalization & \multicolumn{4}{c}{LayerNorm(64, elementwise\_affine=False) }\\
label embedding & \multicolumn{4}{c}{LabelEmbedder(num\_classes, 16) }\\
adaLN modulation & \multicolumn{4}{c}{SiLU $\to$ Linear(16$\to$192) $\to$ chunk(shift, scale, gate) }\\
residual MLP & \multicolumn{4}{c}{SiLU $\to$ Linear(64$\to$64) }\\
modulation formula & \multicolumn{4}{c}{$h = h + \text{gate} \odot \text{res\_mlp}(\text{modulate\_1d}(\text{LN}(h), \text{shift}, \text{scale}))$ }\\
final layer & \multicolumn{4}{c}{Linear(64$\to$1)} \\
\midrule
\multicolumn{5}{l}{\textbf{Training}} \\
GPUs  & \multicolumn{4}{c}{4 Nvidia A100 (80G)} \\
batch size per GPU  & \multicolumn{4}{c}{64} \\
optimizer         & \multicolumn{4}{c}{AdamW} \\
optimizer $\beta_1$ & \multicolumn{4}{c}{0.9} \\
optimizer $\beta_2$ & \multicolumn{4}{c}{0.999} \\
weight decay      & \multicolumn{4}{c}{0.0} \\
learning rate (lr)& \multicolumn{4}{c}{$1 \times 10^{-4}$} \\
lr schedule       & \multicolumn{4}{c}{constant} \\
lr warmup         & \multicolumn{4}{c}{none} \\
interpolation coefficient distribution & \multicolumn{4}{c}{$\mathrm{Uniform}(0,0.999)$} \\
training class dropout & \multicolumn{4}{c}{0.1} \\
training epochs       & \multicolumn{4}{c}{80} \\
test epoch            & \multicolumn{4}{c}{80} \\
$\lambda_{\mathrm{OS}}$ & \multicolumn{4}{c}{1.0} \\
$\lambda_{\mathrm{DE}}$ & \multicolumn{4}{c}{30.0} \\
$\epsilon$ in OSL & \multicolumn{4}{c}{100.0} \\
$c_0$ in DEL & \multicolumn{4}{c}{100.0} \\
\midrule
\multicolumn{5}{l}{\textbf{Inference Hyperparameters}} \\
sphere tracing step size & \multicolumn{4}{c}{0.022} \\
gradient descent step size & 0.64 & 0.64 & 0.64 & 0.68 \\
steps   & \multicolumn{4}{c}{250} \\
\bottomrule
\end{tabular}

\label{tab:dm_imagenet_configs}
\end{table}

\clearpage

\section{Experiment Outcomes}

\subsection{2D Comparison with Energy-Based Models}
\label{app:2d_exp}

From an energy-based modeling perspective, the model learns the energy gradient, so generation can always be viewed as performing a form of optimization~\cite{permenter2023interpreting,balcerak2025energy, wang2025equilibrium}.
While previous papers have pointed out this connection, their energy landscape modeling is often very rough; as a result, the learned landscape is unreliable even in our 2D sanity check, and the generation does not obtain any advantage when evaluated by Wasserstein-2 distance (W2), Hausdorff distance (HD), or Chamfer distance (CD).
This also helps explain why their high-dimensional performance is ad hoc and sensitive when transferring across architectures (see~\Cref{tab:compcfg,tab:cifar-10}).

HD and CD are standard metrics in geometric reconstruction and manifold learning~\cite{niyogi2008finding,attali2022optimal}.
In particular, convergence of point clouds or reconstructed surfaces to an underlying manifold is often formulated in terms of the Hausdorff metric.
W2 provides a complementary distributional metric, capturing both geometric mismatch and mass allocation.

In our 2D experiments, we adopt the 2D TorchCFM architecture~\cite{tong2023improving}.
We append a learnable linear layer to produce a 1D scalar output instead of the original 2D vector output.
Using our distance field, after jumping the first deterministic sphere tracing step, we run a batched Hamiltonian Monte Carlo (HMC) sampler to produce the final samples.
Concretely, we sample an auxiliary momentum $\mathbf{p}\sim\mathcal{N}(\mathbf{0}, m\mathbf{I})$ with $m=0.945$ and define the tempered target distribution
$\pi(\mathbf{x}) \propto \exp\!\left(-U(\mathbf{x})/\sigma^2\right)$ with $\sigma$ linearly annealing among $0.89 \to 0.25 \to 0.04$.
Each HMC proposal is generated by $L=1$ leapfrog step with linearly decreasing step size $\epsilon=1.15 \to 0.78 \to 0.55$,
where the potential gradient is scaled as $\nabla U(\mathbf{x})/\sigma^2$.
We then apply a Metropolis--Hastings accept/reject correction using the Hamiltonian
$H(\mathbf{x},\mathbf{p}) = U(\mathbf{x})/\sigma^2 + \|\mathbf{p}\|_2^2/(2m)$.
Following the baseline, we set $t_{\mathrm{span}}=\mathrm{linspace}(0,1,100)$ and record all $100$ positions. We further show that our advantage over other energy-based models does not come from MCMC, but from accurate geometric signals.

\begin{figure}[!htbp]
  \centering
  \includegraphics[width=0.75\linewidth]{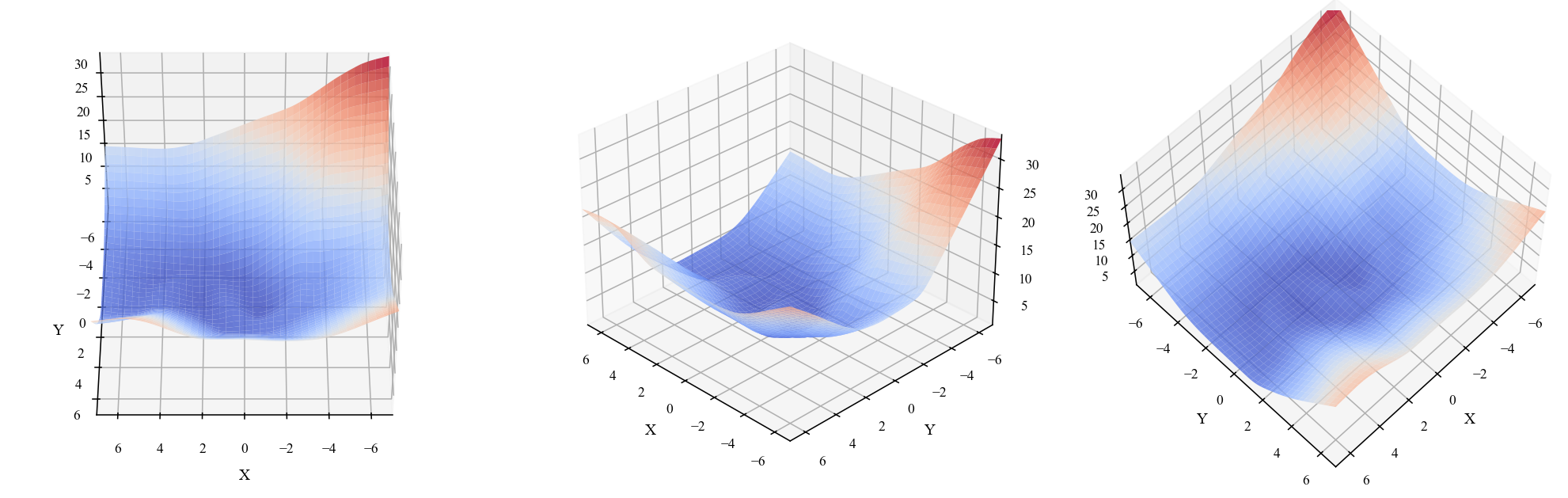}
  \caption{Energy Matching fails to capture fine-grained energy variations near the data manifold and is inconsistent across two locations that are equally distant from the data manifold.}
  \label{fig:energy-matching}
  \vspace{-10pt}
\end{figure}

We first use the open-source 2D toy implementation of Energy Matching~\cite{balcerak2025energy} to generate its 2D samples.
We follow exactly the same experimental setup and consider the same generation task from 8-Gaussians to two-moons.
The resulting level sets and energy landscape are visualized in~\Cref{fig:energy-matching}.
Their method relies on OT coupling~\cite{tong2023improving} and uses a standard ODE solver for energy minimization.
However, the learned energy landscape is inaccurate, and the resulting 2D generation does not obtain advantages when evaluated by W2, HD, or CD.

\begin{figure}[!htbp]
  \centering
  \includegraphics[width=0.75\linewidth]{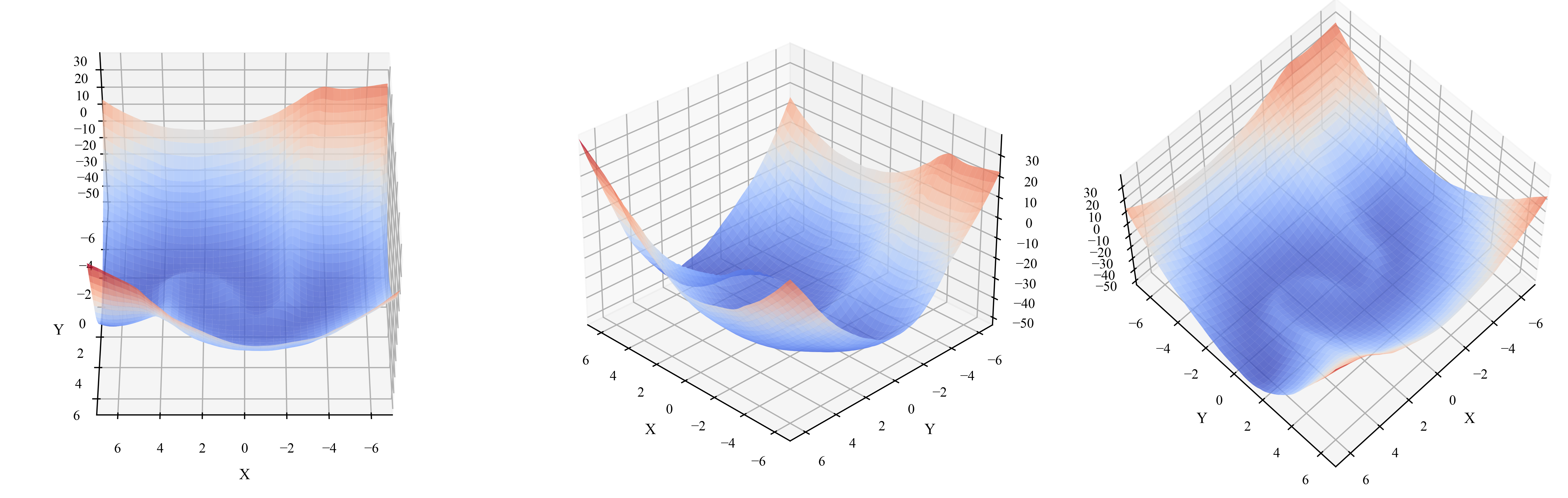}
  \caption{Equilibrium Matching~\cite{wang2025equilibrium} learns $4\times$ flow matching denoising vectors from source to target, with a decaying coefficient when $t>0.8$; for visualization clarity, we divide the negative gradient vectors by $4$ when rendering. Equilibrium Matching misses fine-grained energy structure near the data manifold and is inconsistent across two locations that are equally distant from the data manifold.}
  \label{fig:eqm}
  \vspace{-10pt}
\end{figure}

We next consider Equilibrium Matching~\cite{wang2025equilibrium}.
Since it does not release a 2D model, we implement its objective by replacing our loss with theirs within the same 2D framework.
In short, their learning objective corresponds to a rescaled flow matching vector field, and thus naturally inherits the issues discussed above.
In 2D, this rescaling further amplifies the problem: the bottom of the inferred energy landscape becomes noisy, and the landscape remains inconsistent at two corner locations that are equally distant from the data manifold.
We also try to reproduce their 2D generation using gradient-descent-based sampling, as they mentioned, but do not obtain a satisfactory result.
In~\Cref{fig:eqm_samples} and~\Cref{tab:2d_metrics_eqm}, we report three generation attempts: (i) Gradient descent only; (ii) Gradient descent + Unadjusted Langevin Algorithm (ULA); and (iii) Gradient descent + HMC.

\begin{figure}[!htbp]
  \centering

  \begin{subfigure}[t]{0.15\linewidth}
    \centering
    \includegraphics[width=\linewidth]{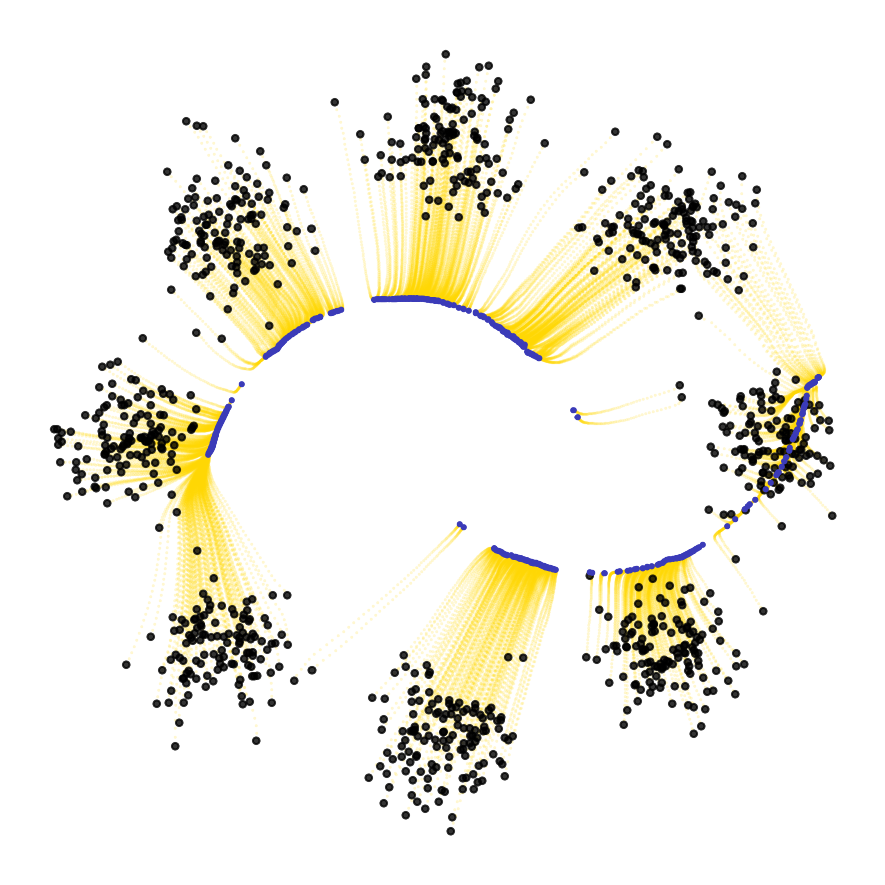}
    \caption{Gradient descent only: no exploration for low-potential areas.}
  \end{subfigure}\hspace{0.6em}%
  \begin{subfigure}[t]{0.15\linewidth}
    \centering
    \includegraphics[width=\linewidth]{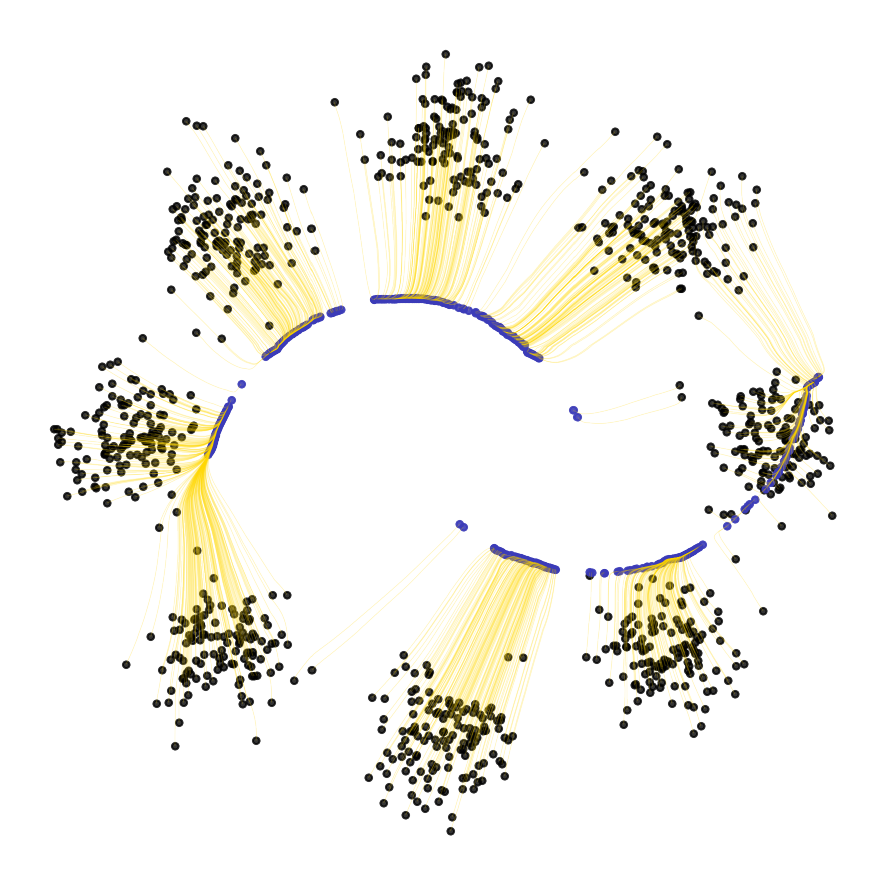}
    \caption{Gradient descent only trajectories.}
  \end{subfigure}\hspace{0.6em}%
  \begin{subfigure}[t]{0.15\linewidth}
    \centering
    \includegraphics[width=\linewidth]{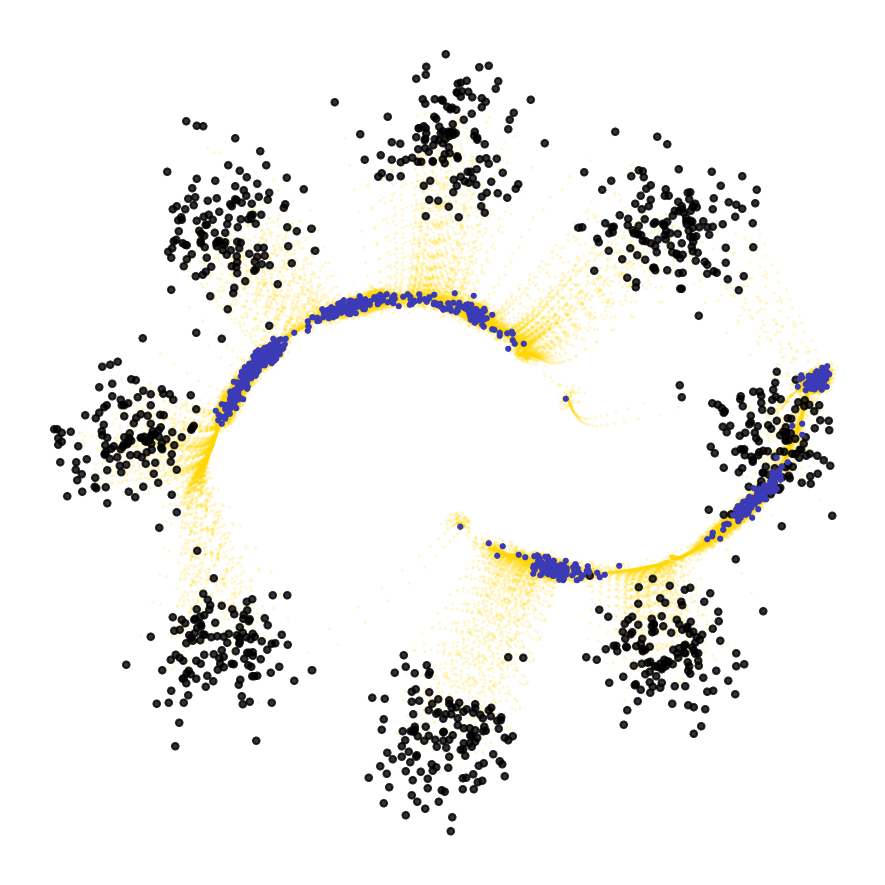}
    \caption{Gradient descent and ULA: insufficient exploration for low-potential areas.}
  \end{subfigure}\hspace{0.6em}%
  \begin{subfigure}[t]{0.15\linewidth}
    \centering
    \includegraphics[width=\linewidth]{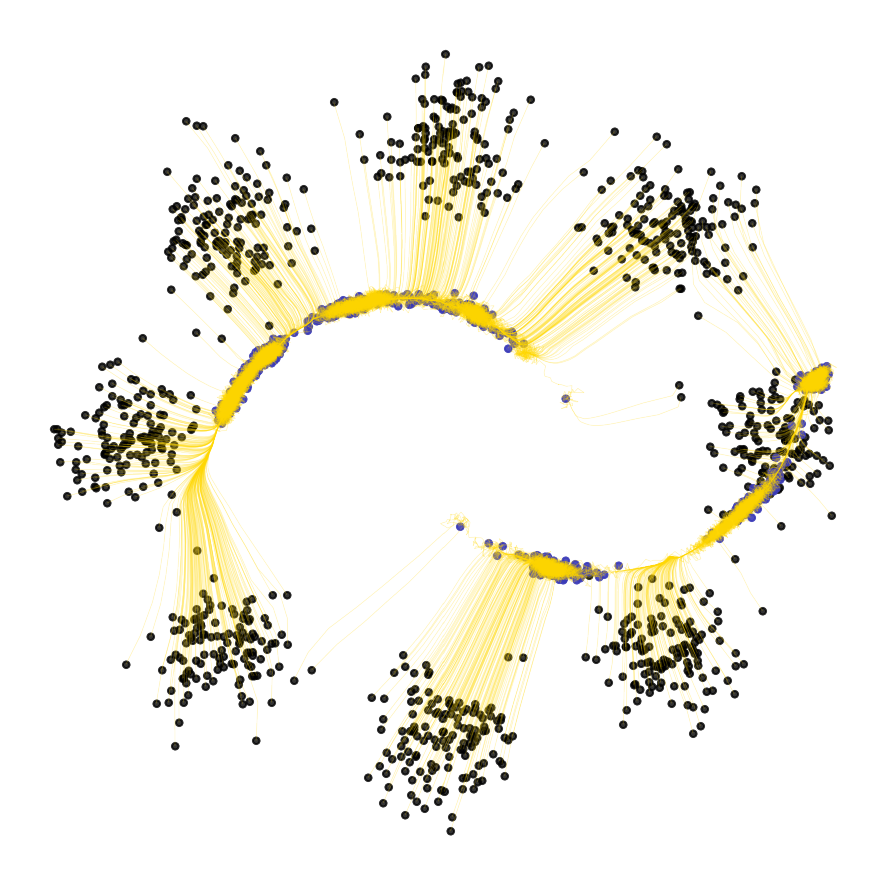}
    \caption{Gradient descent and ULA trajectories.}
  \end{subfigure}\hspace{0.6em}%
  \begin{subfigure}[t]{0.15\linewidth}
    \centering
    \includegraphics[width=\linewidth]{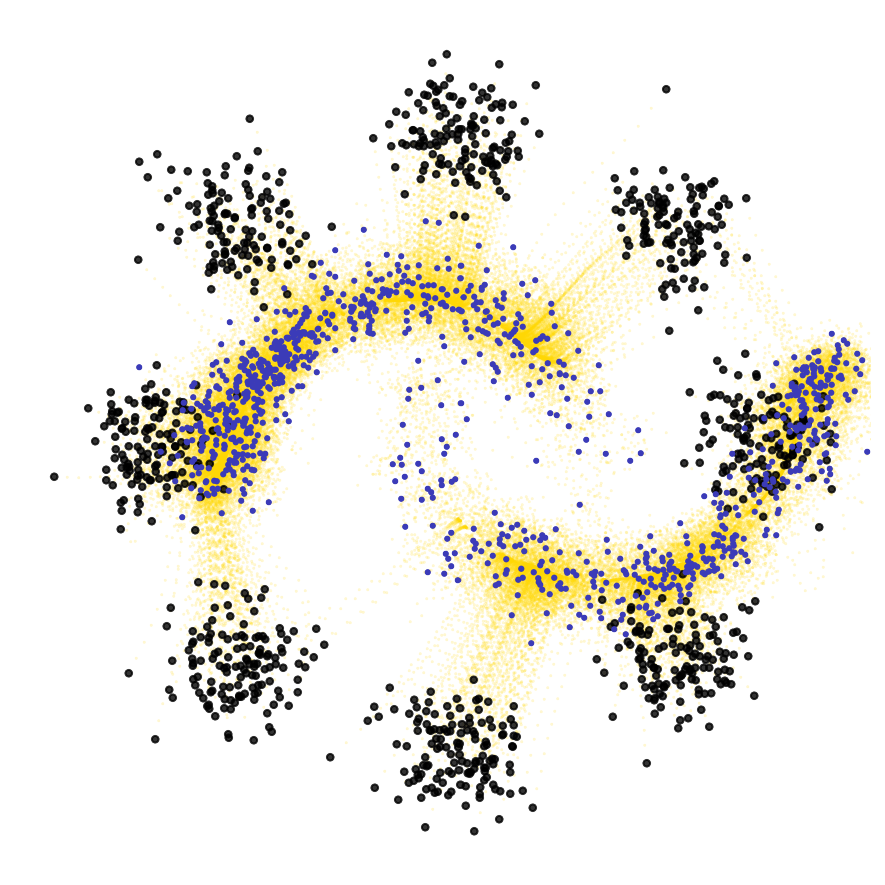}
    \caption{Gradient descent and HMC: long exploration reveals the inaccuracy of the energy modeling.}
  \end{subfigure}\hspace{0.6em}%
  \begin{subfigure}[t]{0.15\linewidth}
    \centering
    \includegraphics[width=\linewidth]{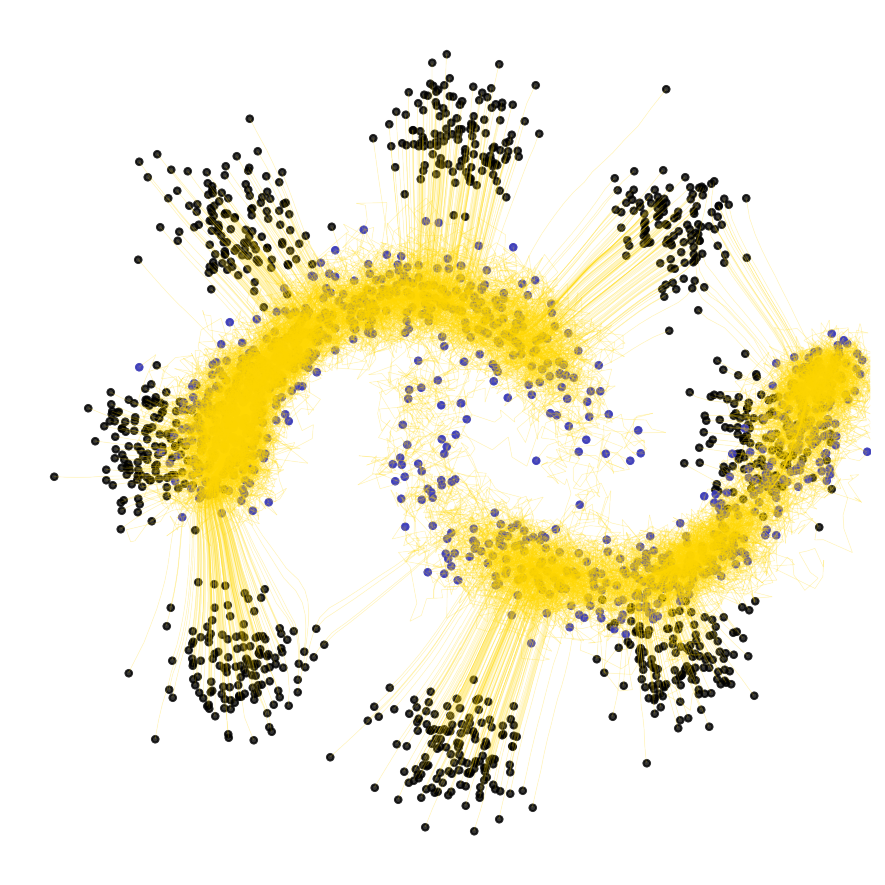}
    \caption{Gradient descent and HMC trajectories.}
  \end{subfigure}

  \caption{We compare different generation methods on the Equilibrium Matching model, but none of them yields competitive results on this task. Steps are recorded as: init (black), mid (yellow), and final (blue). 1{,}024 samples for visualizations.}
  \vspace{-10pt}
  \label{fig:eqm_samples}
\end{figure}

\begin{table}[!htbp]
\centering
\captionsetup{skip=6pt} %
\caption{2D point-cloud distance metrics of~\Cref{fig:eqm_samples} between $10$k generated samples and $10$k target two-moons samples for the 8-Gaussians $\rightarrow$ two-moons task (lower is better). We report the original gradient descent as their best in~\Cref{tab:2d_metrics}.}
\label{tab:2d_metrics_eqm}

\small
\setlength{\tabcolsep}{6pt}
\renewcommand{\arraystretch}{1.05}

\begin{tabular}{>{\raggedright\arraybackslash}p{0.5\columnwidth}rrr}
\toprule
\textbf{Method} & \textbf{W2 $\downarrow$} & \textbf{HD $\downarrow$} & \textbf{CD $\downarrow$} \\
\midrule
Equilibrium Matching (Original GD) & 1.525 & 1.691 & 0.150 \\
Equilibrium Matching (GD+ULA) & 1.852 & 1.341 & 0.093 \\
Equilibrium Matching (GD+HMC) & 1.477 & 1.896 & 0.102 \\
Distance Marching (ours) & 0.440 & 1.069 & 0.014 \\
\bottomrule
\multicolumn{4}{l}{\scriptsize W2: Wasserstein-2; HD: Hausdorff distance; CD: Chamfer distance.}
\end{tabular}
\end{table}

Crucially, 2D MCMC refinements (e.g., ULA/HMC) need faithful energies and gradients as a prerequisite.
If the learned landscape is mis-specified or noisy near the data manifold, the chain may mix poorly or drift toward spurious modes, so MCMC typically fails to improve W2/HD/CD and can even worsen them in practice.

As a comparison, we also present our level sets in~\Cref{fig:2d_motivation} and the corresponding energy landscape in~\Cref{fig:dm_landscape}.
It should also be noted that these methods do not constrain the energy itself but only its gradient, whereas we directly model a distance-like field; consequently, our framework also supports sphere tracing, which they do not provide.

\begin{figure}[!htbp]
  \centering

  \begin{subfigure}[t]{0.62\linewidth}
    \centering
    \includegraphics[width=\linewidth]{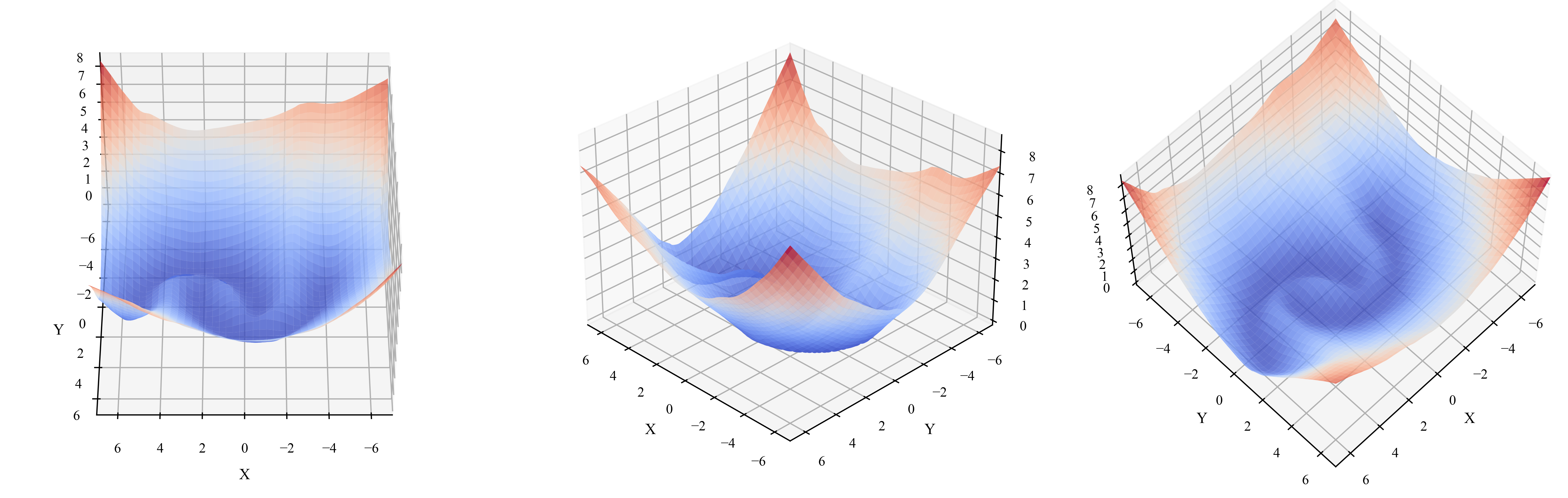}
    \caption{Our losses provide consistent energy modeling with respect to the distance to the data manifold.}
    \label{fig:dm_landscape_landscape}
  \end{subfigure}\hspace{0.9em}%
  \begin{minipage}[t]{0.33\linewidth}
    \vspace{-10.0em} %
    \centering
    \begin{subfigure}[t]{0.47\linewidth}
      \centering
      \includegraphics[width=\linewidth]{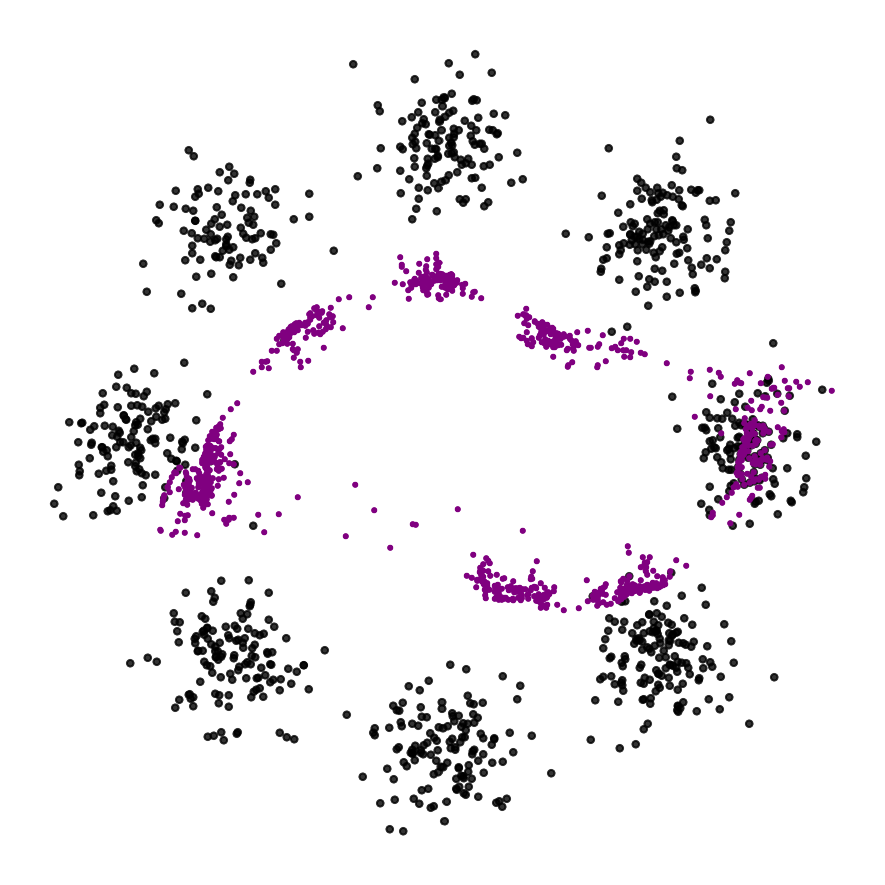}
      \caption{Steps: init (black) and first jump (purple) with sphere tracing.}
      \label{fig:dm_landscape_first}
    \end{subfigure}\hspace{0.6em}%
    \begin{subfigure}[t]{0.47\linewidth}
      \centering
      \includegraphics[width=\linewidth]{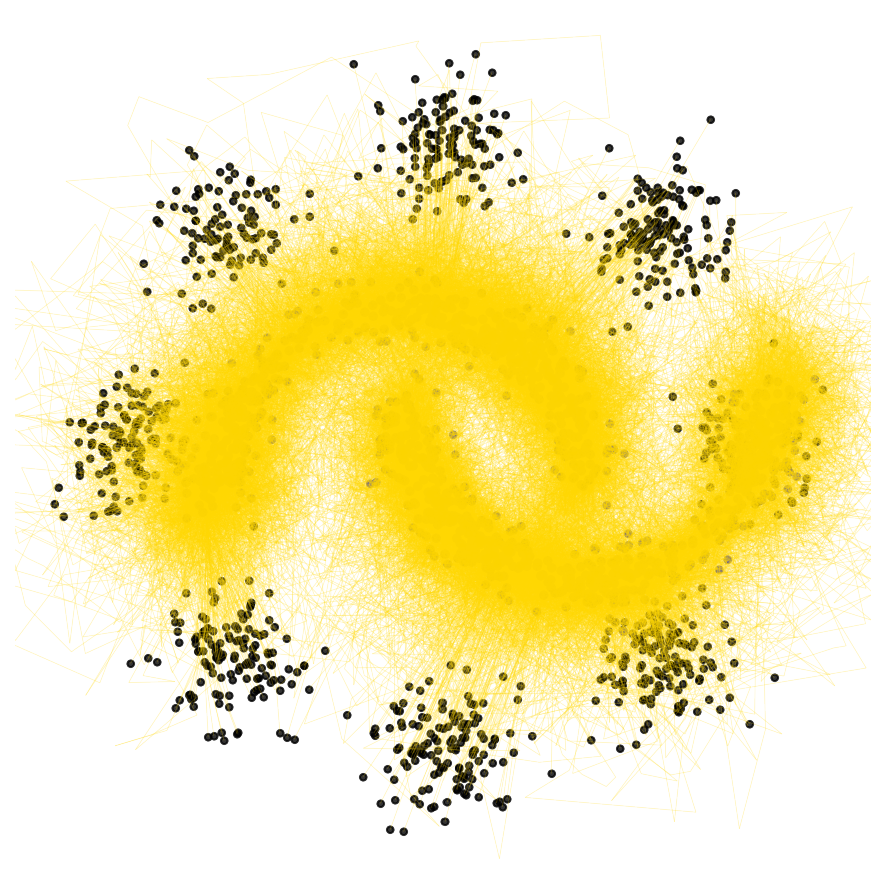}
      \caption{Step connections (yellow) showing trajectories.}
      \label{fig:dm_landscape_lines}
    \end{subfigure}
  \end{minipage}

  \caption{In our 2D toy model, we use a single sphere tracing step to jump to the target region; then Hamiltonian Monte Carlo explores low-potential areas.
  Final samples are presented in~\Cref{fig:2d_motivation}.
  We achieve the lowest Chamfer and Hausdorff distances in \Cref{tab:2d_metrics} with a meaningful energy landscape. 1{,}024 samples for visualizations.}
  \label{fig:dm_landscape}
  \vspace{-10pt}
\end{figure}

\subsection{8D Posterior Analysis with Flow Matching Losses}
\label{subsec:8d}

We use an 8D Gaussian mixture model to further visualize and validate our posterior analysis in~\Cref{sec:analysis}.
Importantly, our distance-field-based generation is deeply connected to geometric learning and support recovery~\cite{li2025scores, shen2026manifold}, focusing on pulling samples toward the closest projection of a compact manifold, and by design does not model the non-uniform density on the manifold.
The resulting distribution is either a uniform distribution on the manifold with long MCMC (\Cref{app:allocation}), or shaped by the initialization distribution (\Cref{app:pushforward} and~\Cref{app:ula}), which we view as a feature of the support-recovery perspective.

Although a Gaussian mixture model does not have compact support, it remains useful for our posterior analysis: it allows efficient importance sampling, and its high-probability region can be visualized as a union of ellipsoids.

Here, we use a mixture of six Gaussians: the initial distribution contains two components, and the target distribution contains four.
To reduce estimator variance, we adopt bidirectional importance sampling~\cite{burke2005bidirectional} to propose matchings between the initial and target components that pass through a given position.
In this context, bidirectional means sampling proposals from both endpoint distributions and reweighting them according to their contribution at the queried position.
With this sampling strategy, we observe empirical convergence in neighborhoods around the Gaussian components.
For visualization, we apply Fisher’s linear discriminant analysis (LDA) to project the 8D space to 2D (see~\Cref{fig:gmm}).
LDA yields a linear projection matrix, enabling a faithful 2D visualization of the projected geometry.
We keep the same projection matrix for all the following visualizations.

\begin{figure}[!htbp]
    \centering
    \includegraphics[width=0.4\linewidth]{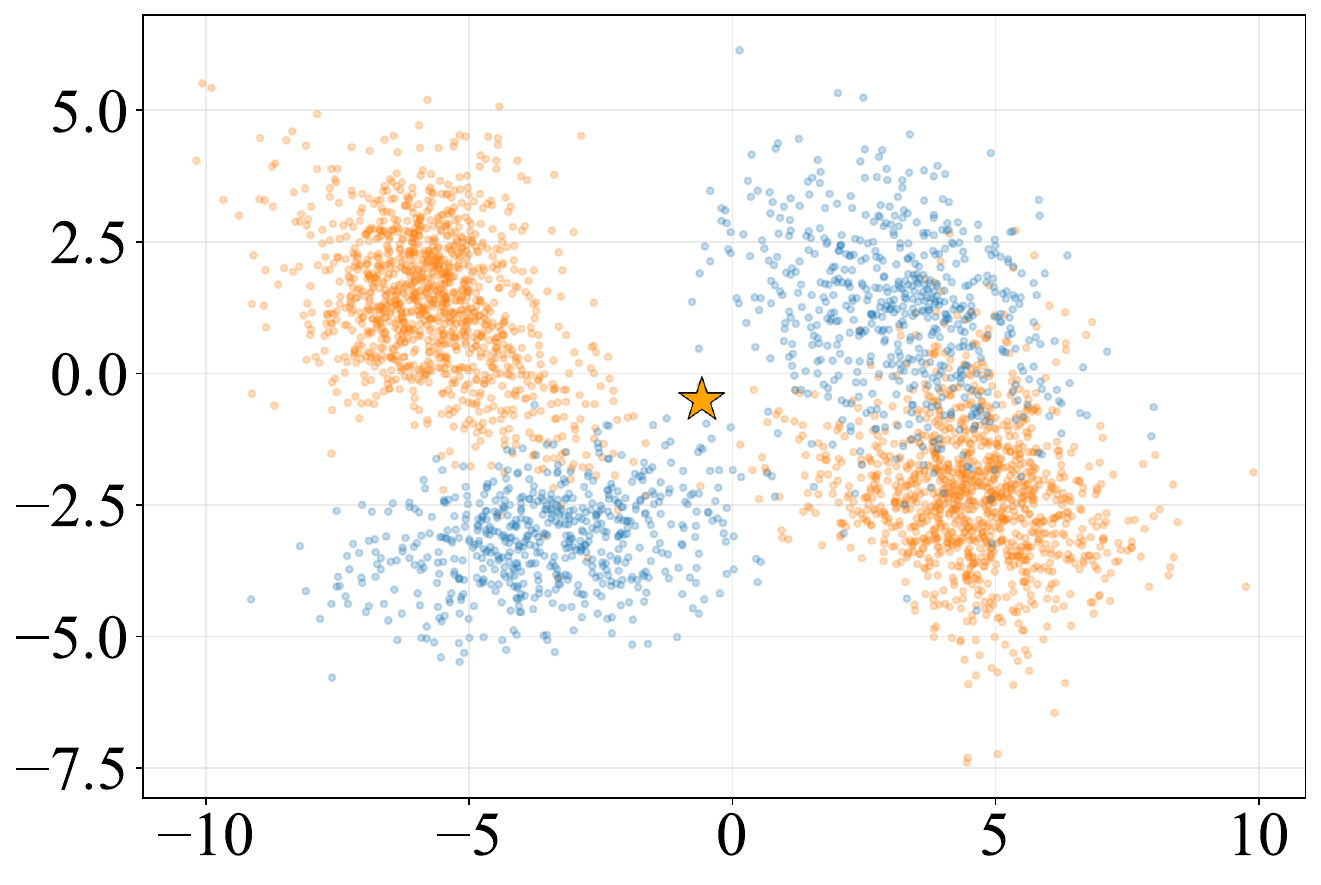}
    \caption{We use LDA to visualize the 8D initial distribution (\textcolor[HTML]{008FFF}{blue}) and the target distribution (\textcolor[HTML]{FA7F57}{orange}) in 2D.
    The data mean of the target distribution is marked using the orange star.}
    \label{fig:gmm}
\end{figure}

We first show that flow matching losses are biased toward the data mean, whereas our losses focus on closer points in~\Cref{fig:posterior_matches_2x2inline}.
Because the initial distribution has two components, we visualize the minimizers of each loss at the mean of each component.
These vectors minimize each loss under the posterior distribution conditioned on the test position.

\begin{figure}[!htbp]
    \centering

    \begin{subfigure}[t]{0.49\linewidth}
        \centering
        \includegraphics[width=0.57\linewidth]{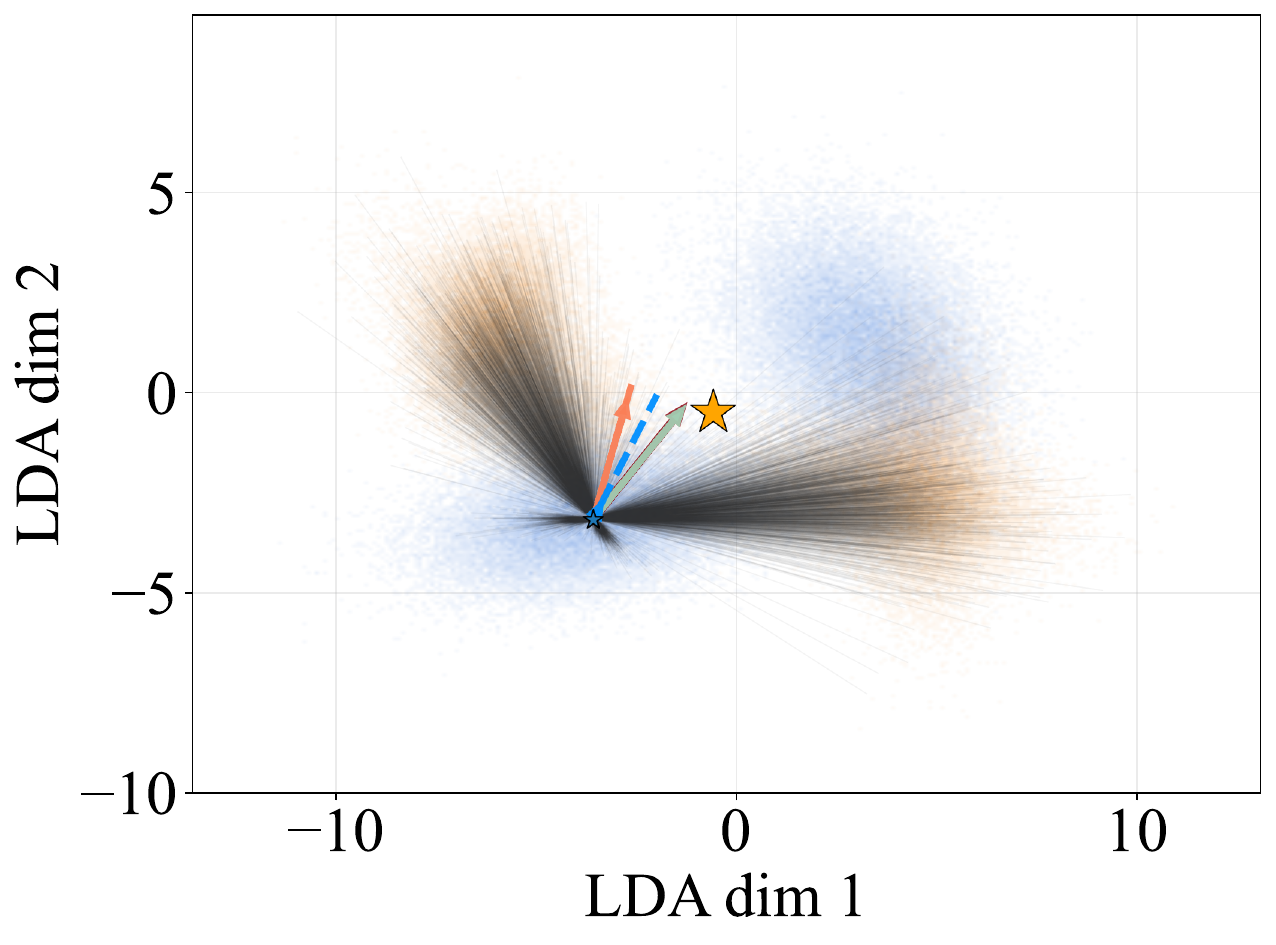}\hfill
        \includegraphics[width=0.43\linewidth]{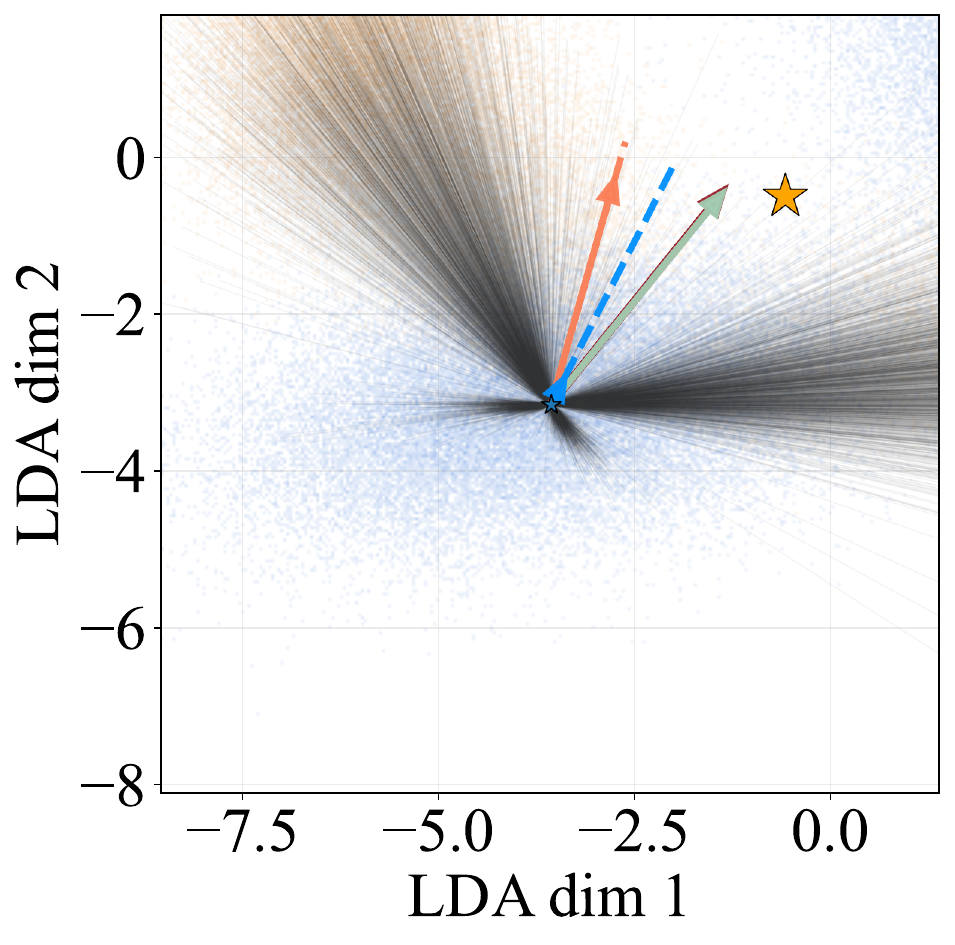}
        \label{fig:posterior_y4}
    \end{subfigure}
    \hfill
    \begin{subfigure}[t]{0.49\linewidth}
        \centering
        \includegraphics[width=0.57\linewidth]{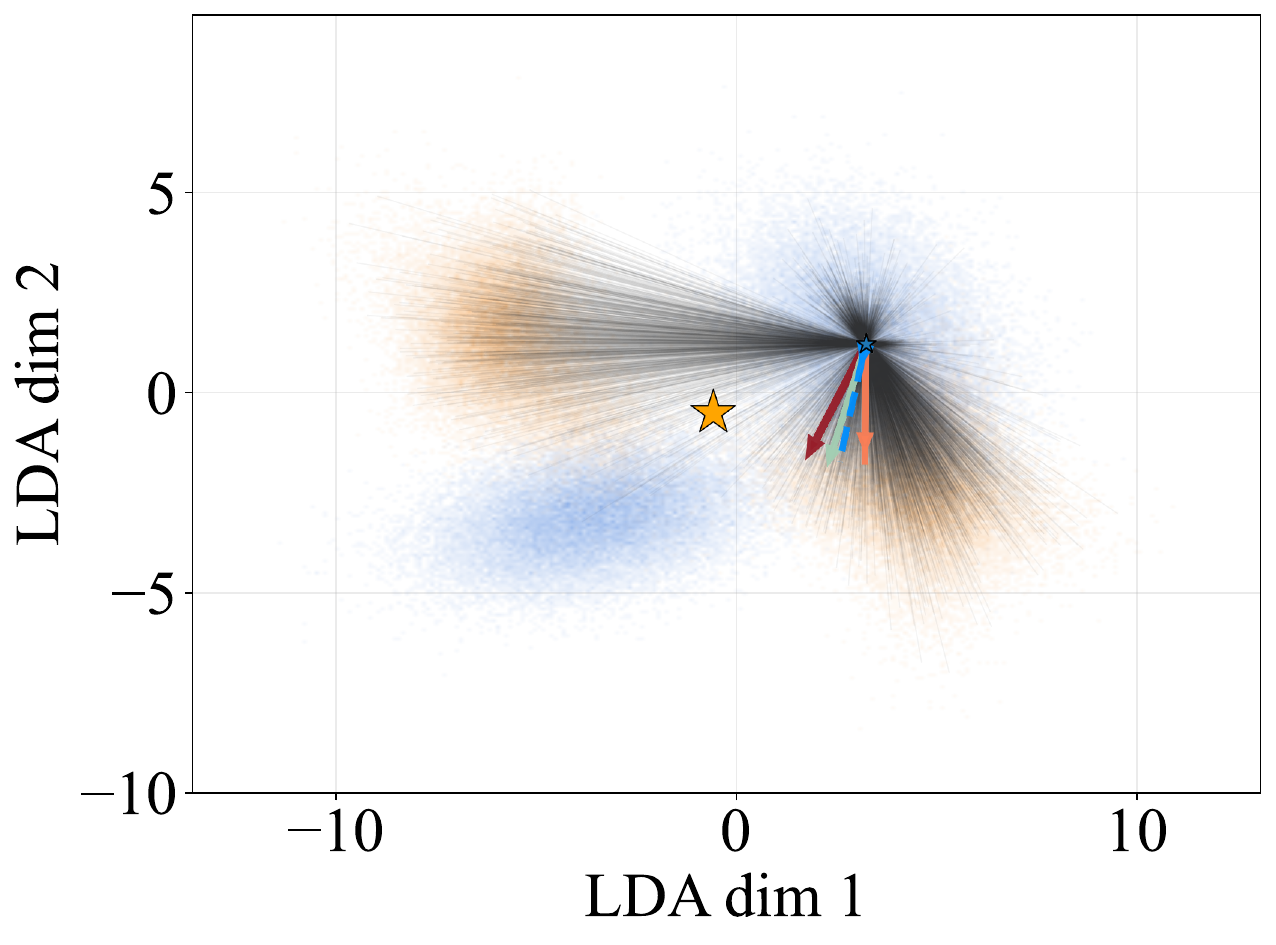}\hfill
        \includegraphics[width=0.43\linewidth]{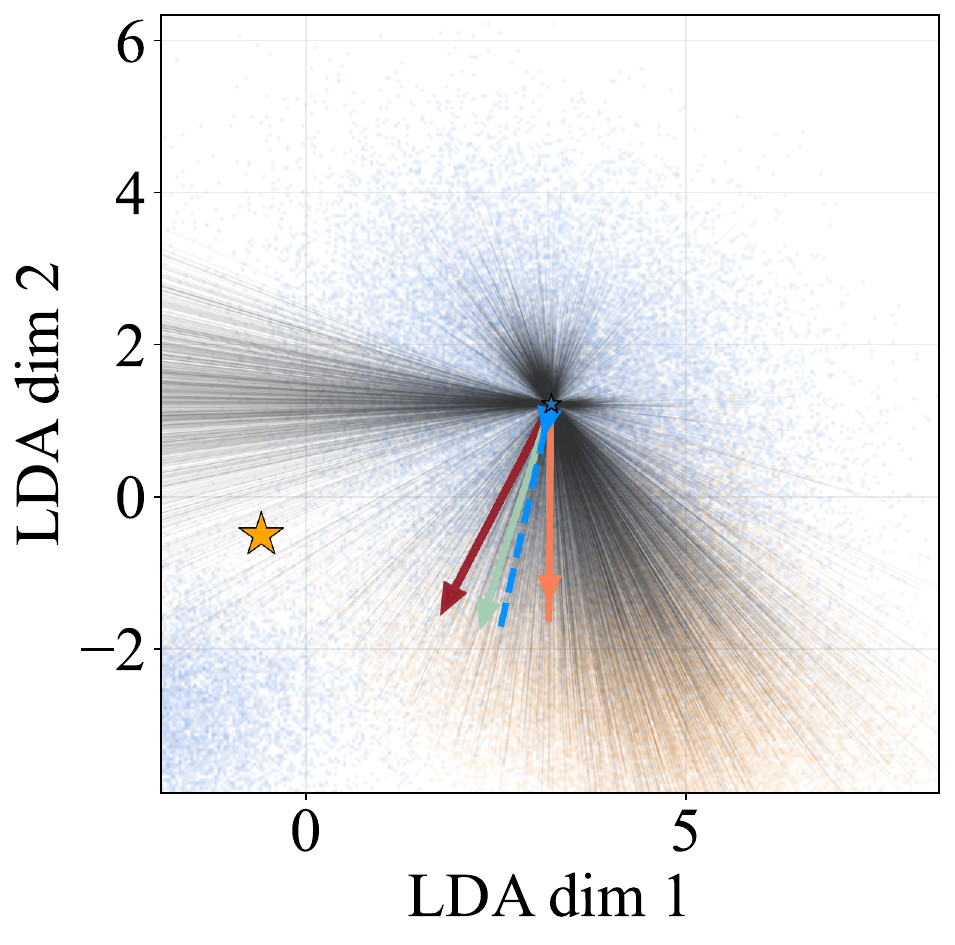}
        \label{fig:posterior_y5}
    \end{subfigure}
    \vspace{-10pt}
    \caption{
    Flow matching losses have minimizers (\textcolor[HTML]{A3CEB3}{green} reweighted, \textcolor[HTML]{971E2A}{red} original) that are strongly biased toward the data mean (\textcolor[HTML]{FA7F57}{orange} star). In contrast, one-step loss (\textcolor[HTML]{FA7F57}{orange}) and directional eikonal loss (\textcolor[HTML]{008FFF}{blue}) yield minimizers pointing to closer target data, less affected by the data mean.
    Global and zoom perspectives shown.
}
    \label{fig:posterior_matches_2x2inline}
\end{figure}

From~\Cref{fig:posterior_matches_2x2inline}, both the standard flow matching loss and the $(1-t)^{-2}$-reweighted flow matching loss yield minimizers strongly biased toward the data mean.
In contrast, the minimizers of the one-step loss and the directional eikonal loss point more closely toward the neighboring data cluster.

As shown in~\Cref{fig:compare_paths_grid}, due to the data-mean distortion, flow matching losses yield worse trajectories with higher curvature and slower improvement, as measured by the target p.d.f.

\begin{figure}[!htbp]
    \vspace{-10pt}
    \centering

    \begin{subfigure}[t]{0.36\linewidth}
        \centering
        \includegraphics[width=\linewidth]{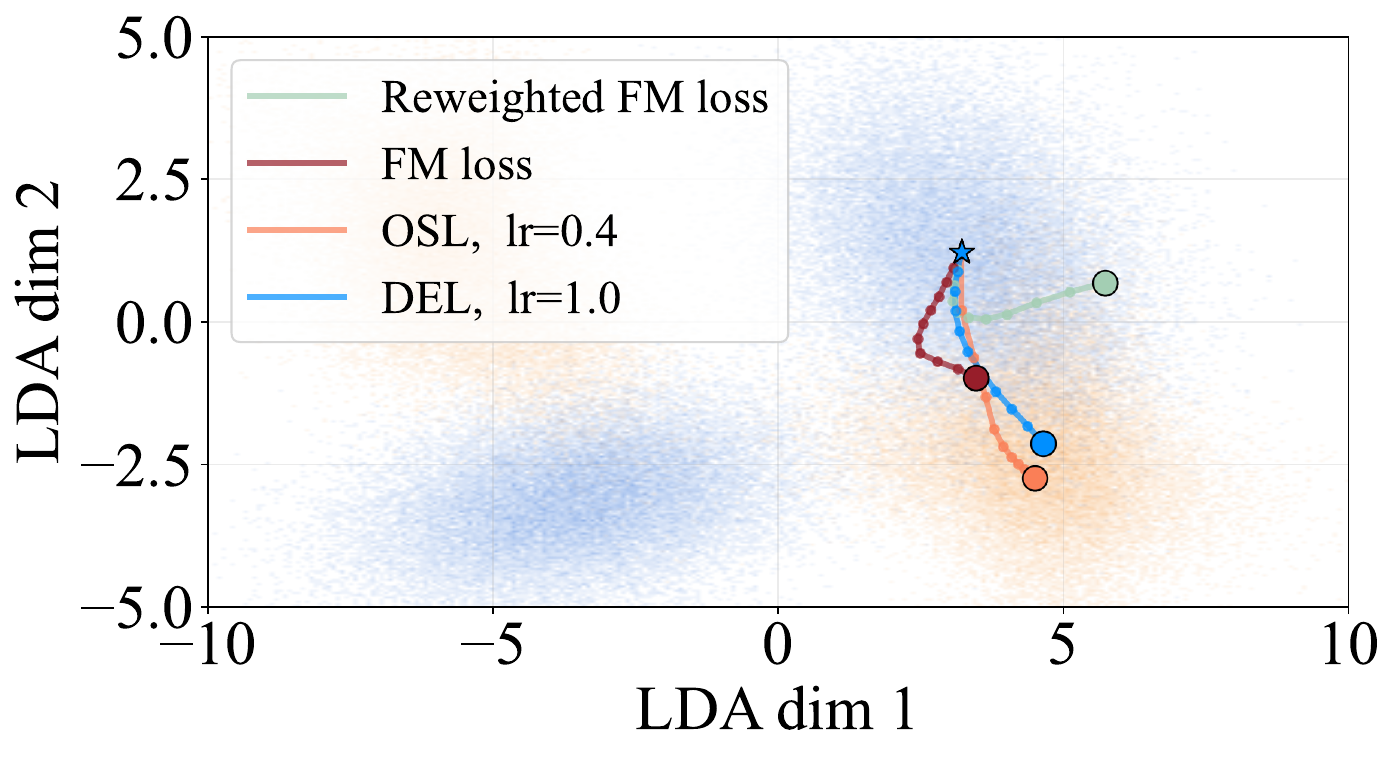}

        \vspace{2mm}

        \includegraphics[width=\linewidth]{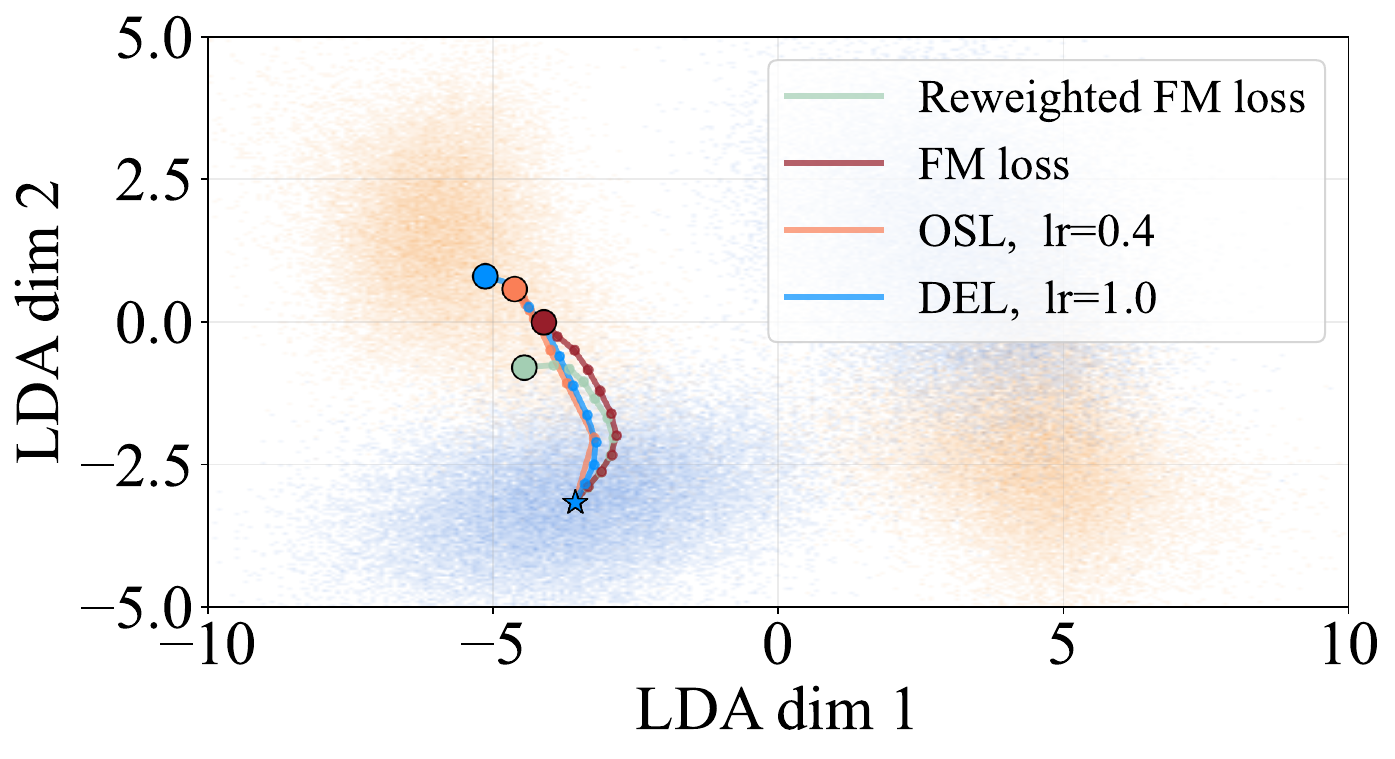}
        \caption{Trajectories from different minimizers. Mean step turning angles (deg) for the top/bottom trajectories in 8D are RFM: 19.59/15.90 and OSL: 8.59/9.36, showing straighter paths for OSL.}
        \label{fig:cmp_col_traj}
    \end{subfigure}
    \hfill
    \begin{subfigure}[t]{0.3\linewidth}
        \centering
        \includegraphics[width=\linewidth]{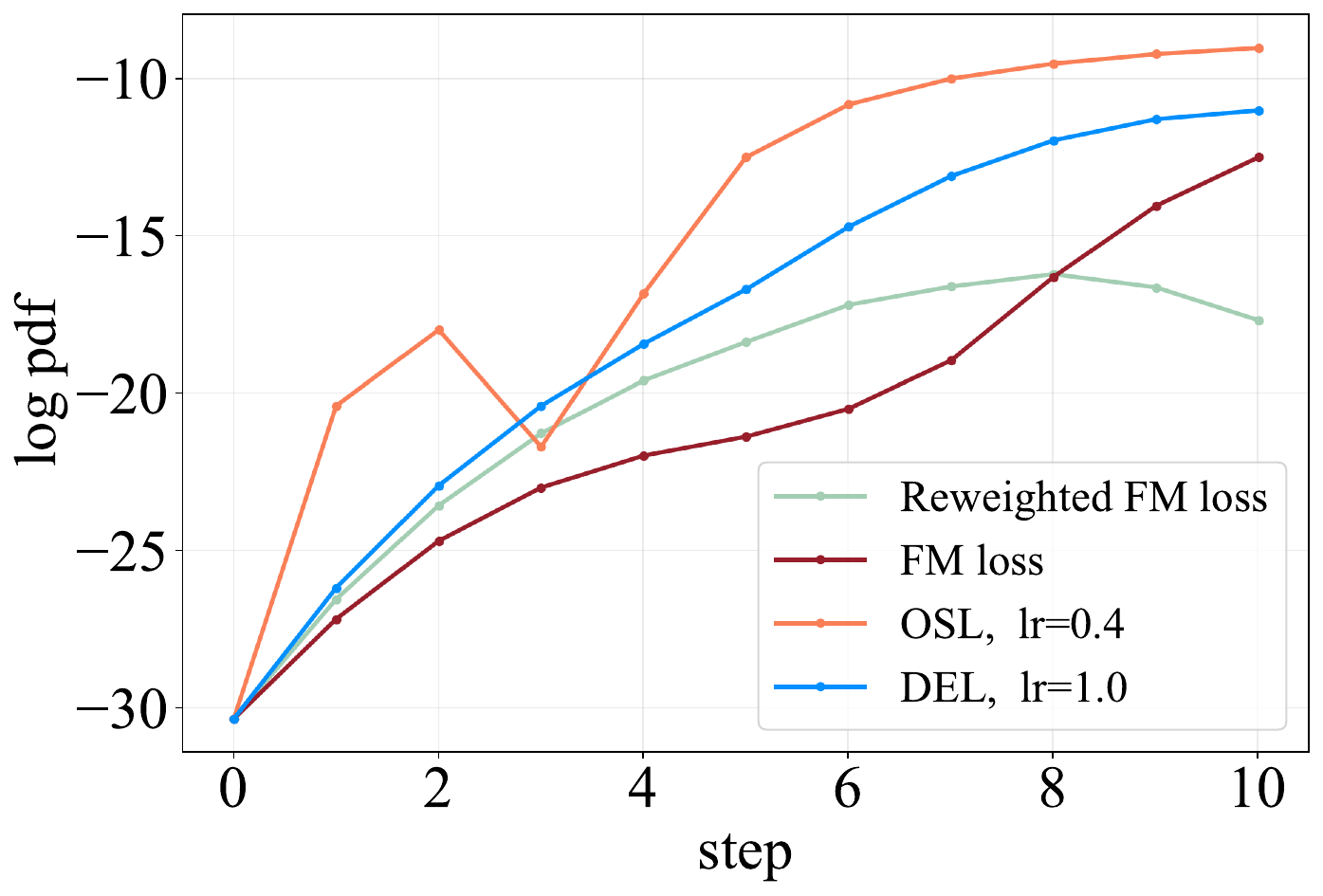}

        \vspace{2mm}

        \includegraphics[width=\linewidth]{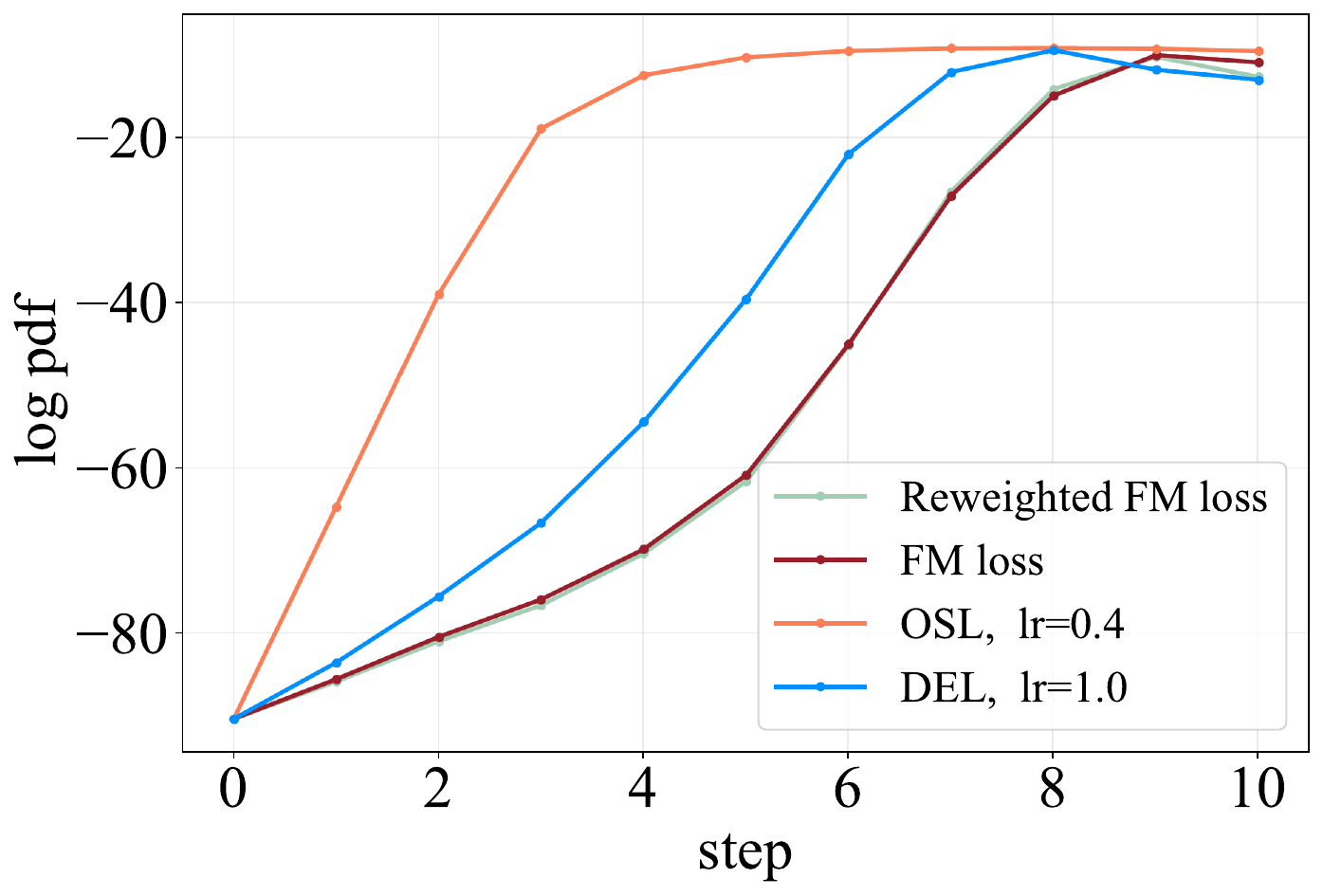}
        \caption{Target log-density (likelihood)$\uparrow$. The reweighted FM curve may be covered by the FM one.}
        \label{fig:cmp_col_logpdf}
    \end{subfigure}
    \hfill
    \begin{subfigure}[t]{0.3\linewidth}
        \centering
        \includegraphics[width=\linewidth]{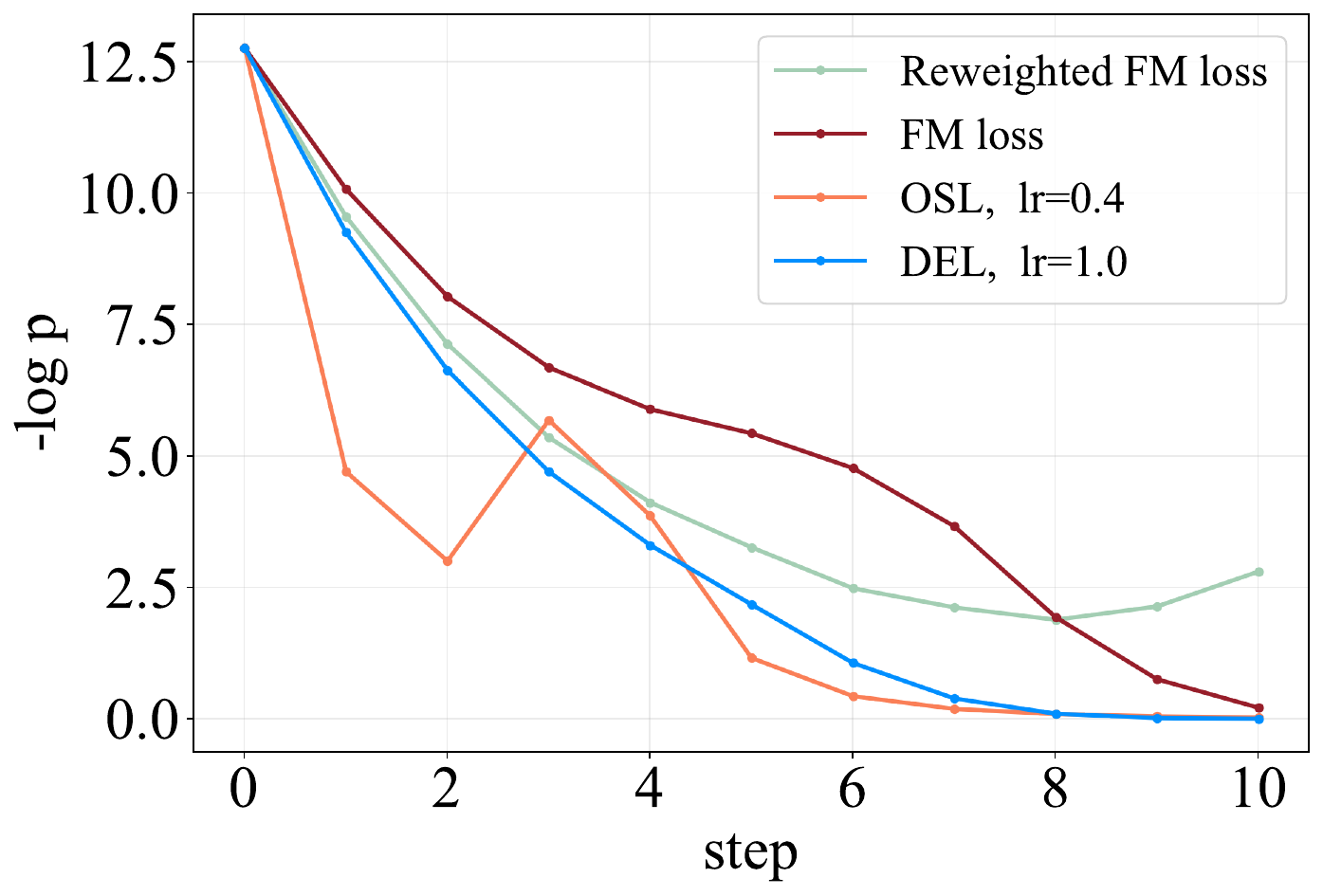}

        \vspace{2mm}

        \includegraphics[width=\linewidth]{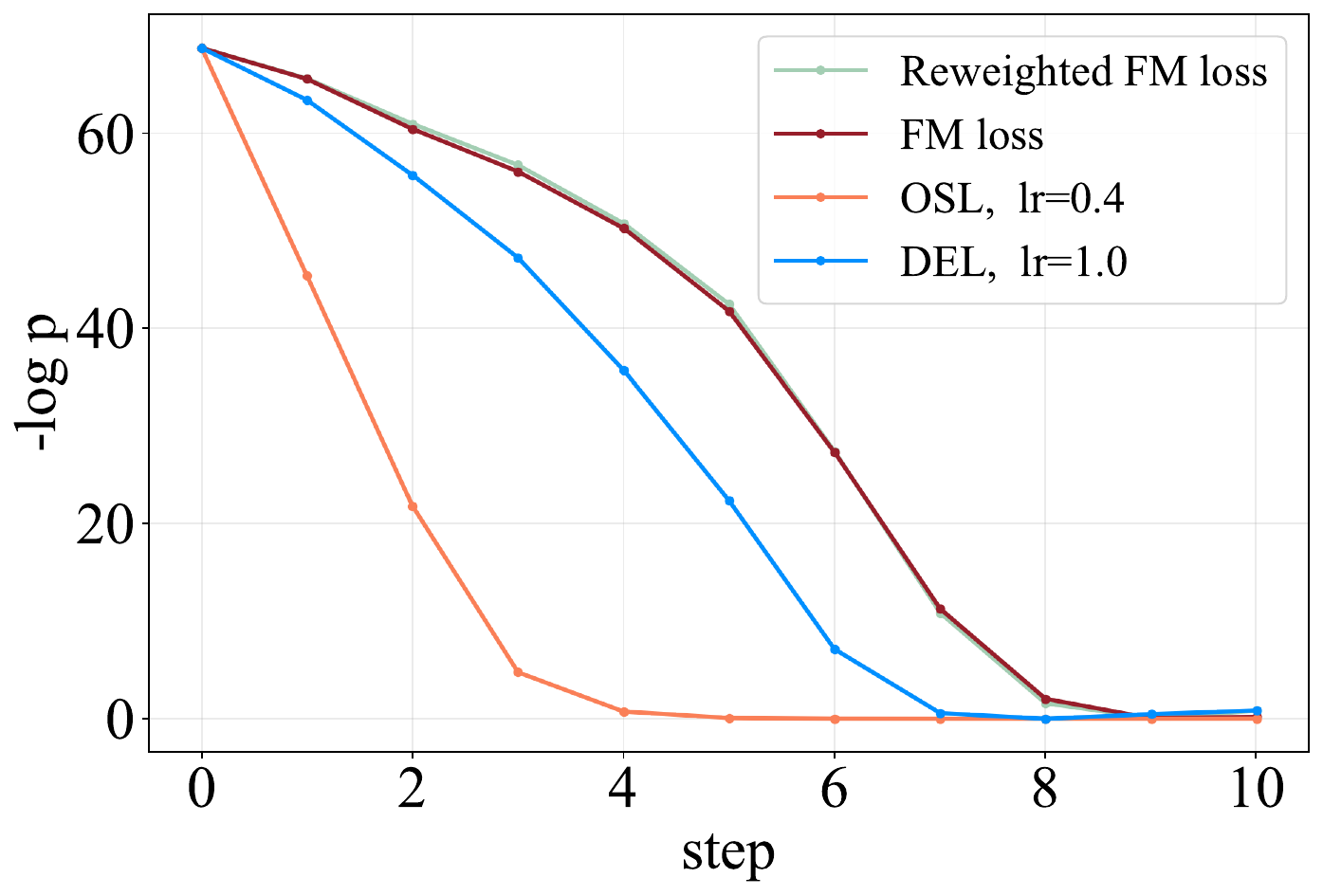}
        \caption{$-\log\!$ p-value (outlierness)$\downarrow$.}
        \label{fig:cmp_col_mahalpval}
    \end{subfigure}
    \vspace{-5pt}
    \caption{
    Flow matching produces more curved trajectories and yields slower improvement in target likelihood than our minimizers.
    We visualize target log-density along the path and report a Mahalanobis-based negative log p-value as a scale-robust outlierness score.
    }
    \label{fig:compare_paths_grid}
\end{figure}

From~\Cref{fig:compare_paths_grid}, our initial steps provide better directions, yielding larger likelihood increases and greater outlierness decreases.
We further show that the one-step loss provides the most stable guidance in~\Cref{fig:lr_sweep_orange_grid}, whereas other minimizers are sensitive to the learning rate/step size $\eta$ in~\Cref{fig:lr_sweep_y4_rgb}.

\begin{figure}[!htbp]
    \vspace{-10pt}
    \centering

    \begin{subfigure}[t]{0.36\linewidth}
        \centering
        \includegraphics[width=\linewidth]{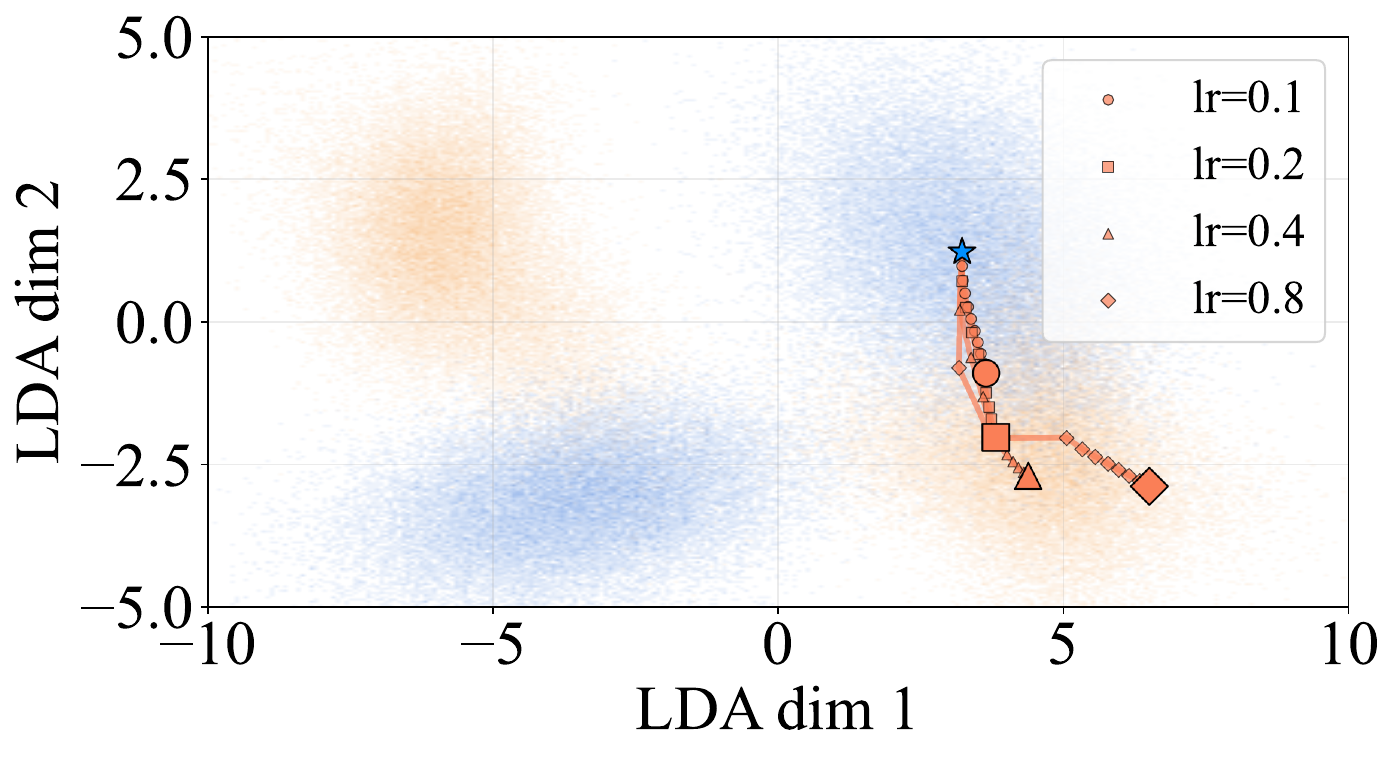}

        \vspace{2mm}

        \includegraphics[width=\linewidth]{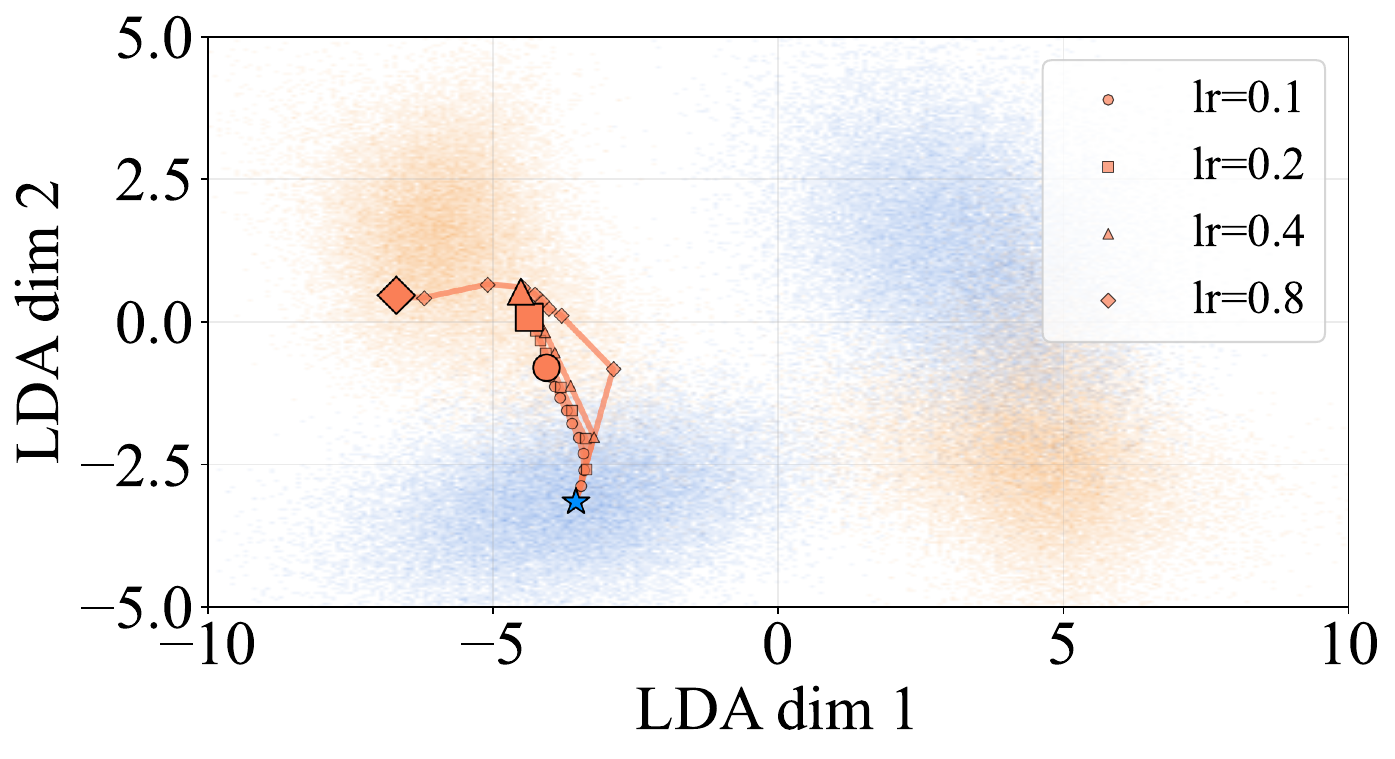}
        \caption{Trajectories under the directions of the OSL minimizer with different $\eta$.}
        \label{fig:lr_col_traj}
    \end{subfigure}
    \hfill
    \begin{subfigure}[t]{0.3\linewidth}
        \centering
        \includegraphics[width=\linewidth]{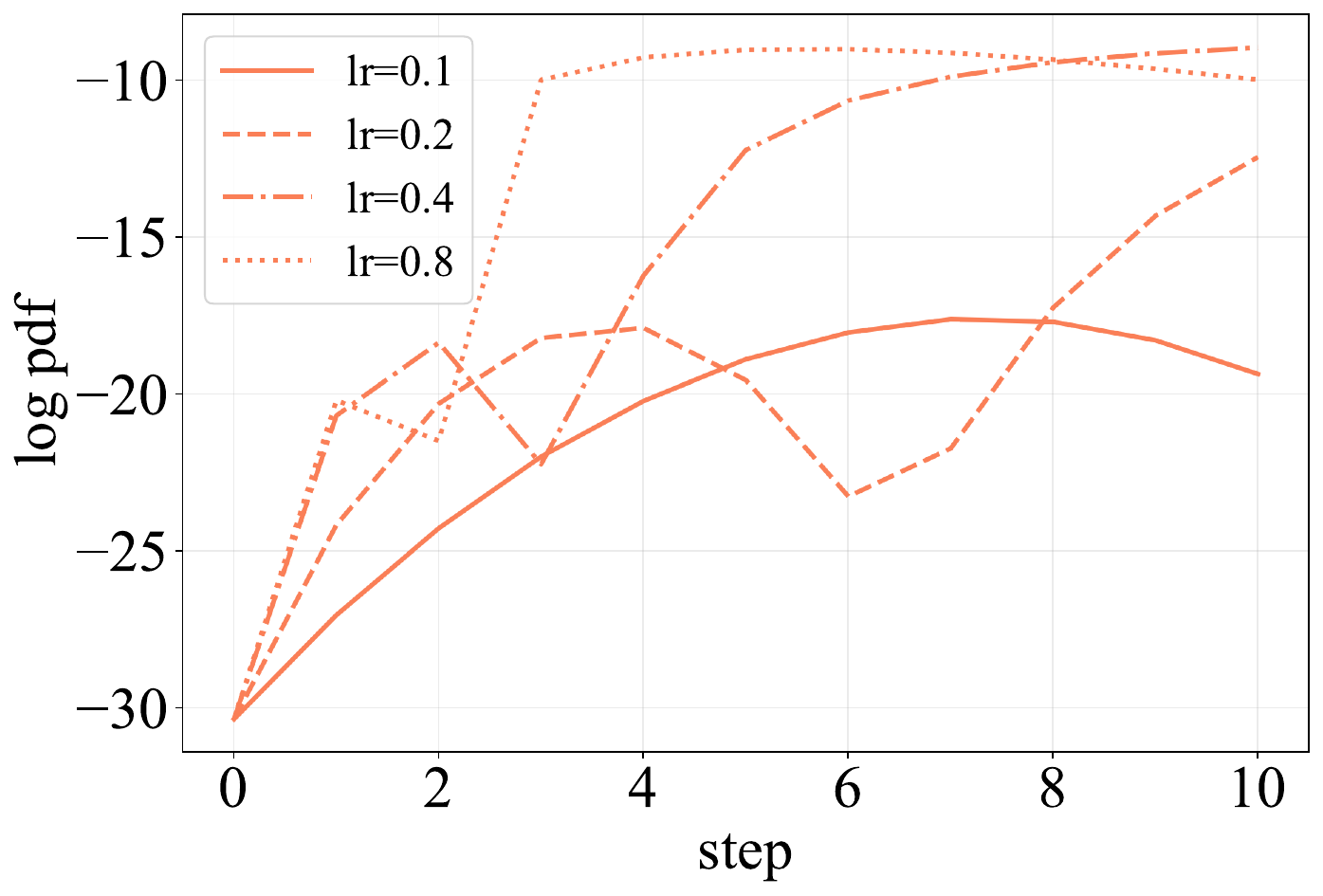}

        \vspace{2mm}

        \includegraphics[width=\linewidth]{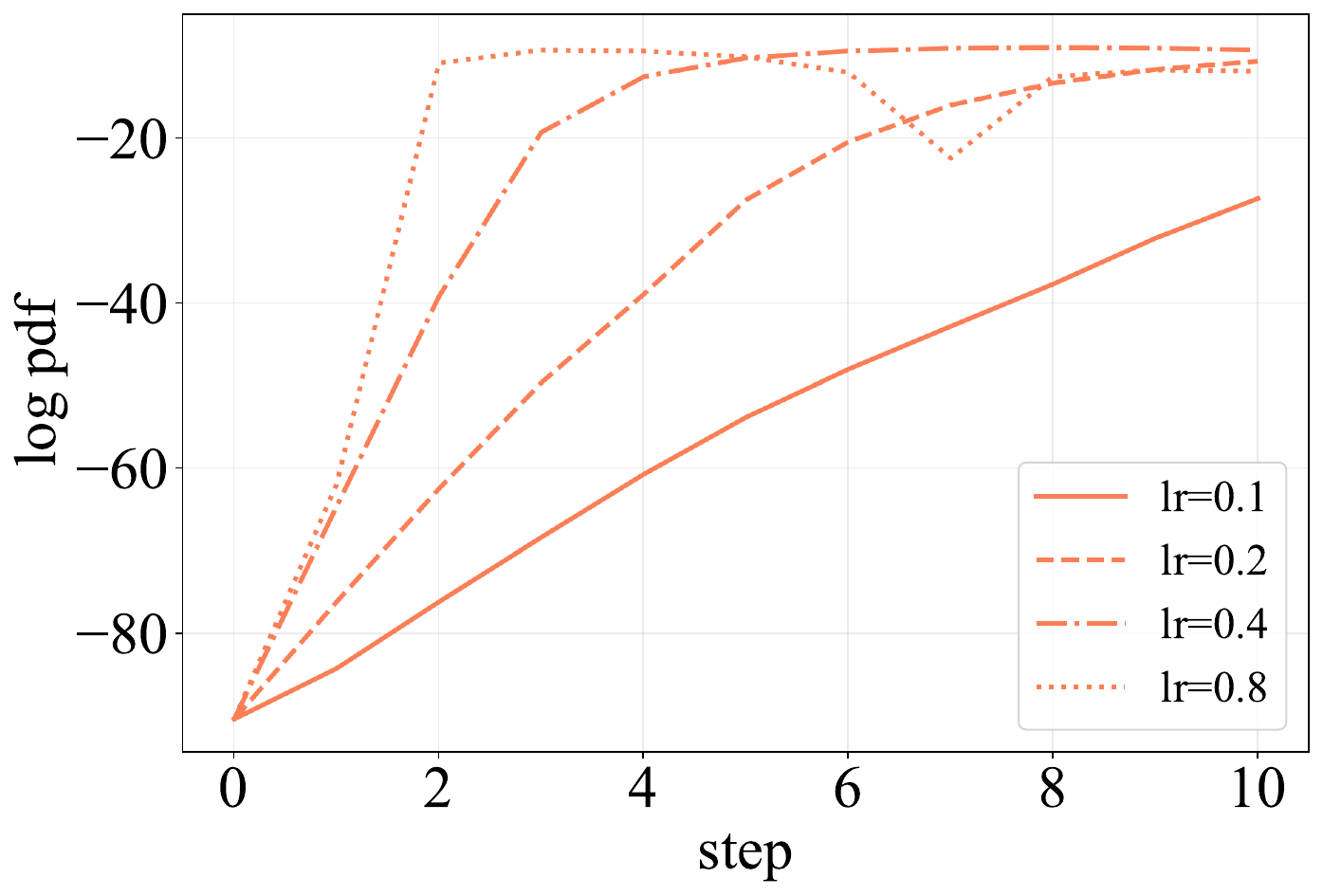}
        \caption{Target log-density (likelihood)$\uparrow$.}
        \label{fig:lr_col_logpdf}
    \end{subfigure}
    \hfill
    \begin{subfigure}[t]{0.3\linewidth}
        \centering
        \includegraphics[width=\linewidth]{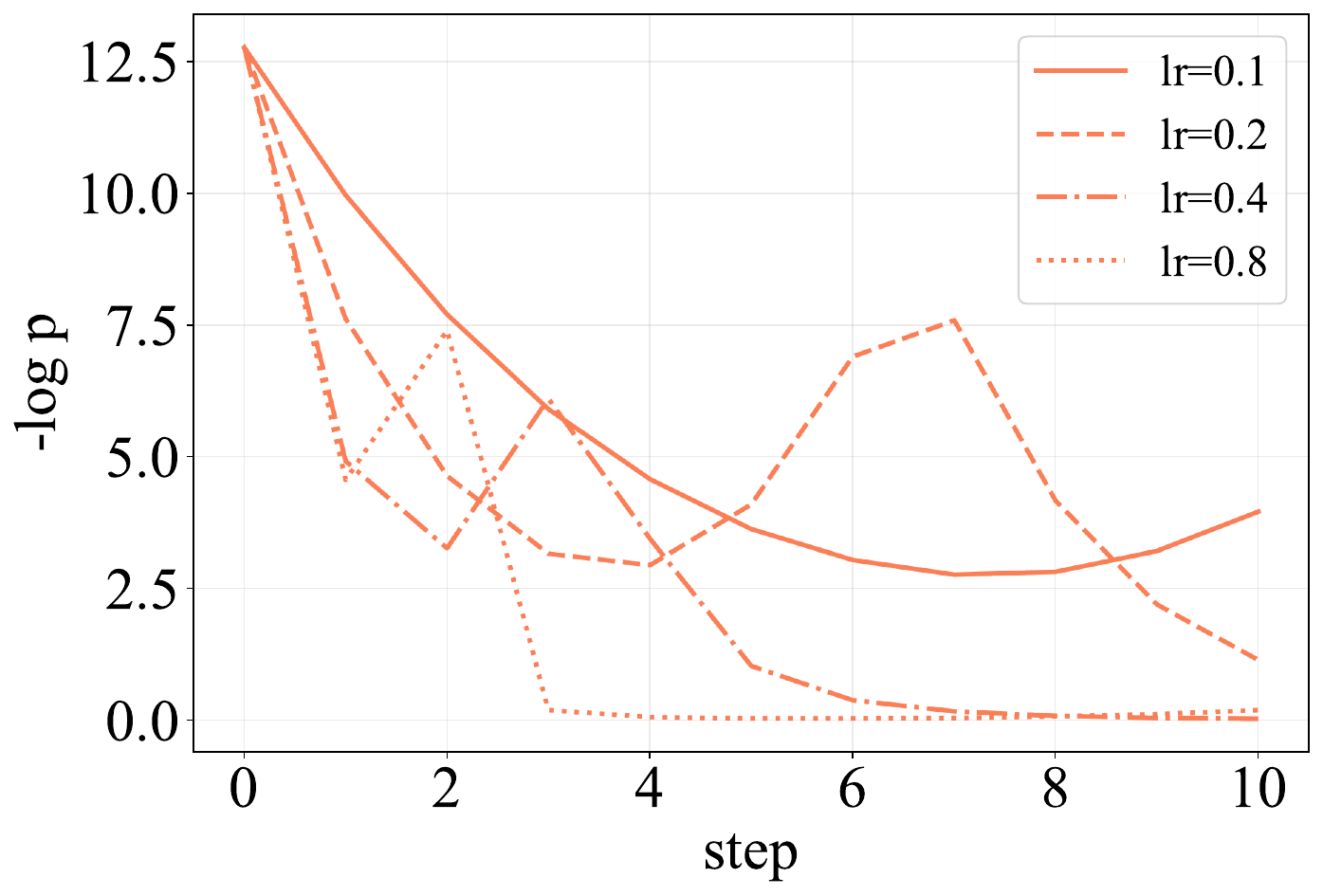}

        \vspace{2mm}

        \includegraphics[width=\linewidth]{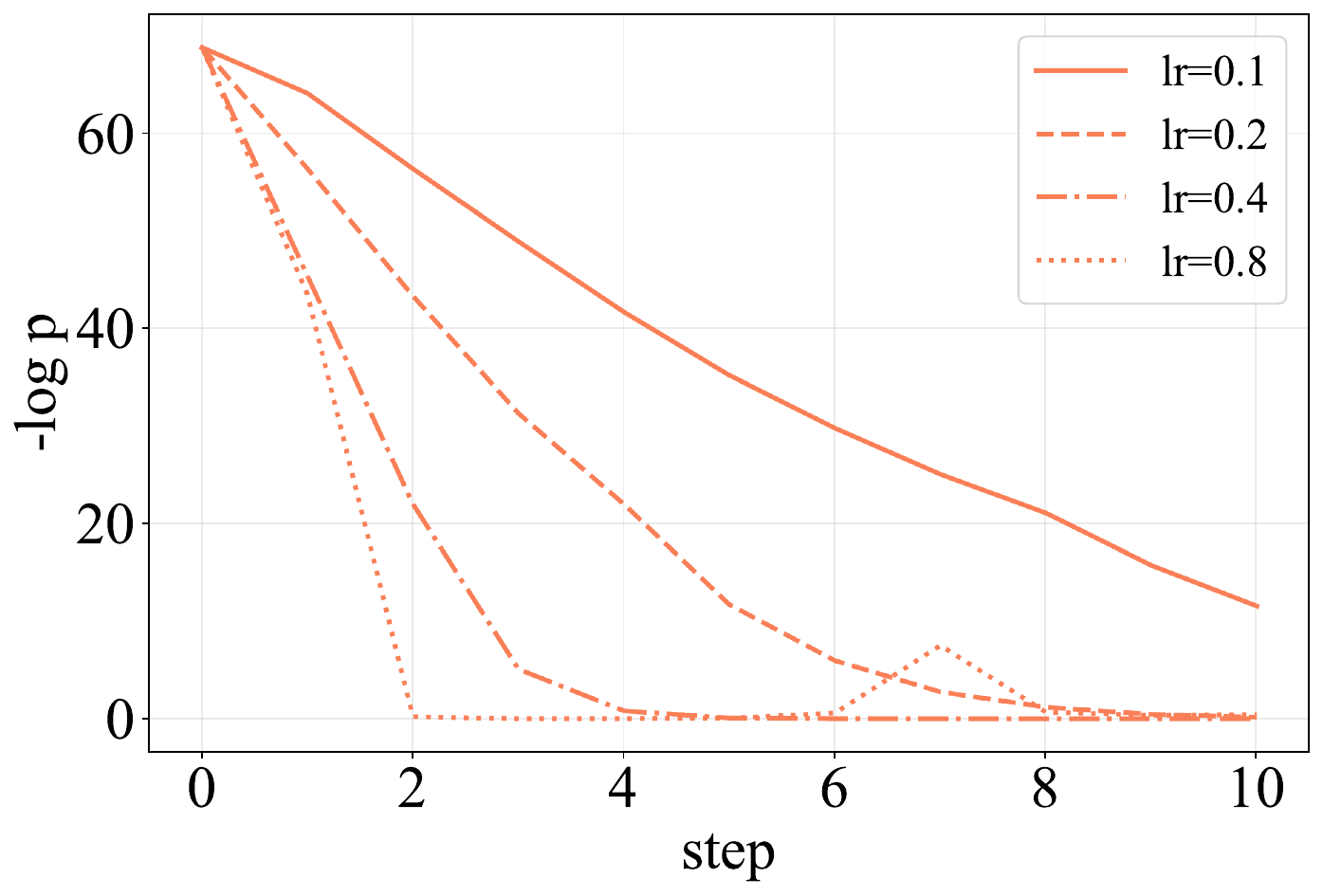}
        
        \caption{$-\log\!$ p-value (outlierness)$\downarrow$.}
        \label{fig:lr_col_mahalpval}
    \end{subfigure}
    \vspace{-5pt}
    \caption{Our one-step loss provides robust guidance across step sizes.
    It converges to near-zero outlierness over a $4\times$ range of step sizes.}
    \label{fig:lr_sweep_orange_grid}
\end{figure}

\begin{figure}[!htbp]
    \vspace{-10pt}
    \centering
    \begin{subfigure}[t]{0.32\linewidth}
        \centering
        \includegraphics[width=\linewidth]{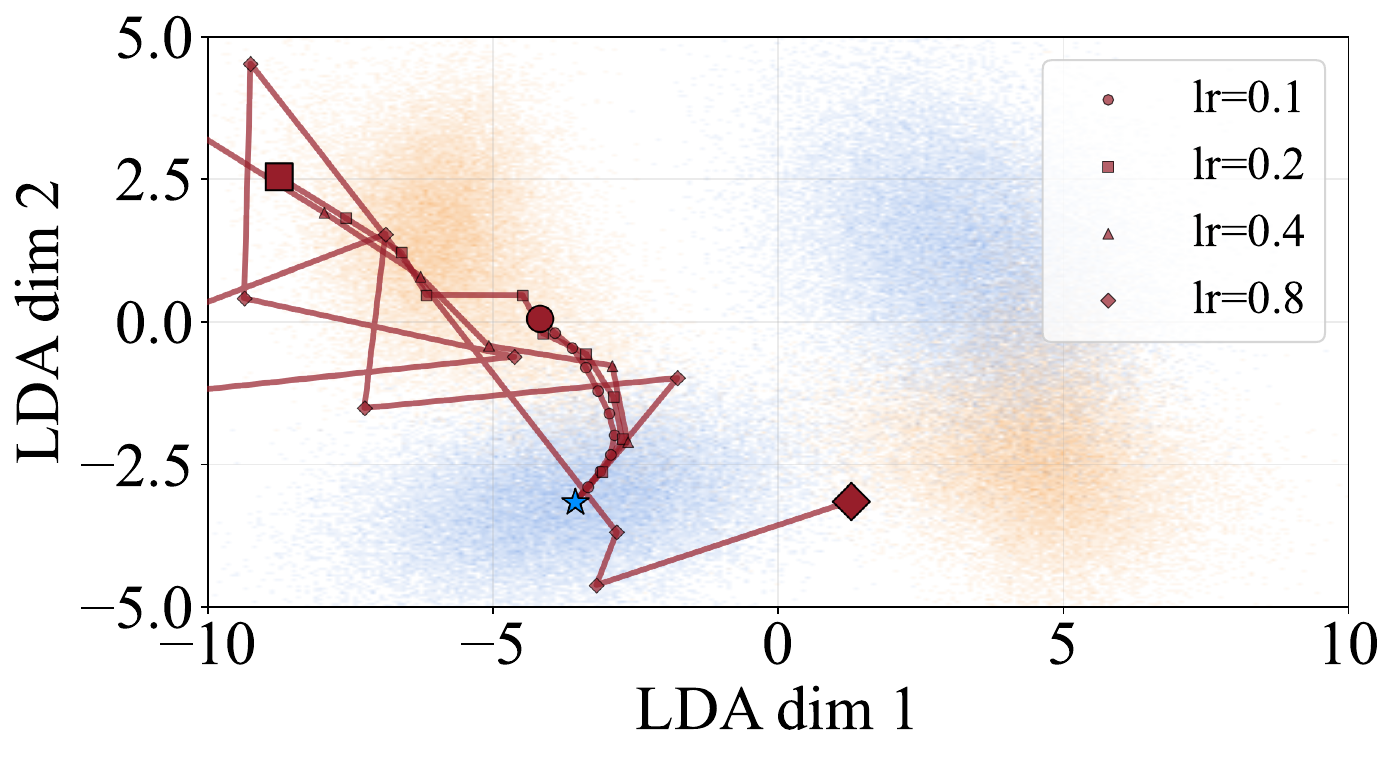}
        \label{fig:lr_y4_red}
    \end{subfigure}
    \hfill
    \begin{subfigure}[t]{0.32\linewidth}
        \centering
        \includegraphics[width=\linewidth]{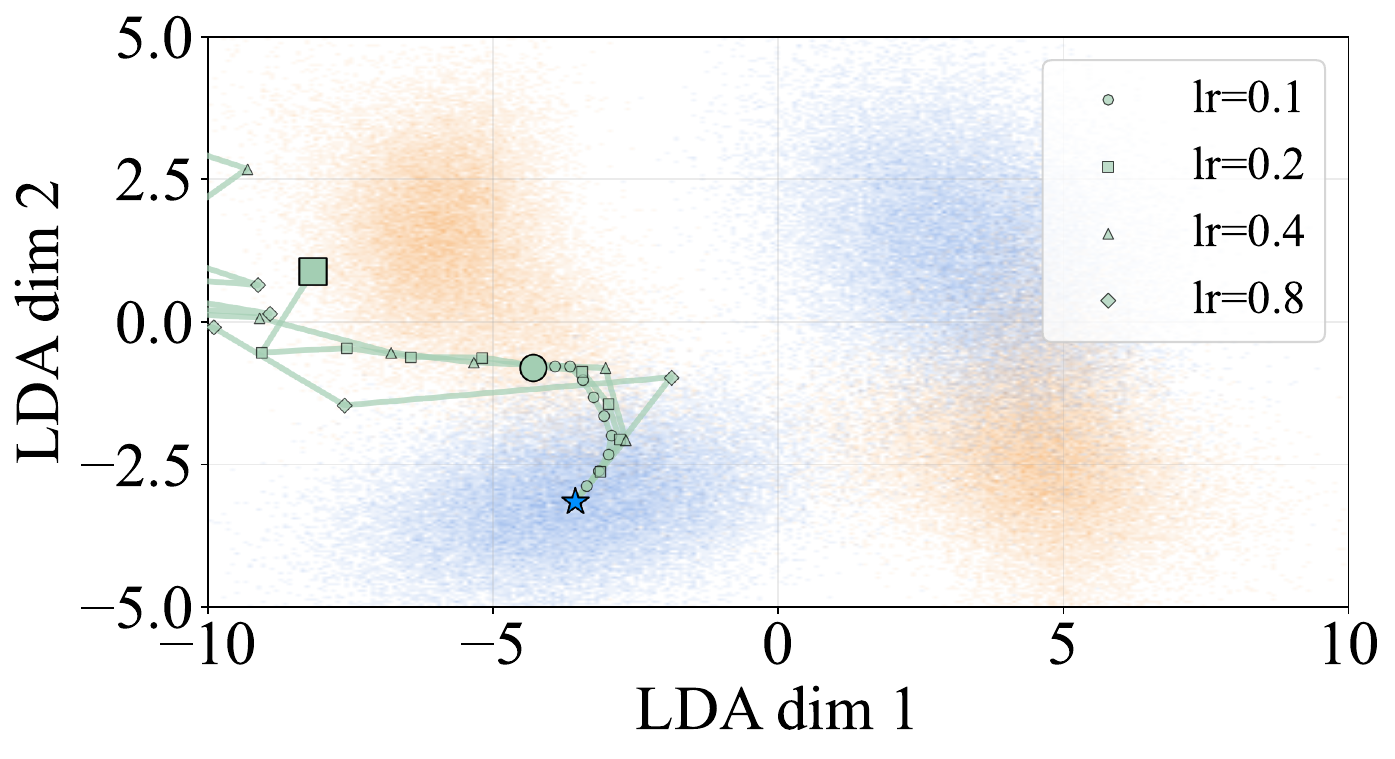}
        \label{fig:lr_y4_mint}
    \end{subfigure}
    \hfill
    \begin{subfigure}[t]{0.32\linewidth}
        \centering
        \includegraphics[width=\linewidth]{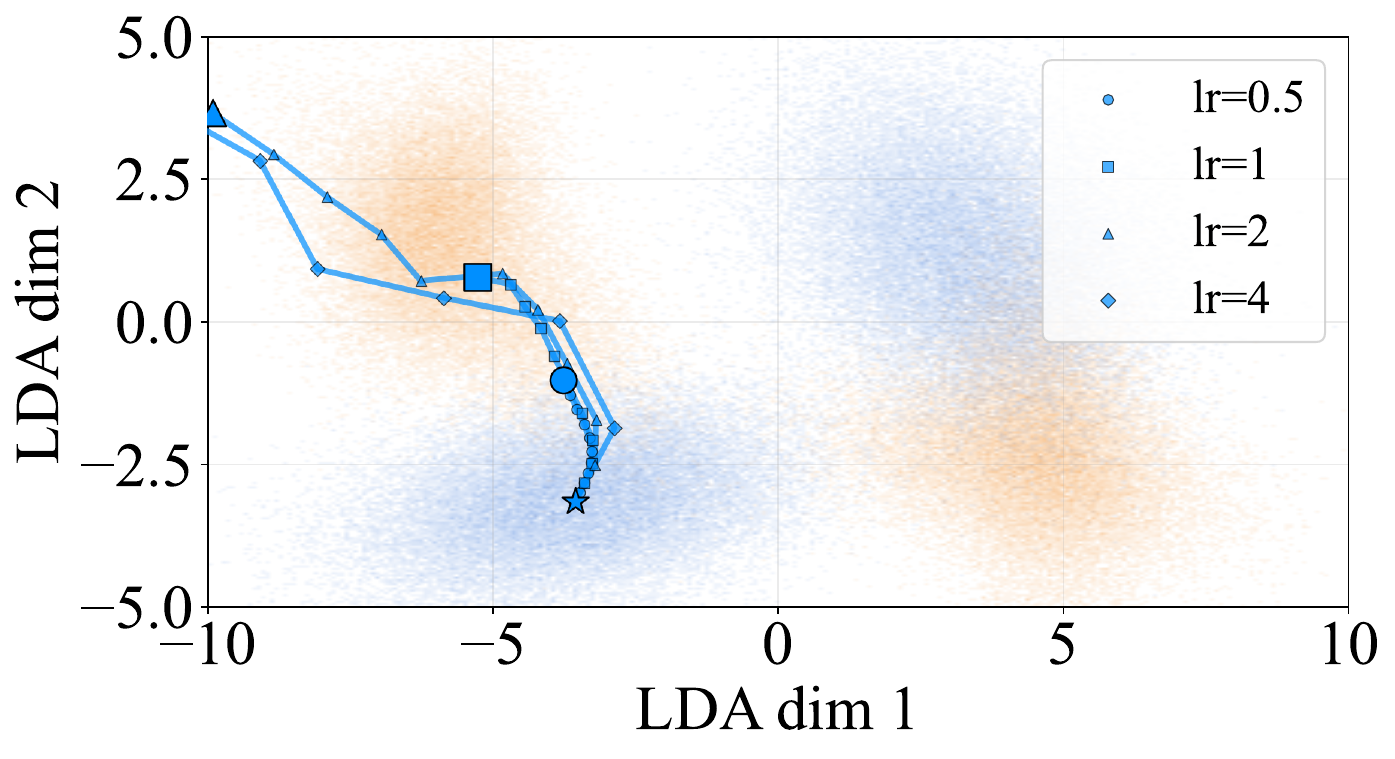}
        \label{fig:lr_y4_blue}
    \end{subfigure}
    \vspace{-20pt}
    \caption{None of the other losses provides reliable guidance over a $4\times$ range of step sizes, as OSL does. \textcolor[HTML]{971E2A}{Red}: FM loss; \textcolor[HTML]{A3CEB3}{Green}: $(1-t)^{-2}$-reweighted FM loss; \textcolor[HTML]{008FFF}{Blue}: directional eikonal loss.}

    \label{fig:lr_sweep_y4_rgb}
    \vspace{-10pt}
\end{figure}

With the step number fixed at 10, flow matching theory suggests using $\eta=1/10$.
In this 8D case, larger $\eta$ does not improve final performance, but severely perturbs the trajectories and causes divergence in~\Cref{fig:lr_sweep_y4_rgb}.

\subsection{Generated Samples}

\begin{figure}[!htbp]
    \centering
    \includegraphics[width=0.8\linewidth]{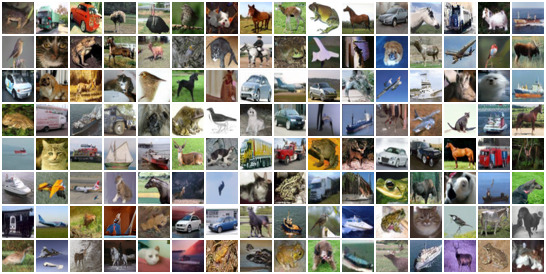}
    \vspace{-5pt}
    \caption{Uncurated samples for time- and class-unconditional generation on CIFAR-10.
\dm achieves the lowest FID among time-unconditional methods on CIFAR-10.}
    \label{fig:uncurated_cifar10}
\end{figure}

\newcommand{\curatedImgW}{0.12\linewidth}

\begin{figure}[!htbp]
\centering
\setlength{\tabcolsep}{0pt} %
\renewcommand{\arraystretch}{0} %

\subfloat[using 250-step gradient descent with $\eta=0.5$]{%
\begin{tabular}{@{}cccc@{}}
\includegraphics[width=\curatedImgW]{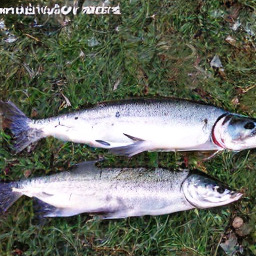} &
\includegraphics[width=\curatedImgW]{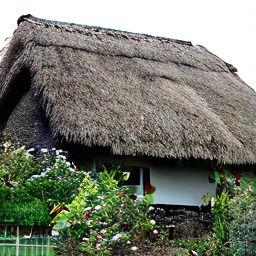} &
\includegraphics[width=\curatedImgW]{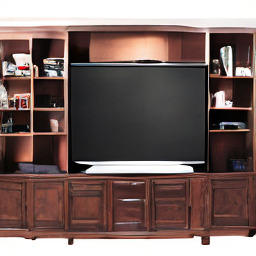} &
\includegraphics[width=\curatedImgW]{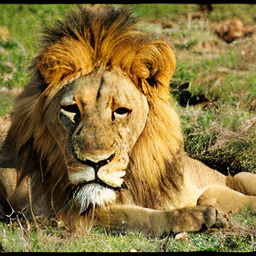} \\
\includegraphics[width=\curatedImgW]{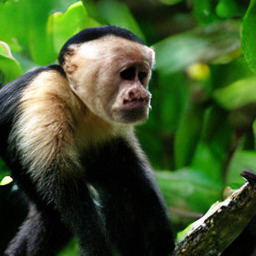} &
\includegraphics[width=\curatedImgW]{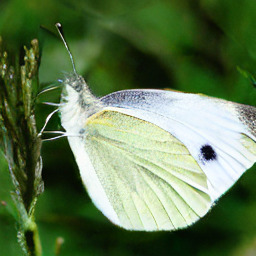} &
\includegraphics[width=\curatedImgW]{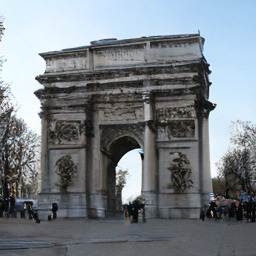} &
\includegraphics[width=\curatedImgW]{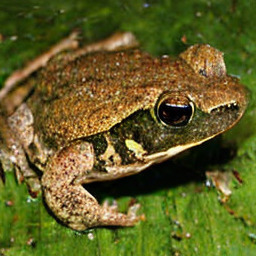} \\
\includegraphics[width=\curatedImgW]{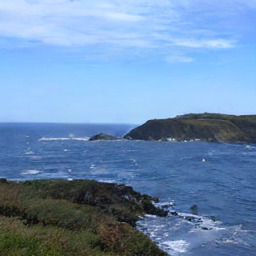} &
\includegraphics[width=\curatedImgW]{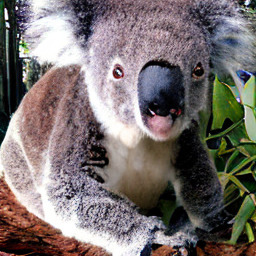} &
\includegraphics[width=\curatedImgW]{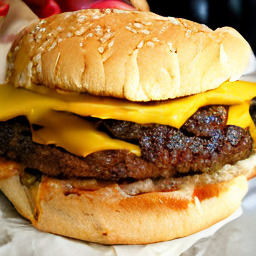} &
\includegraphics[width=\curatedImgW]{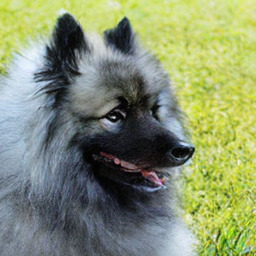} \\
\includegraphics[width=\curatedImgW]{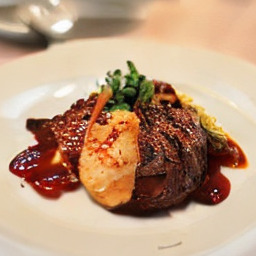} &
\includegraphics[width=\curatedImgW]{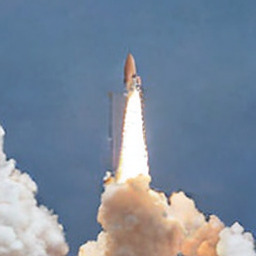} &
\includegraphics[width=\curatedImgW]{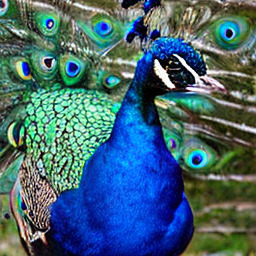} &
\includegraphics[width=\curatedImgW]{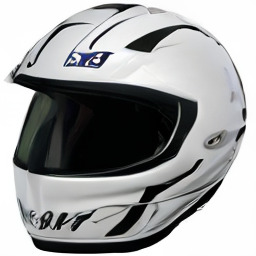} \\
\end{tabular}
}\hspace{0.01\linewidth}
\subfloat[using 250-step sphere tracing with $\eta=0.015$]{%
\begin{tabular}{@{}cccc@{}}
\includegraphics[width=\curatedImgW]{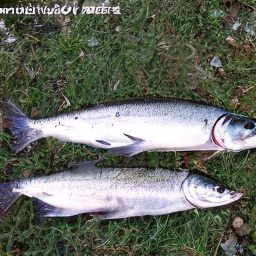} &
\includegraphics[width=\curatedImgW]{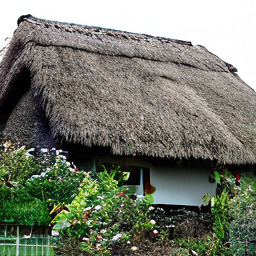} &
\includegraphics[width=\curatedImgW]{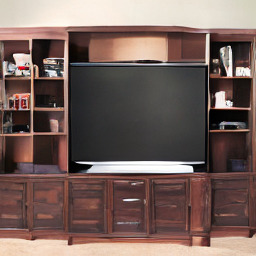} &
\includegraphics[width=\curatedImgW]{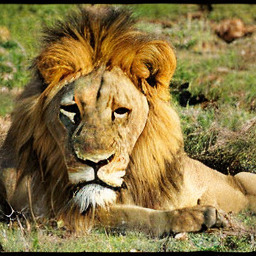} \\
\includegraphics[width=\curatedImgW]{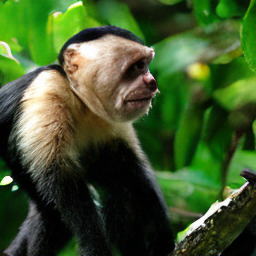} &
\includegraphics[width=\curatedImgW]{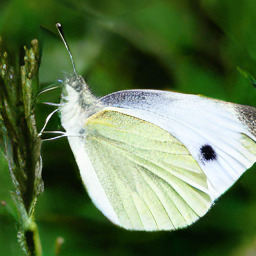} &
\includegraphics[width=\curatedImgW]{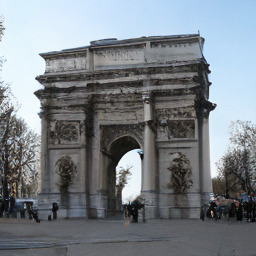} &
\includegraphics[width=\curatedImgW]{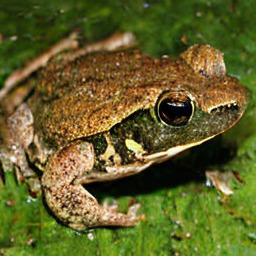} \\
\includegraphics[width=\curatedImgW]{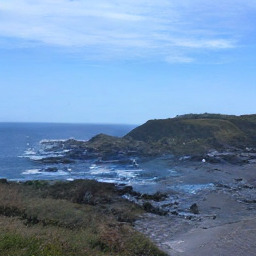} &
\includegraphics[width=\curatedImgW]{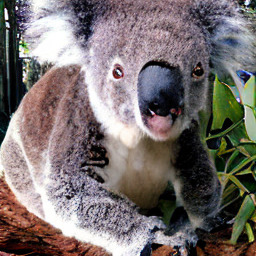} &
\includegraphics[width=\curatedImgW]{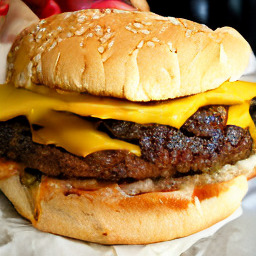} &
\includegraphics[width=\curatedImgW]{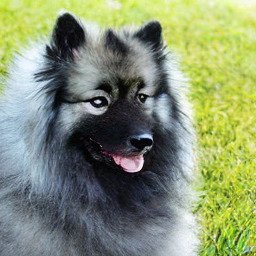} \\
\includegraphics[width=\curatedImgW]{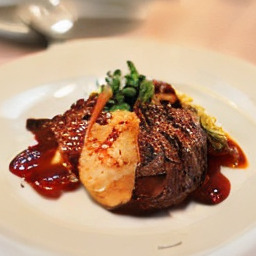} &
\includegraphics[width=\curatedImgW]{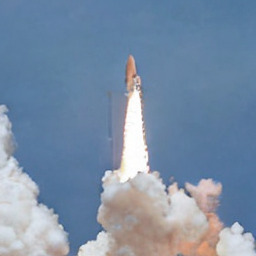} &
\includegraphics[width=\curatedImgW]{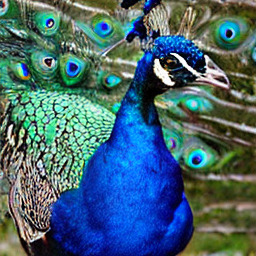} &
\includegraphics[width=\curatedImgW]{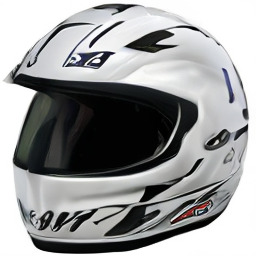} \\
\end{tabular}
}
\vspace{-5pt}
\caption{Curated ImageNet samples generated by our XL model (patch size $2$) with CFG scale $4.0$.
We apply the CFG formula~\cite{ho2022classifier} on $\mathbf{v}_\theta$ in gradient descent and $u_\theta \mathbf{v}_\theta$ in sphere tracing.}
\label{fig:curated_imagenet_grid}
\end{figure}
\vspace{-2pt}

\subsection{Generation Process Comparison}
\label{app:gpc}

\begin{figure}[!htbp]
    \centering
    \includegraphics[width=0.5\linewidth]{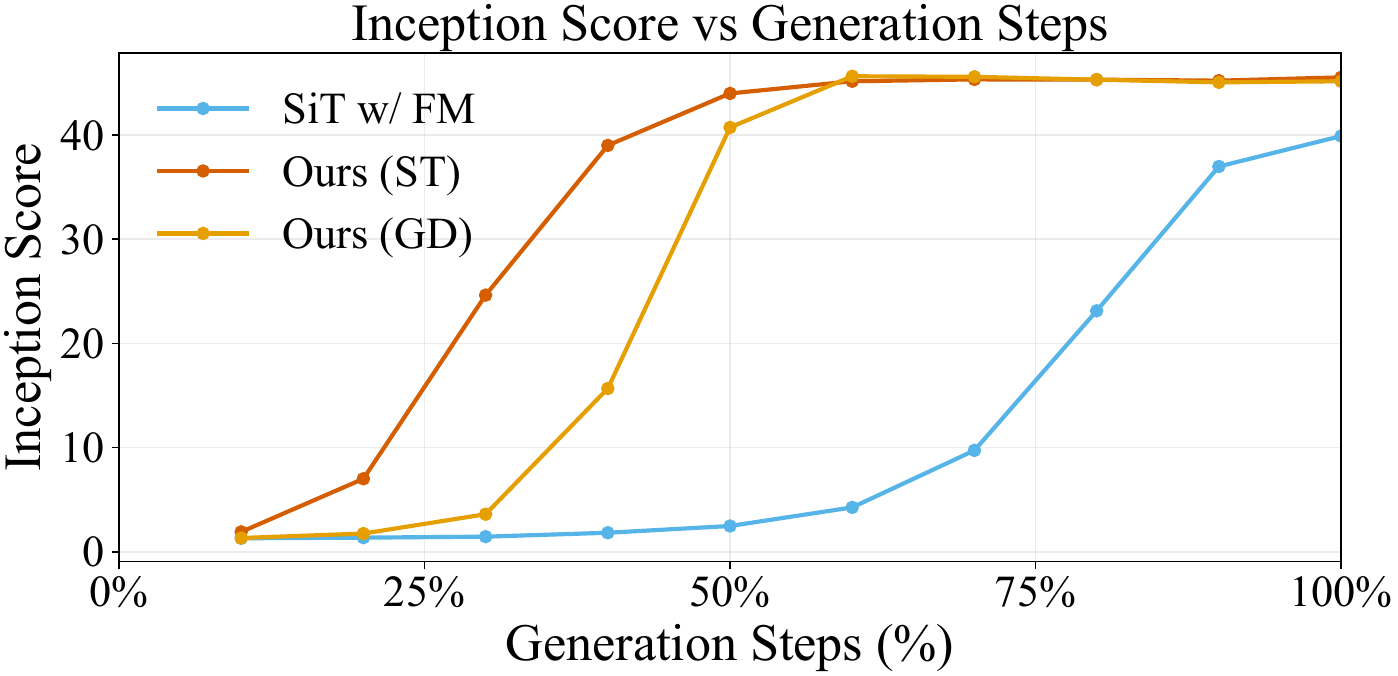}
    \caption{We run the original SiT model~\cite{ma2024sit} with a flow matching setting using their recommended ODE solver dopri5 for 250 steps, matching our step number, and report the curves of inception score as a supplement of~\Cref{fig:faster}.}
    \label{fig:is_vs_steps}
    \vspace{-10pt}
\end{figure}
\vspace{-2pt}

During inference, we observe that our method attains perceptually better samples in fewer steps, whereas flow matching tends to meander in the high-noise regime at the beginning.
More concretely, flow matching can exhibit a non-monotonic FID in the early stage, even though the inception score (IS) keeps increasing (see~\Cref{fig:is_vs_steps}) and coarse structures gradually emerge.
Our method exceeds flow matching's final IS using only 50\% of the sampling steps, converging even faster when measured by IS than by FID.

Similar early-stage drifting behavior has also been noted in prior work. \citet{karras2022elucidating} Fig.~1 and \citet{gagneux2025generation} Fig.~7 report that samples are shifted toward the data mean at the beginning of generation.
Energy Matching~\cite{balcerak2025energy} also implies this meandering behavior: every step uses a $1/100$ scaled full denoising vector, yet achieves the best performance with 325 steps rather than 100 steps.

We visualize intermediate generation slices in~\Cref{fig:evolution} at every $10\%$ progress.
Our analysis in~\Cref{sec:analysis} shows that closest-target rematching can induce mode collapse: fully resolving target ambiguity tends to sacrifice the diversity of targets.
In our final approach, inverse-distance reweighting preserves some ambiguity, so denoising cannot be reliably done in a single step and instead benefits from multiple refinement steps.
We conjecture that our model could also inspire few-step generation.
This is because many distillation models still rely on a strong multi-step teacher, and stop-gradient supervision pipelines may also benefit from faster fine-grained trajectories~\Cref{fig:paths}.

\subsection{Unadjusted Langevin Refinement for Image Generation}
\label{app:ula}

As discussed in~\Cref{app:allocation} and~\Cref{app:pushforward}, in the idealized setting with sufficiently long MCMC, our generated samples are expected to disperse over the image manifold.

In this subsection, we illustrate the practical role of MCMC using CIFAR-10 generation with the unadjusted Langevin algorithm (ULA) as an example.
Operationally, ULA augments each deterministic update with an isotropic Gaussian perturbation whose scale is controlled by a temperature parameter.

Our findings are simple. When the initialization is mapped mainly to a subset of the image manifold, ULA can substantially reduce FID. When the initialization already provides sufficient empirical coverage, however, ULA offers little additional improvement under the same 250-step sampling budget, either visually or in terms of FID.

First, we show that the choice of inference initialization is important and can be modified in our setting.
We train our model using linear interpolations between standard Gaussian noise and real images, following the flow matching baseline.
Under the usual probability allocation view, inference is expected to start from the same source distribution used during training; otherwise, generation may degrade significantly.
In contrast, from a distance-field-modeling perspective, the learned model can still move samples toward the image manifold, even when the initialization distribution is shifted from training Gaussian noise to a uniform distribution.
This behavior highlights a distinction between our distance-field interpretation and probability-flow inference that is tightly coupled to a prescribed source distribution.
We verify it empirically in~\Cref{fig:combined_baseline_range}.

\begin{figure}[!htbp]
    \centering

    \begin{subfigure}{0.7\linewidth}
        \centering
        \includegraphics[width=\linewidth]{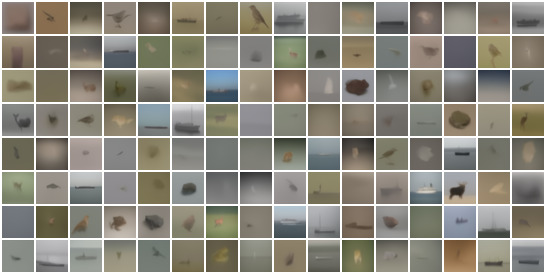}
        \caption{When initialized from uniform noise in $[-1,1]$, the flow matching baseline fails to generate meaningful images in most cases, yielding a high FID of 178.44 over 50,000 generated images.}
        \label{fig:baseline}
    \end{subfigure}

    \vspace{0.5em}

    \begin{subfigure}{0.7\linewidth}
        \centering
        \includegraphics[width=\linewidth]{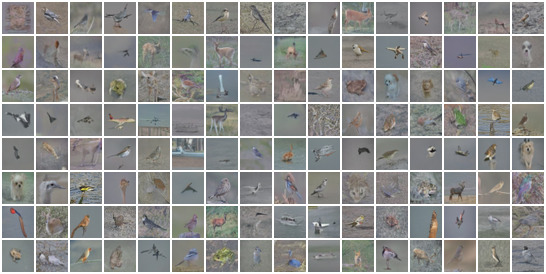}
        \caption{Our model maps uniform initializations in $[-1,1]$ to low-contrast but recognizable images, achieving an FID of 65.48 over 50,000 generated images.}
        \label{fig:rangeonezero}
    \end{subfigure}

    \vspace{0.5em}

    \begin{subfigure}{\linewidth}
        \centering
        \begin{tabular}{@{}c@{}c@{}c@{}c@{}c@{}c@{}c@{}c@{}}
            \includegraphics[width=0.04\linewidth]{figures/mode_collapse/31098.jpg} &
            \includegraphics[width=0.04\linewidth]{figures/mode_collapse/5692.jpg} &
            \includegraphics[width=0.04\linewidth]{figures/mode_collapse/2264.jpg} &
            \includegraphics[width=0.04\linewidth]{figures/mode_collapse/2124.jpg} &
            \includegraphics[width=0.04\linewidth]{figures/mode_collapse/1858.jpg} &
            \includegraphics[width=0.04\linewidth]{figures/mode_collapse/1744.jpg} &
            \includegraphics[width=0.04\linewidth]{figures/mode_collapse/958.jpg} &
            \includegraphics[width=0.04\linewidth]{figures/mode_collapse/682.jpg}
        \end{tabular}
        \caption{Examples of real low-contrast CIFAR-10 images.}
    \end{subfigure}

    \caption{Our model can move samples initialized from uniform noise in $[-1,1]$, despite the mismatch with the Gaussian noise distribution used during training, toward the image manifold.
    Individual outputs are visually consistent with the low-contrast subset of CIFAR-10.
    However, without MCMC, this initialization does not guarantee diverse outputs that cover the full image manifold.
    Both methods use similar model architectures and the same sampling budget.
    }
    \label{fig:combined_baseline_range}
\end{figure}

We then introduce ULA to make our generation better cover the image manifold.
Our sphere tracing and gradient descent samplers can be combined with this refinement by adding an isotropic Gaussian perturbation to each deterministic update.
Let $\boldsymbol{\xi}_i\sim\mathcal{N}(\mathbf{0},\mathbf{I})$ denote independent Gaussian noise, and let $\tau>0$ be the Langevin temperature.
The resulting updates are
\begin{equation}
\begin{aligned}
\text{ST-ULA:}\quad
\mathbf{x}^{(i+1)}
&=
\mathbf{x}^{(i)}
-
\eta\,u_\theta(\mathbf{x}^{(i)})\mathbf{v}_\theta(\mathbf{x}^{(i)})
+
\sqrt{2\eta u_\theta(\mathbf{x}^{(i)}) \tau}\,\boldsymbol{\xi}_{i},
\\
\text{GD-ULA:}\quad
\mathbf{x}^{(i+1)}
&=
\mathbf{x}^{(i)}
-
\eta\,\mathbf{v}_\theta(\mathbf{x}^{(i)})
+
\sqrt{2\eta\tau}\,\boldsymbol{\xi}_{i}.
\end{aligned}
\end{equation}

Here $u_\theta(\mathbf{x}^{(i)})$ denotes the nonnegative distance estimate used by the sampler.
When $\tau$ changes over iterations, the update becomes an annealed ULA variant.
We test ST-ULA on CIFAR-10.
We still start from the uniform distribution on $[-1,1]$ and use a 250-step sampling budget.
We compare two annealing schedules in~\Cref{fig:bad_uniform_noise}.
The slightly longer schedule achieves a much lower FID, showing that MCMC is important when the initialization is poor.
This improvement is not simply caused by injecting Gaussian noise and moving samples to numerically familiar regions.
Otherwise, the shorter MCMC run should already introduce enough Gaussian noise and achieve a similar effect.

\begin{figure}[!htbp]
    \centering

    \begin{subfigure}{0.7\linewidth}
        \centering
        \includegraphics[width=\linewidth]{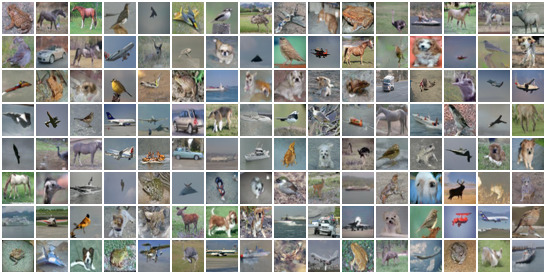}
        \caption{ST-ULA annealing schedule for $\tau$: it stays at 0.1 for the first 5\% of the steps, then decreases linearly to 0.0 by 10\% of the steps. The FID is 25.73 over 50,000 generated images.}
    \end{subfigure}

    \vspace{0.5em}

    \begin{subfigure}{0.7\linewidth}
        \centering
        \includegraphics[width=\linewidth]{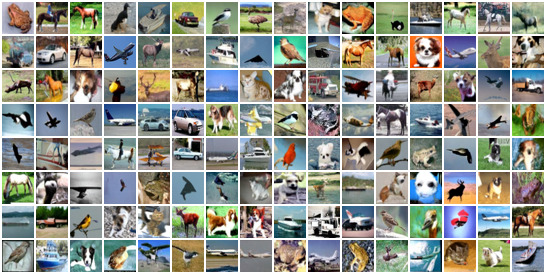}
        \caption{ST-ULA annealing schedule for $\tau$: it stays at 0.1 for the first 10\% of the steps, then decreases linearly to 0.0 by 20\% of the steps. The FID is 11.37 over 50,000 generated images.}
    \end{subfigure}

    \caption{
    With MCMC, poor initializations (uniform distribution in $[-1,1]$) can be spread more effectively over the image manifold, producing more diverse categories.
    Under the same sampling budget, the FID improves from 65.48 to 25.73 and then to 11.37.}
    \label{fig:bad_uniform_noise}
\end{figure}

When the initialization is already strong, the improvement from MCMC under a fixed 250-step sampling budget becomes marginal, and the visual differences become difficult for humans to recognize, as shown in~\Cref{fig:small_improvement}.

\begin{figure}[!htbp]
    \centering

    \begin{subfigure}{\linewidth}
        \centering
        \includegraphics[width=0.7\linewidth]{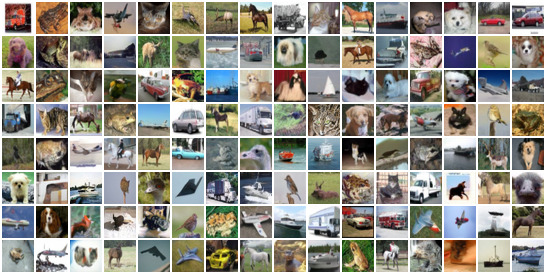}
        \caption{A uniform distribution over $[-1.7, 1.7]$ provides a not-bad initialization for sphere tracing, achieving an FID of 2.56 with 50,000 generated images.}
    \end{subfigure}

    \vspace{0.5em}

    \begin{subfigure}{\linewidth}
        \centering
        \includegraphics[width=0.7\linewidth]{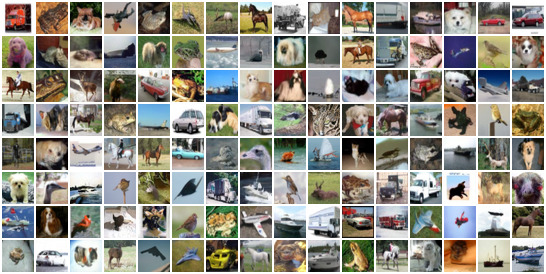}
        \caption{We apply ULA after $[-1.7, 1.7]$ uniform initializations.
        ST-ULA annealing schedule for $\tau$: it starts at 0.02 and linearly decreases to zero at the first 30\% of the steps, providing an FID of 2.35 with 50,000 images.}
    \end{subfigure}

    \vspace{0.5em}

    \begin{subfigure}{\linewidth}
        \centering
        \includegraphics[width=0.7\linewidth]{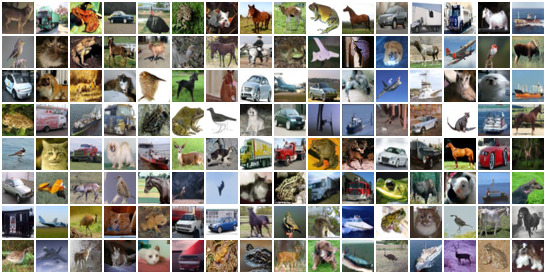}
        \caption{We apply ULA after standard Gaussian initialization. For the ST-ULA annealing schedule, $\tau$ starts at 0.02 and linearly decreases to zero over the first 30\% of the steps, achieving an FID of 2.22 with 50,000 generated images.}
    \end{subfigure}

    \caption{MCMC provides limited visual improvements under a fixed 250-step sampling budget when initialized from either the uniform distribution over $[-1.7,1.7]$ or the standard Gaussian distribution.}
    \label{fig:small_improvement}
\end{figure}

Overall, these observations motivate our choice not to include MCMC in the final sampling pipeline, for reasons similar in spirit to the transition from stochastic DDPM sampling~\cite{ho2020denoising} to deterministic DDIM sampling~\cite{song2020denoising}. Langevin refinement needs enough steps to substantially improve coverage over the image manifold; under a fixed budget, a short stochastic phase may only perturb samples locally while reducing the number of effective deterministic transport steps. Moreover, the resulting non-deterministic process is less convenient for downstream acceleration, especially distillation, where deterministic teacher trajectories provide cleaner targets. Removing stochasticity also reduces sampling variance, makes hyperparameter tuning more stable, and improves the comparability of ablation studies. Thus, we treat MCMC as an optional refinement rather than a core component.

For the deterministic process used in our final implementation, the key question becomes how the initialization distribution affects the induced pushforward distribution. We provide theoretical analysis in~\Cref{app:pushforward}, and discuss this initialization-dependent behavior as a valuable direction for future work in~\Cref{app:limitations_future_work}.

\clearpage
\subsection{Sample Fidelity and Diversity}
\label{app:fidelity_diversity}

We evaluate ImageNet generation across classifier-free guidance scales using 50k samples, 250 sampling steps, and comparable B/2 architectures.
Alongside FID and Inception Score, we report density and coverage, which measure sample fidelity and reference coverage, respectively~\cite{naeem2020reliable}.
Our samplers have the highest coverage at eight of nine guidance scales and higher density than SiT and EqM at all nine.
Sphere tracing has higher density than gradient descent at every scale.

\begin{table}[!htbp]
\centering
\footnotesize
\setlength{\tabcolsep}{3pt}
\caption{FID across guidance scales. Lower is better. Parentheses show standard deviations at each method's best-performing guidance scale.}
\label{tab:fid_cfg}

\begin{adjustbox}{max width=\linewidth}
\begin{threeparttable}
\begin{tabular}{lcccccccccc}
\toprule
CFG & w/o $t$ & 1.25 & 1.5 & 1.75 & 2 & 2.25 & 2.5 & 2.75 & 3 & 4 \\
\midrule
uEDM & \checkmark & -- & -- & -- & 10.66 & -- & 8.15 & \underline{7.96} & 8.15 & 10.28 \\
SiT & \texttimes & 24.30 & 16.38 & 11.67 & 9.11 & 7.93 & \underline{7.66}\,(0.032) & 7.98 & 8.51 & 11.51 \\
EqM & \checkmark & 21.02 & 13.81 & 9.73 & 7.71 & \underline{6.93}\,(0.017) & 6.98 & 7.40 & 8.06 & 10.97 \\
Ours (Sphere Tracing, $\eta=0.022$) & \checkmark & 18.79 & 11.04 & 7.66 & \underline{\textbf{6.59}}\,(0.027) & 6.83 & 7.52 & 8.55 & 9.70 & 13.44 \\
Ours (Gradient Descent, $\eta=0.64$) & \checkmark & 18.19 & 10.73 & 7.60 & \underline{6.67}\,(0.053) & 7.01 & 7.82 & 8.91 & 10.05 & 13.76 \\
\bottomrule
\end{tabular}

\begin{tablenotes}[flushleft]
\footnotesize
\item Underlined: row best. Bold: overall best. The original uEDM paper only reports partial FID results under these settings without releasing code.
\end{tablenotes}
\end{threeparttable}
\end{adjustbox}
\end{table}

\begin{table}[!htbp]
\centering
\footnotesize
\setlength{\tabcolsep}{3pt}
\caption{Inception Score across guidance scales. Higher is better. Parentheses show standard deviations at each method's best-performing guidance scale.}
\label{tab:is_cfg}

\begin{adjustbox}{max width=\linewidth}
\begin{threeparttable}
\begin{tabular}{lcccccccccc}
\toprule
CFG & w/o $t$ & 1.25 & 1.5 & 1.75 & 2 & 2.25 & 2.5 & 2.75 & 3 & 4 \\
\midrule
SiT & \texttimes & 60.95 & 85.66 & 111.65 & 137.04 & 160.71 & 181.44 & 199.56 & 214.07 & \underline{254.77}\,(1.27) \\
EqM & \checkmark & 69.36 & 97.22 & 127.20 & 156.51 & 181.46 & 202.98 & 221.41 & 237.74 & \underline{281.87}\,(1.83) \\
Ours (Sphere Tracing, $\eta=0.022$) & \checkmark & 78.28 & 117.16 & 155.29 & 187.82 & 214.55 & 238.33 & 255.54 & 272.36 & \underline{310.02}\,(1.01) \\
Ours (Gradient Descent, $\eta=0.64$) & \checkmark & 77.93 & 116.69 & 154.17 & 188.08 & 215.18 & 239.08 & 255.99 & 273.80 & \underline{\textbf{311.31}}\,(0.36) \\
\bottomrule
\end{tabular}

\begin{tablenotes}[flushleft]
\footnotesize
\item Underlined: row best. Bold: overall best.
\end{tablenotes}
\end{threeparttable}
\end{adjustbox}
\end{table}

\begin{table}[!htbp]
\centering
\small
\setlength{\tabcolsep}{3pt}
\caption{Coverage across guidance scales. Higher is better. Bold marks the best result at each scale, including ties.}
\label{tab:coverage_cfg}
\begin{tabular}{lccccccccc}
\toprule
Method & 1.25 & 1.5 & 1.75 & 2 & 2.25 & 2.5 & 2.75 & 3 & 4 \\
\midrule
SiT & 0.9019 & 0.9488 & 0.9731 & 0.9858 & 0.9924 & 0.9940 & 0.9968 & 0.9971 & 0.9949 \\
EqM & 0.9311 & 0.9659 & 0.9850 & 0.9917 & 0.9943 & 0.9966 & 0.9972 & 0.9973 & \textbf{0.9963} \\
Ours (Sphere Tracing) & 0.9524 & 0.9799 & \textbf{0.9930} & \textbf{0.9937} & \textbf{0.9965} & \textbf{0.9974} & \textbf{0.9975} & 0.9977 & 0.9961 \\
Ours (Gradient Descent) & \textbf{0.9545} & \textbf{0.9810} & \textbf{0.9930} & \textbf{0.9937} & 0.9964 & \textbf{0.9974} & 0.9974 & \textbf{0.9980} & 0.9960 \\
\bottomrule
\end{tabular}
\end{table}

\begin{table}[!htbp]
\centering
\small
\setlength{\tabcolsep}{3pt}
\caption{Density across guidance scales. Higher is better. Bold marks the best result at each scale.}
\label{tab:density_cfg}
\begin{tabular}{lccccccccc}
\toprule
Method & 1.25 & 1.5 & 1.75 & 2 & 2.25 & 2.5 & 2.75 & 3 & 4 \\
\midrule
SiT & 0.6921 & 0.8657 & 1.0231 & 1.1570 & 1.2571 & 1.3328 & 1.3972 & 1.4443 & 1.5366 \\
EqM & 0.7682 & 0.9454 & 1.0994 & 1.2231 & 1.3152 & 1.3955 & 1.4509 & 1.4913 & 1.5513 \\
Ours (Sphere Tracing) & \textbf{0.8270} & \textbf{1.0516} & \textbf{1.2264} & \textbf{1.3603} & \textbf{1.4487} & \textbf{1.5112} & \textbf{1.5501} & \textbf{1.5759} & \textbf{1.5838} \\
Ours (Gradient Descent) & 0.8246 & 1.0443 & 1.2161 & 1.3461 & 1.4330 & 1.4955 & 1.5272 & 1.5507 & 1.5586 \\
\bottomrule
\end{tabular}
\end{table}

\clearpage

\subsection{Evidence Against Memorization}
\label{app:memo}

As discussed in~\Cref{app:allocation}, annealed MCMC can be incorporated to improve density reconstruction. 
However, recent work~\cite{li2025scores, shen2026manifold} formulates image generation from a geometric support recovery perspective, which illustrates that MCMC is not necessary for generating realistic and novel samples starting with Gaussian noise.

Our method aims to recover the data support through a learned distance field, together with two principled generation procedures: sphere tracing and gradient descent. 
In this subsection, we present a nearest-neighbor analysis in~\Cref{fig:nn-memo} showing that the generated samples are not copying training images, providing further evidence for the support recovery theory.

\newcommand{\nnpath}{figures/curated_imagenet_4.0_st}
\newcommand{\nnw}{0.115\textwidth}
\newcommand{\nnrowpair}[8]{%
  \includegraphics[width=\nnw]{\nnpath/#1} &
  \includegraphics[width=\nnw]{\nnpath/#1_nn1} &
  \includegraphics[width=\nnw]{\nnpath/#1_nn2} &
  \includegraphics[width=\nnw]{\nnpath/#1_nn3} &
  \includegraphics[width=\nnw]{\nnpath/#5} &
  \includegraphics[width=\nnw]{\nnpath/#5_nn1} &
  \includegraphics[width=\nnw]{\nnpath/#5_nn2} &
  \includegraphics[width=\nnw]{\nnpath/#5_nn3} \\[-2pt]
  & {\scriptsize $\ell_2{=}\mathbf{#2}$} & {\scriptsize $\ell_2{=}#3$} & {\scriptsize $\ell_2{=}#4$} &
  & {\scriptsize $\ell_2{=}\mathbf{#6}$} & {\scriptsize $\ell_2{=}#7$} & {\scriptsize $\ell_2{=}#8$} \\[4pt]
}

\begin{figure}[!htbp]
\centering
\setlength{\tabcolsep}{1pt}
\begin{tabular}{cccc@{\hskip 4pt}cccc}
\small\textbf{Generated} & \small\textbf{NN 1} & \small\textbf{NN 2} & \small\textbf{NN 3} &
\small\textbf{Generated} & \small\textbf{NN 1} & \small\textbf{NN 2} & \small\textbf{NN 3} \\[3pt]
\nnrowpair{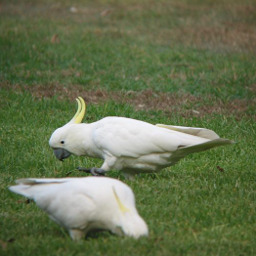}{27516}{28432}{28580}{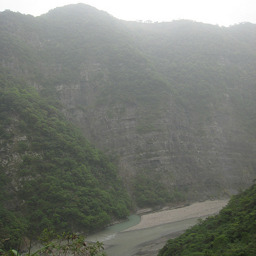}{24034}{24042}{24073}
\nnrowpair{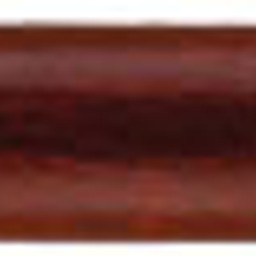}{26892}{27482}{28075}{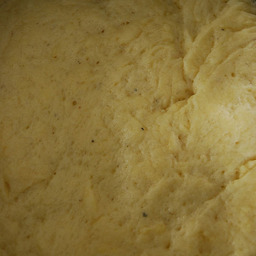}{24842}{24953}{24966}
\nnrowpair{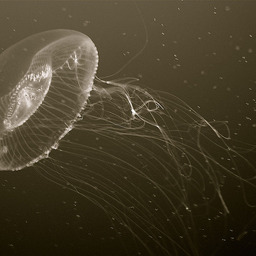}{25001}{25498}{25675}{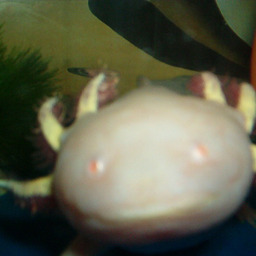}{26501}{27650}{27750}
\nnrowpair{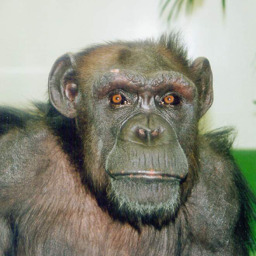}{24713}{25320}{26056}{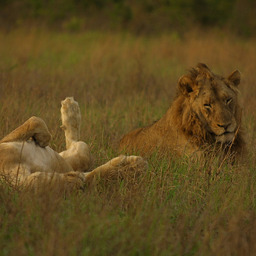}{21824}{21846}{21921}
\nnrowpair{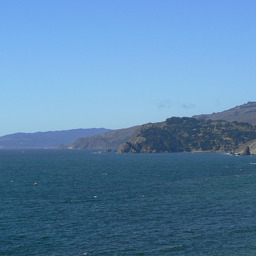}{16130}{17417}{17653}{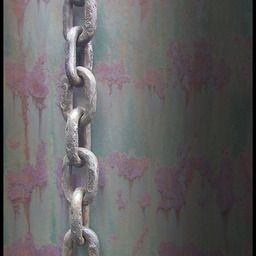}{29697}{30043}{30145}
\nnrowpair{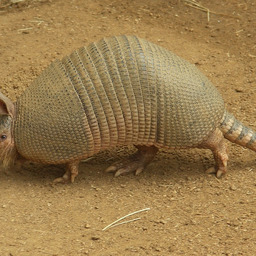}{30298}{30443}{30547}{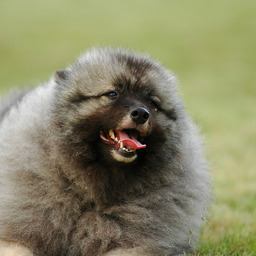}{22731}{23778}{24142}
\nnrowpair{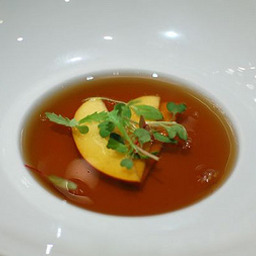}{25691}{26110}{26993}{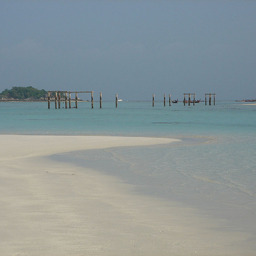}{17204}{17479}{17514}
\nnrowpair{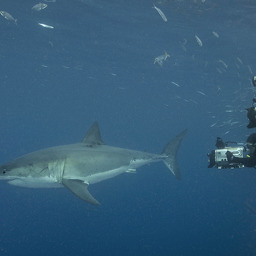}{28584}{28888}{28895}{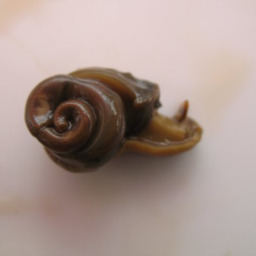}{32941}{34503}{34878}
\end{tabular}
\caption{Nearest-neighbor analysis of the sphere tracing method in~\Cref{fig:curated_imagenet_grid}. Each group shows a generated sample and its top-3 nearest training images by image-space $\ell_2$ distance. The closest distance is bolded. The large distances and visually distinct content confirm that our model does not memorize training data.}
\label{fig:nn-memo}
\end{figure}

\clearpage
\subsection{Accuracy and Geometric Consistency of the Distance Head}
\label{app:distance_head}

The distance head predicts distances accurately, and its gradient agrees with the vector head along the main denoising direction.
We show this on CIFAR-10 through distance prediction errors, image-space level sets, and comparisons between the two heads.

\paragraph{Distance prediction accuracy.}
We draw 25{,}600 noisy inputs by interpolating between Gaussian noise and paired real images.
We sample the interpolation coefficient uniformly from $[0,1]$.
For each input $\mathbf{x}$ and its paired image $\mathbf{s}$, we compare $u_\theta(\mathbf{x})$ with the smoothed distance $\sqrt{\|\mathbf{x}-\mathbf{s}\|_2^2+c_0}$, using $c_0=4$.
We also compare the composite vector length $\|u_\theta(\mathbf{x})\mathbf{v}_\theta(\mathbf{x})\|_2$ with $\|\mathbf{x}-\mathbf{s}\|_2$.
The mean absolute percentage error (MAPE) is 2.44\% for the scalar head and 6.92\% for the composite vector length.
\Cref{fig:dist_error} shows how these errors vary with the reference distance.
These small errors show that the scalar head accurately predicts distances to the paired images.
The composite vector also closely matches the target displacement in length.

\begin{figure}[!htbp]
    \centering
    \includegraphics[width=0.5\linewidth]{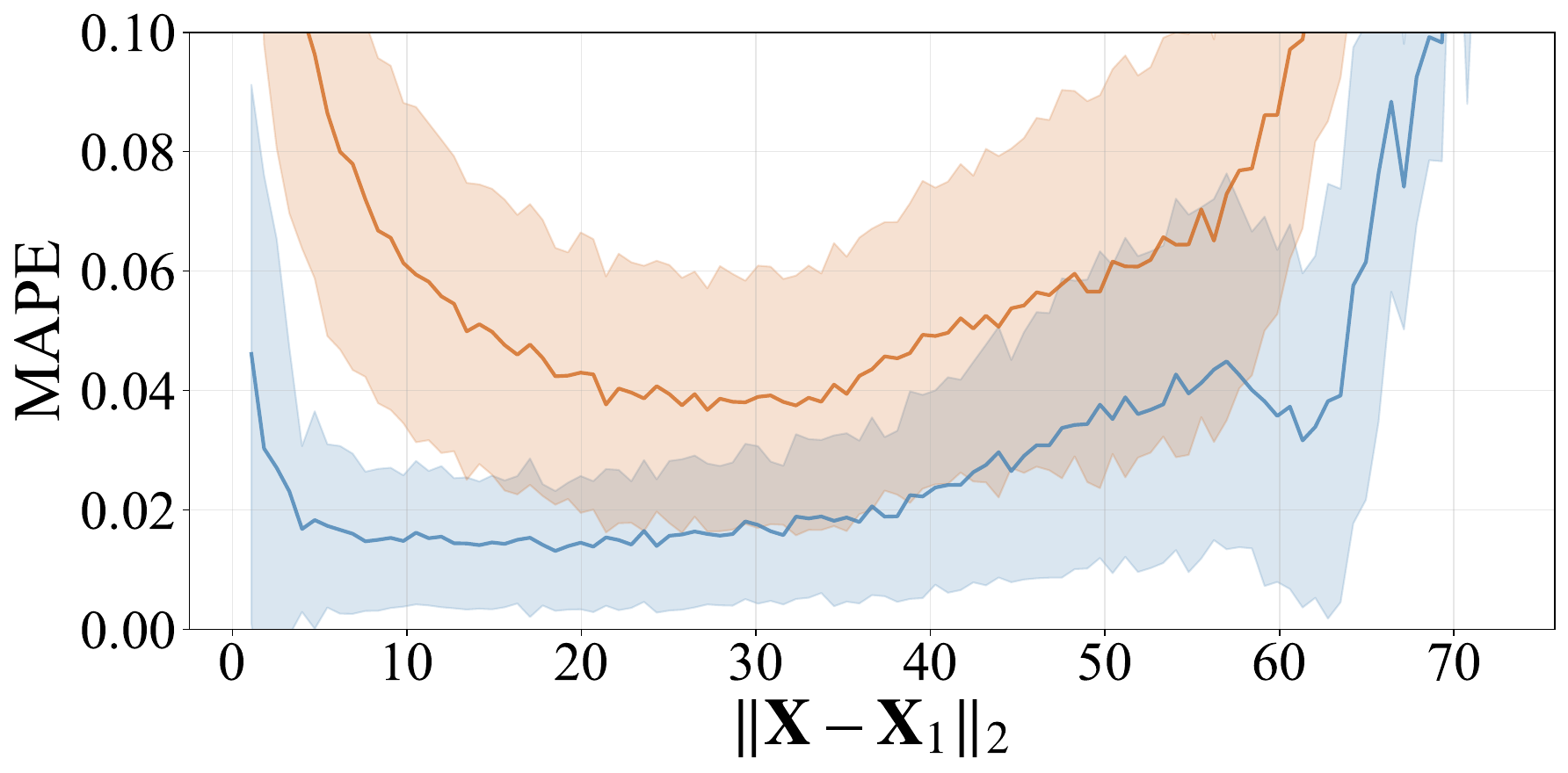}
    \caption{Distance estimation MAPE. Blue: $u_\theta$ vs.\ $\sqrt{\|\mathbf{x}-\mathbf{s}\|_2^2 + c_0}$. Orange: $\|u_\theta\mathbf{v}_\theta\|_2$ vs.\ $\|\mathbf{x}-\mathbf{s}\|_2$. Shaded: $\pm 1$ standard deviation within each bin.}
    \label{fig:dist_error}
\end{figure}

\paragraph{Local image-space geometry.}
To visualize the distance field around a generation trajectory, we examine the plane spanned by its first two principal components.
We generate 32 CIFAR-10 trajectories from independent Gaussian initializations, using 250 sphere tracing steps with step size $0.031$.
For each trajectory, we apply principal component analysis (PCA) to its recorded states.
The first two principal components $\mathbf{p}_1,\mathbf{p}_2$ explain, on average, 99.40\% and 0.54\% of the trajectory variance.
Most of the trajectory variation therefore lies along a single direction.
The directions are orthogonal and have unit Euclidean norm, with no scaling by their singular values.
Let $P=[\mathbf{p}_1,\mathbf{p}_2]$ and let $\mathbf{x}_{\mathrm{end}}$ be the generated endpoint.
We evaluate $u_\theta$ on a grid in the plane $\mathbf{x}_{\mathrm{end}}+P\mathbf{z}$, where $\mathbf{z}\in\mathbb{R}^2$ gives the plane coordinates.

\Cref{fig:distance_head_trajectory} shows a projected generation trajectory, with intermediate images every 50 steps.
These images show how generation progressively removes noise.
\Cref{fig:distance_head_neighborhood} shows the inputs $\mathbf{x}_{\mathrm{end}}+P\mathbf{z}$ at nine evenly spaced locations in the same plane.
Moving along PC1 toward smaller predicted distances reveals the image structure, while varying PC2 changes contrast and fine detail.

We also initialize 25 trajectories on a $5\times5$ grid in the same plane and run the same sphere tracing sampler in the full image space.
The projected trajectories in~\Cref{fig:distance_head_neighborhood}(b) move mainly along PC1, while their endpoints remain spread along PC2.
The distance field thus captures a common denoising direction across nearby inputs.
Following this direction removes noise while preserving differences along PC2.

\begin{figure}[!t]
    \centering
    \includegraphics[width=\linewidth]{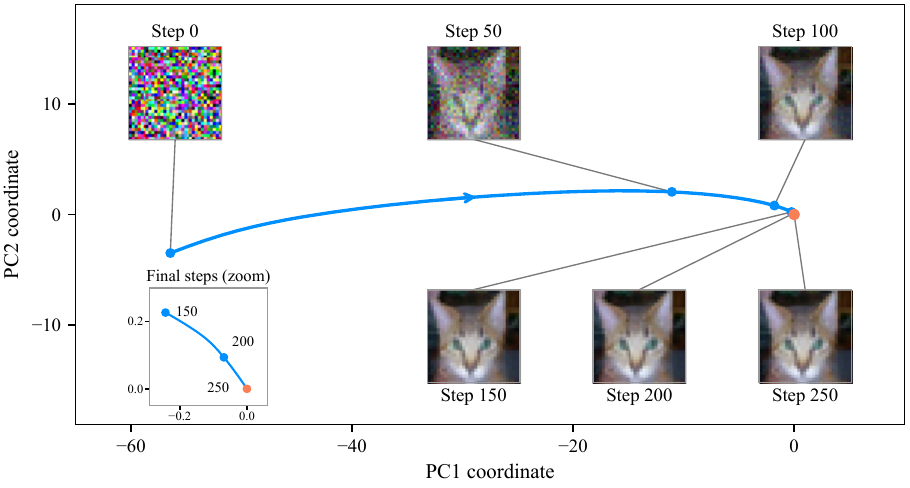}
    \caption{Generation trajectory in the PCA plane for CIFAR-10.
    The blue curve traces generation from the Gaussian initialization to the orange endpoint.
    Thumbnails show the actual images every 50 steps. Gray lines connect them to their projected positions.
    The inset magnifies the final steps. Both views use equal axis scales.}
    \label{fig:distance_head_trajectory}
\end{figure}

\begin{figure}[!t]
    \centering
    \includegraphics[width=\linewidth]{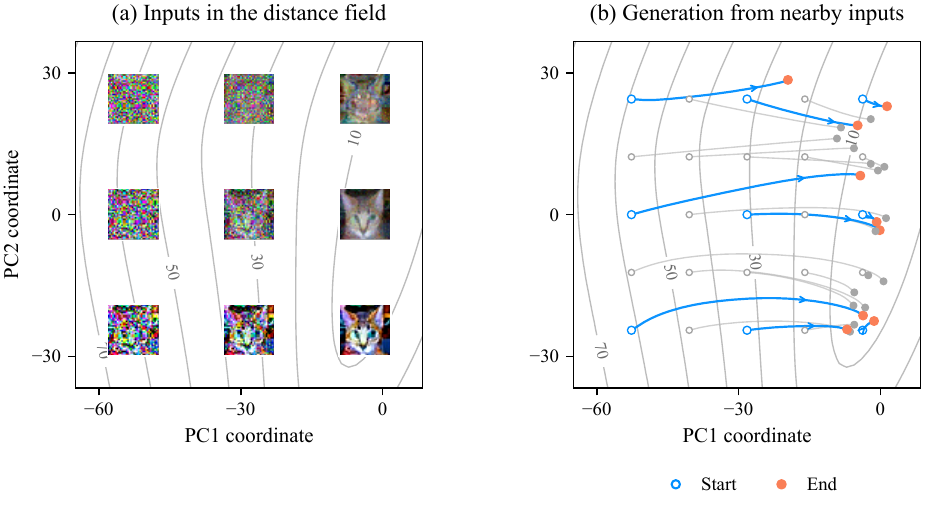}
    \caption{Nearby inputs share a common denoising direction in the PCA plane.
    \textbf{(a)} Gray contours show the predicted distance $u_\theta$.
    Thumbnails at nine evenly spaced locations show the corresponding image-space inputs, clipped to the display range.
    \textbf{(b)} We initialize 25 trajectories on a $5\times5$ grid and run sphere tracing in the full image space, then project the resulting states onto the same plane.
    The nine trajectories starting at the thumbnail locations in (a) are blue, with open start markers and orange endpoints. The remaining trajectories are gray.
    Both panels use the same coordinates as in~\Cref{fig:distance_head_trajectory}.}
    \label{fig:distance_head_neighborhood}
\end{figure}

\clearpage
\paragraph{Agreement between the two heads.}
We compute $\nabla u_\theta$ by automatic differentiation at a $5\times5$ grid of sample points inside each plane.
We then compare the projected directions $P^{\mathsf{T}}\nabla u_\theta$ and $P^{\mathsf{T}}\mathbf{v}_\theta$.
Across all 800 sample points, the mean angle is 8.90$^\circ$.
The PC1 components have the same sign at 98.75\% of sample points.
\Cref{fig:distance_head_geometry} shows the level sets and descent directions for the same trajectory as in~\Cref{fig:distance_head_trajectory}, together with the angle distribution over all sample points.
The two heads therefore agree on the main denoising direction.
In these planes, the sampling updates point toward lower values of the scalar field.
This agreement supports our view of generation as marching along a learned distance field.
Despite this alignment in the first two principal directions, sampling with the scalar-field gradient alone fails to produce usable images.
We attribute this to the tiny distance head's limited capacity.
This supports our design of separating distance and direction networks.

\begin{figure}[!t]
    \centering
    \includegraphics[width=\linewidth]{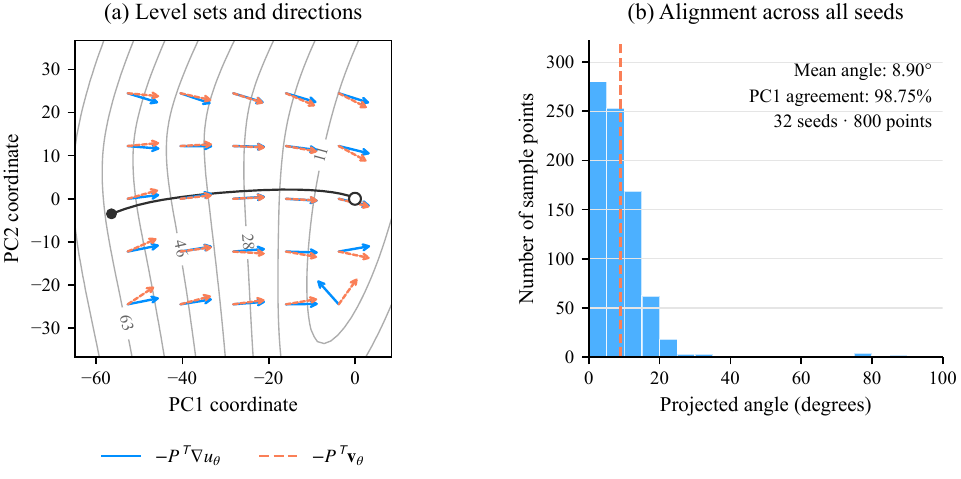}
    \caption{Local geometry of the learned distance field on CIFAR-10.
    \textbf{(a)} Gray contours show $u_\theta$ in the same PCA plane as in~\Cref{fig:distance_head_trajectory,fig:distance_head_neighborhood}.
    The black curve shows the same generation trajectory, from the filled marker to the open marker.
    Blue solid arrows show $-P^{\mathsf{T}}\nabla u_\theta$ and orange dashed arrows show $-P^{\mathsf{T}}\mathbf{v}_\theta$, each normalized to the same display length.
    The axes have equal scale to preserve angles.
    \textbf{(b)} Histogram of the projected angles across all 800 sample points from 32 seeds, with $5^\circ$ bins.
    The dashed vertical line marks the mean angle of $8.90^\circ$.
    The PC1 sign agreement is 98.75\%, and PC1 explains 99.40\% of trajectory variance on average.}
    \label{fig:distance_head_geometry}
\end{figure}

\newpage
\subsection{Convergence and Robustness}
\label{app:convergence}

We first report the standard deviations in~\Cref{tab:fid_cfg,tab:is_cfg} to assess the robustness of our method.
For CIFAR-10, the sample standard deviations of FID are 0.021 for sphere tracing and 0.031 for gradient descent.

One limitation of our current work is that we do not structurally constrain $\mathbf{v}_\theta$ to be a conservative field that is curl-free.
In the worst case, a field with strong rotational components can cause trajectories to cycle or oscillate forever.
Hence, we use this subsection to empirically illustrate our convergence and robustness.

\begin{figure}[!htbp]
    \centering
    \includegraphics[width=0.5\linewidth]{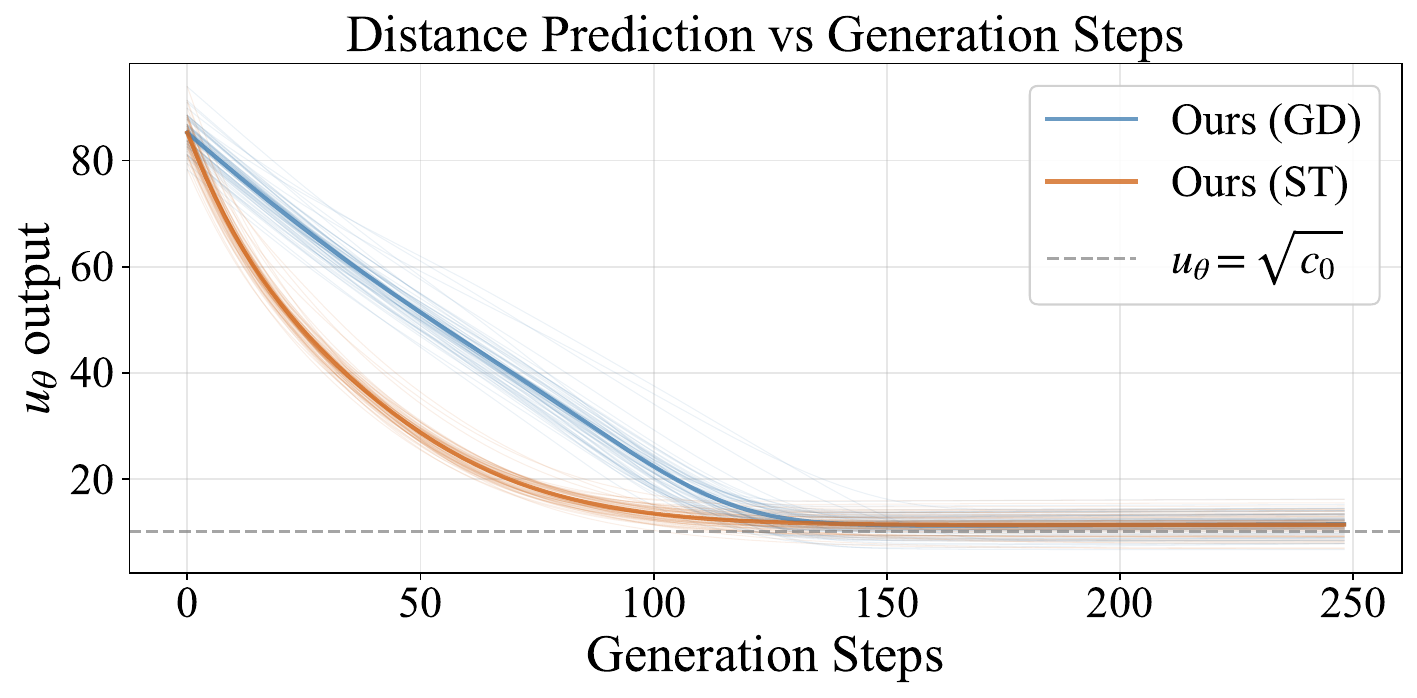}
    \vspace{-10pt}
    \caption{During ImageNet generation, the distance predicted by $u_\theta$, as smoothed by~\Cref{eq:analytical}, converges to the nearby $\sqrt{c_0}$ as predicted by our analysis.
    We sample the same 64 Gaussian noise vectors and run gradient descent and sphere tracing in parallel; the bold lines show their mean.
    Gradient descent provides a linear descent pattern, while sphere tracing shows an exponential descent.}
    \label{fig:dist_process}
\end{figure}

In~\Cref{fig:dist_process}, we first plot the distance curves for the 16 Gaussian noise vectors used in~\Cref{fig:evolution}, along with 48 additional initial noise vectors, to show that the convergence behaves like real optimization in low-dimensional distance fields.

As a comparison, during inference, previous flow matching and diffusion models essentially divide $[0,1]$ by the step budget and maintain a step counter as the $t$ input.
When $t$ reaches 1, denoising stops.
In their theory, the user can only specify the step budget, with no direct control over the step size.
We argue that this may not be the optimal perspective, because the model-generated states may not exactly match the $t$ implied by the step index, which is often summarized as exposure bias~\cite{ning2023elucidating, li2023alleviating}.

\begin{figure}[!htbp]
    \centering
    \includegraphics[width=0.5\linewidth]{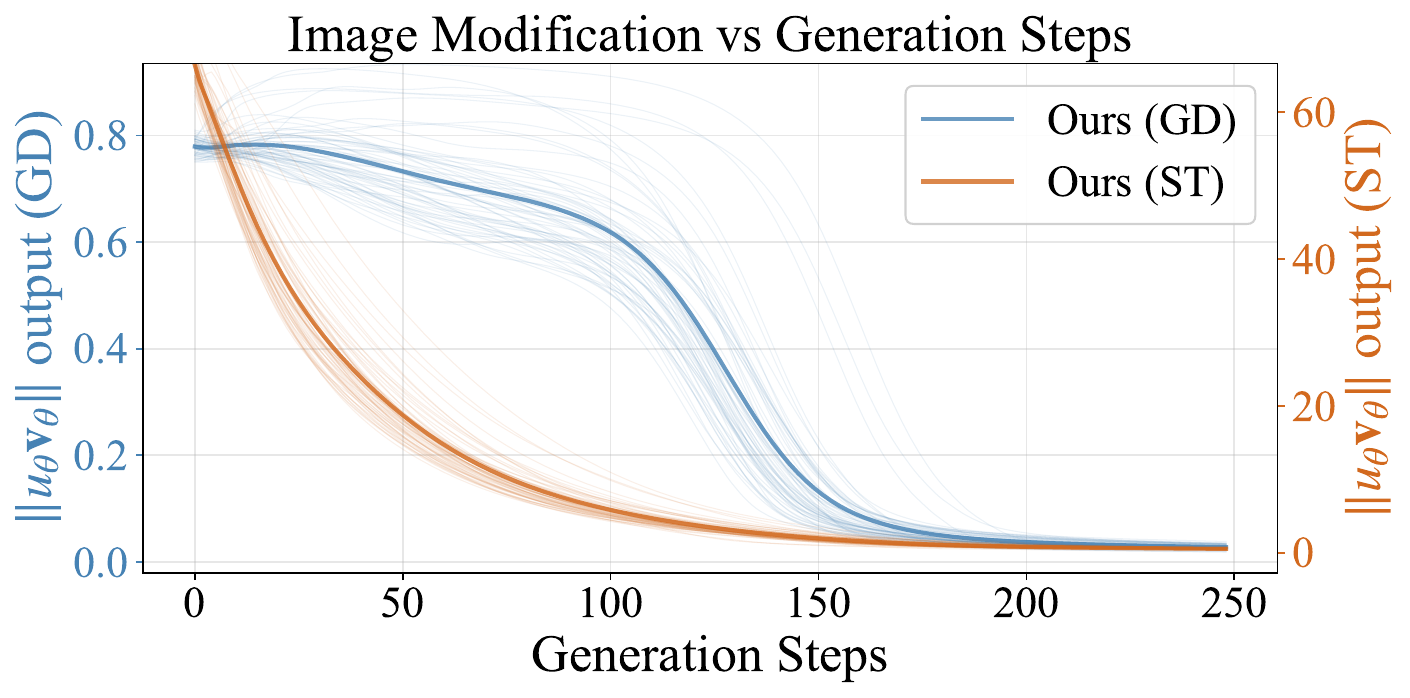}
    \vspace{-10pt}
    \caption{In the first 40\%, gradient descent maintains a steady denoising pace and then rapidly converges to zero modification. Sphere tracing still follows an exponential pattern and has the largest modifications at the beginning, consistent with our analysis in~\Cref{app:gpc}.}
    \label{fig:modification}
\end{figure}

In contrast, our inference methods sphere tracing and gradient descent are principally aligned with our loss design.
We observe that the image modification, measured by $\| \mathbf{v}_\theta \|$ (gradient descent) and $u_\theta \| \mathbf{v}_\theta \|$ (sphere tracing), also converges to near zero in~\Cref{fig:modification}.

Having one more hyperparameter to tune does not make our method fragile.
In practice, \dm has a wide basin of convergence across step sizes, yielding reliable convergence and strong robustness, visualized in~\Cref{fig:fid_vs_stepsize}.

\begin{figure}[t]
    \centering
    \begin{subfigure}{0.3\linewidth}
        \centering
        \includegraphics[width=\linewidth]{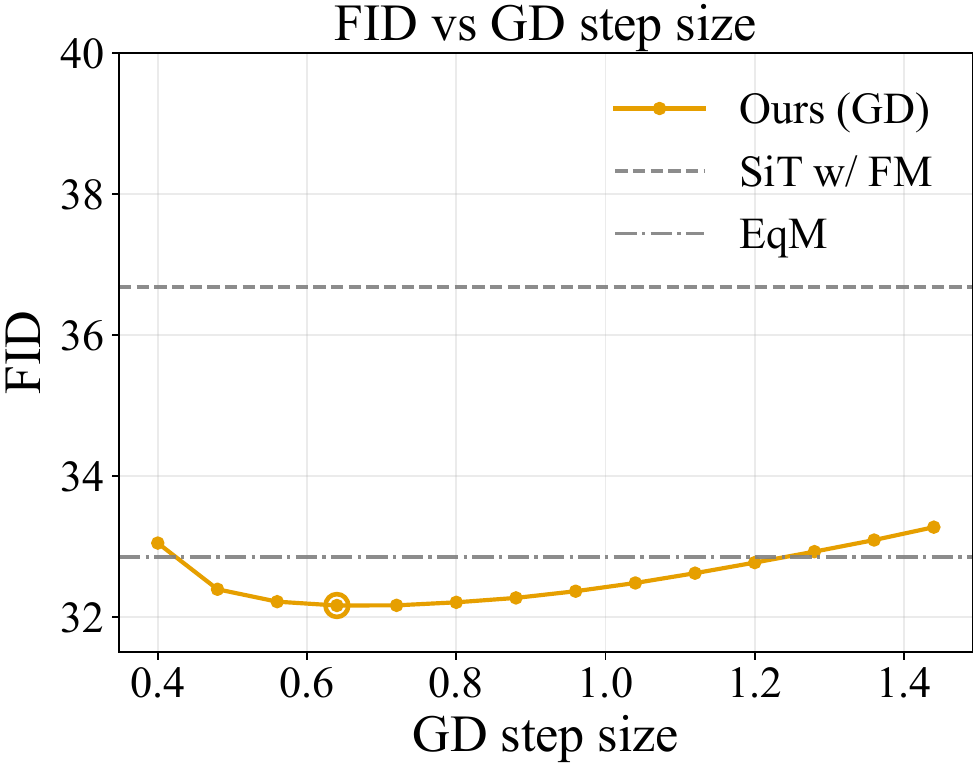}
    \end{subfigure}\hspace{0.05\linewidth}
    \begin{subfigure}{0.3\linewidth}
        \centering
        \includegraphics[width=\linewidth]{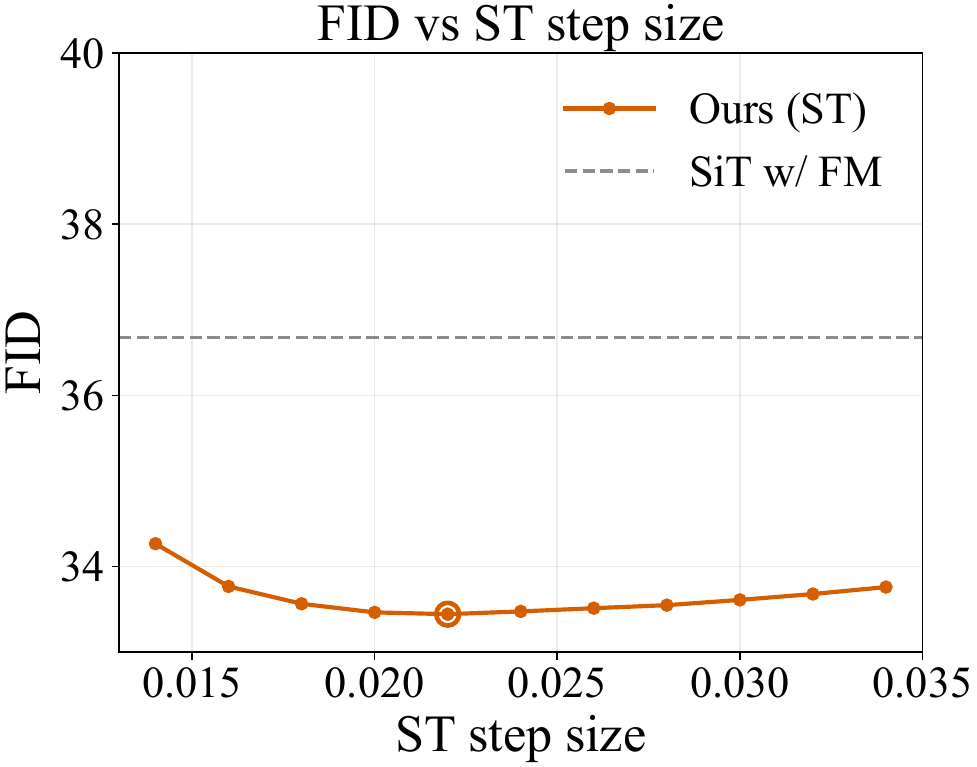}
    \end{subfigure}
    \caption{\dm is robust to step-size choices.
    We plot results for our backbone SiT~\cite{ma2024sit} and the second-best EqM~\cite{wang2025equilibrium} in gray as references, using the same base model size, training epochs, and evaluation protocol (250 steps, 50k samples).
    The best step size is highlighted with an additional circle.}
    \vspace{-10pt}
    \label{fig:fid_vs_stepsize}
\end{figure}

We also find that prior work~\cite{balcerak2025energy,wang2025equilibrium} does not perform well when using a structural constraint (i.e., taking autograd from a scalar network) to make the directions curl-free.
We conjecture that current generative model architectures are not designed to preserve directional information after the autograd operator.
Moreover, applying constraints on the autograd output, which requires backpropagating through second-order derivatives, is highly unstable~\cite{czarnecki2017sobolev,li2024neural,chetan2025accurate}.
On the other hand, efficiently computing or estimating curl in high-dimensional spaces remains an open problem, let alone using losses to enforce zero curl.
We anticipate future work will go deeper into this discussion.

\subsection{Ablation Study}
\label{app:ablation}

\paragraph{Verification of Mode Collapse and Condition Mismatch.}

We compare our original random linear interpolation against two alternative rematching strategies during training: (i) OT coupling and (ii) minibatch closest coupling.
Both variants lead to degraded performance (see~\Cref{fig:ablation_coupling} and~\Cref{tab:ot_failure}), consistent with our discussion in~\Cref{sec:analysis}.

\begin{table}[!htbp]
\centering
\captionsetup{skip=6pt}
\caption{Quick test on ImageNet generation with $1024$ samples. We keep sphere-tracing sampling unchanged (same step budget) and only vary the training-time matching on a B/2 SiT model with 80 epochs: random linear interpolation as in~\Cref{def:genproc}, OT coupling~\cite{tong2023improving} (i.e., nearest-target matching without replacement), or minibatch closest coupling (i.e., nearest-target matching within each minibatch with replacement).}
\label{tab:ot_failure}

\small
\setlength{\tabcolsep}{6pt}
\renewcommand{\arraystretch}{1.05}

\begin{tabular}{l c c c}
\toprule
\textbf{Method} & \textbf{Inception Score $\uparrow$} & \textbf{FID $\downarrow$} & \textbf{sFID $\downarrow$} \\
\midrule
no coupling & 41.77 & 71.23 & 206.2 \\
w/ OT coupling      & 36.96 & 77.19 & 208.2 \\
w/ minibatch closest coupling & 26.31 & 90.79 & 217.6 \\
\bottomrule
\end{tabular}
\end{table}

\begin{figure}[!htbp]
  \centering

  \begin{subfigure}[!htbp]{0.27\linewidth}
    \centering
    \includegraphics[width=\linewidth]{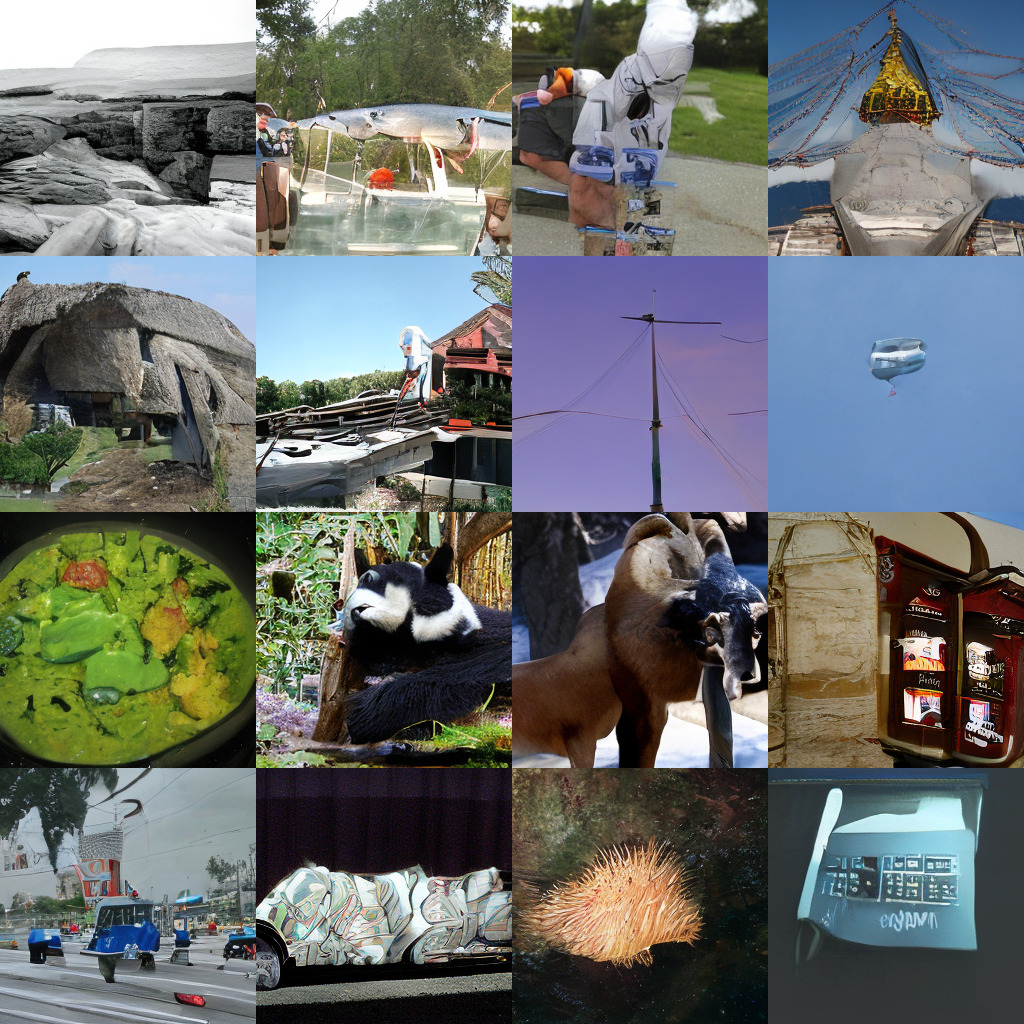}
    \caption{Original random linear matching.}
    \label{fig:ablation_coupling:a}
  \end{subfigure}\hspace{0.02\linewidth}
  \begin{subfigure}[!htbp]{0.27\linewidth}
    \centering
    \includegraphics[width=\linewidth]{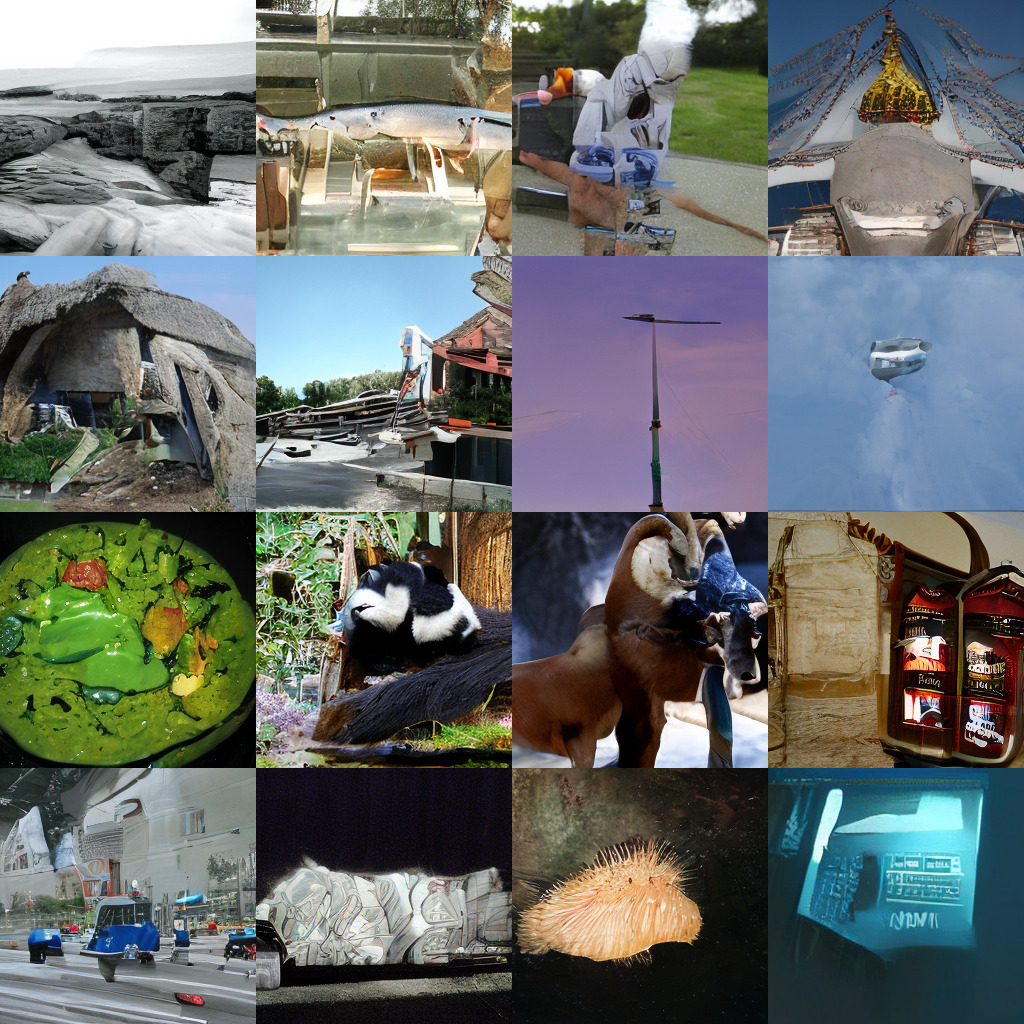}
    \caption{OT coupling.}
    \label{fig:ablation_coupling:b}
  \end{subfigure}\hspace{0.02\linewidth}
  \begin{subfigure}[!htbp]{0.27\linewidth}
    \centering
    \includegraphics[width=\linewidth]{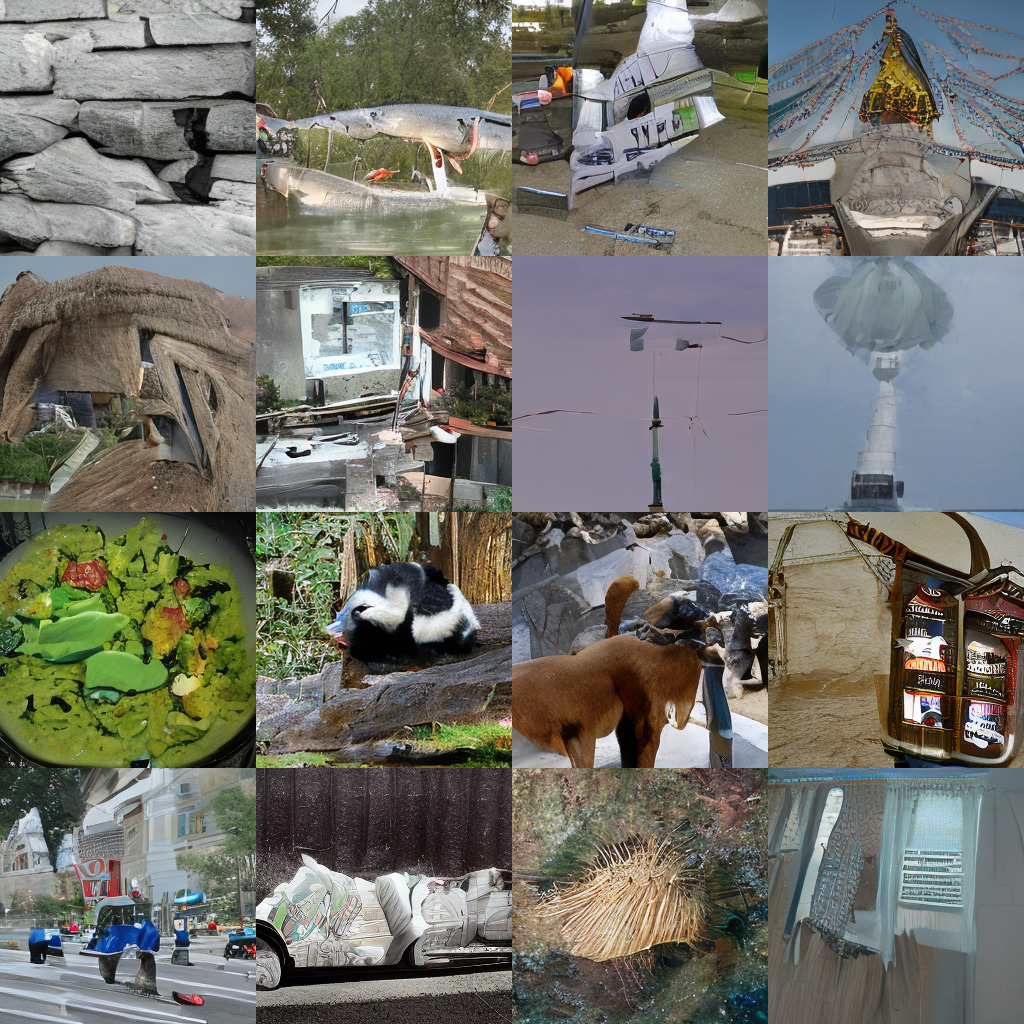}
    \caption{Minibatch closest coupling.}
    \label{fig:ablation_coupling:c}
  \end{subfigure}

  \caption{Comparison of generative models trained by different rematching algorithms, highlighting both condition mismatch and mode collapse:
  (\hyperref[fig:ablation_coupling:b]{b}) OT coupling introduces more unnatural artifacts in conditional generation;
  (\hyperref[fig:ablation_coupling:c]{c}) minibatch closest coupling makes the outputs noticeably low-contrast and grayish.
  We keep CFG scale $=1$ (i.e., no CFG) to align with other experiments.}

  \label{fig:ablation_coupling}
  \vspace{-12pt}
\end{figure}

\paragraph{Training without directional eikonal loss.}

On CIFAR-10, sphere tracing achieves lower FID than gradient descent both with and without DEL~(\Cref{tab:sampler_ablation}).
With sphere tracing, adding directional eikonal loss further improves FID from 2.33 to 2.23~(\Cref{tab:remove_del}).
Removing one-step loss leaves the distance part $u_\theta$ essentially unconstrained and makes training unstable.
We therefore do not include this setting.

\vspace{-2pt}
\begin{table}[!htbp]
\centering
\captionsetup{skip=6pt}
\caption{CIFAR-10 FID with and without directional eikonal loss.
Sphere tracing achieves lower FID than gradient descent both with and without DEL.
All results use 50k generated samples and 250 sampling steps.
We tune $\eta$ for each method and report the lowest FID achieved.}
\label{tab:sampler_ablation}

\small
\setlength{\tabcolsep}{6pt}
\renewcommand{\arraystretch}{1.05}

\begin{tabular}{l c c}
\toprule
\textbf{Sampler} & \textbf{w/ DEL} & \textbf{w/o DEL} \\
\midrule
Sphere tracing & 2.23 & 2.33 \\
Gradient descent & 2.54 & 2.38 \\

\bottomrule
\end{tabular}
\end{table}

\begin{table}[!htbp]
\centering
\captionsetup{skip=6pt}
\caption{FID under sphere tracing with and without directional eikonal loss, using the evaluation protocol in~\Cref{tab:sampler_ablation}.
Adding directional eikonal loss improves FID on both datasets, with a larger reduction on ImageNet.}
\label{tab:remove_del}

\small
\setlength{\tabcolsep}{6pt}
\renewcommand{\arraystretch}{1.05}

\begin{tabular}{l c c}
\toprule
\textbf{Dataset} & \textbf{w/ DEL} & \textbf{w/o DEL} \\
\midrule
CIFAR-10 & 2.23 & 2.33 \\
ImageNet & 33.44 & 37.01 \\

\bottomrule
\end{tabular}
\end{table}
\vspace{-2pt}

\paragraph{Loss Variations}

We apply two different loss designs on the CIFAR-10 generation task: (i) set $c_0=\epsilon=0$; (ii) unnormalized one-step loss without the denominator.
Consistent with our previous analysis, both degrade generation performance (see~\Cref{tab:loss_variation}).

\vspace{-2pt}
\begin{table}[!htbp]
\centering
\vspace{-6pt}
\captionsetup{skip=6pt}
\caption{Ablation study on CIFAR-10 of different loss designs. We tune $\eta$ for each method and report the lowest FID achieved.}
\label{tab:loss_variation}

\small
\setlength{\tabcolsep}{6pt}
\renewcommand{\arraystretch}{1.05}

\begin{tabular}{l c c}
\toprule
\textbf{Method} & \textbf{FID$\downarrow$ (ST)} & \textbf{FID$\downarrow$ (GD)} \\
\midrule
original OSL, $c_0=\epsilon=4$ (baseline) & 2.23 & 2.54 \\
original OSL, $c_0=\epsilon=0$ & 2.37 & 3.84 \\
unnormalized OSL w/o denominator, $c_0=4$  & 3.88 & 4.63 \\

\bottomrule
\end{tabular}
\vspace{-8pt}
\end{table}
\vspace{-2pt}

As in our previous analysis, setting $\epsilon,\, c_0>0$ both prevents division by zero (and the resulting numerical instability) and discourages excessive sensitivity near the data manifold.
Neural networks have limited capacity.
If we set $\epsilon=c_0=0$, the objective effectively demands that the model learn the denoising direction accurately in a tiny neighborhood around the data.
In that neighborhood, the correct denoising direction can change very sharply over a very small distance.
To represent such sharp local variation, the model must allocate a large portion of its capacity there, which can be an inefficient use of capacity for generation tasks.
Even with $\epsilon=c_0=0$, the FID does not degrade substantially: we still achieve the best performance among energy-based models and rank second among all time-unconditional methods, indicating robustness to these choices.

By contrast, removing the denominator not only loses the inverse-distance reweighting we proved in~\Cref{app:properties}, but also makes fine-grained denoising difficult to learn.
In the unnormalized form, large displacements dominate the training signal, since they consistently incur larger one-step loss errors, which is unavoidable given the ambiguous posterior discussed in~\Cref{sec:analysis}.
As a result, the model allocates less learning capacity to moderate and lightly noised inputs, whose displacements are inherently smaller.
This imbalance leads to a substantially worse FID for the unnormalized one-step loss.

\newcommand{\gridImgW}{0.19\linewidth} %
\newcommand{\gridGap}{1.2mm}           %
\newcommand{\subFigVGap}{2.5mm}        %

\begin{figure*}[t]
\centering
\setlength{\tabcolsep}{0pt}
\renewcommand{\arraystretch}{0}

\subcaptionbox{\dm with sphere tracing, $\eta=0.022$, no CFG.\label{fig:st-grid}}{%
\begin{tabular}{@{}ccccc@{}}
\includegraphics[width=\gridImgW]{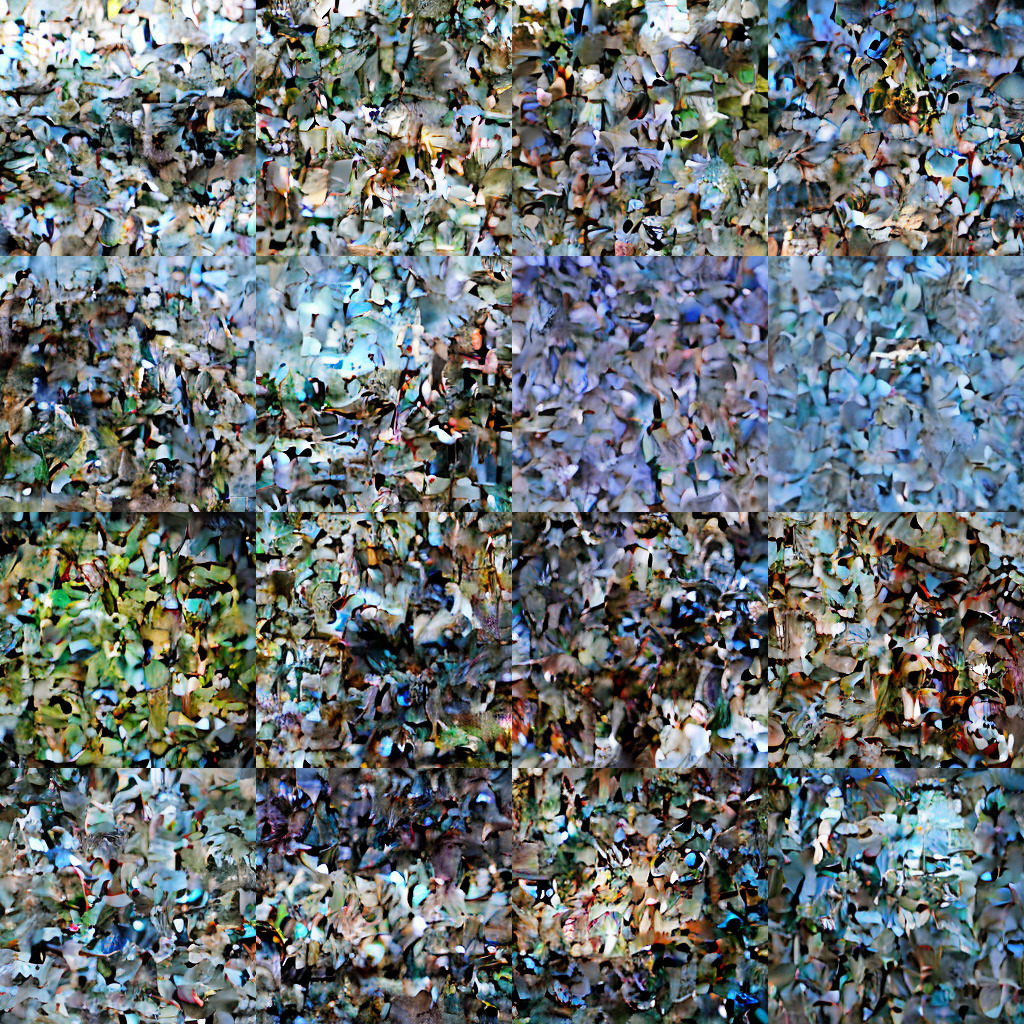} &
\hspace{\gridGap}\includegraphics[width=\gridImgW]{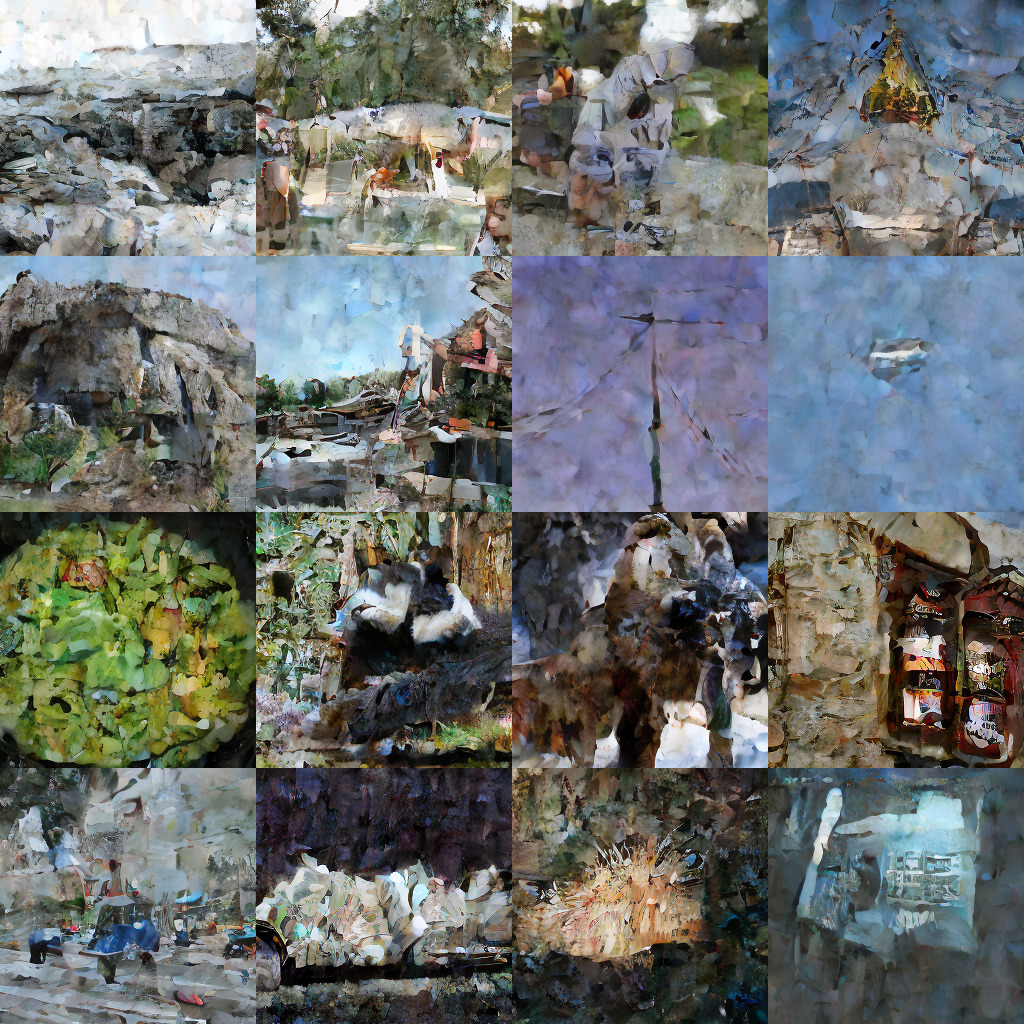} &
\hspace{\gridGap}\includegraphics[width=\gridImgW]{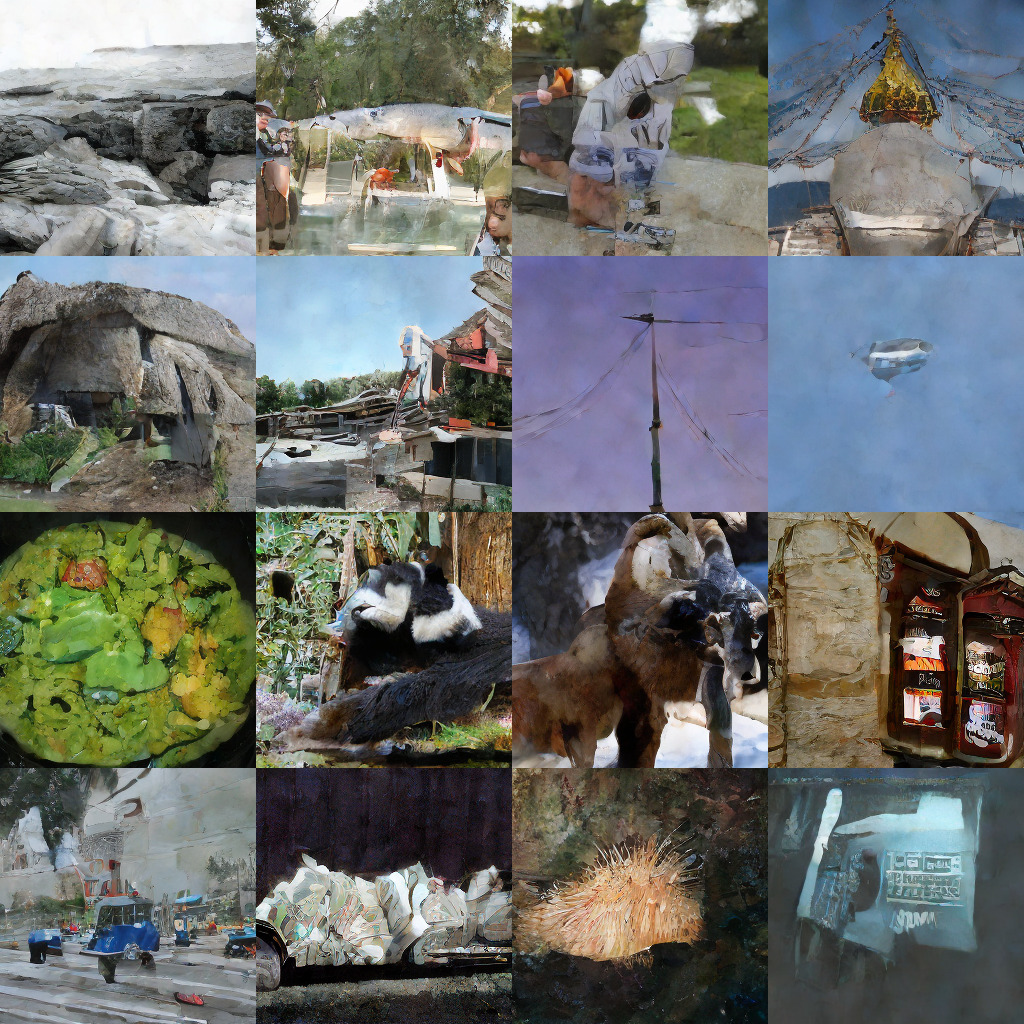} &
\hspace{\gridGap}\includegraphics[width=\gridImgW]{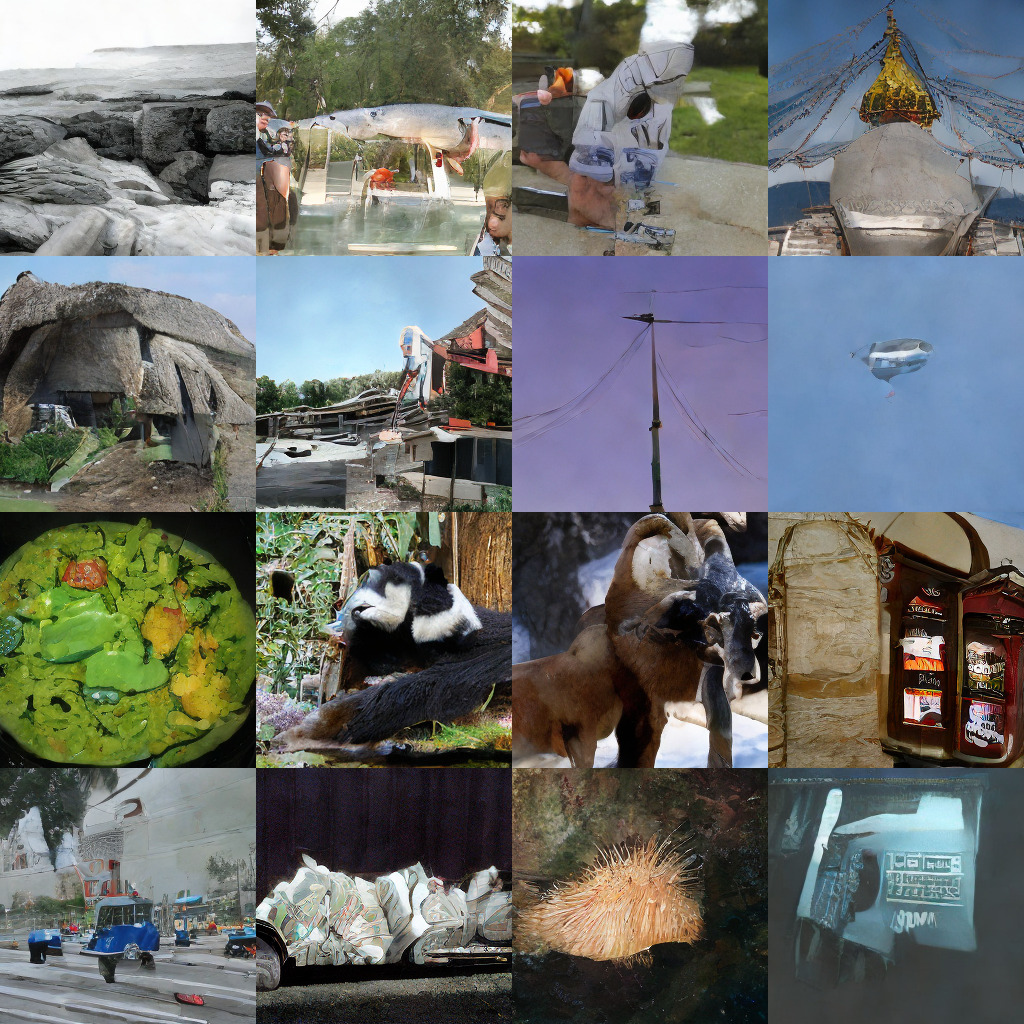} &
\hspace{\gridGap}\includegraphics[width=\gridImgW]{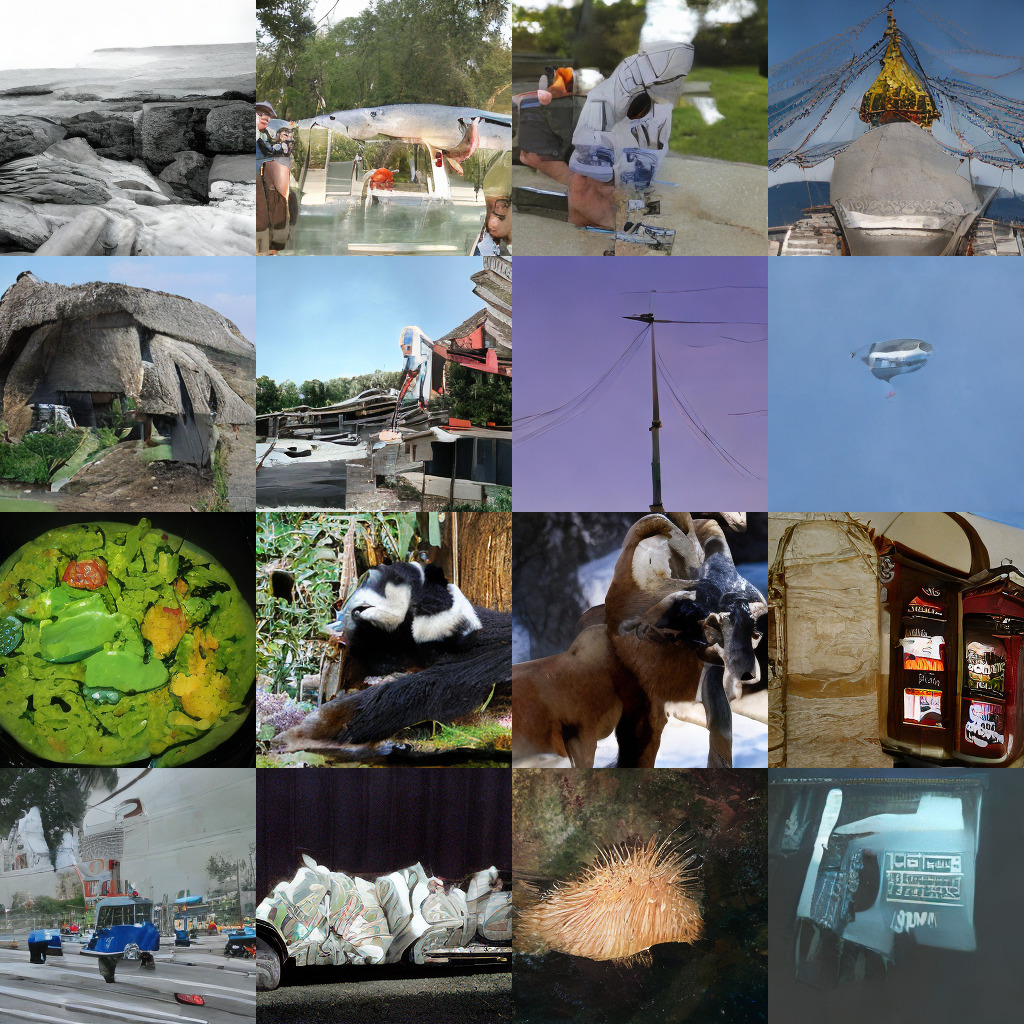} \\
\multicolumn{5}{@{}l@{}}{\vspace{\gridGap}}\\
\includegraphics[width=\gridImgW]{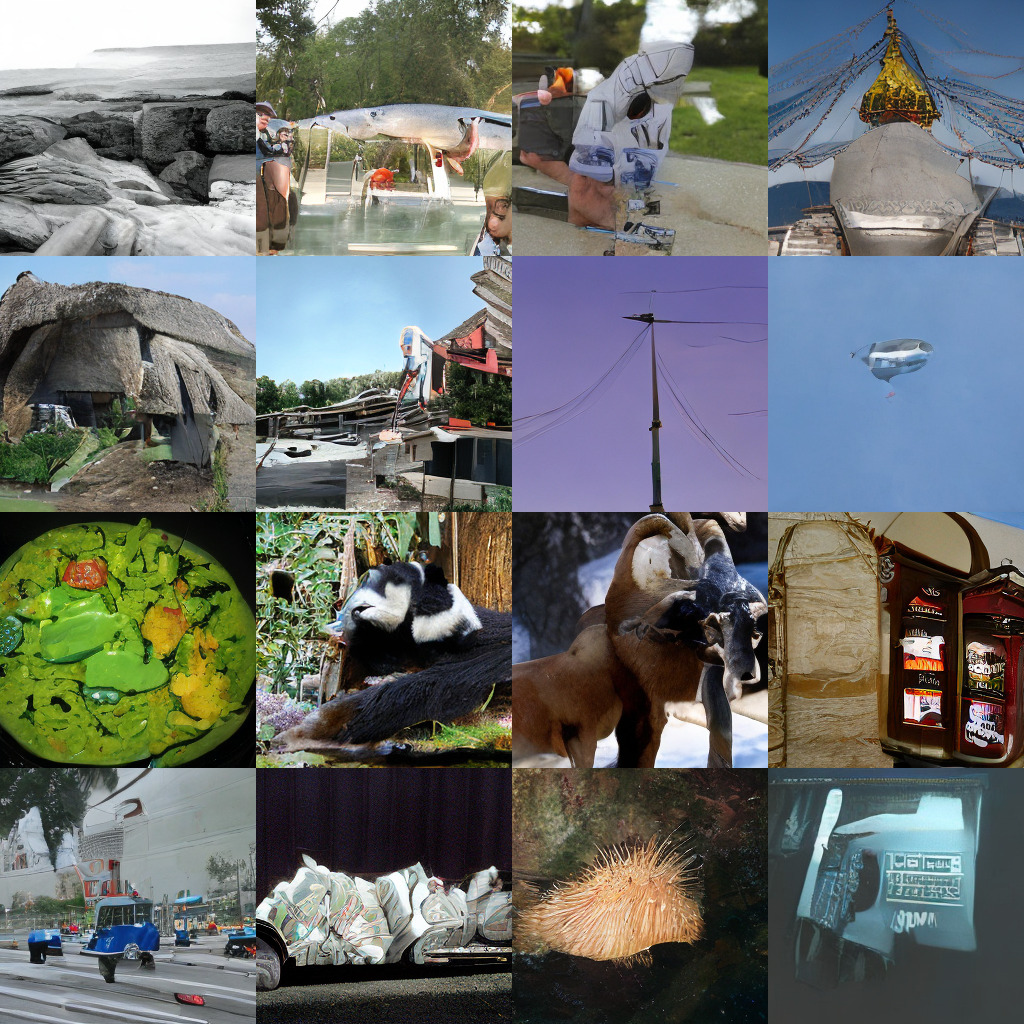} &
\hspace{\gridGap}\includegraphics[width=\gridImgW]{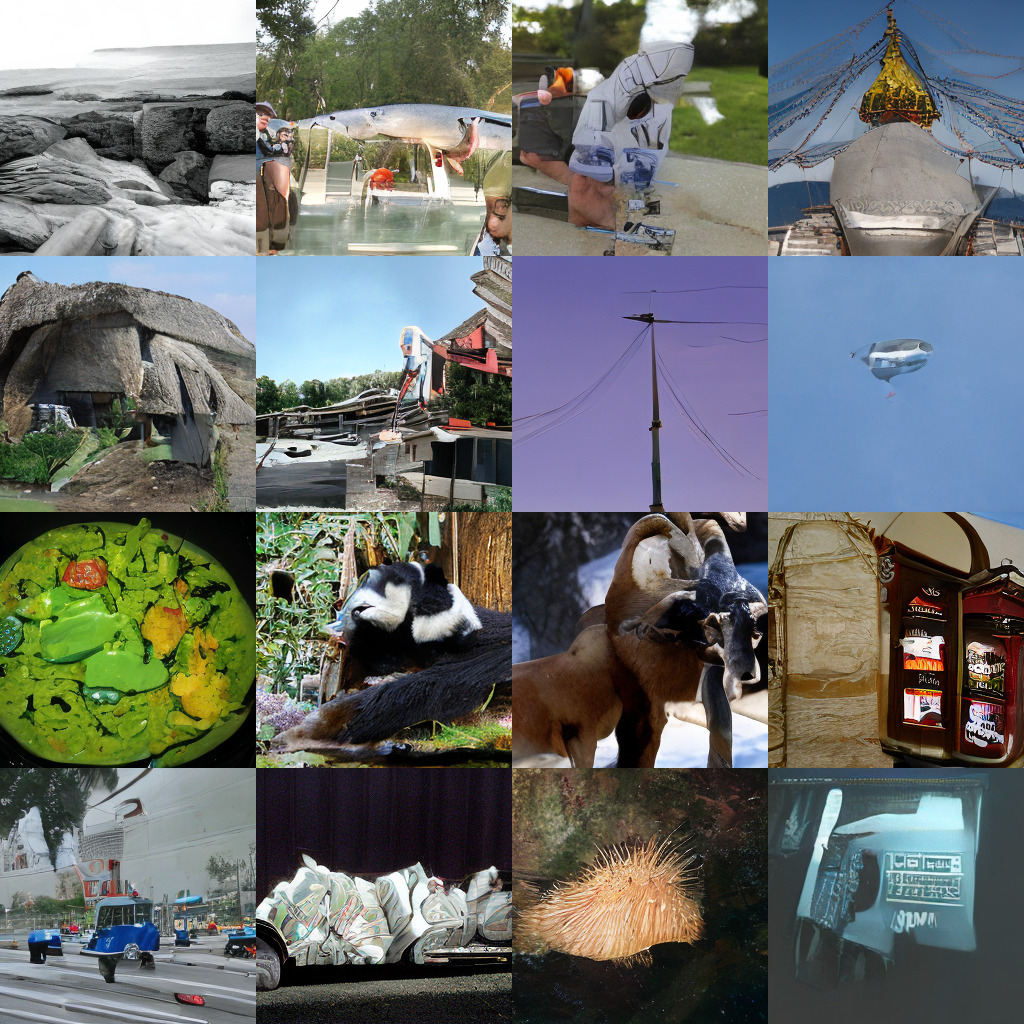} &
\hspace{\gridGap}\includegraphics[width=\gridImgW]{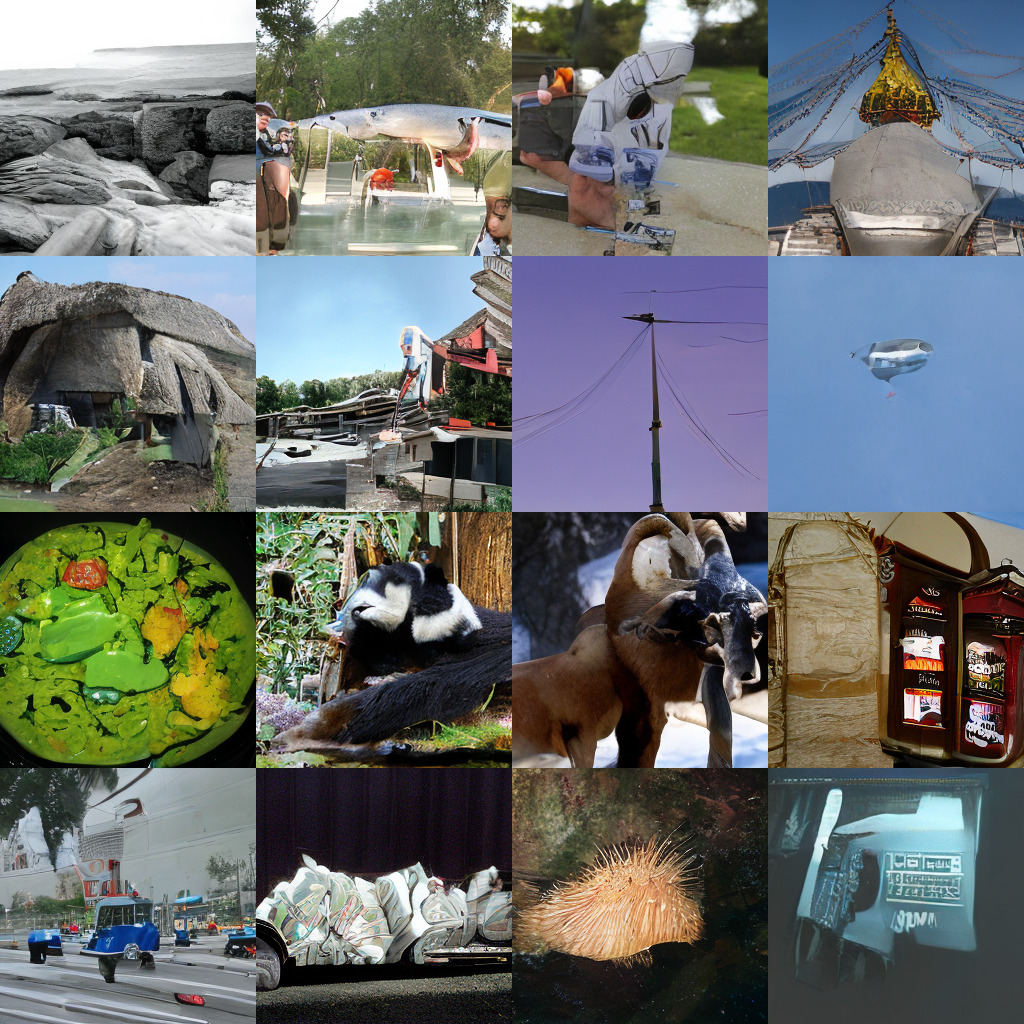} &
\hspace{\gridGap}\includegraphics[width=\gridImgW]{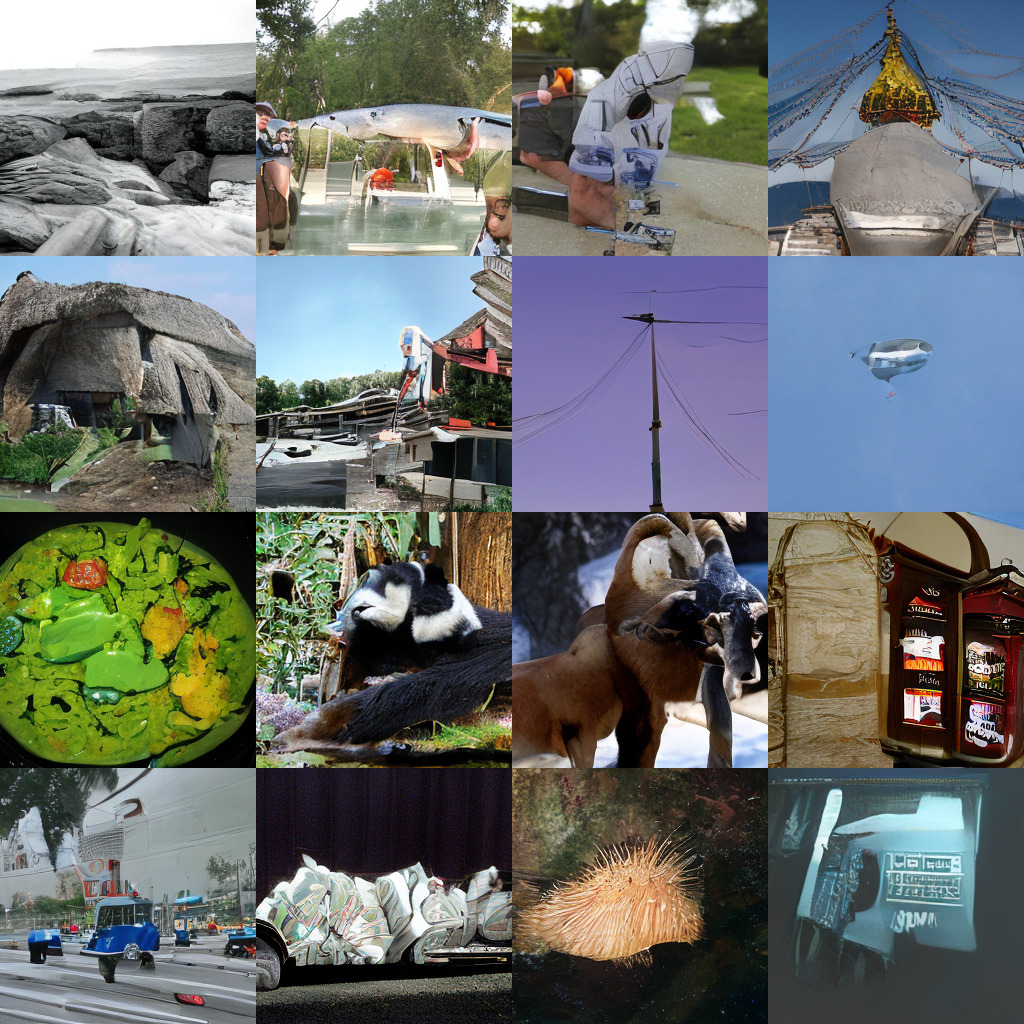} &
\hspace{\gridGap}\includegraphics[width=\gridImgW]{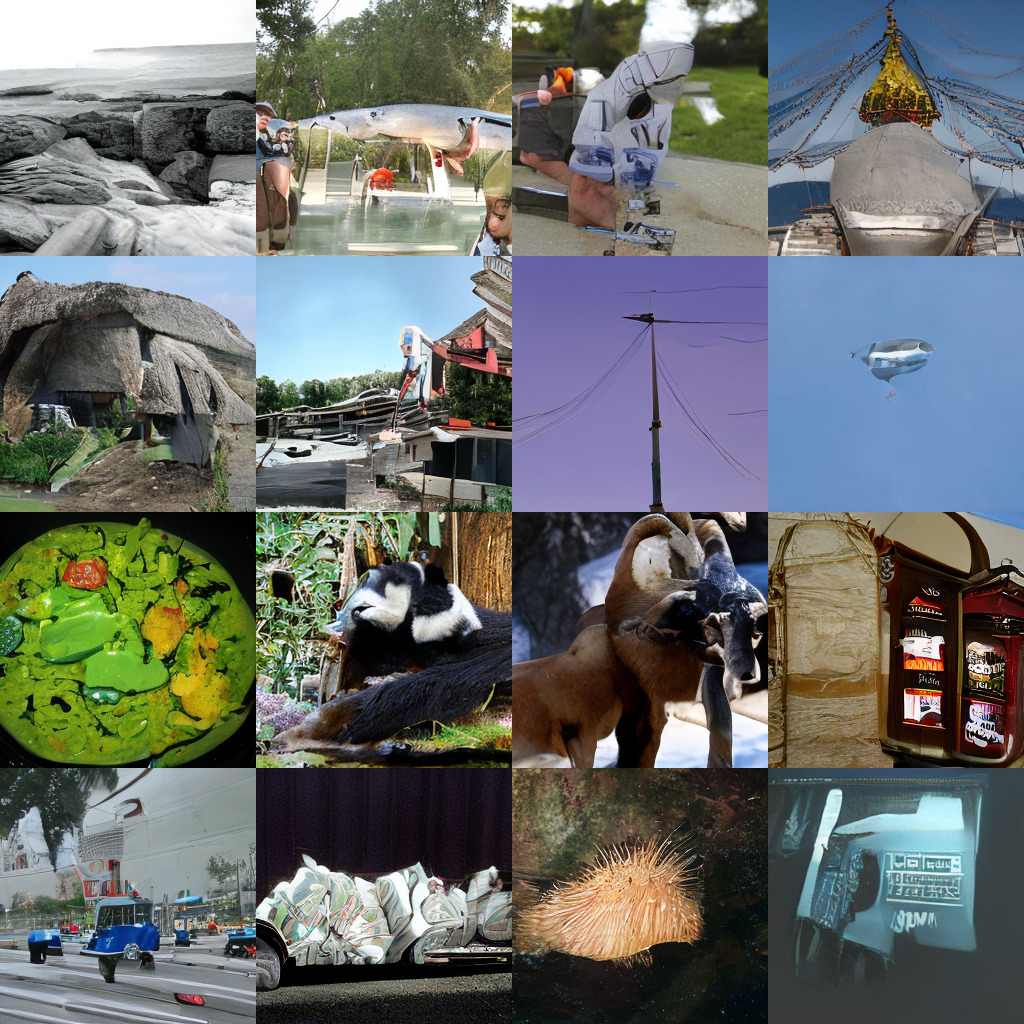} \\
\end{tabular}%
}

\vspace{\subFigVGap}

\subcaptionbox{\dm with gradient descent, $\eta =0.64$, no CFG.\label{fig:gd-grid}}{%
\begin{tabular}{@{}ccccc@{}}
\includegraphics[width=\gridImgW]{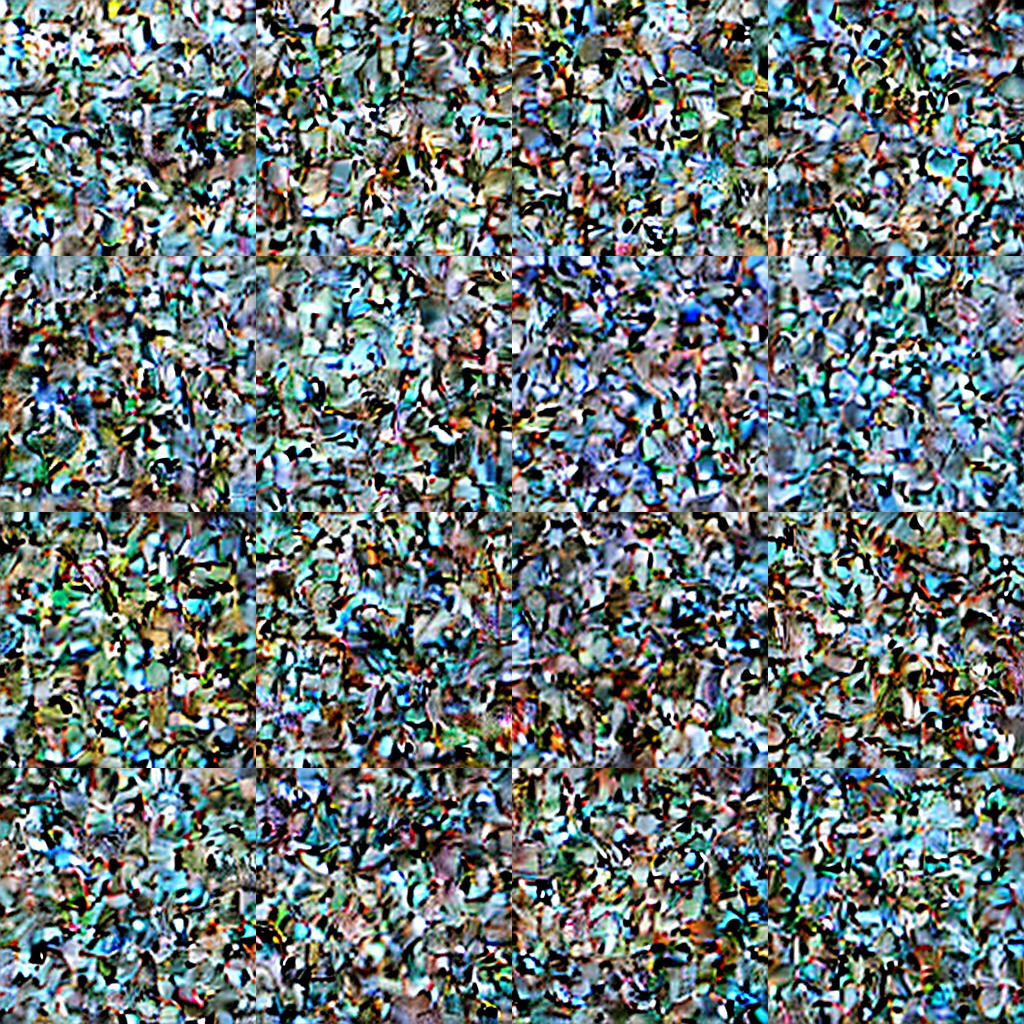} &
\hspace{\gridGap}\includegraphics[width=\gridImgW]{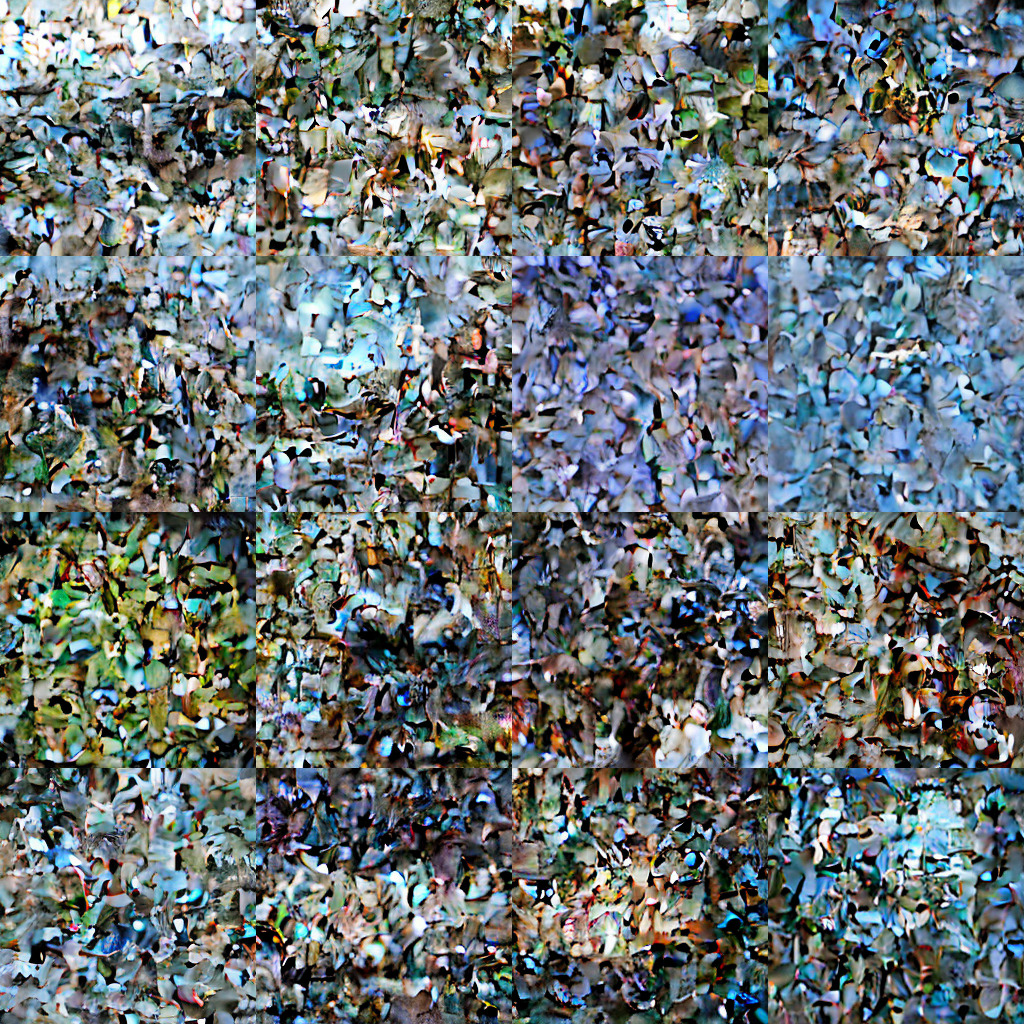} &
\hspace{\gridGap}\includegraphics[width=\gridImgW]{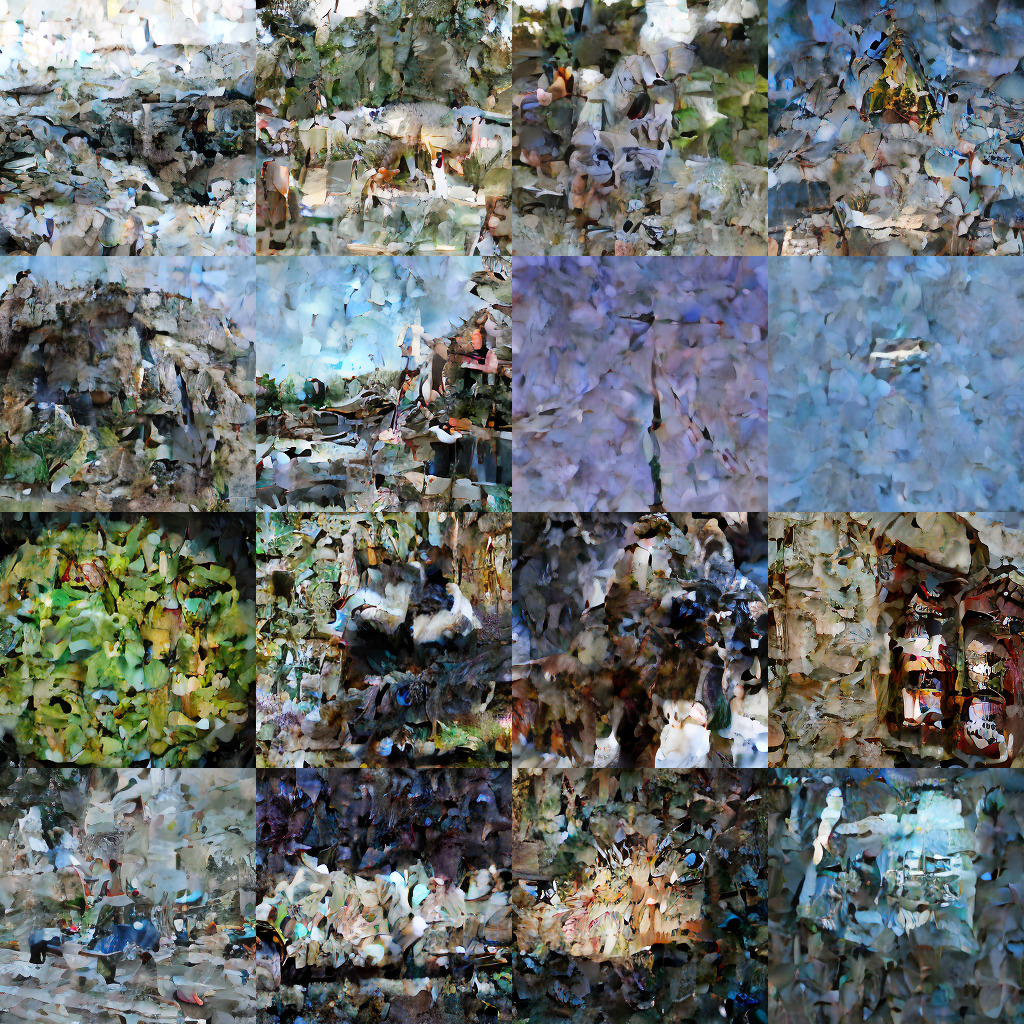} &
\hspace{\gridGap}\includegraphics[width=\gridImgW]{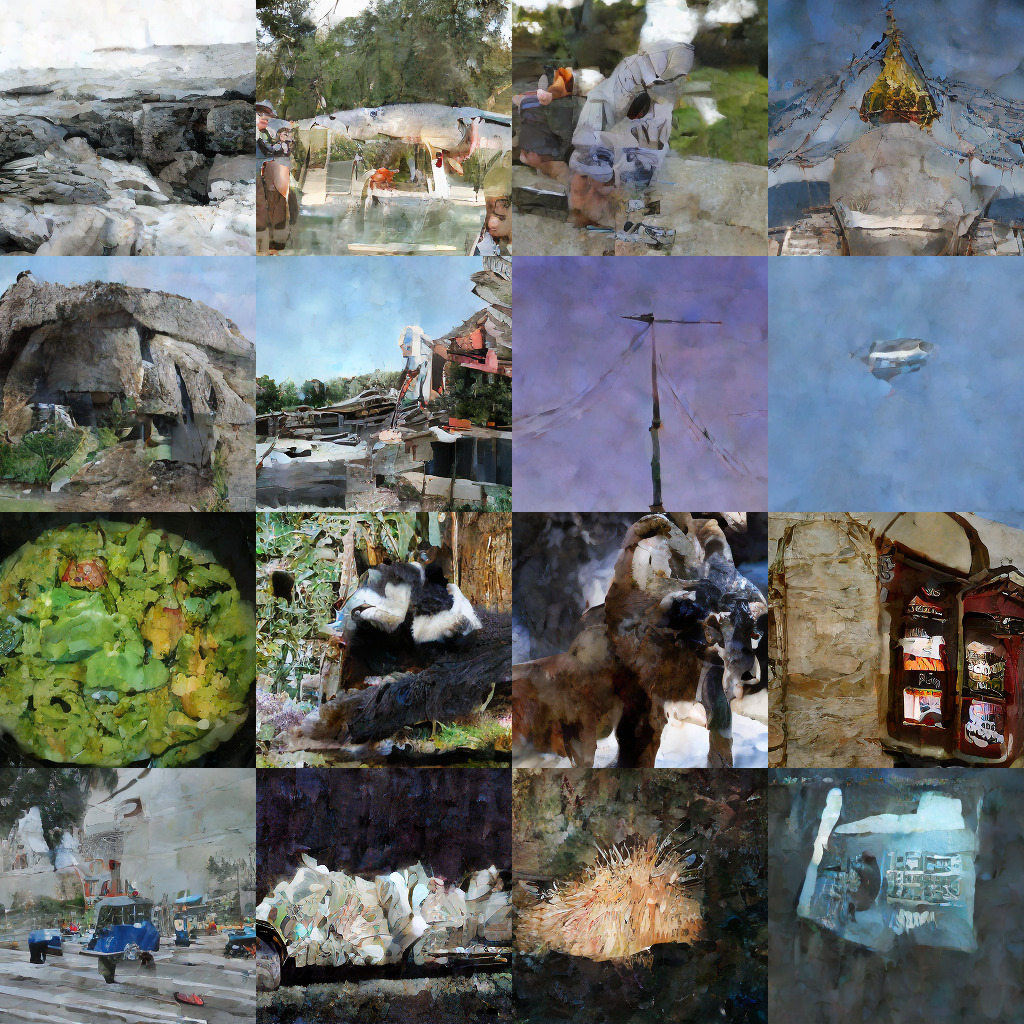} &
\hspace{\gridGap}\includegraphics[width=\gridImgW]{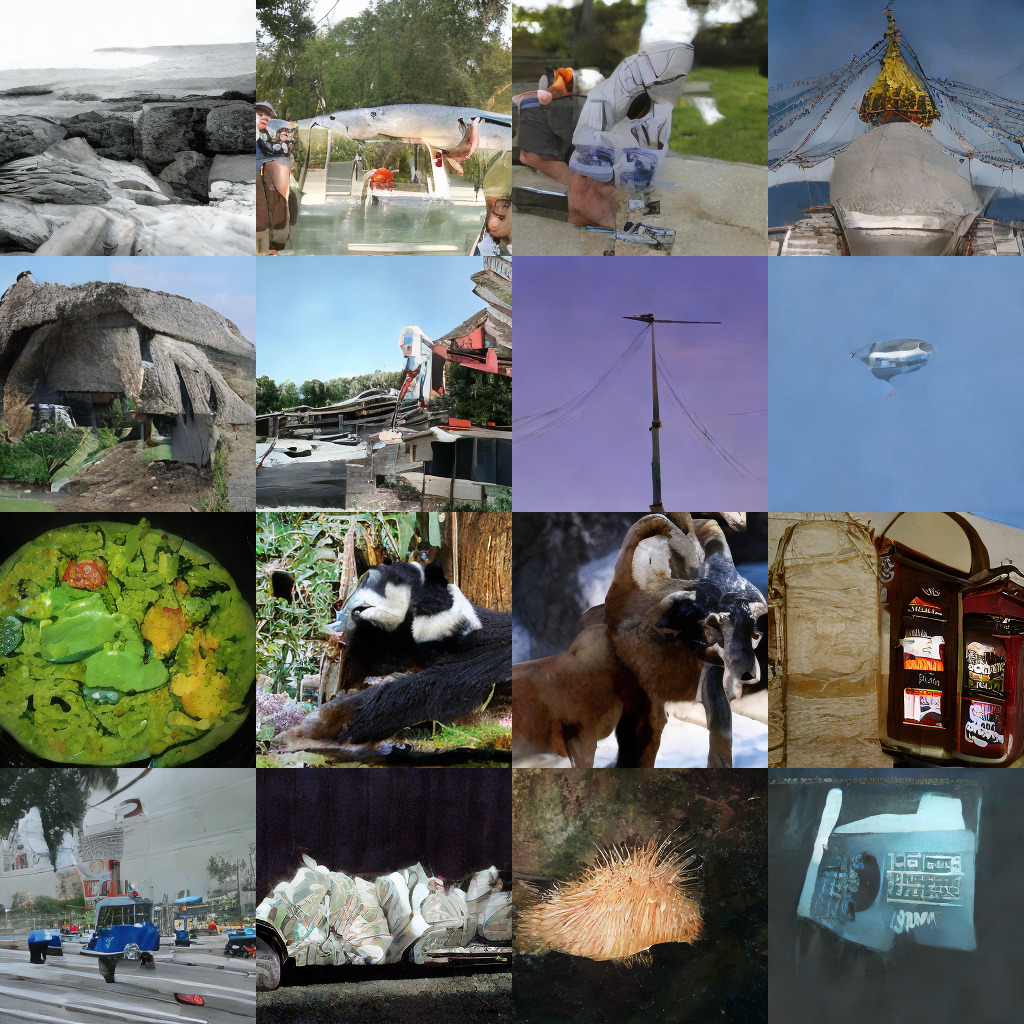} \\
\multicolumn{5}{@{}l@{}}{\vspace{\gridGap}}\\
\includegraphics[width=\gridImgW]{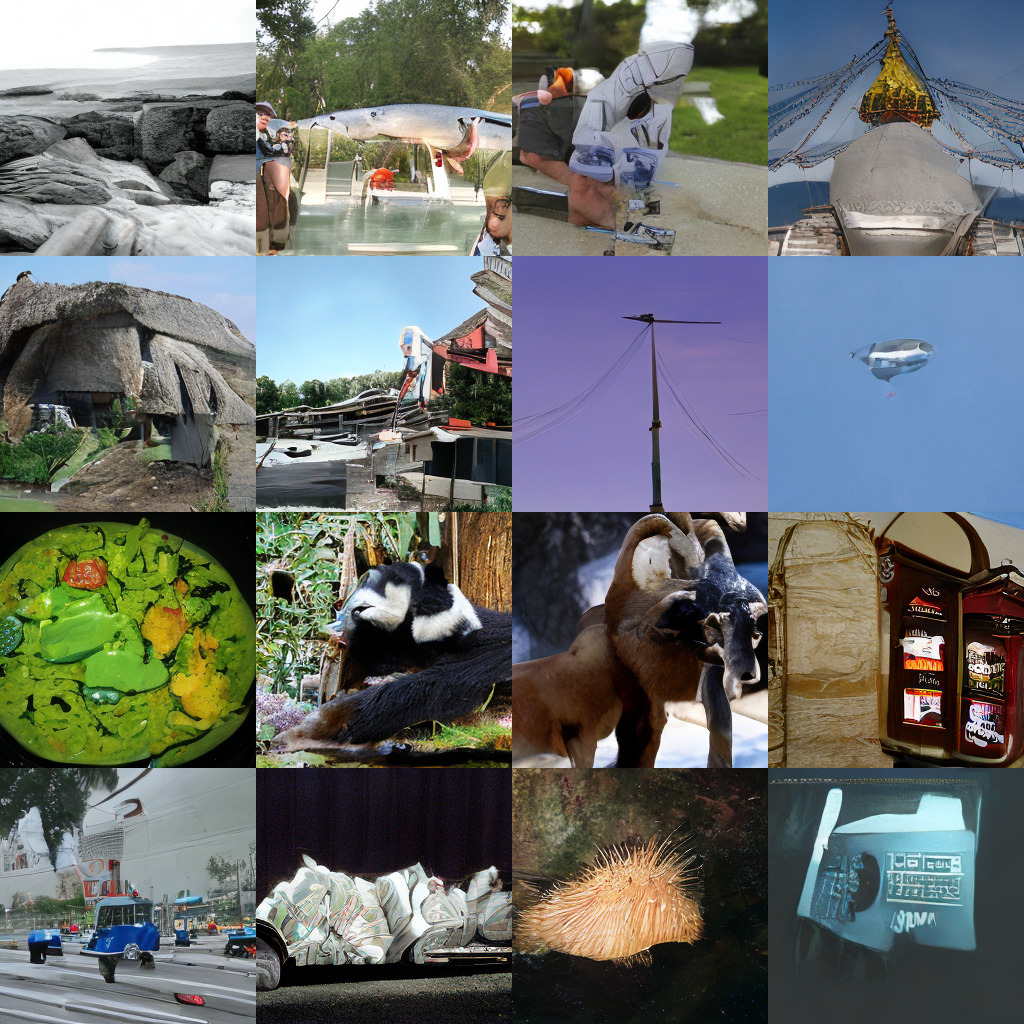} &
\hspace{\gridGap}\includegraphics[width=\gridImgW]{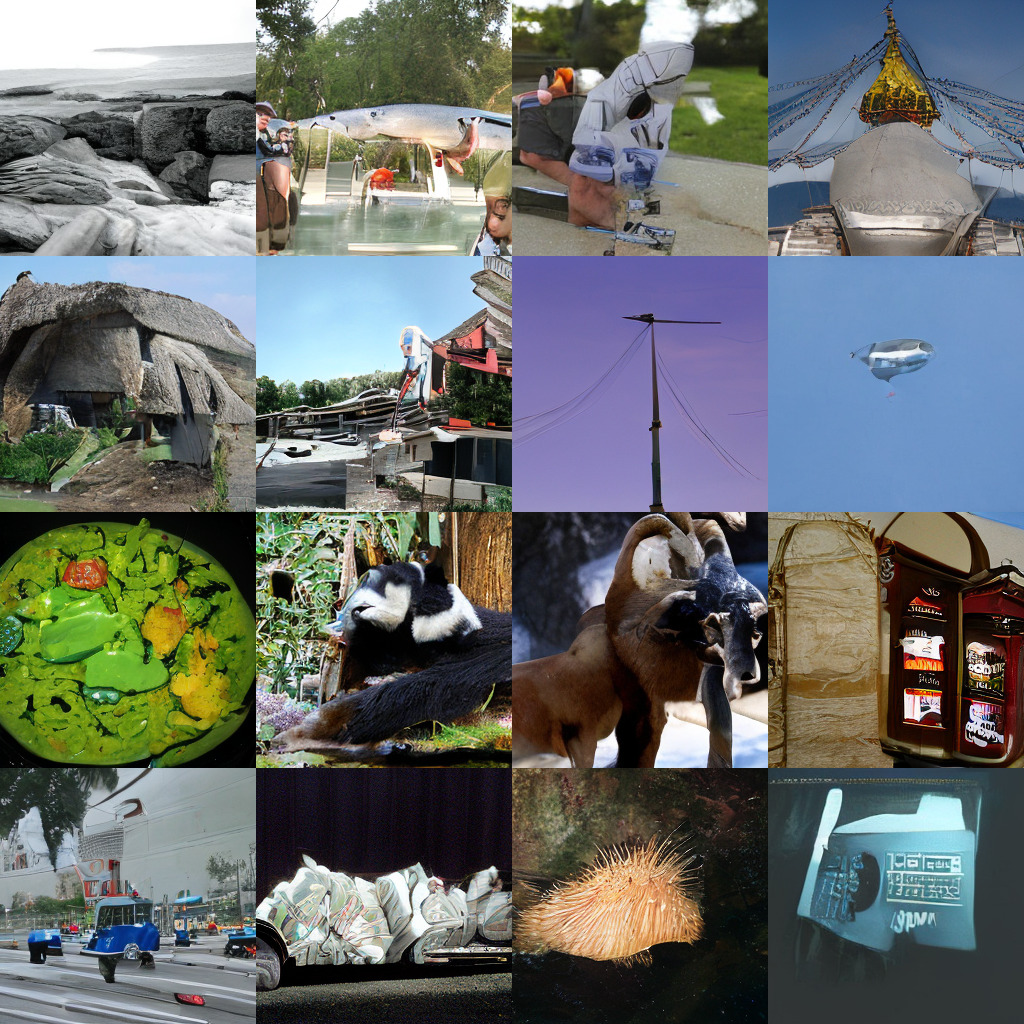} &
\hspace{\gridGap}\includegraphics[width=\gridImgW]{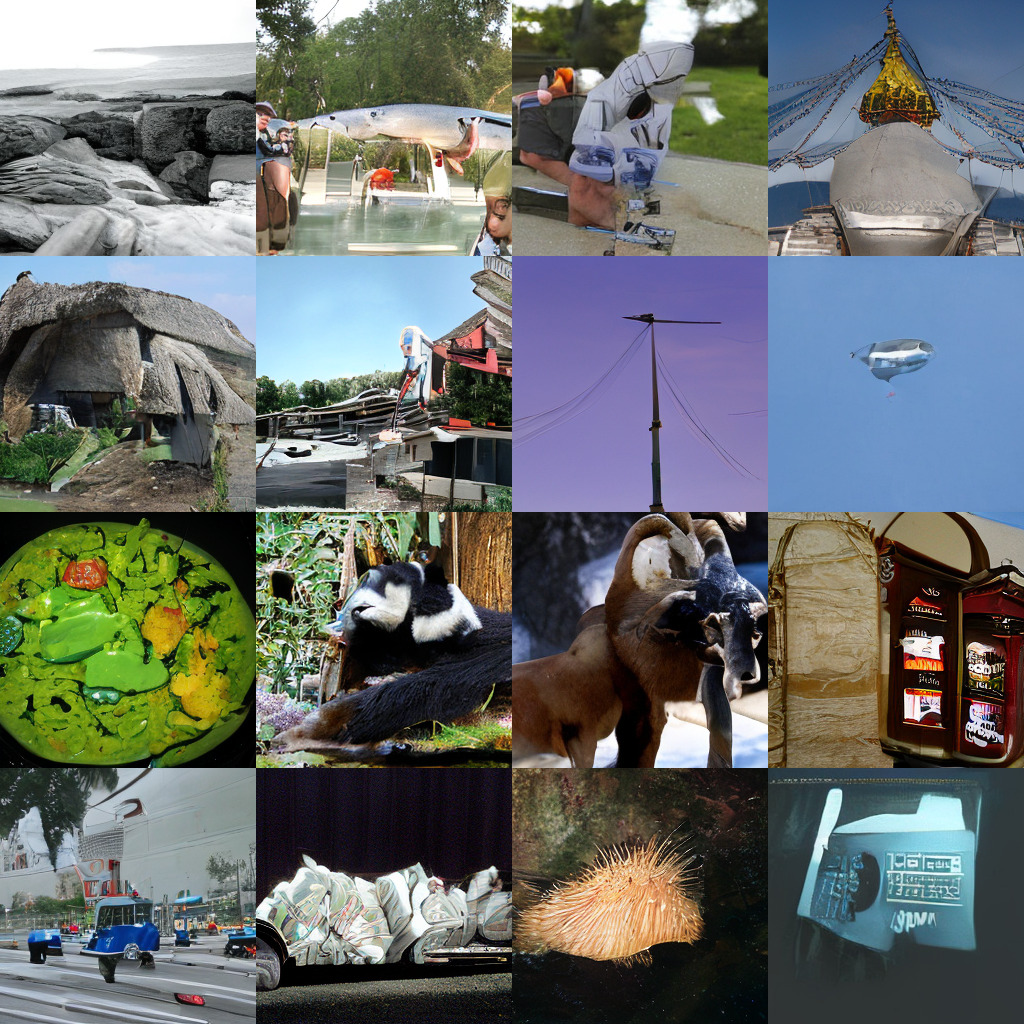} &
\hspace{\gridGap}\includegraphics[width=\gridImgW]{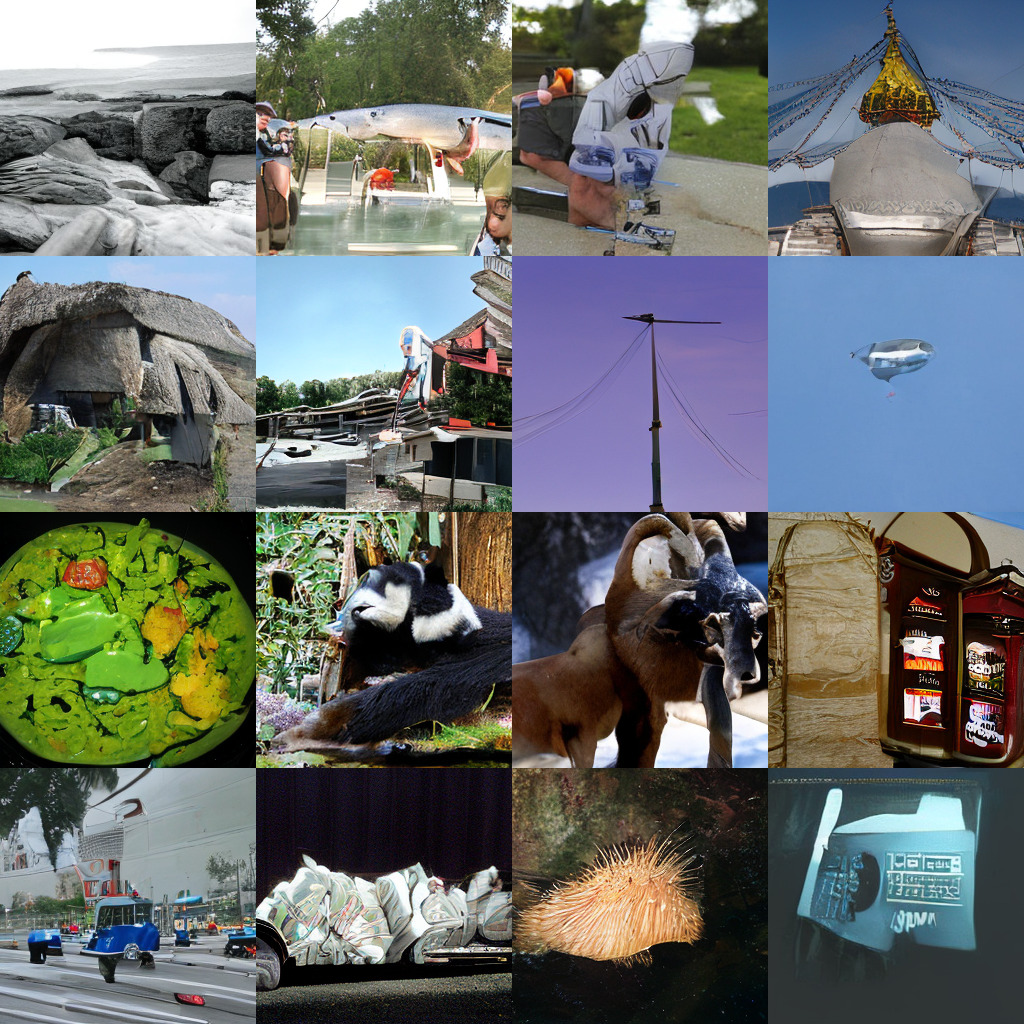} &
\hspace{\gridGap}\includegraphics[width=\gridImgW]{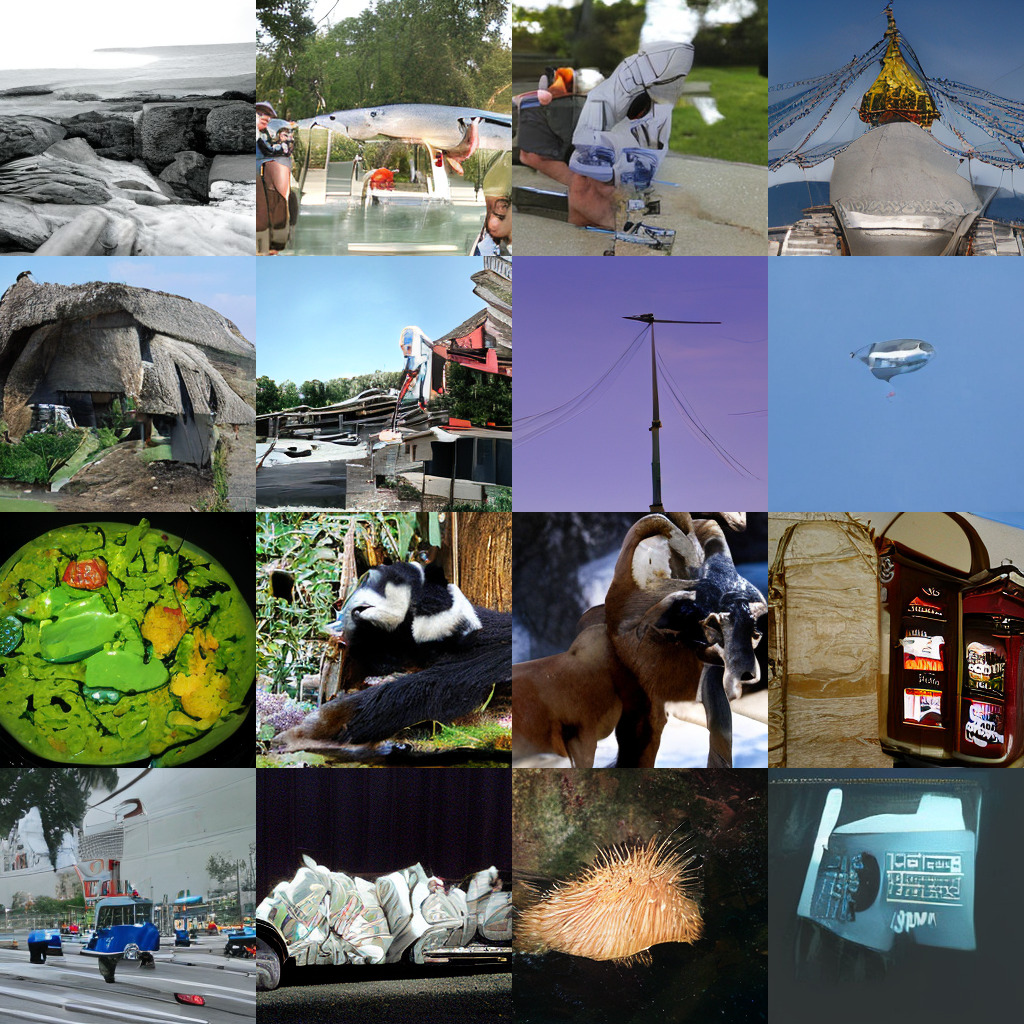} \\
\end{tabular}%
}

\vspace{\subFigVGap}

\subcaptionbox{Flow matching SiT with ODE solver dopri5, no CFG.\label{fig:sit-grid}}{%
\begin{tabular}{@{}ccccc@{}}
\includegraphics[width=\gridImgW]{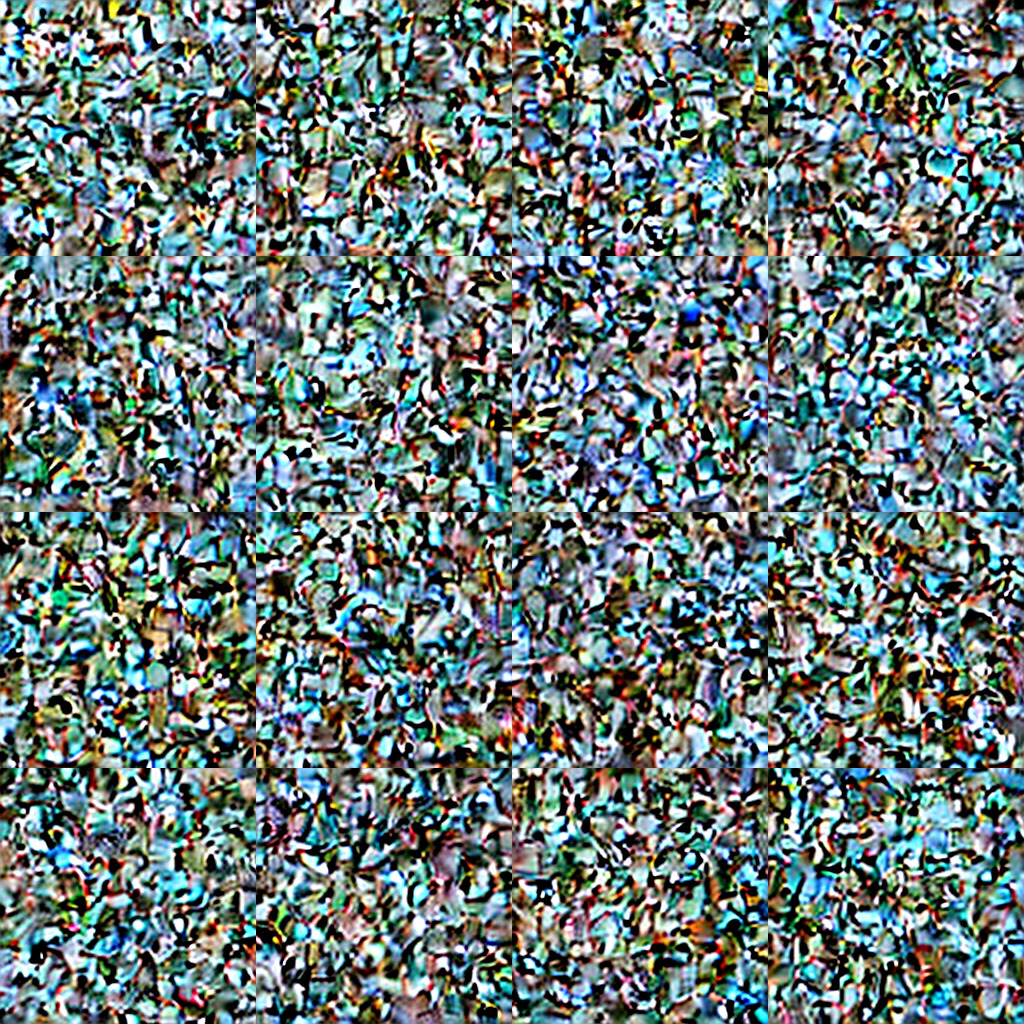} &
\hspace{\gridGap}\includegraphics[width=\gridImgW]{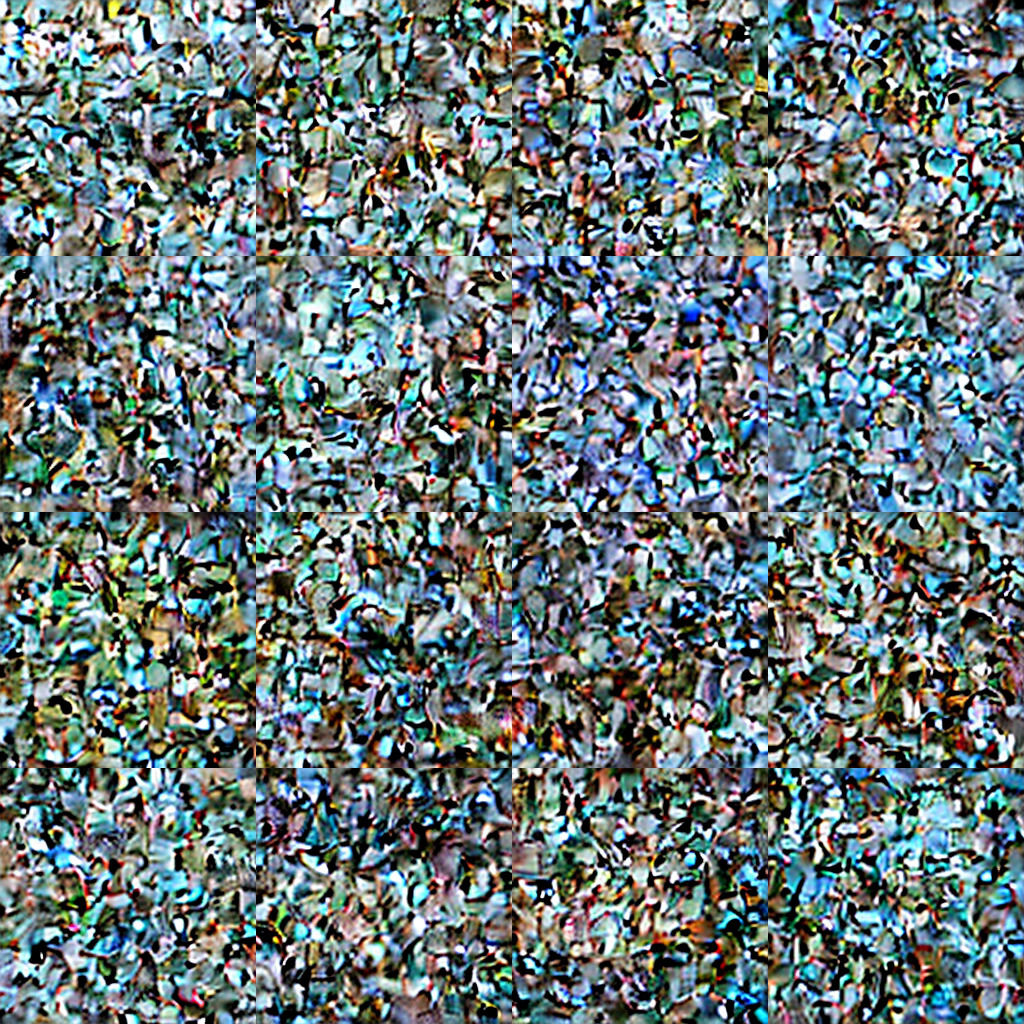} &
\hspace{\gridGap}\includegraphics[width=\gridImgW]{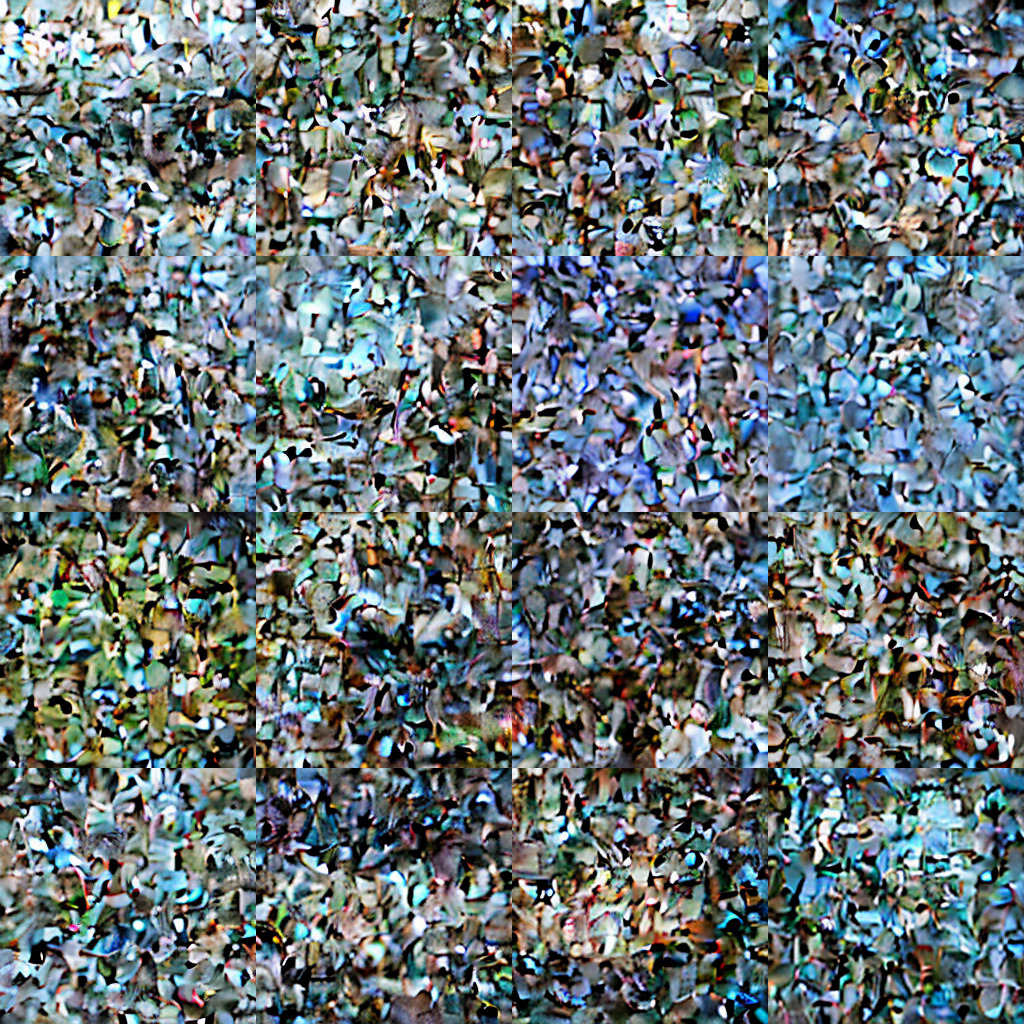} &
\hspace{\gridGap}\includegraphics[width=\gridImgW]{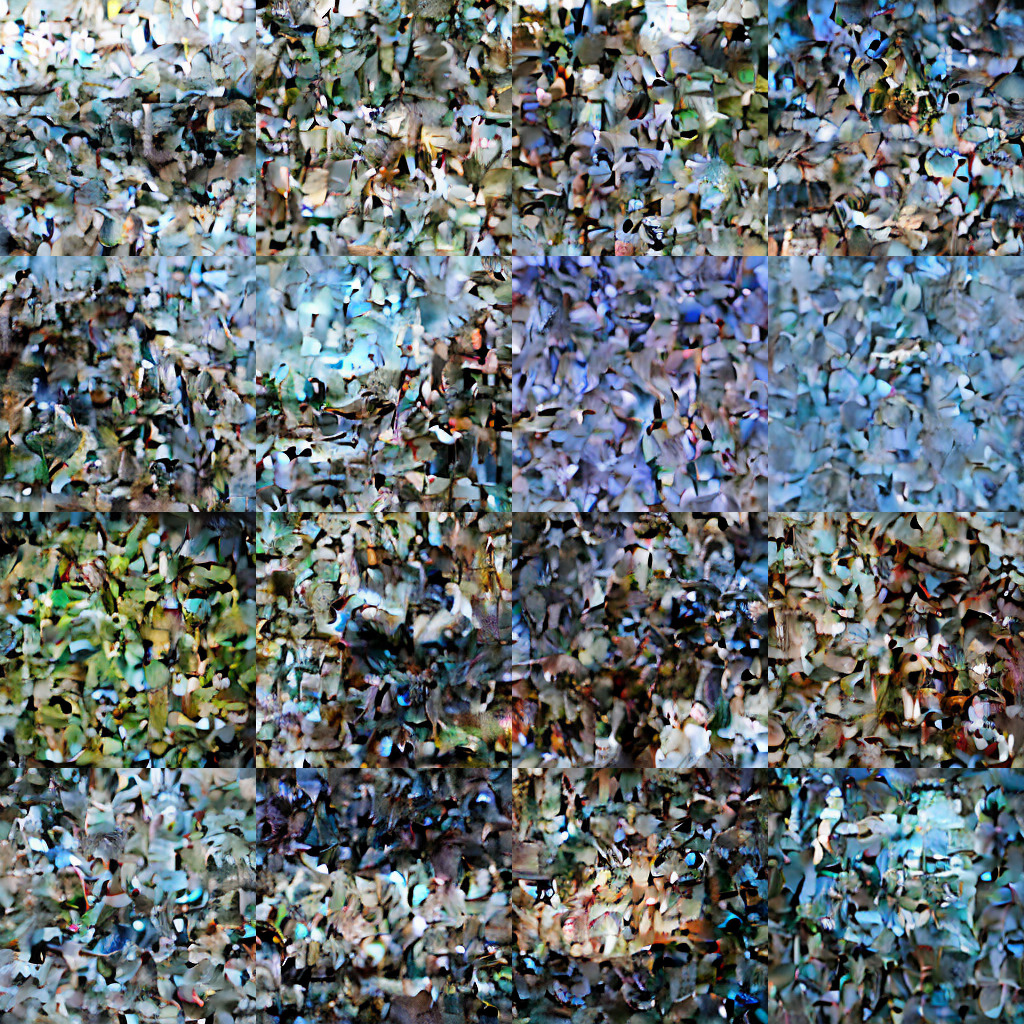} &
\hspace{\gridGap}\includegraphics[width=\gridImgW]{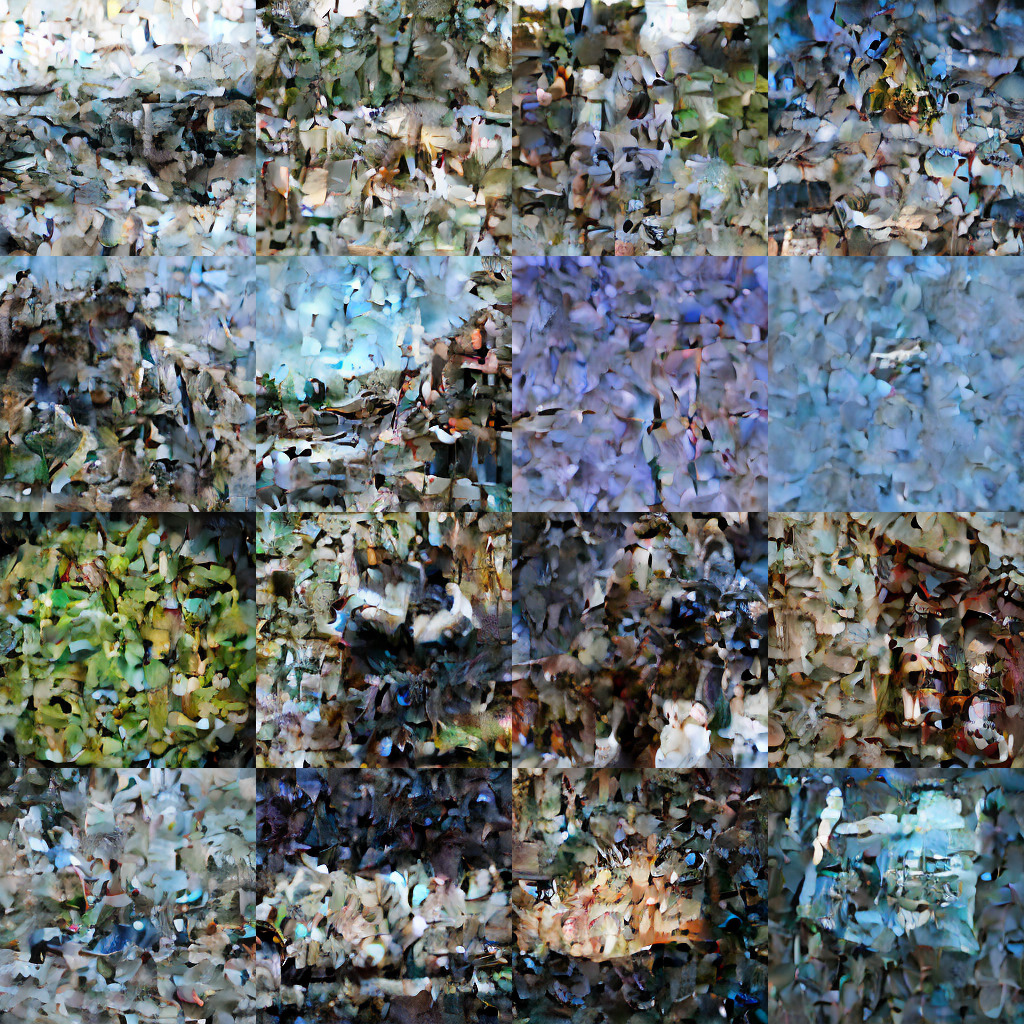} \\
\multicolumn{5}{@{}l@{}}{\vspace{\gridGap}}\\
\includegraphics[width=\gridImgW]{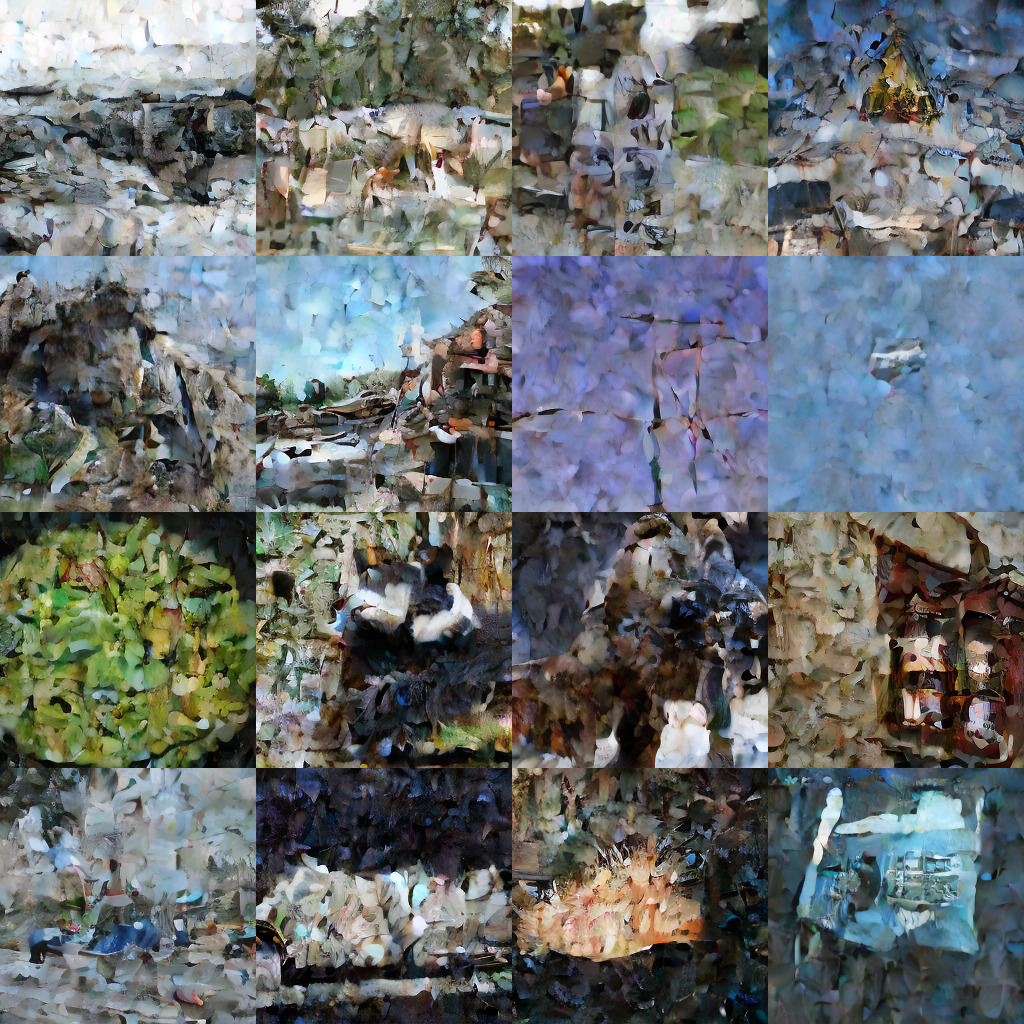} &
\hspace{\gridGap}\includegraphics[width=\gridImgW]{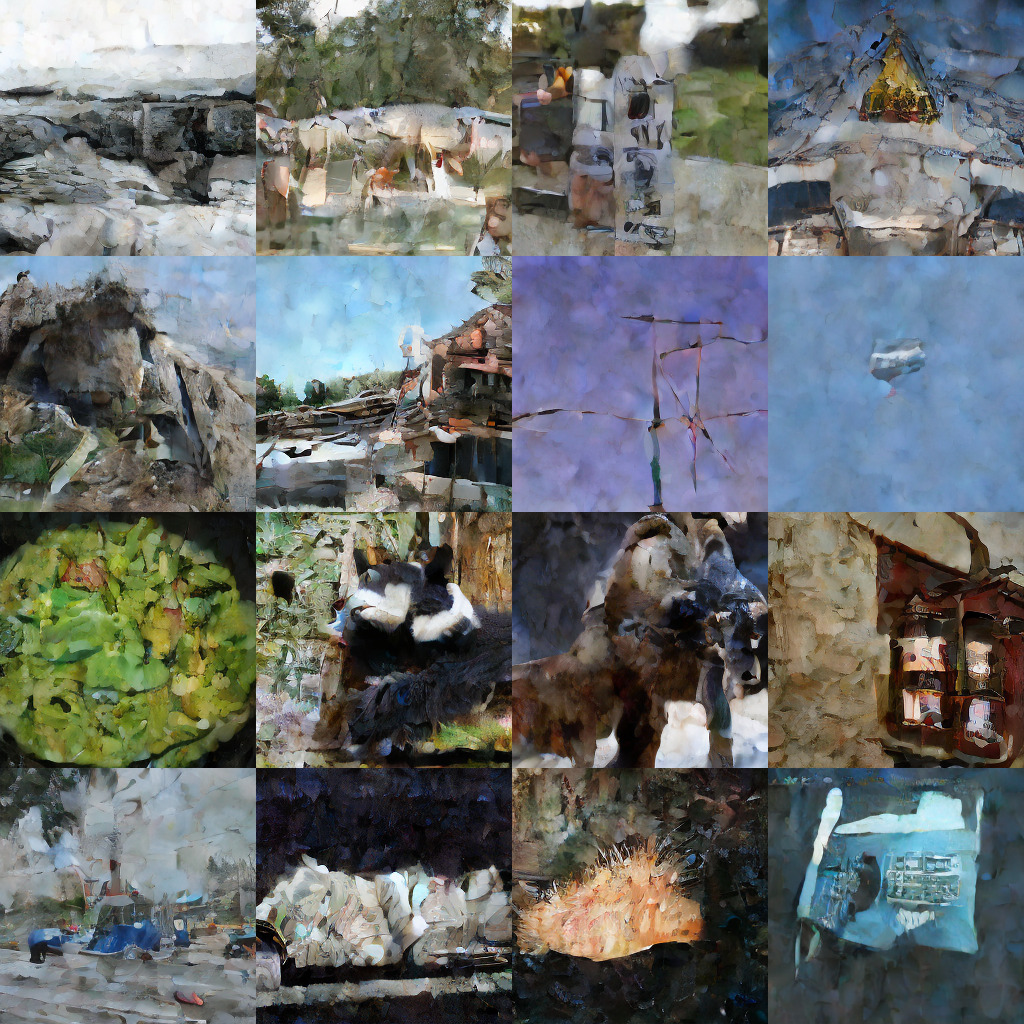} &
\hspace{\gridGap}\includegraphics[width=\gridImgW]{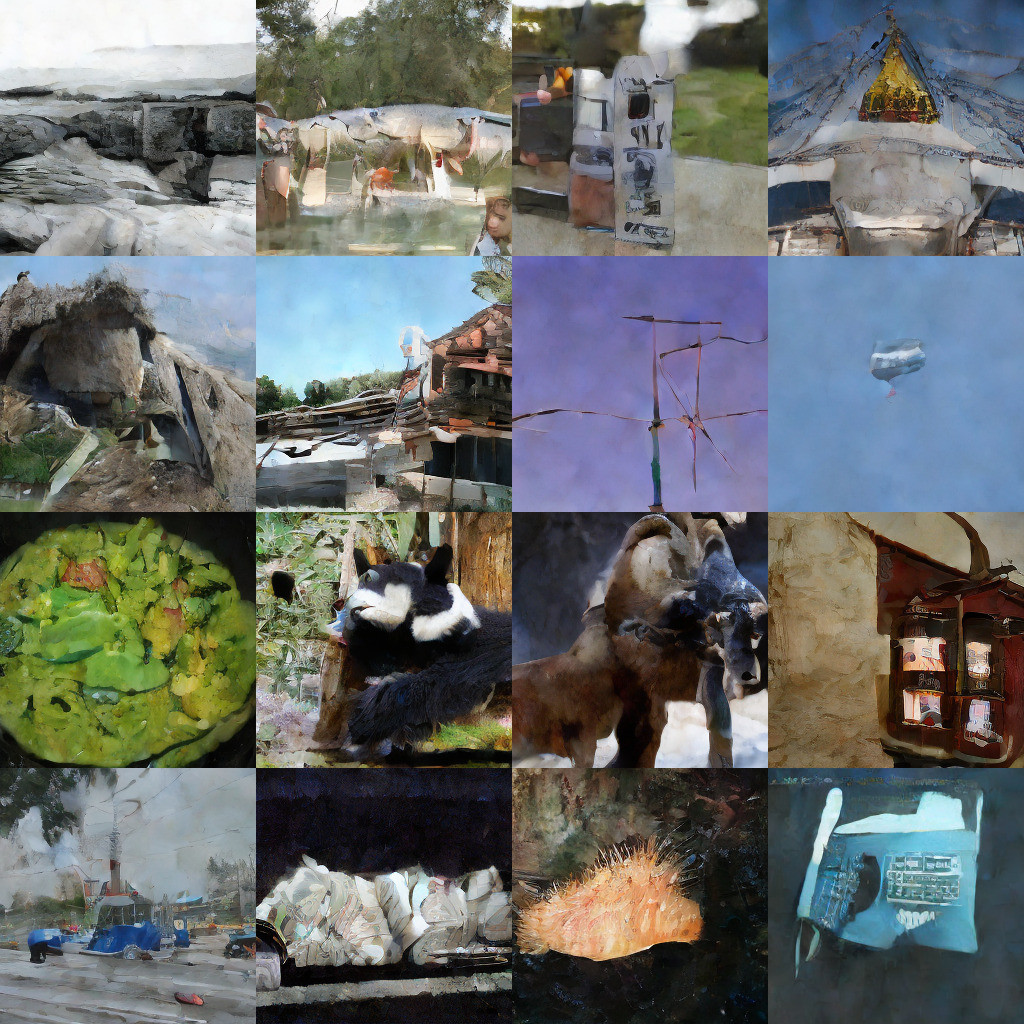} &
\hspace{\gridGap}\includegraphics[width=\gridImgW]{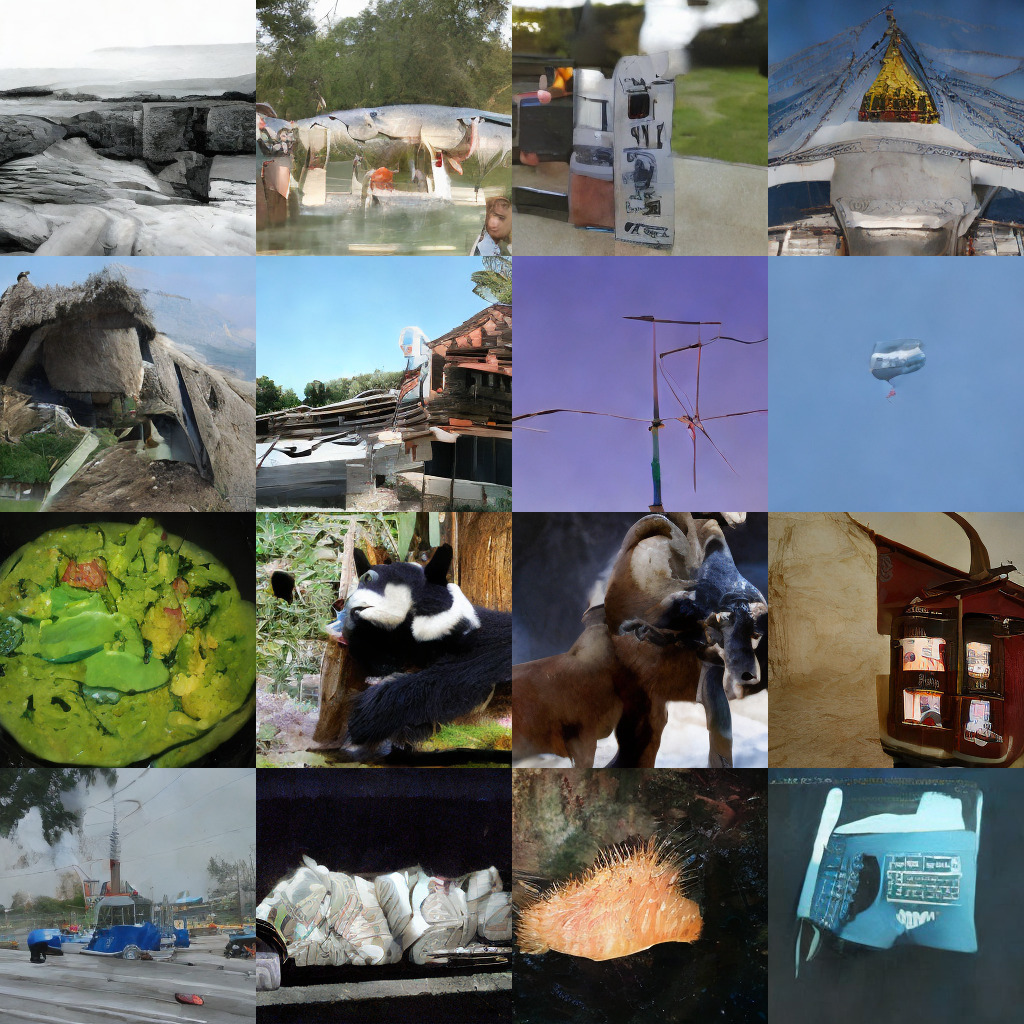} &
\hspace{\gridGap}\includegraphics[width=\gridImgW]{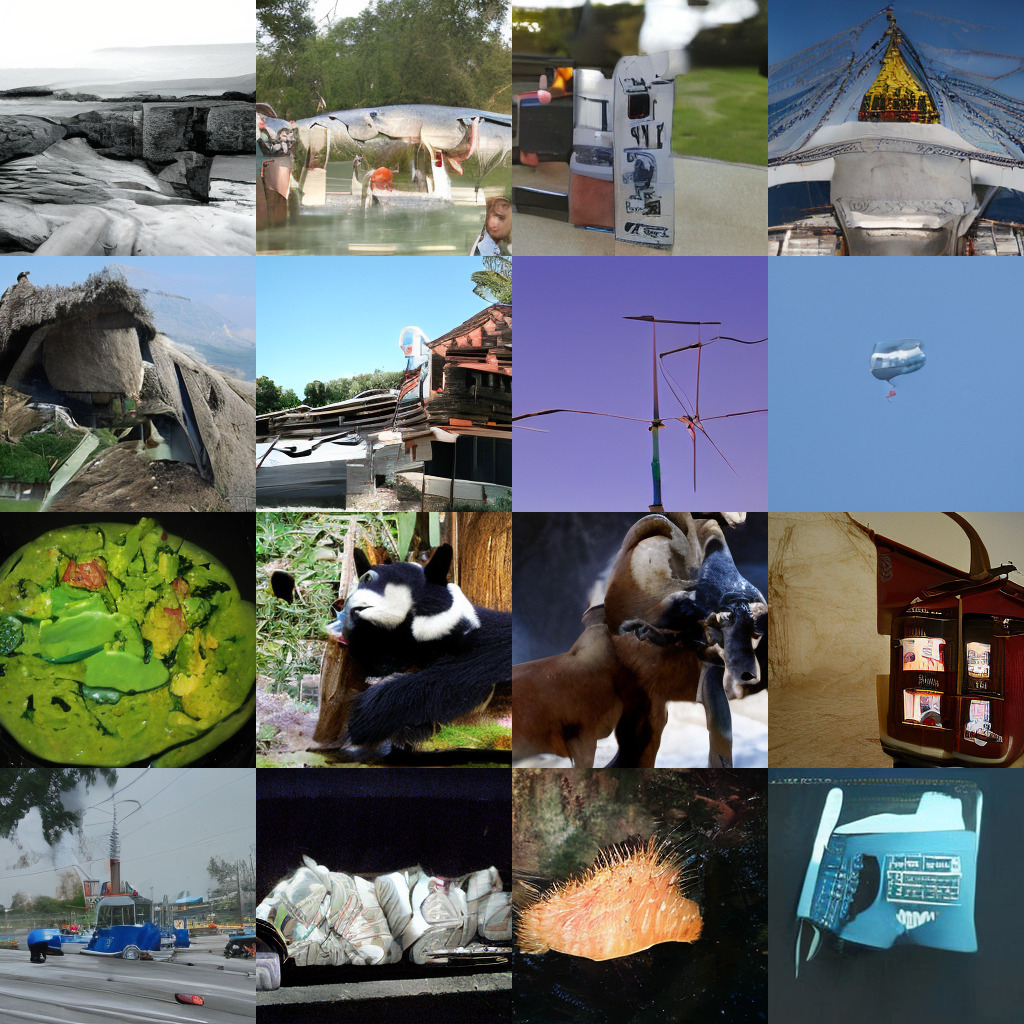} \\
\end{tabular}%
}

\caption{We perform 250-step generation using SiT-B/2 models trained for 80 epochs under their respective losses, starting from the same initial noise at inference.
For each method, the step size is chosen to maximize the final performance after 250 steps.
We visualize intermediate outputs at 10\%--100\% progress in 10\% increments.}
\label{fig:evolution}
\end{figure*}

\clearpage

\subsection{Supplementary Tables and Figures}
\label{sec:exp_out}

For several baselines, we directly extract numeric values from the vector graphics metadata in Fig.~3 of~\citet{wang2025equilibrium}, Fig.~8 of~\citet{sun2025noise}, and Fig.~8 of~\citet{peebles2023scalable}.
We report these values to two decimal places and expect them to match the originals up to this precision.
Due to space constraints in the main paper, we provide additional results here.

\begin{figure}[!htbp]
    \centering
    \includegraphics[width=0.55\linewidth]{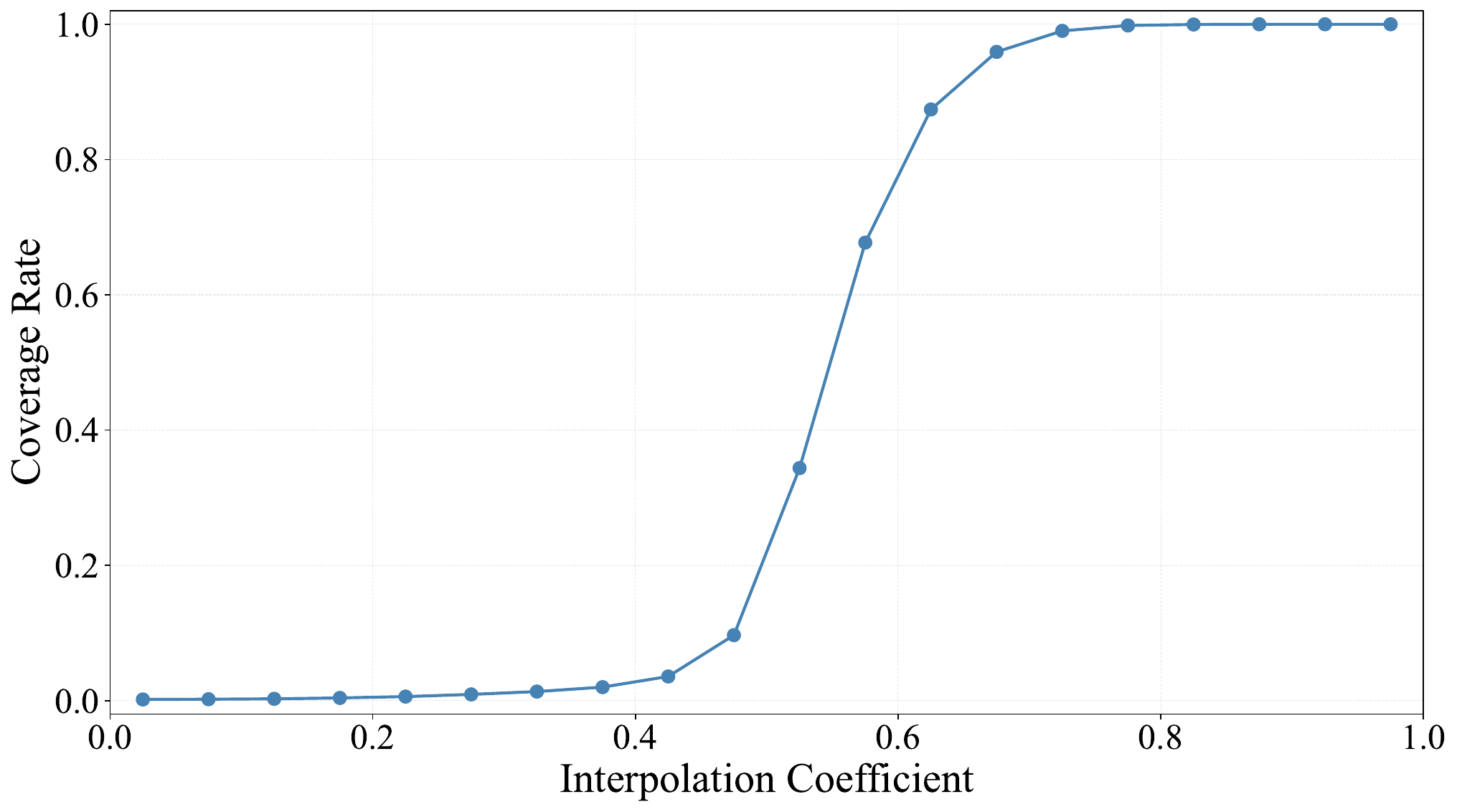}
    \caption{We bin the interpolation coefficient $t\in[0,1]$ uniformly; in each bin, we draw 50k Gaussian noise samples, form linear interpolations, and select their nearest neighbors in the training set, and then compute the coverage.
    Coverage rate changes abruptly in the middle, remaining near $0$ on the noisy side and saturating near $1$ on the data side, indicating a highly non-uniform nearest-neighbor visitation across the interpolation.}
    \label{fig:mode_collapse_curve}
\end{figure}

\begin{figure}[!htbp]
    \centering
    \includegraphics[width=0.55\linewidth]{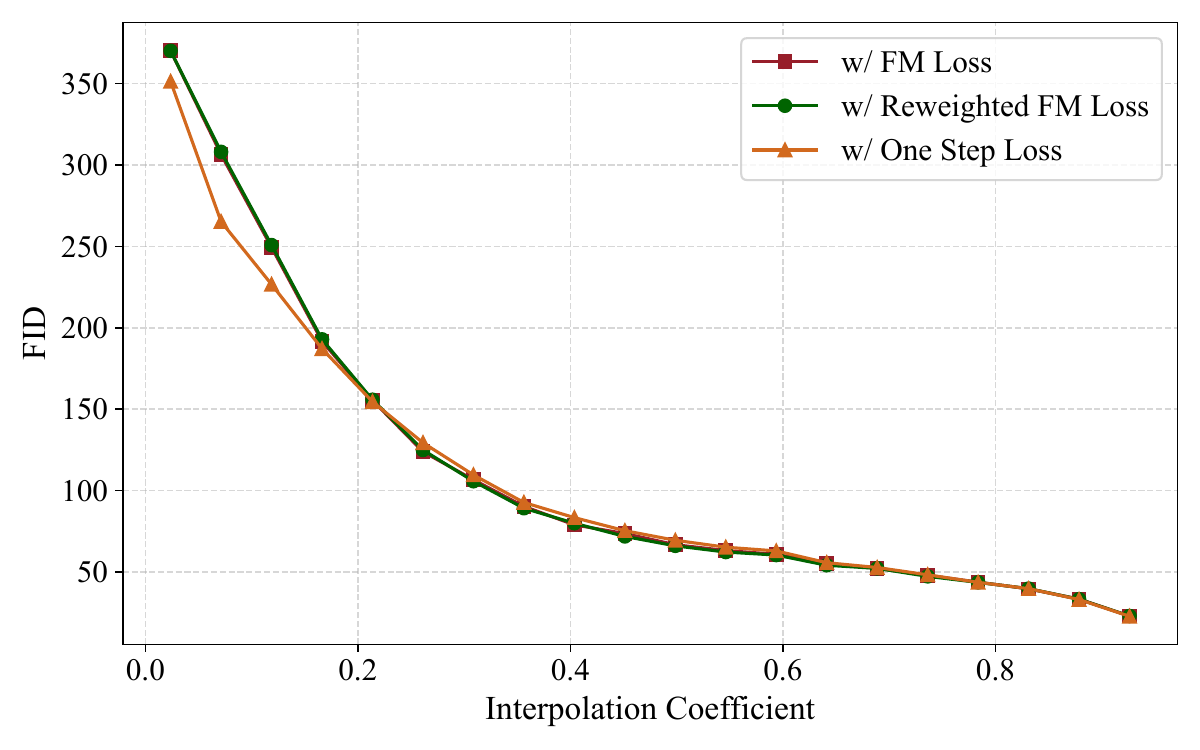}
    \caption{
\textbf{Full FID curves across diffusion time.}
We partition $t\in[0.0,\,0.95]$ into 20 uniform bins (same bin count as~\Cref{fig:fid-compare}, which is restricted to $t\in[0.0,\,0.25]$).
For each bin, we sample $T$ uniformly within the bin and generate 2048 noisy samples following~\Cref{def:genproc}; FID is computed on the corresponding denoised outputs for each objective.
We exclude directional eikonal loss from this comparison because it does not specify a denoising destination; instead, we report in~\Cref{fig:verify-equa} that its solution lies between the one-step loss minimizer and the flow matching minimizer.
Beyond the small-$t$ regime, the performance gaps shrink, and the FID curves across objectives become much less distinguishable.
The red and green curves always overlap so that they are hard to distinguish.}
    \label{fig:full_fid_comparison}
\end{figure}

\clearpage

\stopappendixtoc

\newpage
\section*{NeurIPS Paper Checklist}

\begin{enumerate}

\item {\bf Claims}
    \item[] Question: Do the main claims made in the abstract and introduction accurately reflect the paper's contributions and scope?
    \item[] Answer: \answerYes{} %
    \item[] Justification: The abstract and introduction state the core problem (time-marginalized ambiguity), the proposed method \dm{} (distance-style losses without explicit time conditioning; gradient descent and sphere tracing), and empirical claims spelled out with references in~\Cref{sec:intro} and~\Cref{sec:exp}. Quantitative statements (FID vs.\ flow matching on ImageNet across scales, fewer sampling steps than FM, CFG deltas, time-unconditional CIFAR comparison) agree with reported tables and figures; adaptive stopping and OOD uses are documented in~\Cref{sec:exp}. Items framed as outlook (e.g., MCMC) are not presented as achieved contributions.

\item {\bf Limitations}
    \item[] Question: Does the paper discuss the limitations of the work performed by the authors?
    \item[] Answer: \answerYes{} %
    \item[] Justification: Yes. Appendix (proofs): ``Limitations and Future Work'' lists open theory, density vs.\ geometric coverage, benchmarks (mainly CIFAR-10 and ImageNet-256 SiT), the $u_\theta/\mathbf{v}_\theta$ split vs.\ enforcing $\mathbf{v}_\theta=\nabla u_\theta$, and future MCMC/distillation work. Appendix (outcomes): $\mathbf{v}_\theta$ is not curl-free. Main text mentions distance-estimate bias and locality; the conclusion discusses MCMC for diversity.

\item {\bf Theory assumptions and proofs}
    \item[] Question: For each theoretical result, does the paper provide the full set of assumptions and a complete (and correct) proof?
    \item[] Answer: \answerYes{} %
    \item[] Justification: Formal theorems/lemmas/remark statements and~\texttt{assumption} blocks live in the proofs appendix with numbering and~\verb|\cref| cross-references, each followed by a complete appendix proof or explicit citation where appropriate.


\item {\bf Experimental result reproducibility}
    \item[] Question: Does the paper fully disclose all the information needed to reproduce the main experimental results of the paper to the extent that it affects the main claims and/or conclusions of the paper (regardless of whether the code and data are provided or not)?
    \item[] Answer: \answerYes{} %
    \item[] Justification: \Cref{sec:exp} covers datasets (CIFAR-10 / ImageNet-256), pretrained latent VAE for ImageNet, backbone choices aligned with prior open codebases~\cite{lipman2024flow,ma2024sit}, loss definition and sampling variants. Appendix~\Cref{sec:exp_set} tabulates architectures, optimizer, learning rate schedule, epochs, hyperparameters $\lambda_{\mathrm{OS}}$, $\lambda_{\mathrm{DE}}$, $\epsilon$, $c_0$, rescalings for $\mathbf{v}_\theta$, per-step budgets, step sizes $\eta$, and hardware (GPU type/count). Metrics (e.g.\ FID protocol, sample counts) appear in~\Cref{sec:exp} and~\Cref{sec:exp_set}. Standard benchmarks are externally obtained with citations; nothing essential to the headline claims rests on unpublished data.

    \item[] Guidelines:
    \begin{itemize}
        \item The answer \answerNA{} means that the paper does not include experiments.
        \item If the paper includes experiments, a \answerNo{} answer to this question will not be perceived well by the reviewers: Making the paper reproducible is important, regardless of whether the code and data are provided or not.
        \item If the contribution is a dataset and\slash or model, the authors should describe the steps taken to make their results reproducible or verifiable. 
        \item Depending on the contribution, reproducibility can be accomplished in various ways. For example, if the contribution is a novel architecture, describing the architecture fully might suffice, or if the contribution is a specific model and empirical evaluation, it may be necessary to either make it possible for others to replicate the model with the same dataset, or provide access to the model. In general, releasing code and data is often one good way to accomplish this, but reproducibility can also be provided via detailed instructions for how to replicate the results, access to a hosted model (e.g., in the case of a large language model), releasing of a model checkpoint, or other means that are appropriate to the research performed.
        \item While NeurIPS does not require releasing code, the conference does require all submissions to provide some reasonable avenue for reproducibility, which may depend on the nature of the contribution. For example
        \begin{enumerate}
            \item If the contribution is primarily a new algorithm, the paper should make it clear how to reproduce that algorithm.
            \item If the contribution is primarily a new model architecture, the paper should describe the architecture clearly and fully.
            \item If the contribution is a new model (e.g., a large language model), then there should either be a way to access this model for reproducing the results or a way to reproduce the model (e.g., with an open-source dataset or instructions for how to construct the dataset).
            \item We recognize that reproducibility may be tricky in some cases, in which case authors are welcome to describe the particular way they provide for reproducibility. In the case of closed-source models, it may be that access to the model is limited in some way (e.g., to registered users), but it should be possible for other researchers to have some path to reproducing or verifying the results.
        \end{enumerate}
    \end{itemize}

\item {\bf Open access to data and code}
    \item[] Question: Does the paper provide open access to the data and code, with sufficient instructions to faithfully reproduce the main experimental results, as described in supplemental material?
    \item[] Answer: \answerYes{} %
    \item[] Justification: We provide code for CIFAR-10 experiments at \url{https://anonymous.4open.science/r/DM_CIFAR10-0657/} and ImageNet-$256\!\times\!256$ experiments at \url{https://anonymous.4open.science/r/DM_ImageNet-0246/}. Each includes a \texttt{README} with environment setup, data preparation, and commands to train and evaluate, together with a \texttt{requirements.txt} pinning dependencies. The implementations gratefully build on Yaron Lipman et al.'s Flow Matching guide and reference code~\cite{lipman2024flow}, the Equilibrium Matching (EqM) codebase~\cite{wang2025equilibrium}, and the SiT transformer backbone and training recipe~\cite{ma2024sit}, under their respective licenses; our repositories state these attributions and how to obtain the standard public datasets and frozen VAE used in~\Cref{sec:exp}.
    \item[] Guidelines:
    \begin{itemize}
        \item The answer \answerNA{} means that paper does not include experiments requiring code.
        \item Please see the NeurIPS code and data submission guidelines (\url{https://neurips.cc/public/guides/CodeSubmissionPolicy}) for more details.
        \item While we encourage the release of code and data, we understand that this might not be possible, so \answerNo{} is an acceptable answer. Papers cannot be rejected simply for not including code, unless this is central to the contribution (e.g., for a new open-source benchmark).
        \item The instructions should contain the exact command and environment needed to run to reproduce the results. See the NeurIPS code and data submission guidelines (\url{https://neurips.cc/public/guides/CodeSubmissionPolicy}) for more details.
        \item The authors should provide instructions on data access and preparation, including how to access the raw data, preprocessed data, intermediate data, and generated data, etc.
        \item The authors should provide scripts to reproduce all experimental results for the new proposed method and baselines. If only a subset of experiments are reproducible, they should state which ones are omitted from the script and why.
        \item Providing as much information as possible in supplemental material (appended to the paper) is recommended, but including URLs to data and code is permitted.
    \end{itemize}

\item {\bf Experimental setting/details}
    \item[] Question: Does the paper specify all the training and test details (e.g., data splits, hyperparameters, how they were chosen, type of optimizer) necessary to understand the results?
    \item[] Answer: \answerYes{} %
    \item[] Justification: \Cref{sec:exp} states datasets, pretrained components (frozen VAE for ImageNet~\cite{rombach2022high}), backbone families~\cite{lipman2024flow,ma2024sit}, architectures for the auxiliary distance head, sampling pipelines (sphere tracing vs.\ GD, steps, CFG where used), evaluation metrics/protocol (e.g., FID at 50k samples). Appendix~\Cref{sec:exp_set} publishes full training tables---optimizer hyperparameters ($\beta_1/\beta_2$, weight decay, LR), epochs, interpolation-sampled $t$ distribution, losses ($\lambda_{\mathrm{OS}}$, $\lambda_{\mathrm{DE}}$, $\epsilon$, $c_0$), output rescalings, per-model $\eta$ and NFE, and GPUs.


\item {\bf Experiment statistical significance}
    \item[] Question: Does the paper report error bars suitably and correctly defined or other appropriate information about the statistical significance of the experiments?
    \item[] Answer: \answerYes{} %
    \item[] Justification: Appendix~\Cref{tab:fid_cfg,tab:is_cfg} annotate the classifier-free-guidance sweep with parentheses reporting standard deviations at each method's best CFG scale.~\Cref{app:convergence} gives numeric sample standard deviations for CIFAR-10 FID under our sphere tracing and gradient-descent generators.~\Cref{fig:dist_error} shades $\pm1$ standard deviation per bin; the caption ties the band to dispersion within bins.


\item {\bf Experiments compute resources}
    \item[] Question: For each experiment, does the paper provide sufficient information on the computer resources (type of compute workers, memory, time of execution) needed to reproduce the experiments?
    \item[] Answer: \answerYes{} %
    \item[] Justification: Appendix~\Cref{sec:exp_set} reports hardware for training (4 Nvidia A100 80G GPUs) and, for ImageNet throughput/latency, inference on one Nvidia A10 (\Cref{tab:speed_comparison}). Per-dataset training duration is implied by epochs and published step throughputs (e.g., CIFAR-10 3000 epochs, ImageNet $256{\times}256$ 80 epochs in~\Cref{tab:dm_cifar10_config,tab:dm_imagenet_configs}). Memory is indicated by GPU model (80G A100). As is common, the full research process included additional exploratory runs and ablations not all tabulated; the reported experiments dominate compute relative to early debugging.
    
\item {\bf Code of ethics}
    \item[] Question: Does the research conducted in the paper conform, in every respect, with the NeurIPS Code of Ethics \url{https://neurips.cc/public/EthicsGuidelines}?
    \item[] Answer: \answerYes{} %
    \item[] Justification: The study uses publicly available image benchmarks and standard ML engineering practice; it does not involve human subjects, invasive data collection, or deceptive participant interaction. Authors intend to comply with dataset and model licenses (\Cref{sec:exp_set}) and with NeurIPS research integrity expectations.
    \item[] Guidelines:
    \begin{itemize}
        \item The answer \answerNA{} means that the authors have not reviewed the NeurIPS Code of Ethics.
        \item If the authors answer \answerNo, they should explain the special circumstances that require a deviation from the Code of Ethics.
    \end{itemize}

\item {\bf Broader impacts}
    \item[] Question: Does the paper discuss both potential positive societal impacts and negative societal impacts of the work performed?
    \item[] Answer: \answerYes{} %
    \item[] Justification: \Cref{subsec:broader_impact} discusses positive uses of generative models, including artistic prototyping and potential extensions to scientific research and robotics, as well as negative risks such as disinformation from synthesized imagery. It also mentions detection and watermarking as mitigation-oriented future directions.
    
\item {\bf Safeguards}
    \item[] Question: Does the paper describe safeguards that have been put in place for responsible release of data or models that have a high risk for misuse (e.g., pre-trained language models, image generators, or scraped datasets)?
    \item[] Answer: \answerNA{} %
    \item[] Justification: We do not release a new large scraped dataset, unrestricted web-scale crawl, or a new hosted API/service in this submission. Experiments rely on established academic image benchmarks and a publicly available frozen VAE; any future code or checkpoint release is intended to follow the same license and usage constraints as the upstream assets.

\item {\bf Licenses for existing assets}
    \item[] Question: Are the creators or original owners of assets (e.g., code, data, models), used in the paper, properly credited and are the license and terms of use explicitly mentioned and properly respected?
    \item[] Answer: \answerYes{} %
    \item[] Justification: Datasets, the pretrained VAE, and third-party code (flow matching~\cite{lipman2024flow}, SiT~\cite{ma2024sit}, EqM~\cite{wang2025equilibrium}) are cited in the main text and~\Cref{sec:exp_set}, which notes that we follow each distributor's and repository's terms and licenses.

\item {\bf New assets}
    \item[] Question: Are new assets introduced in the paper well documented and is the documentation provided alongside the assets?
    \item[] Answer: \answerYes{} %
    \item[] Justification: We release code for our method (the two repositories above) with \texttt{README} and \texttt{requirements.txt}, describing scope, dependencies, training and evaluation entry points, and alignment with the hyperparameter tables in~\Cref{sec:exp_set}. No new human-derived dataset is introduced; benchmarks and upstream checkpoints remain subject to their original licenses as cited in the paper.
    \item[] Guidelines:
    \begin{itemize}
        \item The answer \answerNA{} means that the paper does not release new assets.
        \item Researchers should communicate the details of the dataset\slash code\slash model as part of their submissions via structured templates. This includes details about training, license, limitations, etc. 
        \item The paper should discuss whether and how consent was obtained from people whose asset is used.
    \end{itemize}

\item {\bf Crowdsourcing and research with human subjects}
    \item[] Question: For crowdsourcing experiments and research with human subjects, does the paper include the full text of instructions given to participants and screenshots, if applicable, as well as details about compensation (if any)? 
    \item[] Answer: \answerNA{} %
    \item[] Justification: The paper does not conduct crowdsourcing, human evaluation studies, or other research with human participants.

\item {\bf Institutional review board (IRB) approvals or equivalent for research with human subjects}
    \item[] Question: Does the paper describe potential risks incurred by study participants, whether such risks were disclosed to the subjects, and whether Institutional Review Board (IRB) approvals (or an equivalent approval/review based on the requirements of your country or institution) were obtained?
    \item[] Answer: \answerNA{} %
    \item[] Justification: No human-subjects protocol was required because the work does not enroll participants or collect identifiable human data beyond standard public benchmarks.
    \item[] Guidelines:
    \begin{itemize}
        \item The answer \answerNA{} means that the paper does not involve crowdsourcing nor research with human subjects.
        \item Depending on the country in which research is conducted, IRB approval (or equivalent) may be required for any human subjects research. If you obtained IRB approval, you should clearly state this in the paper. 
        \item We recognize that the procedures for this may vary significantly between institutions and locations, and we expect authors to adhere to the NeurIPS Code of Ethics and the guidelines for their institution. 
    \end{itemize}

\item {\bf Declaration of LLM usage}
    \item[] Question: Does the paper describe the usage of LLMs if it is an important, original, or non-standard component of the core methods in this research? Note that if the LLM is used only for writing, editing, or formatting purposes and does \emph{not} impact the core methodology, scientific rigor, or originality of the research, declaration is not required.
    \item[] Answer: \answerNA{} %
    \item[] Justification: Large language models are not part of the proposed training objective, architecture, or inference pipeline. If writing assistance was used for prose polishing only, it does not affect methodology, claims, or empirical results under the NeurIPS disclosure policy.

\end{enumerate}

\end{document}